%% file: ms.tex
\documentclass{article}
\usepackage[utf8]{inputenc}
\usepackage[nonatbib, preprint]{neurips_2021}

\usepackage[numbers]{natbib}

\usepackage{hyperref}       
\usepackage{url}            
\usepackage{booktabs}       
\usepackage{amsfonts}       
\usepackage{nicefrac}       
\usepackage{microtype}      
\usepackage{color}          

\usepackage{graphicx} 
\usepackage{amssymb}
\usepackage{amsmath}
\usepackage{amsthm}
\usepackage{mathtools}
\usepackage{bbm}

\usepackage{stmaryrd}
\usepackage{enumitem}
\usepackage{multirow}
\usepackage{algorithm}
\usepackage{algorithmic}
\usepackage{tabu}
\usepackage{tikz}
\usetikzlibrary{shapes,backgrounds}
\usepackage{url}
\usepackage{xspace}
\usepackage{thmtools, thm-restate}
\usepackage{diagbox}
\usepackage{xspace}
\usepackage{graphicx} 
\usepackage{caption}
\usepackage{subcaption}

\theoremstyle{definition}
\declaretheorem{definition}
\theoremstyle{plain}
\declaretheorem{lemma}
\declaretheorem{theorem}

\newcommand{\NDCGk}{\operatorname{NDCG}\!\mathit{@k}}
\newcommand{\Pk}{\operatorname{P@k}}

\makeatletter
\DeclareRobustCommand\onedot{\futurelet\@let@token\@onedot}
\def\@onedot{\ifx\@let@token.\else.\null\fi\xspace}

\def\iid{{i.i.d}\onedot}
\def\eg{{e.g}\onedot} 
\def\ie{{i.e}\onedot} 
 
\def\etc{{etc}\onedot}

\makeatother

\newcommand{\myparagraph}[1]{\noindent\textbf{#1.}\ }

\graphicspath{ {images/} }

\title{Fairness Through Regularization \\ for Learning to Rank}

\author{%
  Nikola Konstantinov \\
  IST Austria \\
  \texttt{nikola.konstantinov@ist.ac.at} \\
  \And
  Christoph H. Lampert\\
  IST Austria \\
  \texttt{chl@ist.ac.at}\\
}

\begin{document}

\maketitle

\allowdisplaybreaks
\setlength{\abovedisplayskip}{4pt plus 1.0pt minus 2.0pt}
\setlength{\belowdisplayskip}{3pt plus 1.0pt minus 2.0pt}
\setlength{\abovedisplayshortskip}{0.0pt plus 2.0pt}
\setlength{\belowdisplayshortskip}{3.0pt plus 2.0pt minus 2.0pt}
\setlength{\textfloatsep}{8pt plus 1.0pt minus 2.0pt}
\setlength{\dbltextfloatsep}{9pt plus 1.0pt minus 2.0pt}

\begin{abstract}
Given the abundance of applications of ranking in recent years, 
addressing fairness concerns around automated ranking systems 
becomes necessary for increasing the trust among end-users. 
Previous work on fair ranking has mostly focused on application-specific 
fairness notions, often tailored to online advertising, 
and it rarely considers learning as part of the process. In this work, we show how to transfer numerous fairness notions 
from binary classification to a learning to rank setting. 
Our formalism allows us to design methods for incorporating 
fairness objectives with provable generalization guarantees. 
An extensive experimental evaluation shows that our method can 
improve ranking fairness substantially with no or only little 
loss of model quality.
\end{abstract}

\section{Introduction}
\label{sec:introduction}

\input{introduction2}

\vspace{-0.15cm}
\section{Related work}
\label{sec:related_work}
\vspace{-0.15cm}
\input{related_work}

\section{Preliminaries}
\label{sec:preliminaries}

\input{preliminaries}

\section{Fairness in Learning-to-Rank}
\label{sec:fairness_notions}

\input{fairness_notions}

\vspace{-0.2cm}
\subsection{Generalization}
\label{sec:concentration}

\input{concentration}

\vspace{-0.2cm}
\section{Experiments}
\label{sec:experiments}

\input{experiments}

\vspace{-0.2cm}
\section{Conclusion}
\label{sec:conclusion}

\input{conclusion}

\bibliography{ms}
\bibliographystyle{abbrvnat}

\onecolumn
\appendix

\centerline{\Huge Supplementary Material}
\bigskip

The supplementary material is structured as follows:

\textbf{Section \ref{app:proof}} contains the complete proof of Theorem \ref{thm:uniform_bound}. In particular, Section \ref{subsec:chromatic_bounds} discusses the chromatic concentration bounds of \cite{janson2004large} that we need to address the dependence between the samples. Section \ref{subsec:non_uniform_bound_proof} uses this tool and a conditional probability concentration technique of \cite{woodworth2017learning} to prove a concentration bound for a single classifier. Finally, in Section \ref{subsec:uniform_bound_proof} we make this bound uniform over the hypothesis space by evoking a variant of the classic symmetrization argument (\eg~\cite{vapnik2013nature}), while carefully accounting for the dependence between the samples. We conclude with a brief discussion in Section \ref{sec:app_discussion} on the bounds and with some ideas for potential improvements and extensions of the theoretical analysis.

\textbf{Section \ref{app:experimental_details}} provides more technical details on our experiments, in particular on the feature extraction procedures used and on the computational costs.

\textbf{Section \ref{app:further_experimental_results}} contains further experimental results, including results for $\Pk$ (Section \ref{app:results_pk}), results for all three fairness measures in the setup of Figure \ref{fig:plots} (Section \ref{sec:app_plots_for_three_measures}), splits of Table \ref{table:results_ndcg_averages} according to the value of $k$ (Section \ref{sec:results_split_over_k}) and the plots of the performance of our algorithm on all experimental setups (Section \ref{sec:all_plots}). 

\section{Proof of Theorem \ref{thm:uniform_bound}}
\label{app:proof}
\input{appendix_proof}

\section{Details of Experimental Setup}
\label{app:experimental_details}

\input{appendix_experimental_details}

\section{Further experimental results}
\label{app:further_experimental_results}

\input{appendix_results}

\end{document}

%% file: introduction2.tex
Ranking problems are abundant in many contemporary subfields of machine learning and artificial intelligence, including web search, question answering, candidate/reviewer allocation, recommender systems and bid phrase suggestions~\citep{manning2008introduction}. Decisions taken by such ranking systems affect our everyday life and this naturally leads 
to concerns about the fairness of ranking algorithms.

Indeed, ranking systems are typically designed to maximize utility and return the results most likely correct 
for each query~\cite{robertson1977probability}.
This can have potentially harmful down-stream effects. 
For example, in 2015 Google became the target of heavy 
criticism after news reports that when searching for "CEO" 
in Google's image search, the first image of a women appeared 
only in the twelfth row, requiring two page scrolls to reach, 
and it actually did not show a real person but a Barbie doll~\cite{ceo2}.
Similar problems still exist in other ranking applications, 
such as product recommendations or online dating. 

Potentially biased or otherwise undesirable results are particularly 
problematic in the \emph{learning to rank (LTR)} 
setting~\cite{liu2011learning, mitra2018introduction}, where 
a machine learning model is trained to predict the relevances 
of the items for any query at test time.
Training data for these systems is typically obtained from users 
interacting with another ranking system. Therefore, a biased 
selection of items can lead to disadvantageous winner-takes-it-all 
and rich-get-richer dynamics~\cite{tabibian2020design, tsintzou2018bias}.

A number of ranking-related fairness notions were proposed to 
make ranking systems more fair~\cite{biega2018equity, singh2018fairness, zehlike2017fa}. 
However, these were tailored to specific applications, such as 
online advertising. Some are also not well suited to the 
LTR situation, because their associated algorithms 
work by directly manipulating the order of returned items for a 
given query. 
In contrast, in the LTR setting one would hope that
the system \emph{learns to be fair}, such that a manipulation 
of the predicted relevance scores or the order of items are not 
necessary.
This latter aspect brings up the problem of generalization for 
fairness: a ranking method could appear fair on the training data 
but turn out unfair at prediction time.

In this paper, we address these challenges and develop 
fairness-aware algorithms for LTR that 
provably generalize.
To this end, we exploit connections between ranking and classification. 
Indeed, in contrast to fair learning in information retrieval, fair classification is a widely studied area where both the algorithmic and learning theoretic challenges of learning fair models are rather well understood \cite{barocas-hardt-narayanan, cotter2019training, woodworth2017learning}. 
Importantly, many different notions of fairness have been proposed, 
which describe different properties that are desirable in various applications \cite{mehrabi2019survey}. 

We provide a formalism for translating such 
well-established and well-understood fairness notions from 
classification to ranking by phrasing the LTR problem as a binary classification problem for every query-item pair. 
We exemplify our approach on three fairness notions that emerge 
naturally in the ranking setting and correspond to popular concepts 
in classification: 
\emph{demographic parity}, \emph{equalized odds}, and \emph{equality of opportunity}.
We then formulate corresponding \emph{fairness regularization terms},
which can be incorporated with minor overhead into many standard 
LTR algorithms.

Besides its flexibility, another advantage of our approach is that it 
makes the task of fair LTR readily amendable to a learning-theoretic analysis. 
Specifically, we show generalization bounds for the three considered 
fairness notions, using a chromatic concentration bound for sums of 
dependent random variables~\cite{janson2004large} to overcome the 
challenge that training samples for the same query are not independent.

Finally, we demonstrate the practical usefulness of our method 
for training fair models. Experiments on two ranking datasets 
confirm that training with our regularizers indeed 
yields models with greatly improved fairness at prediction 
time, often with little to no reduction of ranking quality. In contrast, prior fair ranking methods are unable to consistently improve our fairness notions.

%% file: related_work.tex
\myparagraph{Fairness in classification} Algorithmic fairness is well explored in the context of binary classification, see~\cite{barocas-hardt-narayanan} for a detailed introduction. 
In this work we show how to extend three popular group fairness notions -- demographic parity, equalized odds and equality of opportunity \cite{hardt2016equality} -- to the ranking setting. In principle, our formalism is applicable to other group fairness notions, as well as individual \cite{dwork2012fairness} and causal \cite{kusner2017counterfactual} fairness notions. 
We defer the exploration of these to future work. On the methodological level, we opt for a highly adaptive and scalable regularization approach, inspired by successful regularization methods for fair classification \cite{kamishima2011fairness, kamishima2012fairness, zafar2017fairness}. More generally any other fair classification technique, \eg \cite{baharlouei2019renyi, cho2020fair, kim2020fact, rezaei2020fairness, tan2020learning, zhang2018mitigating}, may be applicable to our framework.

\myparagraph{Fairness in ranking}
Fairness in ranking has so far received less attention 
than fairness in classification. For an overview of recent 
techniques, see~\cite{castillo2019fairness}.
Most existing works concentrate on application-specific 
(single-purpose) fairness notions.
One popular concept is \emph{fairness of exposure}~\cite{biega2018equity, gorantla2020ranking, morik2020controlling, sapiezynski2019quantifying, singh2018fairness, singh2019policy, yadav2019fair, zehlike2020reducing}.
It states that exposure/attention received by a group of items or an individual item should be proportional to its utility. 
Other works aim at ensuring sufficient representation of items 
from different groups in the top-$k$ positions of a 
ranking~\cite{celis2018ranking, celis2020interventions, geyik2019fairness, yang2017measuring, zehlike2017fa}. 
Besides group fairness also fair treatment of individuals 
has been studied in the context of ranking~\cite{bower2021individually,yang2019balanced}.

Among papers considering broader notions of fairness in 
ranking, \cite{asudeh2019designing} designs learning algorithms 
that can work with any fairness oracle. 
The framework however is limited to linear classifiers and 
the authors do not propose specific fairness notions. 
\cite{singh2017equality} introduces a number of fair ranking 
definitions and draws parallels to equalized odds and 
demographic parity from fair classification. 
However, it does not provide a formal framework from studying 
the correspondence between the two setups, and does not study 
how to optimize these measures in a learning to rank context. 
Moreover, its fairness measures concern fair rankings for a fixed 
query, which also holds for the causal fairness notion in~\cite{wu2018discrimination}.
In contrast, our notion of ranking fairness is amortized across 
queries, similarly to~\cite{biega2018equity}. 

Another related line of work is the one of pairwise 
fairness~\cite{beutel2019fairness, kuhlman2019fare, narasimhan2020pairwise}.
These works also describe ranking as a classification task in 
order to define fairness.
However, the considered task is the proxy commonly
employed by pairwise ranking methods, namely predicting which one of 
two items is more relevant than the other for a given query. 
In contrast, we define fairness in direct relation to the downstream 
task of deciding whether to return an item as relevant for a query or not. 
\cite{kallus2019fairness, vogel2020learning} introduce fairness 
notions for bipartite ranking. These are also based on pairwise 
comparisons between points, but aim at learning fair continuous 
risks scores.

Overall, the main difference of our work to previous ones on 
ranking fairness is that we do not introduce a new fairness 
notion or algorithm. Instead, the formalism we introduce allows 
transferring existing fairness notions from classification to 
ranking. 
A second distinction is that only a minority of prior works 
considers fairness in the context of learning, and those who
do usually propose new training techniques.
Instead, the fairness regularizers we introduce can be combined 
with any existing training procedure that can be formulated as 
learning a score function by minimizing a cost function.
Finally, no prior works provides generalization guarantees for fair ranking as we do.

\myparagraph{Fairness in recommender systems}
For recommender systems, fairness can be studied with respect to the 
consumers/users (known as C-fairness) or with respect to the 
providers/items (known as P-fairness)~\cite{burke2017multisided}.
\cite{steck2018calibrated, tsintzou2018bias} consider calibration 
and bias disparity within recommender systems with respect to 
recommended items. 
In \cite{burke2018balanced, farnadi2018fairness, zhu2018fairness, chakraborty2019equality, peysakhovich2019fair} various hybrid approaching for achieving both C-fairness and P-fairness are presented.
In contrast to our paper, these works are specific to collaborative filtering or 
tensor-based algorithms and do not carry over to approaches based
on supervised learning.

A concept from recommender systems related to demographic parity fairness is that of \textit{neutrality} \cite{kamishima2012enhancement}, in which one aims to provide recommendations that are independent of a certain viewpoint. In particular, \cite{kamishima2012enhancement, kamishima2014correcting} apply a neutrality enhancing regularizer to a recommender system model. The focus of these works, however, lies on dealing with filter bubble problems and no formal links to classification or fairness are made.

\myparagraph{Diversity in ranking}
Another related topic is the one of diversifying the output 
set of ranking system, see,\eg,~\cite{radlinski2009redundancy}.
However, diversifying rankings generally has the goal of 
improving the user experience, not a fair treatment of items. 
A discussion on the relationship between fairness and 
ranking diversity can be found in~\cite{singh2018fairness}.

%% file: preliminaries.tex
In this section we introduce some background information
on the learning to rank (LTR) task.
A thorough introduction can be found, \eg, in~\cite{liu2011learning, mitra2018introduction}.

\myparagraph{Learning to rank}
Let $\mathcal{Q}$ be a set of possible \emph{queries} to a ranking system,
and let $\mathcal{D}$ be a set of \emph{items} (historically \emph{documents}) 
that are meant to be ranked according to their relevance for any query. 
Typically, we think of the query set as practically infinite, 
\eg natural language phrases, whereas the item set is finite 
and fixed, \eg a database of products or customers. 
These are not fundamental constraints, though, and extensions are 
possible, \eg items appearing or disappearing over time. 

A dataset in the LTR setting typically has the form 
$S = \{(q_i, d^i_j, r_{j}^i)\}_{i\in [N], j\in [m_i]}$, \ie 
for each of $N$ queries, $q_1, \ldots, q_N\in\mathcal{Q}$, a subset of 
the items $D_{q_i} = \{d^i_1, d^i_2, \ldots, d^i_{m_i}\} \subset \mathcal{D}$ 
are annotated with binary labels $r_j^i = r(q_i,d^i_j) \in \{0,1\}$ that
indicate if item $d^i_j$ is relevant to query $q_i$ or not.
In most real-world scenarios, 
$m_i$ will be much smaller than $|\mathcal{D}|$,
since it is typically impractical to determine the relevance of every item
for a query.

The goal of learning to rank is to use a given training set 
to learn a \emph{ranking procedure} that, for any future query, 
can return a set of items as well as their order. 
That is, the learner has to construct a \emph{subset selection function},
\begin{equation}
R:\mathcal{Q} \to \mathfrak{P}(\mathcal{D}),  
\end{equation}
where $\mathfrak{P}$ denotes the powerset operation, as 
well as an ordering of the predicted item set. 
In practice, both steps are typically combined by learning a 
\emph{score function}, $s: \mathcal{Q}\times \mathcal{D}\to \mathbb{R}$.
For any fixed $q$, $s(q,\cdot)$ induces a total ordering of 
$\mathcal{D}$, and the set of predicted items is obtained by 
thresholding or top-$k$ prediction. The function $s$ is usually learned by minimizing a loss function on the quality of the resulting ranking on the train data. 
Classic examples of this construction are 
SVMRank~\cite{joachims2002optimizing}
or WSABIE~\cite{weston2011wsabie}. 
Most other pointwise, pairwise and listwise methods can 
also be phrased in the above way, with differences mainly
in how the loss is defined and how the score 
function is learned numerically~\cite{liu2011learning}.

\myparagraph{Evaluation measures}
Many measures exist for evaluating the 
quality of a ranking system. Arguably the simplest 
is to measure the fraction of correctly predicted 
relevant items. 
\begin{definition}\label{def:precision_at_k}
Let $S$ be a test set in the format introduced above. 
For any query $q_i$, let $d^i_1,d^i_2,\dots$ be a 
ranking of the items in $\mathcal{D}_{q_i}$ with 
associated ground-truth values $r(q_i,d^i_j)$. 
Then, for any $k\in\mathbb{N}\setminus\{0\}$, 
the \emph{precision at $k$} is defined as 
$\Pk = \frac{1}{N}\sum_{i=1}^N P@k(q_i)$ with
\begin{equation}\label{eqn:precision_at_k}
\Pk(q_i) = \frac{1}{k}\sum_{j=1}^k r(q_i, d^i_j).
\end{equation}
\end{definition}

For any $k$, the $\Pk$ value reflects only which 
items appear in the top-$k$ list, but not their ordering. 
Furthermore, $\Pk$ is automatically small for datasets 
in which queries have only few relevant documents.
To mitigate these shortcomings, one can add position-dependent 
weights and normalize by the score of a \emph{best-possible} ranking.
\begin{definition}\label{def:ndcg_at_k}
In the same setting as for Definition~\ref{def:precision_at_k},
the \emph{normalized discounted cumulative gain at $k$} is defined as 
$\NDCGk = \frac{1}{N}\sum_{i=1}^N \NDCGk(q_i)$ for 
\begin{equation}\label{eqn:NDCG_at_k}
\NDCGk(q_i) = \Big(\sum_{j=1}^k \frac{r(q_i, d^i_j)}{\log_2(j+1)}\Big) \left.  \middle/ \right. \Big(\sum_{j=1}^{\min(k,K_i)}\frac{1}{\log_2(j+1)}\Big),
\end{equation}
where $K_i=|\{d\in\mathcal{D}_{q_i}\ : \ r(q_i,d)=1\}|$ 
is the number of relevant items for query $q_i$.
Queries with no relevant items are excluded from the average, 
as the measure is not well-defined for these. 
\end{definition}

%% file: fairness_notions.tex
We now introduce our framework for group fairness in ranking. 
The main step is to exploit a correspondence between ranking and 
multi-label learning, a view that has previously been employed for 
practical tasks, \eg, in \emph{extreme classification}~\cite{bengio2019extreme},
but not --to our knowledge-- to make LTR benefit from prior work 
on classification fairness.

Specifically, we study how a set of relevant items for any query 
can be selected in a fairly.  
Analogously to the discussion in Section~\ref{sec:preliminaries}, 
this originally means learning a \emph{subset selection function}
$R:\mathcal{Q}\to\mathfrak{P}(\mathcal{D})$,
where $R(q)$ is the predicted set of selected items for a query $q$.
The objects for which we want to impose fairness, the items, occur as 
outputs of the learned functions. 
This makes it hard to leverage fairness notions from classification, where fairness is defined with 
respect to the inputs. 

We advocate an orthogonal viewpoint: for any fixed 
query $q$, we treat the items not as elements of the predictor's 
output, but as the inputs to a query-dependent classifier: 
$f_q:\mathcal{D}\to\{0,1\}$, where $f_q(d)=1$, if item $d$ 
is should be returned for query $q$, and $f_q(d)=0$ otherwise. 
As the query is a priori unknown, this means one ultimately has 
to find an \emph{item selection function}
\begin{equation}\label{eqn:indicator_function}
f: \mathcal{Q}\times \mathcal{D} \to \{0,1\}.
\end{equation}
While, of course, both views are equivalent, the latter 
one allows us to readily integrate notions of classification fairness 
into the LTR paradigm. 
Here we focus on the inclusion of \emph{group fairness},
and leave the derivation of \emph{individual fairness}~\cite{dwork2012fairness} 
or \emph{counterfactual fairness}~\cite{kusner2017counterfactual} 
to future work.

Note that even though the item selection function $f(q,d)$ and the relevance 
label $r(q,d)$ have the same signature, their roles are different. 
$r$ specifies if an item is relevant for a query or not. $f$ indicates 
if the item should be returned as a result. 
These concepts differ when other aspects besides relevance are 
meant to influence the ranking, such as an upper bound on how many items can be retrieved per query or fairness and diversity considerations.

\subsection{Group fairness in learning to rank}
Notions of group fairness in classification are typically 
based on an underlying probabilistic framework that allows 
statements about (conditional) independence relations~\cite{barocas-hardt-narayanan}.
The same is true in the ranking situation, where we assume 
$\mathbb{P}\in\mathcal{P}(\mathcal{Q}\times\mathcal{D}\times\{0,1\})$ 
to be an unknown but fixed distribution over query/document/relevance 
triplets.
In the rest of our work, all statements about probabilities of 
events, denoted by $\Pr$, will be with respect to $\mathbb{P}(q,d,r(q,d))$. 
Note that $\mathbb{P}$ characterizes only the marginal distribution
of observing individual data points. 
It does not further specify how sets of many points, \eg a training 
dataset, would be sampled.  In particular, as we will discuss later, 
datasets for ranking tasks are typically not sampled \iid from $\mathbb{P}$, 
but exhibit strong statistical dependencies.

Analogously to the situation of classification, we assume that any 
item $d\in\mathcal{D}$ has a \emph{protected attribute}, $A(d)$, 
which denotes the group membership for which fairness should be 
ensured. For example, $A(d)$ can correspond to gender, when the retrieved items are images of people, or to the country of origin of an Amazon product.
In this work, we assume binary-valued protected attributes, but 
this is only for simplicity of presentation, not a fundamental 
limitation of our framework.

A plausible notion of fairness in the context of ranking is: 
\textbf{For any relevant item the probability of being included 
in the ranker's output should be independent of its protected attribute}. 
This intuition is easy to formulate in our formalism, 
resulting a direct analog of the \emph{equality of opportunity} 
principle from fair classification~\cite{hardt2016equality}.

\begin{definition}[\textbf{Equality of opportunity for LTR}]\label{def:equal_opportunity}
An item selection function $f:\mathcal{Q}\times\mathcal{D}\to \{0,1\}$
fulfills the \emph{equality of opportunity} condition, if 
\begin{equation}\label{eqn:equal_opportunity_ours_population}
\begin{split}
&\Pr(f(q, d)  = 1 | A(d) = 0, r(q,d) = 1) = \Pr(f(q, d) = 1| A(d) = 1, r(q,d) = 1).
\end{split}
\end{equation}
where $A(d)$ denotes the protected attribute of a document $d$. 
\end{definition}

The above definition provides a formal criterion of
what it means for a ranking system to be fair. 
In practice, a ranker will rarely achieve 
perfect fairness. 
Therefore, we also introduce a quantitative version 
of Definition~\ref{def:equal_opportunity} in the form of 
\emph{a fairness deviation measure}~\cite{williamson2019fairness,woodworth2017learning}
that reports a ranking procedure's \emph{amount of unfairness} 
(or \emph{lack of fairness}) by means of its 
\emph{mean difference score}~\cite{calders2010three}.

\begin{definition}\label{def:eop-violation}
The \emph{equality of opportunity (EOp) violation} of 
any item selection function, $f:\mathcal{Q}\times\mathcal{D}\to \{0,1\}$, 
is 
\begin{equation*}\label{eqn:equal_opportunity_ours_measure}
\begin{split}
\Gamma^{\textrm{EOp}}(f) &= 
\Big|\Pr(f(q, d)\!=\!1| A(d) \!=\! 0, r(q,d)\!=\!1) - \Pr(f(q, d)\!=\!1| A(d) \!=\! 1, r(q,d)\!=\!1) \Big|. 
\end{split}
\end{equation*}
\end{definition}

Clearly, $f$, is fair in the sense of 
Definition~\ref{def:equal_opportunity} if and only if 
it fulfills $\Gamma^{\textrm{EOp}}(f)=0$. 

\myparagraph{Other fairness measures}
As discussed extensively in the literature, different notions 
of fairness are appropriate under different circumstances. 
For example, to check the \emph{equality of opportunity} condition 
one needs to know which items are relevant for a query, and this 
can be problematic, \eg, if the available data itself exhibits a 
bias in this respect.

A major advantage of our formalism compared to prior fair ranking methods is that it is not partial to a 
specific fairness measure. 
Besides \emph{equality of opportunity}, many other notions of 
group fairness can be expressed by simply translating 
the corresponding expressions from classification. 

For example, one can avoid the problem of a data bias 
by simply demanding:
\textbf{The probability of any item to be selected 
should be independent of its protected attribute}
(disregarding its relevance to the
query). 
In our formalism, this condition is a direct analog 
of \emph{demographic parity} \cite{calders2009building}.

\begin{definition}[\textbf{Demographic Parity for LTR}]\label{def:demographic_parity}
An item selection function $f:\mathcal{Q}\times\mathcal{D}\to \{0,1\}$
fulfills the \emph{demographic parity} condition, if 
\begin{align}
\label{eqn:demographic_parity_ours_population}
\Pr(f(q, d) \!=\! 1| A(d) \!=\! 0) &\!=\! \Pr(f(q, d)  \!=\! 1| A(d) \!=\! 1).
\end{align}
As associated quantitative measure we define 
the \emph{demographic parity (DP) violation} of $f$ as 
\begin{align*}
\begin{split}
\label{eqn:demographic_parity_deviation}
\Gamma^{\textrm{DP}}(f) &= \Big|\Pr(f(q, d)=1| A(d) = 0) - \Pr(f(q, d)=1| A(d) = 1) \Big|.
\end{split}
\end{align*}
\end{definition}

Another meaningful notion of fairness in ranking is: \textbf{The 
probability of any item to be selected should be independent 
of its protected attribute, individually for all relevant and 
for all irrelevant items.}  
This condition yields the ranking analog 
of the \emph{equality odds} criterion ~\cite{hardt2016equality}.

\begin{definition}[\textbf{Equalized Odds for LTR}]\label{def:equal_odds}
An item selection function $f:\mathcal{Q}\times\mathcal{D}\to \{0,1\}$
fulfills the \emph{equalized odds} condition, if for all $r\in \{0,1\}$:
\begin{equation}\label{eqn:equal_odds_ours_population}
\begin{split}
&\Pr(f(q, d) = 1 | A(d) = 0, r(q,d) = r) = \Pr(f(q, d) = 1| A(d) = 1, r(q,d) = r)
\end{split}
\end{equation}
The \emph{equalized odds (EOd) violation} of $f$ is 
\begin{align*}\label{eqn:equal_odds_ours_measure}
\Gamma^{\textrm{EOd}}(f) &= \frac12\!\!\!\sum_{r\in\{0,1\}}\!\!\!
\Big|\Pr(f(q, d)\!=\!1| A(d) \!=\! 0, r(q,d)\!=\!r)- \Pr(f(q, d)\!=\!1| A(d) \!=\! 1, r(q,d)\!=\!r) \Big|. 
\end{align*}
\end{definition}

\subsection{Training fair rankers}\label{subsec:training_fair_rankers}
The above definitions do not only allow measuring the fairness of 
a fixed ranking system, but any of them can also be used to enforce 
the fairness of a LTR system during the training phase.
For this, we create empirical variants of the fairness violation 
measures and add them as a regularizer during the training step ~\cite{agarwal2018reductions,kamishima2011fairness}.
For this construction to make sense, we have to answer two 
questions: \emph{Can we solve the resulting optimization 
efficiently?} and \emph{Does the inclusion  
of a regularizer generalize, \ie ensure fairness also on future predictions?}  
In rest of this section, we will answer the first question.
The second question we will address in Section~\ref{sec:concentration}.

To allow for gradient-based optimization, we parametrize 
the binary-valued item selection function in a differentiable way 
using a real-valued score function $s:\mathcal{Q}\times \mathcal{D}\to [0,1]$,
similar as introduced in Section~\ref{sec:preliminaries}.
Our inspiration, however, comes from the classification 
setting, such as logistic regression, and we assume that $s$ is 
not arbitrary real-valued, but that it parameterizes the probability that $d$ is selected for $q$, 
\ie $s(q,d)=\Pr(f(q,d)=1)$. 

\myparagraph{Empirical fairness measures}
For a given training set, $S$, in the format discussed in 
Section~\ref{sec:preliminaries}, we obtain empirical estimates 
of the previously introduced fairness violation measures.
For any $a\in\{0,1\}$, $r\in\{0,1\}$, denote by $S_a$ the subset of 
data points $(q,d,r(q,d))$ in $S$ with $A(d)=a$, and by $S_{a,r}$ 
the subset of data points in $S$ with $A(d)=a$ and $r(q,d)=r$.

\begin{definition}[\textbf{Empirical fairness violation measures}]\label{def:eop-violation-empirical}
For any function $s:\mathcal{Q}\times\mathcal{D}\to [0,1]$, its 
\emph{empirical equality of opportunity violation on a dataset $S$} is 
\begin{equation}\label{eqn:empirical_fairness_measures}
\begin{split}
{\Gamma^{\textrm{EOp}}}(s;S) &= \Big|
\frac{1}{|S_{0,1}|}\!\!\sum_{(q,d)\in S_{0,1}}\!\!\!\!\!\!s(q, d)
- \frac{1}{|S_{1,1}|}\!\!\sum_{(q,d)\in S_{1,1}}\!\!\!\!\!\!s(q, d)\Big|.
\end{split}
\end{equation}
The \emph{empirical demographic parity violation of $s$ on $S$} is 
\begin{align}
{\Gamma^{\textrm{DP}}}(s;S) &= 
\Big|\frac{1}{|S_0|}\!\!\sum_{(q,d)\in S_0}\!\!\!\!\!s(q, d)-
\frac{1}{|S_1|}\!\!\sum_{(q,d)\in S_1}\!\!\!\!\!s(q, d) \Big|.
\end{align}
and the \emph{empirical equalized odds violation of $s$ on $S$} is 
\begin{align}
\begin{split}
{\Gamma^{\textrm{EOd}}}(s;S) &= \frac{1}{2}\sum_{r\in\{0,1\}}\Big|
\frac{1}{|S_{0,r}|}\!\!\sum_{(q,d)\in S_{0,1}}\!\!\!\!\!s(q, d) -\frac{1}{|S_{1,r}|}\!\!\sum_{(q,d)\in S_{1,1}}\!\!\!\!\!s(q, d) \Big|.
\end{split}
\end{align}
\end{definition}

These expressions can be derived readily as approximations of the 
conditional probabilities of the individual fairness measures by 
fractions of the corresponding examples in $S$. This is done by assuming that 
the marginal probability of any data point in $S$ is $\mathbb{P}$, 
and inserting the assumed relation $s(p,q)=\Pr(f(p,q)=1)$. 
Note that Definition~\ref{def:eop-violation-empirical} applies 
also to binary-valued functions, so it can also be used to evaluate 
the fairness of a learned item selection function on a dataset.

\myparagraph{Learning with fairness regularization}
Let $L(s, S)$ be any loss function ordinarily used to 
train an LTR system. 
Instead of optimizing solely this fairness-agnostic loss, 
we propose to optimize a fairness-aware regularized objective: 
\begin{equation}\label{eqn:weighted_objective}
L^{\text{fair}}(s;S) = L(s, S) + \alpha {\Gamma}(s, S)
\end{equation}
for $\alpha\geq 0$, where ${\Gamma}(s;S)$ is any of  
the empirical measures of fairness violation. 
The larger the value of $\alpha$, the more the resulting 
rankers will take also the fairness of its decisions into 
account rather than just their utility. 
In the case that constraints on the desired fairness of the 
system are given, \eg the often cited \emph{four-fifth rule}~\cite{biddle2006adverse}, 
then a suitable value of $\alpha$ can be determined by classic 
model selection, \eg using a validation set.  
In general, we expect the desired trade-off between 
utility and fairness to be influenced also by subjective 
factors and we leave $\alpha$ as a free parameter.
However, as our experiments in Section~\ref{sec:experiments} 
show, and as it has been observed in the context of
classification \cite{NEURIPS2019_373e4c5d}, the relation between fairness
and ranking quality is not necessarily adversarial. 

\myparagraph{Optimization}
The fairness regularization terms, $\alpha\Gamma(s,S)$, are absolute
values between differences of weighted sums over the score 
functions. 
Consequently, their values and their gradients can be computed 
efficiently using standard numerical frameworks. 
In large-scale settings, where ordinary gradient descent optimization 
is infeasible due to memory and computational limitations, the 
regularized objective~\eqref{eqn:weighted_objective} can also 
be optimized by stochastic gradient steps over mini-batches, 
as long as the unregularized loss function $L(s,S)$, supports this 
as well. 
The resulting per-batch gradient updates are not unbiased estimators 
of the full gradient, though, so the characteristics of the 
fairness notion changes depending on the batch size. 
For example, if batches were always formed of a single query 
with all associated documents, fairness would be enforced 
individually for each query, while the original objective 
enforces it averaged across all queries. 
In our experiments, however, we did not observe any deleterious 
effect of stochastic training when using a moderate batch size 
of 100.

%% file: concentration.tex
In this section we show that --given enough data-- our train-time regularization procedure will also ensure fairness at prediction time. 
Specifically, we prove a generalization bound by means
of a uniform concentration argument, showing that the
fairness on future decisions is bounded by the sum
of the fairness on the training set and a complexity
term, where the latter decreases monotonically towards
zero with the number of queries in the training set. 
Our results are similar to the ones in \cite{woodworth2017learning}
for the classification setting. However, in the case of ranking 
data there is additional dependence between the samples, which 
complicates the analysis and influences the complexity term.

\myparagraph{Data generation process}
To study the generalization properties of our fairness 
measures at training time versus prediction time, we first have to formally define the statistical properties of the 
training data. 
We assume the following data generation process
which is consistent with the structure of 
LTR datasets, 
with the only simplifying assumption that the item sets
for all queries are of equal size $m$. 

For a given data distribution $\mathbb{P}(q,d,r)$, a dataset 
$S = \{(q_i, d^i_j, r_{j}^i)\}_{i\in [N], j\in [m]}$, is sampled 
as follows: 1) queries, $q_1,\dots,q_N$, are sampled \iid 
from the marginal distribution $\mathbb{P}(q)$; 2) for each query 
$q_i$ independently a set of items, $D_{q_i}=\{d^i_1,\dots,d^i_{m}\}$, 
is sampled \emph{in an arbitrary way} with the only restriction 
that the marginal distribution of each individual $d^i_j$ should be 
$\mathbb{P}(d|q_i)$; 3) for each pair $(q_i,d^i_j)$ independently, 
the relevance $r^i_j$ is sampled from $\mathbb{P}(r | q_i, d^i_j)$. 

Note that each data point of the resulting training set has marginal 
distribution $\mathbb{P}$. Nevertheless, a lot of flexibility 
remains about how the actual items per query are chosen. 
In particular, the item set can have dependencies, such as avoiding repetitions or diversity constraints. 
While this choice of generating process complicates the theoretical 
analysis, we believe that it is necessary, because we want to make 
sure that real-world ranking data is covered, which typically 
is far from \iid.

We now characterize the generalization properties of the fairness 
regularizers. Let $\mathcal{F}\subset\{ f: Q\times D \to \{0,1\}\}$
be a set of item selection functions that make independent 
deterministic decisions per item (\eg, by thresholding a learned 
score function). 
Then, the following theorem holds:

\begin{theorem}\label{thm:uniform_bound}
Let $S$ be a dataset sampled as described above with $2Nm > v$ for
$v=\text{VCdim}(\mathcal{F})$.
Let $P = \min_{r,a}\big( \mathbb{P}(r(q,d) = r \wedge A(d) = a)\big)$ and $Q = \min_{a}\big(\mathbb{P}(A(d) = a)\big)$.
Then, for any $\delta>0$, each of the following inequalities holds
with probability at least $1-\delta$ over the sampling of $S$,
uniformly for all $f\in\mathcal{F}$:
\begin{align*}\label{eqn:result_vc}
\Gamma^{\textup{EOp}}(f) &\leq {\Gamma^{\textup{EOp}}}(f, S) + C_1,
\qquad
\Gamma^{\textup{EOd}}(f) \leq {\Gamma^{\textup{EOd}}}(f, S) + C_2,
\qquad
\Gamma^{\textup{DP}}(f) \leq {\Gamma^{\textup{DP}}}(f, S) + C_3,
\\
\text{with}\quad C_1 &= C_2 = 8 \sqrt{2\frac{v\log(\frac{2eNm}{v}) +
\log(\frac{48}{\delta})}{NP^2}},
\qquad
C_3 = 8 \sqrt{2\frac{v\log(\frac{2eNm}{v}) +
\log(\frac{24}{\delta})}{NQ^2}}.
\end{align*}
\end{theorem}

\textit{Proof sketch.}\quad 
The proof consists of two parts. First, for any fixed item selection
function a bound is shown on the gap between the conditional 
probabilities contributing to fairness measure and their empirical 
estimations. 
For this, we build on the technique of \cite{woodworth2017learning} 
for showing concentration of fairness quantities. We combine this with  
the large deviations bounds for sums of dependent random variables in terms of the chromatic number of their dependence graph
of~\cite{janson2004large}. 
Next, the bounds are extended to hold uniformly over the full 
hypothesis space by evoking a variant of the classic symmetrization 
argument (\eg~\cite{vapnik2013nature}), while carefully accounting 
for the dependence between the samples. 
A complete proof can be found in the supplementary material.  

\myparagraph{Discussion}
Theorem~\ref{thm:uniform_bound} bounds the fairness violation
on future data by the fairness on the training set plus an 
explicit complexity term, uniformly over all item selection
functions. 
Consequently, any item selection function with low fairness 
violation on the training set will have a similarly low fairness 
violation on new data, provided that enough data was used 
for training. Indeed, the complexity term decreases like $\sqrt{\log N/N}$ as $N\to \infty$, which is the 
expected behavior for a VC-based bound. We refer to the supplementary material for a more detailed discussion of the bound.

%% file: experiments.tex
We report on some experiments to validate the 
practicality and performance of our method for training 
fair LTR systems, including a large-scale setting.
Our emphasis lies on studying the interaction between model 
quality and fairness, the effectiveness of our 
proposed method for optimizing both of these notions on 
real data and on the comparison to previous fair ranking algorithms. 
For space reasons, we only provide a high-level description 
of the experimental setting here. Technical details, \eg on 
feature extraction, can be found in 
the supplemental material. 

\subsection{Datasets and experimental setup}\label{subsec:data_description}
We experiment on two datasets: the TREC Fairness data and MSMARCO.
As a measure of ranking quality we use $\NDCGk$ for 
$k\in\{1,2,3,4,5\}$, but also report results for 
$\Pk$ in the supplemental material.
To quantify fairness, we evaluate the three different 
empirical measures of fairness violation.

\myparagraph{TREC Fairness data} 
We use the training data of TREC 2019 Fairness track 
dataset~\cite{trec-fair-ranking-2019}. It consists of 
652 real-world queries taken from the Semantic Scholar 
search engine, together with a set of scientific papers 
for each query and binary labels for the relevance of 
every query-paper pair. 
The average number of labeled papers per query is $7.1$,
out of which $3.4$ are relevant on average. 
Because of the rather small number of queries, we 
use five-fold cross-validation to evaluate our method
and report averages and standard errors across the 
folds. As an exemplary \textit{protected attribute} we use a proxy of the 
authors' seniority. We split the set of documents 
into two groups based on whether the mean of their authors' 
$i10$-index proxies (as provided in the TREC data) 
exceeds a threshold $t$ or not. For $t\in\{3,4,5\}$ we 
get different amounts of group imbalance, with the 
minority group consisting of approximately $46\%, 26\%$ 
and $9\%$ of all papers, respectively.

\myparagraph{MSMARCO} We use the passage ranking dataset v2.1 of MSMARCO~\cite{nguyen2016ms}. 
It  consists of approximately one million natural language questions, which serve as queries, 
associated sets of potentially relevant passages 
from Internet sources, and binary relevance labels for all provided 
query-document pairs. 
On average, there are $8.8$ passages per question, 
and the average number of relevant ones is $0.65$.
For training and evaluation we use the default 
train-development split and report average and 
standard deviation over $10$ random seeds. To create a \textit{protected attribute}, we split the passages into 
two groups based on their top-level domains, thinking of it 
as a proxy of the answers' geographic origin.
Specifically, we split by \texttt{".com vs other"} 
(denoted by \emph{com}) and by \texttt{".com/.org/.gov/.edu/.net vs other"}
(denoted by \emph{ext}). 
Their minority groups are of size 32\% and 5\% of all
passages, respectively.

\begin{figure*}[t]
\begin{subfigure}{.24\textwidth}
  \centering
  \includegraphics[width=.95\linewidth]{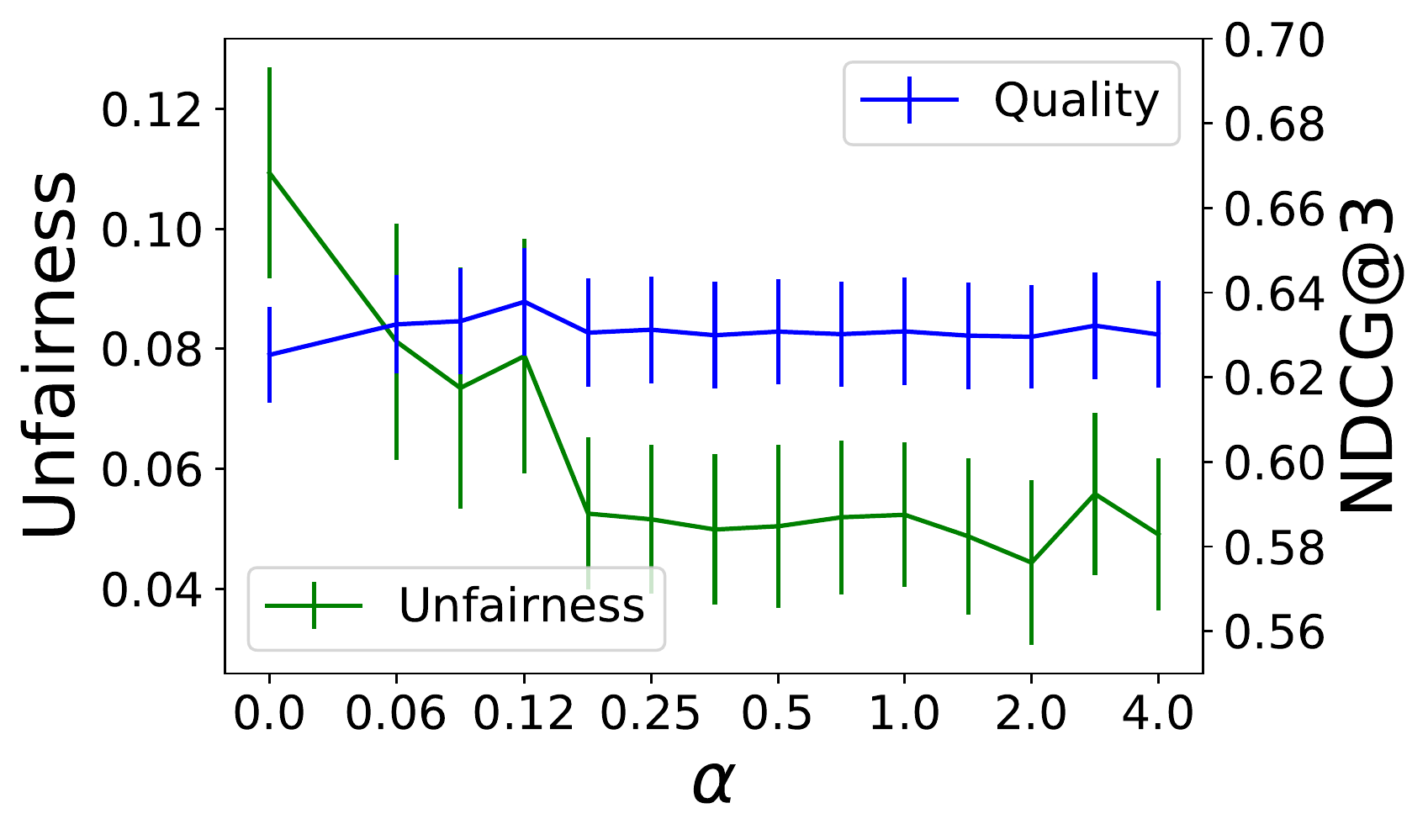}
  \caption{Ours, TREC}
  \label{fig:trec_amortized_demographic_parity}
\end{subfigure}%
\hfill
\begin{subfigure}{.24\textwidth}
  \centering
  \includegraphics[width=.95\linewidth]{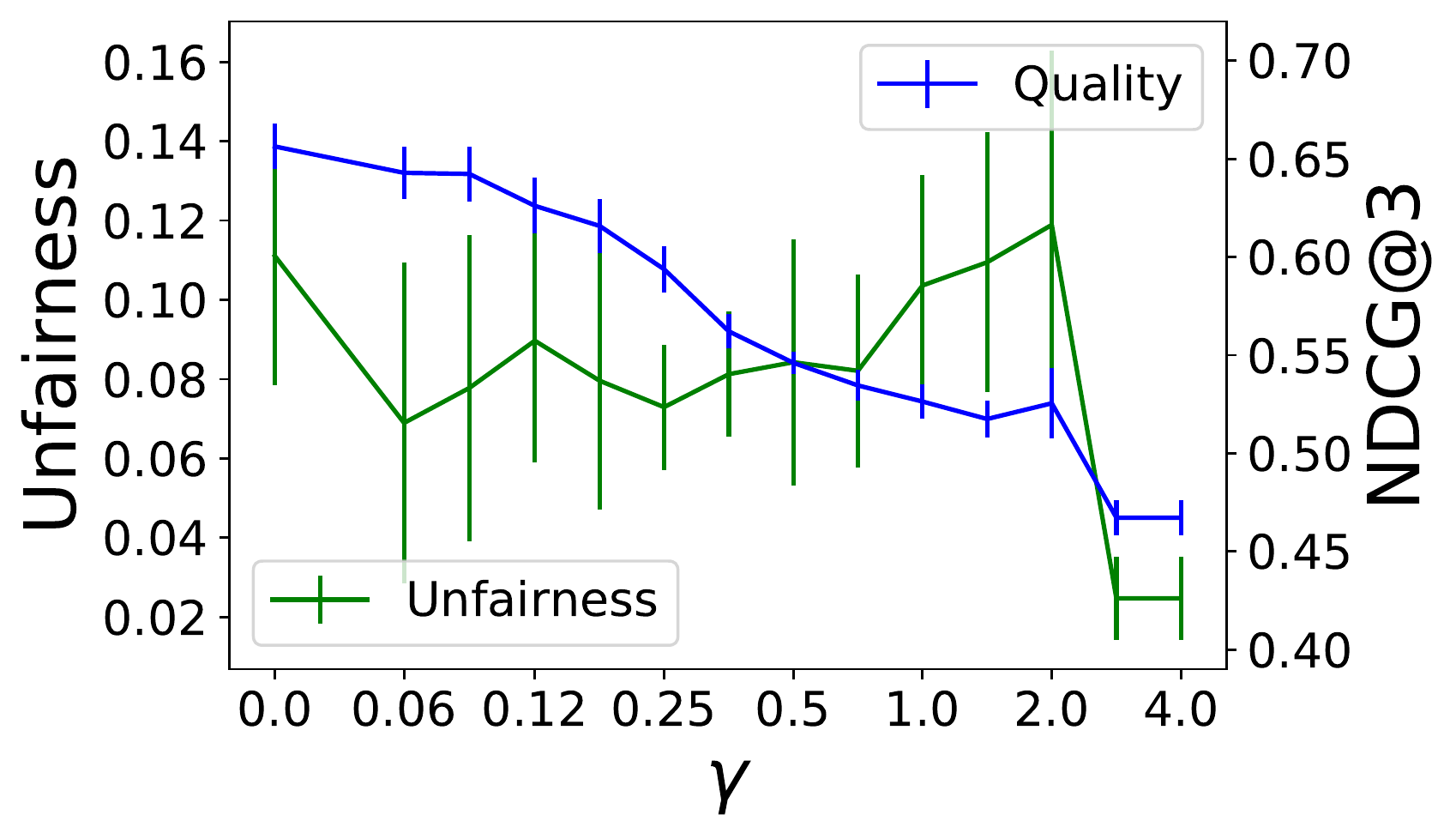}
  \caption{DELTR, TREC}
  \label{fig:trec_amortized_equal_odds}
\end{subfigure}%
\hfill
\begin{subfigure}{.24\textwidth}
  \centering
  \includegraphics[width=.95\linewidth]{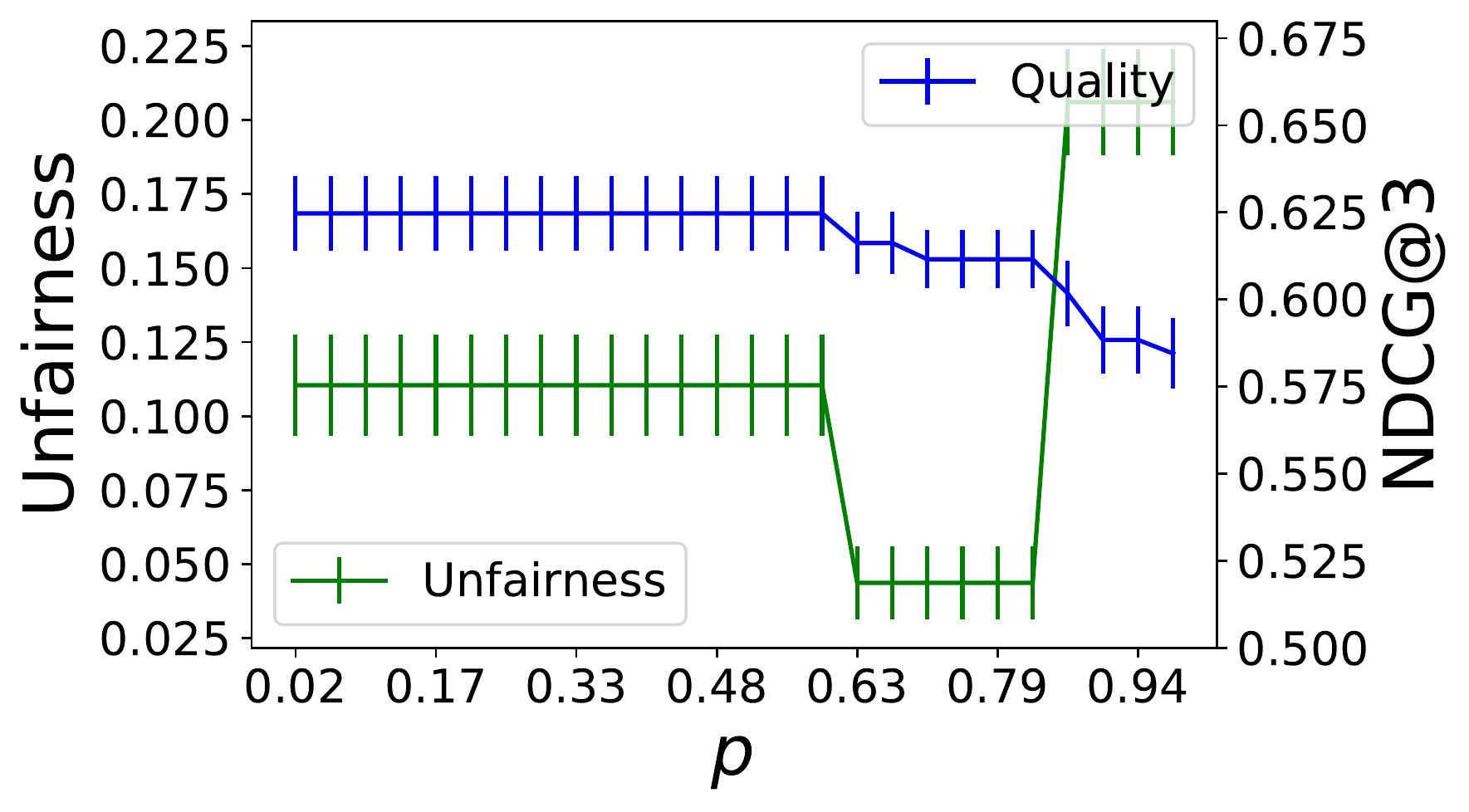}
  \caption{FA*IR, TREC}
  \label{fig:trec_amortized_equal_opp}
\end{subfigure}
\begin{subfigure}{.24\textwidth}
  \centering
  \includegraphics[width=.95\linewidth]{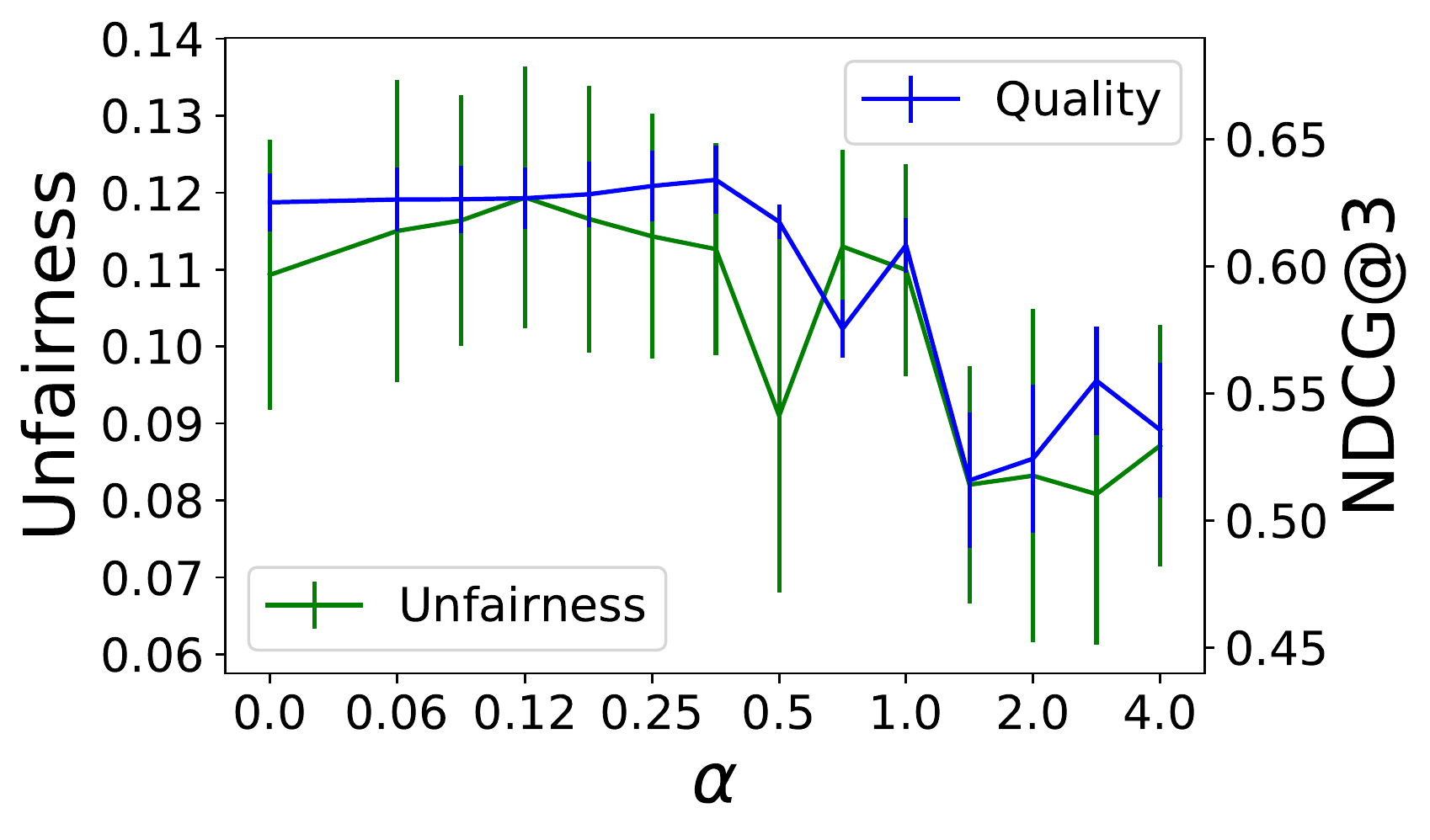}
  \caption{Per-query, TREC}
  \label{fig:trec_amortized_demographic_parity}
\end{subfigure}%
\vskip\baselineskip
\begin{subfigure}{.33\textwidth}
  \centering
  \includegraphics[width=.691\linewidth]{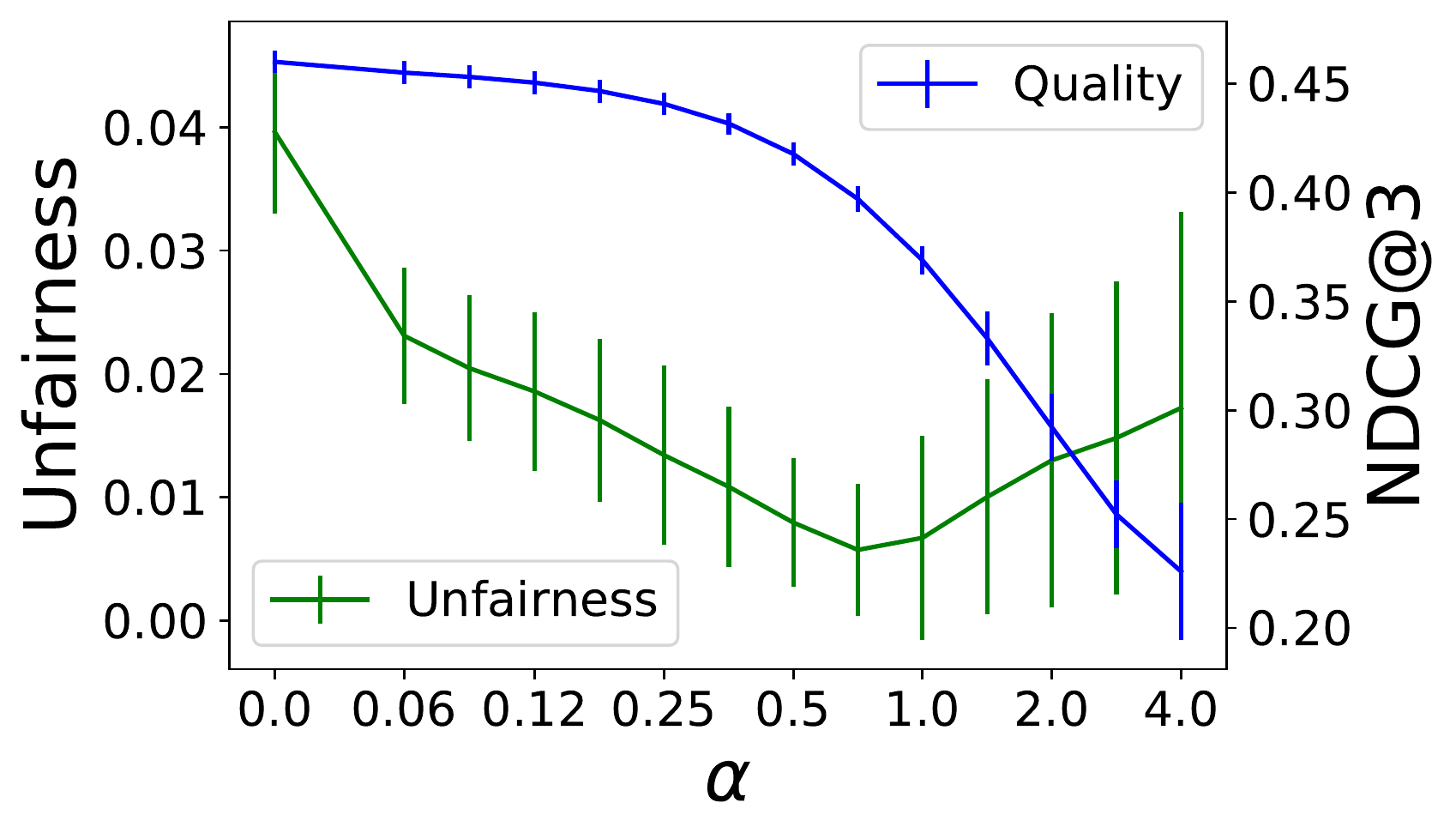}
  \caption{Ours, MSMARCO}
  \label{fig:trec_amortized_demographic_parity}
\end{subfigure}%
\hfill
\begin{subfigure}{.33\textwidth}
  \centering
  \includegraphics[width=.691\linewidth]{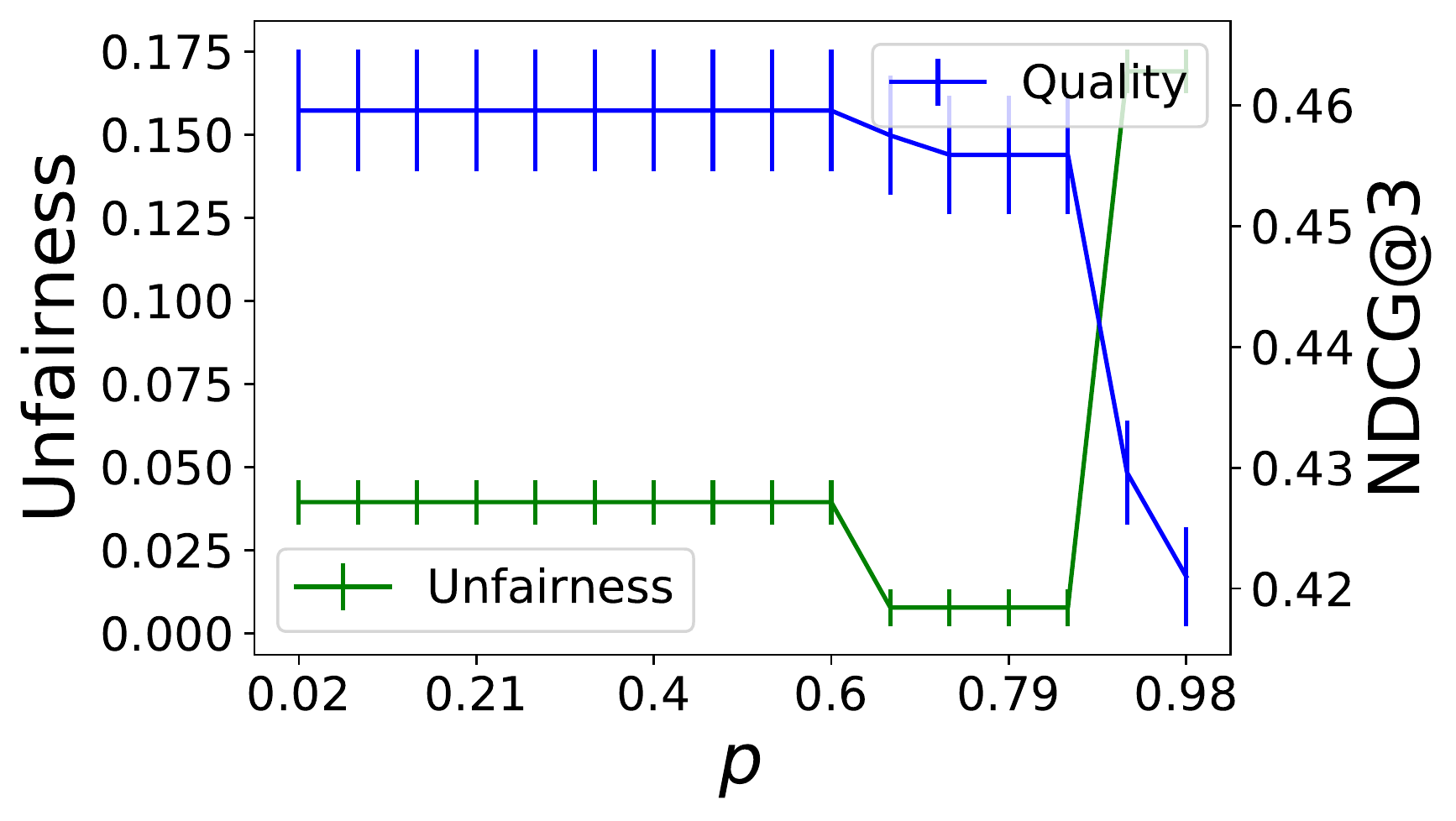}
  \caption{FA*IR, MSMARCO}
  \label{fig:trec_amortized_equal_opp}
\end{subfigure}
\begin{subfigure}{.33\textwidth}
  \centering
  \includegraphics[width=.691\linewidth]{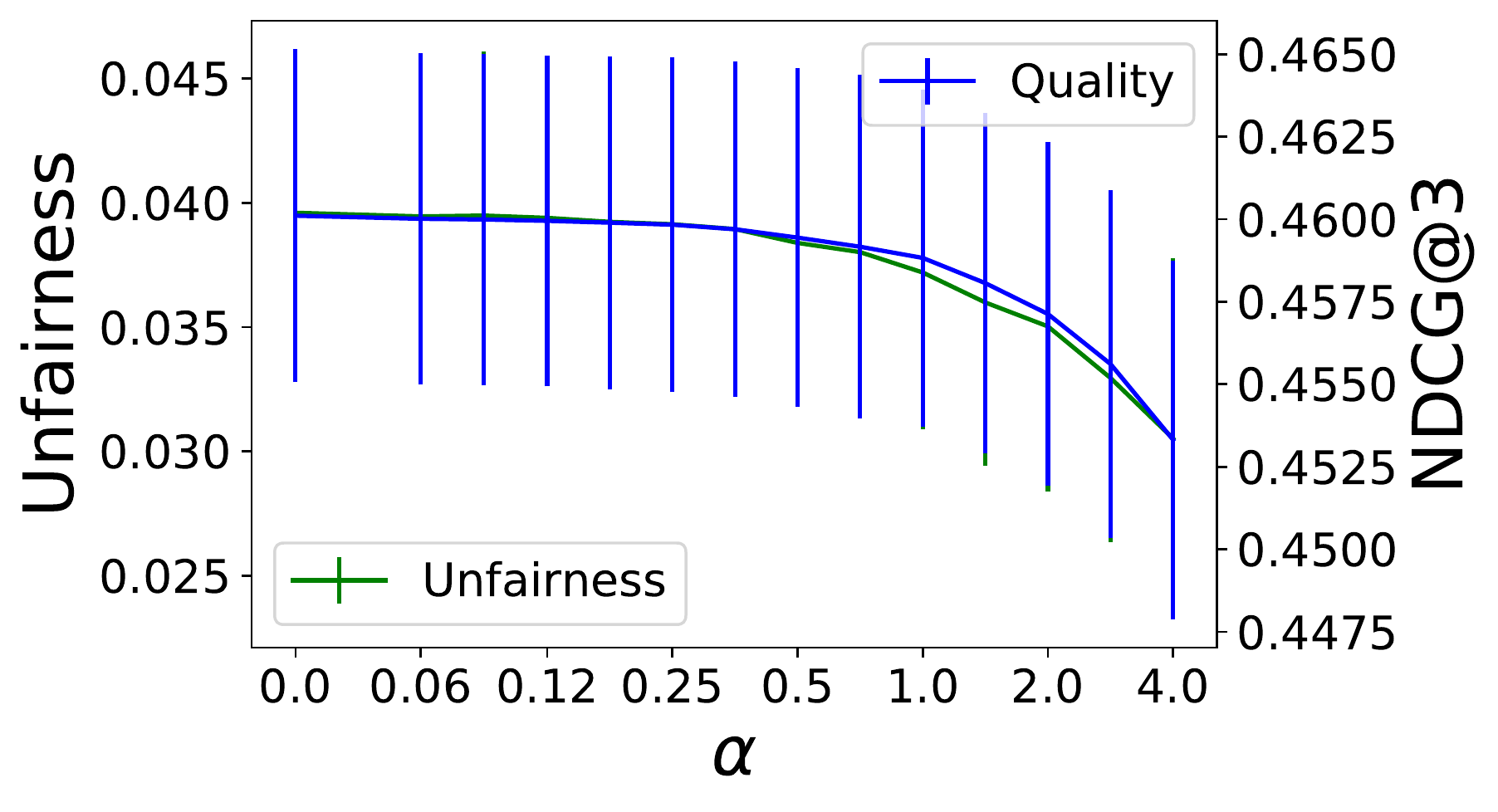}
  \caption{Per-query, MSMARCO}
  \label{fig:trec_amortized_demographic_parity}
\end{subfigure}%
\caption{Test-time performance of fair rankers with equal opportunity fairness, achieved by our algorithm and the baselines: unfairness (left $y$-axes) and NDCG@3 ranking 
quality (right $y$-axes); after training with different regularization 
strengths ($x$-axis). See the supplementary material for larger versions and for results for demographic parity and equalized odds.}
\label{fig:plots}
\end{figure*}

\vspace{-0.25cm}
\subsection{Learning to rank models} 
\myparagraph{Our algorithm} We adopt a classical pointwise LTR approach with 
a generalized linear score function, $s(d,q)=\langle \theta,\phi(q,d)\rangle$, for 
a predefined feature function, $\phi:\mathcal{Q}\times\mathcal{D}\to\mathbb{R}^D$
(see supplemental material). 
As loss function of ranking quality, $L(s,S)$, we use the squared loss 
between the relevance labels and the predictions of $s$ over all data. To optimize for both ranking quality and fairness, we train 
with a weighted loss, as in equation (\ref{eqn:weighted_objective}). 
For TREC we train all models by $1500$ steps of gradient 
descent with a learning rate of $0.003$. 
In the MSMARCO experiments we train with $5$ epochs of SGD with 
a batch size of $100$ queries and $10$ passages per query and a 
learning rate of $0.0001$.

\myparagraph{Baselines and ablation studies}
Our method is the first to enforce the well-established fairness notions from classification in a LTR setting. Hence previously developed methods for fair ranking aim to optimize for other (often single-purporse) fairness notions. Nevertheless, it is informative to see how such algorithms perform against our method, in order to understand the relationship between ours and previous fair ranking works. Therefore, we consider two recent methods for fair ranking, DELTR \cite{zehlike2020reducing} and FA*IR \cite{zehlike2017fa}, using the implementation provided by the authors \cite{zehlike2020fairsearch}.

DELTR is a state-of-the-art algorithm for fair LTR. At train time, a linear version of ListNet is trained, together with a regularizer tailored to a notion of disparate exposure \cite{calders2009building, singh2018fairness}. We use the same feature representations as for our method, as well as the same range for their regularization parameter $\gamma$, to ensure a fair comparison. Unfortunately, the implementation of \cite{zehlike2020fairsearch} does not scale to MSMARCO.

FA*IR, on the other hand, is an algorithm that \textit{changes the ranking query by query, at prediction time}, by ensuring that whenever $k$ items are retrieved, the proportion of retrieved items from a protected group is not smaller that the $\alpha$-th quantile of a binomial distribution $Bin(k,p)$, for fixed parameters $p, \alpha \in [0,1]$. We use $\alpha = 0.1$ and $p\in [0.02, 0.04, \ldots, 0.98]$. Note that in our LTR setting the true relevances of items at test time are unknown, so we first train via our method with $\alpha = 0$ and then, at test time, use the relevances predicted by our method as ground truth and apply FA*IR on top.

We also perform an ablation study by considering a version of our algorithm that learns to enforce fairness on the \textit{per-query level}. This is inspired by \cite{singh2017equality}, who, however, do not propose an algorithm for enforcing such per-query fairness notions. Within our framework this is achieved by regularizing with a separate term for every query in a batch and then averaging over the batch afterwards.

\vspace{-0.15cm}
\subsection{Results}
Figure \ref{fig:plots} shows the results when imposing different 
amounts of the equal opportunity fairness in typical settings 
for TREC ($t=3, k=3$; top row) and MSMARCO (\textit{com}, $k=3$; bottom row).
We also report the ranking quality and equal opportunity unfairness for the three baselines. As one can see, our method is able to consistently improve fairness. 
For TREC, this comes at no loss in ranking quality (here NDCG). For 
MSMARCO the loss is quite small for small to medium values of $\alpha$. 
As the figure shows, these observations are robust across the 
different amounts of regularization. In contrast, the fairness curves of the baselines behave erratically with respect to the trade-off parameters.

The possibility of increasing the fairness of learning  models 
without damaging their accuracy has been previously observed in 
the context of supervised learning \cite{NEURIPS2019_373e4c5d}. To the best of our knowledge, we are the first to 
observe this in a ranking context. 
This effect is more expressed in the experiment on the 
TREC data than for MSMARCO, possibly due to the higher number of relevant items per query in TREC, which results in more flexibility to rearrange items without decreasing the ranking quality.

\begin{table}[t]\footnotesize
\caption{Maximal and mean relative fairness increase, achievable without a significant decrease of ranking quality, for our algorithm and the baselines. See main text for details.}
\centering\setlength{\tabcolsep}{3.3pt}
 \begin{tabular}{| c c || c | c || c | c || c | c || c | c |} 
 \hline
 \multirow{2}{*}{TREC} & & \multicolumn{2}{c||}{Ours}  &  \multicolumn{2}{c||}{DELTR}  & \multicolumn{2}{c||}{FA*IR}  & \multicolumn{2}{c|}{Per-query} \\ 
  & & Max & Mean &  Max & Mean  & Max & Mean  & Max & Mean \\ 
 \hline\hline 
\multirow{3}{*}{\parbox{.1\textwidth}{equality of \newline opportunity}} & $t = 3$ & 48\% & 34\% & 39\% & 23\% & 56\% & 18\% & 19\% & -5\%\\
& $t = 4$ & 46\% & 37\% & 2\% & -8\% & 56\% & 11\% & 14\% & -4\%\\
& $t = 5$ & 46\% & 32\% & 18\% & 7\% & 6\% & -17\% & 8\% & -13\%\\
 \hline
\multirow{3}{*}{\parbox{.1\textwidth}{demographic \newline parity}} & $t = 3$ & 27\% & 17\% & 55\% & 36\% & 83\% & 24\% & 15\% & -1\%\\
& $t = 4$ & 44\% & 32\% & 12\% & 6\% & 56\% & 10\% & 27\% & 5\%\\
& $t = 5$ & 57\% & 40\% & 26\% & 14\% & 11\% & -35\% & 24\% & 1\%\\
 \hline
\multirow{3}{*}{\parbox{.1\textwidth}{equalized \newline odds}} & $t = 3$ & 20\% & 13\% & 48\% & 31\% & 57\% & 16\% & 14\% & -3\%\\
& $t = 4$ & 30\% & 21\% & 9\% & 4\% & 40\% & 3\% & 13\% & -1\%\\
& $t = 5$ & 29\% & 21\% & 21\% & 12\% & 0\% & -35\% & 7\% & -4\%\\
\hline
\multicolumn{2}{|c|}{average} & 39\% & 27\% & 26\% & 14\% & 41\% & -1\% & 16\% & -3\%\\
\hline
 \hline
 \multirow{2}{*}{MSMARCO} & & \multicolumn{2}{c||}{Ours} &  \multicolumn{2}{c||}{DELTR}  &  \multicolumn{2}{c||}{FA*IR}  & \multicolumn{2}{c|}{Per-query} \\ 
  & & Max & Mean &  Max & Mean &  Max & Mean & Max & Mean \\ 
 \hline\hline 
\multirow{2}{*}{\parbox{.1\textwidth}{equality of \newline opportunity}} &\emph{com} & 55\% & 36\% & NA & NA & 64\% & 11\% & 24\% & 6\%\\
& \emph{ext} & 19\% & 10\% & NA & NA & 0\% & -112\% & 2\% & 0\%\\
 \hline
\multirow{2}{*}{\parbox{.1\textwidth}{demographic \newline parity}} & \emph{com} & 42\% & 27\% & NA & NA & 39\% & -50\% & 0\% & 0\%\\
& \emph{ext} & 20\% & 13\% & NA & NA & 0\% & -168\% & 7\% & 3\%\\
 \hline
\multirow{2}{*}{\parbox{.1\textwidth}{equalized \newline odds}} & \emph{com} & 61\% & 41\% & NA & NA & 45\% & -12\% & 14\% & 4\%\\
& \emph{ext} & 28\% & 17\% & NA & NA & 0\% & -142\% & 1\% & 0\%\\
\hline
\multicolumn{2}{|c|}{average} & 37\% & 24\% & NA & NA & 25\% & -79\% & 8\% & 2\%\\
\hline
\end{tabular}
\label{table:results_ndcg_averages}
\end{table}

We obtained very similar results also for the other setups, \eg different values of $k$, fairness measures and protected attributes and for $\Pk$. Plots for these can be found in the supplemental material. \textit{Table~\ref{table:results_ndcg_averages} summarizes some of the results in a compact form.} For different fairness notions and splits into protected groups (rows), it reports the maximal and mean reduction of the 
fairness violation measure over the range of values of the trade-off parameter for which the 
corresponding model's prediction quality is not significantly 
worse than for a model trained without a fairness regularizer 
(\ie $\alpha = 0$). 
Here we call a model significantly worse than another 
if the difference of the mean quality values of the two models is 
larger than the sum of the standard errors/deviations, for TREC/MSMARCO respectively, around 
those averages (that is, if the error bars, as in Figure~\ref{fig:plots}, 
would not intersect). The results are averaged over $k \in \{1,2,3,4,5\}$, with individual versions in the supplementary material.

Intuitively, the max values quantify how much an algorithm can improve fairness without decreasing the ranking quality, while the mean values report the average improvement of fairness over the values of the trade-off parameters, as a more robust measure. The results confirm that in all cases our proposed training method is able 
to greatly reduce the unfairness in the test time ranking without 
majorly damaging ranking quality. In comparison, the baselines behave inconsistently between experiments and are less robust to the choice of the trade-off parameter, indicating that training with the right regularization, as integrated in our method, is indeed beneficial for test time fairness.

We refer to the supplementary material for further results, including splits over the values of $k$, experiments for $\Pk$ and plots of the performance of our algorithm in other scenarios.

%% file: conclusion.tex
We introduced a framework for transferring classification fairness notions to the context of LTR, by rephrasing ranking as a collection of 
query-dependent classification problems. 
This viewpoint, while technically elementary, opens a wide 
range of  possibilities for expanding the optimization methods 
and proof techniques from the fair classification literature to 
ranking and multi-label learning. 
In particular, we report the first -- to our 
knowledge -- generalization bound for group fairness in 
the context of ranking. We further show in our experiments that including a suitable regularizer 
during training can greatly improve the fairness 
of rankings with no or minor reduction in model quality. 
This effect seems even more pronounced than what had been
observed in classification tasks, especially if the set of 
relevant items for any query is large.
Therefore, we hypothesize that the multi-label nature 
of the ranking task naturally allows for more fairness without 
adverse effects on accuracy, and we deem making this intuition 
formal an interesting direction for future research.

%% file: appendix_proof.tex
To prove Theorem \ref{thm:uniform_bound}, we first introduce some classic definitions and concentration results for sums of dependent random variables from \cite{janson2004large} in Section \ref{subsec:chromatic_bounds}. Next we show in Section \ref{subsec:non_uniform_bound_proof} how these can be used to derive large deviation bounds for the three fairness notions, given a fixed classifier. The proof is similar to the corresponding i.i.d. result of \cite{woodworth2017learning}, however an application of the results from \cite{janson2004large} is needed because of the dependence between the samples. Finally, in Section \ref{subsec:uniform_bound_proof} we show how these bounds can be made uniform over the hypothesis space by adapting the classic symmetrization argument (e.g. \cite{vapnik2013nature}) to a dependent data scenario.

\subsection{Concentration inequalities for sums of dependent random variables}
\label{subsec:chromatic_bounds}
To deal with the dependence between the samples, we will use the following framework from \cite{janson2004large}. Let $Y_{\alpha}$ be a set of random variables, with $\alpha$ ranging over some index set $\mathcal{A}$. Let $X = \sum_{\alpha \in \mathcal{A}} Y_{\alpha}$. To derive concentration bounds for $X$, the following notions are useful:
\begin{definition}[\cite{janson2004large}]
Given $\mathcal{A}$ and $\{Y_{\alpha}\}_{\alpha \in \mathcal{A}}$:
\begin{itemize}
\item A subset $\mathcal{A}' \subset \mathcal{A}$ is independent if the random variables $\{Y_{\alpha}\}_{\alpha \in \mathcal{A}'}$ are (jointly) independent.
\item A family $\{\mathcal{A}_j\}_{j}$ is a cover of $\mathcal{A}$ if $\cup_j \mathcal{A}_j = \mathcal{A}$. A cover is proper if each set $\mathcal{A}_j$ is independent.
\item $\chi (\mathcal{A})$ is the size of the smallest proper cover of $\mathcal{A}$, that is the smallest integer $m$, such that $\mathcal{A}$ can be written as the union of $m$ independent subsets.
\end{itemize}
\end{definition}
\noindent Then the following result holds, similar to the Hoeffding inequality, but accounting for the amount of dependence between the random variables $\{Y_{\alpha}\}_{\alpha\in \mathcal{A}}$:
\begin{theorem}[\cite{janson2004large}]
Let $Y_{\alpha}$ and $X$ be as above, with $a_{\alpha} \leq Y_{\alpha} \leq b_{\alpha}$ for every $\alpha \in \mathcal{A}$, for some real numbers $a_{\alpha}$ and $b_{\alpha}$. Then, for every $t>0$:
\begin{equation}
\label{eqn:chromatic_concentration_sum}
\mathbb{P}(X \geq \mathbb{E}(X) + t) \leq \exp \left(-2 \frac{t^2}{\chi (\mathcal{A}) \sum_{\alpha \in \mathcal{A}} (b_{\alpha} - a_{\alpha})^2}\right). 
\end{equation}
The same upper bound holds for $\mathbb{P}(X \leq \mathbb{E}(X) - t)$.
\end{theorem}
\noindent If instead one considers the mean of $\{Y_{\alpha}\}_{\alpha \in \mathcal{A}}$, namely $\bar{X} = \frac{1}{|\mathcal{A}|} \sum_{\alpha \in \mathcal{A}} Y_{\alpha}$, then the following holds:
\begin{equation}
\label{eqn:chromatic_concentration_mean}
\mathbb{P}(\bar{X} \geq \mathbb{E}(\bar{X}) + t) \leq \exp \left(-2 \frac{t^2 |\mathcal{A}|^2}{\chi (\mathcal{A}) \sum_{\alpha \in \mathcal{A}} (b_{\alpha} - a_{\alpha})^2}\right).
\end{equation}
Specifically, if the $Y_{\alpha}$ are Bernoulli random variables:
\begin{equation}
\label{eqn:chromatic_concentration_bernoulli}
\mathbb{P}(\bar{X} \geq \mathbb{E}(\bar{X}) + t) \leq \exp \left(-2 \frac{t^2 |\mathcal{A}|}{\chi (\mathcal{A})}\right).
\end{equation}

\subsection{Non-uniform bounds}
\label{subsec:non_uniform_bound_proof}

First we use the tools from the previous section and a technique of \cite{woodworth2017learning, agarwal2018reductions} to show a non-uniform Hoeffding-type bound for equal opportunity and equalized odds:
\begin{lemma}
\label{lemma:hoeffding_type_bound-opp-odds}
Fix $\delta\in (0,1)$ and a binary predictor $f:Q\times D \to \{0,1\}$. Suppose that $N > \frac{8\log(8/\delta)}{P^2}$, where $P = \min_{ar} \mathbb{P}(A(d) = a, r(q,d) = r)$, then:
\begin{align}
\label{eqn:hoeffding_type_bound_equal_opp}
\mathbb{P}\left(|\Gamma^{\textrm{EOp}}(f,S) - \Gamma^{\textrm{EOp}}(f)| > 2\sqrt{\frac{\log(8/\delta)}{NP}}\right)  \leq \delta.
\end{align}
and 
\begin{align}
\label{eqn:hoeffding_type_bound_equal_odds}
\mathbb{P}\left(|\Gamma^{\textrm{EOd}}(f,S) - \Gamma^{\textrm{EOd}}(f)| > 2\sqrt{\frac{\log(16/\delta)}{NP}}\right)  \leq \delta.
\end{align}
\end{lemma} 

\begin{proof}
Denote by $I_{ar} = \{(i,j): A(d^i_j) = a, r(q_i, d^i_j) = r\}$ the set of indexes of the training data for which the document belongs to the group $a$ and the relevance of the query-document pair is $r$. Notice that $I_{ar}$ is a random variable and that $|I_{ar}| = |S_{a,r}|$. We first bound the probability of a large deviation of $$\gamma^S_{ar}(f) \vcentcolon = \frac{1}{|I_{ar}|}\sum_{(i,j)\in I_{ar}} f(q_i, d^i_j)$$ from $\gamma_{ar}(f) \vcentcolon = \mathbb{P}(f(q,d) = 1 | A(d) = a, r(q,d) = r)$, for each pair $r\in\{0,1\}, a \in\{0,1\}$. Since $f$ is fixed here, we omit the dependence of $\gamma_{ar}(f), \gamma^S_{ar}(f), \Gamma^{\textrm{EOp}}(f), \Gamma^{\textrm{EOd}}(f)$, \etc on $f$ for the rest of this proof. 

For any fixed $I_{ar}$:
\begin{align}
\mathbb{E}\left(\gamma^S_{ar}|I_{ar}\right) = \mathbb{E}\left(\frac{1}{|I_{ar}|}\sum_{(i,j)\in I_{ar}} f(q_i, d^i_j)\right) = \mathbb{P}(f(q, d) = 1| A(d) = a, r(q, d) = r) = \gamma_{ar}(f),
\end{align}
since the marginal distribution of every $(q_i, d^i_j, r(q_i, d^i_j))$ is $\mathbb{P}$. It is also easy to see that if $\mathcal{A} = \{(i,j): i \in [N], j \in [m]\}$ is the index set of the random variables $Y_{(i,j)} = f(q_i, d^i_j)$, then $\chi (\mathcal{A}) = m$. Therefore, for any fixed set $I_{ar} \subset \mathcal{A}$, we have $\chi(I_{ar}) \leq \chi(\mathcal{A}) = m$. Now conditional on $I_{ar}$:
\begin{align}
\label{eqn:equal_opp_lemma1}
\mathbb{E}(|\gamma^S_{ar} - \gamma_{ar}| > t| I_{ar}) = \mathbb{E}\left(\left|\frac{1}{|I_{ar}|}\sum_{(i,j)\in I_{ar}} f(q_i, d^i_j) - \gamma_{ar}\right| > t\right) \leq 2 \exp \left(-2\frac{t^2 |I_{ar}|}{m}\right).
\end{align}
Similarly, $|I_{ar}| = \sum_{i\in [N]}\sum_{j \in [m]} \mathbbm{1}(r(q_i, d^i_j) = r, A(d^i_j) = a)$ is the sum of $Nm$ Bernoulli random variables indexed by $\mathcal{A} = \{(i,j)\}_{i\in [N], j \in [m]}$, such that $\chi (\mathcal{A}) = m$. Denote by $P_{ar} = \mathbb{P}(A(d) = a, r(q,d) = r)$ and recall the notation $P = \min_{ar}P_{ar}$. Then $\mathbb{E}(|I_{ar}|) = P_{ar}Nm$. Therefore,
\begin{align*}
\mathbb{P}\left(|I_{ar}| \leq P_{ar}Nm - t\right) \leq \exp\left(-2\frac{t^2}{Nm^2}\right).
\end{align*}
Setting $t = P_{ar}Nm/2$, we obtain:
\begin{align}
\label{eqn:equal_opp_lemma2}
\mathbb{P}\left(|I_{ar}| \leq \frac{P_{ar}}{2}Nm\right) \leq \exp \left(- \frac{P^2_{ar}N}{2}\right).
\end{align}
Now assume that $N \geq \frac{2\log(8/\delta)}{P^2}$. Then for any $r\in\{0,1\}, a\in \{0,1\}$:
\begin{align*}
\mathbb{P}(|\gamma^S_{ar} - \gamma_{ar}| > t) & = \sum_{I_{ar}} \mathbb{P}(|\gamma^S_{ar} - \gamma_{ar}| > t | I_{ar})\mathbb{P}(I_{ar}) \\ & \leq \mathbb{P}(|I_{ar}| \leq \frac{P_{ar}}{2}Nm) + \sum_{I_{ar}: |I_{ar}| \geq \frac{P_{ar}Nm}{2}}\mathbb{P}(|\gamma^S_{ar} - \gamma_{ar}| > t | I_{ar})\mathbb{P}(I_{ar}) \\ & \leq \exp \left(-\frac{P^2_{ar}N}{2}\right) + \sum_{I_{ar}: |I_{ar}| \geq \frac{P_{ar}Nm}{2}} 2 \exp\left(-2\frac{t^2|I_{ar}|}{m}\right) \mathbb{P}(S_{ar}) \\ & \leq \frac{\delta}{8} + 2\exp\left(-t^2 NP_{ar}\right).
\end{align*}
The rest of the proof proceeds as in \cite{woodworth2017learning}. For a fixed $r\in\{0,1\}$ the triangle law gives:
\begin{align*}
||\gamma^S_{0r} - \gamma^S_{1r}| - |\gamma_{0r} - \gamma_{1r}|| \leq |\gamma^S_{0r} - \gamma^S_{1r} - \gamma_{0r} + \gamma_{1r}| \leq |\gamma^S_{0r} - \gamma_{0r}| + |\gamma^S_{1r} - \gamma_{1r}|.
\end{align*}
Therefore,
\begin{align*}
\mathbb{P}(||\gamma^S_{0r} - \gamma^S_{1r}| - |\gamma_{0r} - \gamma_{1r}|| > 2t) & \leq \mathbb{P}(|\gamma^S_{0r} - \gamma_{0r}| + |\gamma^S_{1r} - \gamma_{1r}| > 2t) \\ & \leq \mathbb{P}((|\gamma^S_{0r} - \gamma_{0r}| > t) \lor (|\gamma^S_{1r} - \gamma_{1r}| > t)) \\ & \leq \mathbb{P}(|\gamma^S_{0r} - \gamma_{0r}| > t) + \mathbb{P}(|\gamma^S_{1r} - \gamma_{1r}| > t) \\ & \leq \frac{\delta}{4} + 4\exp (-t^2N P).
\end{align*}
Setting $t = t_0 = \sqrt{\frac{\log(16/\delta)}{NP}}$ gives:
\begin{align*}
\mathbb{P}\left(||\gamma^S_{0r} - \gamma^S_{1r}| - |\gamma_{0r} - \gamma_{1r}|| > 2\sqrt{\frac{\log(16/\delta)}{NP}}\right)  \leq \frac{\delta}{4} + 4\frac{\delta}{16} = \frac{\delta}{2}.
\end{align*}
Setting $r = 1$ gives the first result.

For the second result, note that taking the union bound over $r\in\{0,1\}$ shows that with probability at least $1-\delta$  both $||\gamma^S_{00} - \gamma^S_{10}| - |\gamma_{00} - \gamma_{10}|| \leq 2t_0$ and $||\gamma^S_{01} - \gamma^S_{11}| - |\gamma_{01} - \gamma_{11}|| \leq 2t_0$ hold.

Under this event we have:
\begin{align*}
|\Gamma^{\textrm{EOd}}(f,S) - \Gamma^{\textrm{EOd}}(f)| & = \left|\frac{1}{2}\left(|\gamma^S_{00} - \gamma^S_{10}| + |\gamma^S_{01} - \gamma^S_{11}|\right) - \frac{1}{2}\left(|\gamma_{00} - \gamma_{10}| + |\gamma_{01} - \gamma_{11}|\right)\right| \\
& =  \left|\frac{1}{2}\left(|\gamma^S_{00} - \gamma^S_{10}| - |\gamma_{00} - \gamma_{10}|\right) + \frac{1}{2}\left(|\gamma^S_{01} - \gamma^S_{11}| - |\gamma_{01} - \gamma_{11}| \right)\right| \\
& \leq \frac{1}{2}\left||\gamma^S_{00} - \gamma^S_{10}| - |\gamma_{00} - \gamma_{10}|\right| + \frac{1}{2}\left||\gamma^S_{01} - \gamma^S_{11}| - |\gamma_{01} - \gamma_{11}|\right| \\
& \leq 2t_0
\end{align*}
and hence the result follows.
\end{proof}

An identical argument, by conditioning on the values of the set $I_{a} = \{(i,j): A(d^i_j) = a\} $ gives a similar result for demographic parity:

\begin{lemma}
\label{lemma:hoeffding_type_bound_demog_par}
Fix $\delta\in (0,1)$ and a binary predictor $f:Q\times D \to \{0,1\}$. Suppose that $N > \frac{8\log(8/\delta)}{Q^2}$, where $Q = \min_{a} \mathbb{P}(A(d) = a)$, then:
\begin{align}
\label{eqn:hoeffding_type_bound_demog_par}
\mathbb{P}\left(|\Gamma^{\textrm{DP}}(f,S) - \Gamma^{\textrm{DP}}(f)| > 2\sqrt{\frac{\log(8/\delta)}{NQ}}\right)  \leq \delta.
\end{align}
\end{lemma}

\subsection{Uniform bounds}
\label{subsec:uniform_bound_proof}

In this section we show how to formally extend the non-uniform bounds from the previous section to hold uniformly over the hypothesis space $\mathcal{H}$.

Let $S' = \{(q'_i, d'^i_j, r(q'_i, d'^i_j))\}_{i\in[N], j\in [m]}$ be a ghost sample independent of $S$ and also sampled via the same procedure as $S$, as described in the main body of the paper. In the proof of Lemma \ref{lemma:hoeffding_type_bound-opp-odds} we showed that for any classifier $f$ and any $t \in (0,1)$:
\begin{align}
\label{eqn:actual_hoeffding_bound_opp}
\mathbb{P}\left(|\Gamma^{\textrm{EOp}}(f) - \Gamma^{\textrm{EOp}}(f, S)| > 2t\right) \leq 2 \exp\left(-\frac{P^2N}{2}\right) + 4\exp\left(-\frac{t^2NP}{2}\right) \leq 6\exp\left(-\frac{t^2NP^2}{2}\right)\\
\label{eqn:actual_hoeffding_bound_odds}
\mathbb{P}\left(|\Gamma^{\textrm{EOd}}(f) - \Gamma^{\textrm{EOd}}(f, S)| > 2t\right) \leq 4 \exp\left(-\frac{P^2N}{2}\right) + 8\exp\left(-\frac{t^2NP}{2}\right) \leq 12\exp\left(-\frac{t^2NP^2}{2}\right)\\
\end{align}
Similarly, from the proof of Lemma 2
\begin{align}
\label{eqn:actual_hoeffding_bound_par}
\mathbb{P}\left(|\Gamma^{\textrm{DP}}(f) - \Gamma^{\textrm{DP}}(f, S)| > 2t\right) \leq 2 \exp\left(-\frac{Q^2N}{2}\right) + 4\exp\left(-\frac{t^2NQ}{2}\right) \leq 6\exp\left(-\frac{t^2NQ^2}{2}\right)
\end{align}
We will use these in particular to prove the following symmetrization lemma:
\begin{lemma}
\label{lemma:symmetrization}
For any $1 > t \geq 4\sqrt{\frac{2\log(12)}{NP^2}}$,
\begin{align}
\label{eqn:symmetrization_opp}
\mathbb{P}_S \left(\sup_{f\in\mathcal{F}} (\Gamma^{\textrm{EOp}}(f) - \Gamma^{\textrm{EOp}}(f, S)) \geq t\right) \leq 2\mathbb{P}_{S, S'}\left(\sup_{f\in\mathcal{F}} (\Gamma^{\textrm{EOp}}(f, S') - \Gamma^{\textrm{EOp}}(f, S)) \geq t/2\right).
\end{align}
For any $1 > t \geq 4\sqrt{\frac{2\log(24)}{NP^2}}$:
\begin{align}
\label{eqn:symmetrization_odds}
\mathbb{P}_S \left(\sup_{f\in\mathcal{F}} (\Gamma^{\textrm{EOd}}(f) - \Gamma^{\textrm{EOd}}(f, S)) \geq t\right) \leq 2\mathbb{P}_{S, S'}\left(\sup_{f\in\mathcal{F}} (\Gamma^{\textrm{EOd}}(f, S') - \Gamma^{\textrm{EOd}}(f, S)) \geq t/2\right).
\end{align}
For any $1 > t \geq 4\sqrt{\frac{2\log(12)}{NQ^2}}$:
\begin{align}
\label{eqn:symmetrization_par}
\mathbb{P}_S \left(\sup_{f\in\mathcal{F}} (\Gamma^{\textrm{DP}}(f) - \Gamma^{\textrm{DP}}(f, S)) \geq t\right) \leq 2\mathbb{P}_{S, S'}\left(\sup_{f\in\mathcal{F}} (\Gamma^{\textrm{DP}}(f, S') - \Gamma^{\textrm{DP}}(f, S)) \geq t/2\right).
\end{align}
\end{lemma}
\begin{proof}
We show the result for the equal opportunity fairness measure, the rest follow in an identical manner. 

Let $f^*$ be the function achieving the supremum on the left-hand side \footnote{If the supremum is not attained, this argument can be repeated for each element of a sequence of classifiers approaching the supremum}. Note that:
\begin{align*}
\mathbbm{1}(\Gamma^{\textrm{EOp}}(f^{*}) - \Gamma^{\textrm{EOp}}(f^*,S) \geq t) & \mathbbm{1}(\Gamma^{\textrm{EOp}}(f^*) - \Gamma^{\textrm{EOp}}(f^*, S') < t/2)\\
& = \mathbbm{1}(\Gamma^{\textrm{EOp}}(f^{*}) - \Gamma^{\textrm{EOp}}(f^*,S) \geq t \wedge \Gamma^{\textrm{EOp}}(f^*,S') - \Gamma^{\textrm{EOp}}(f^*) > -t/2) \\ & \leq \mathbbm{1}(\Gamma^{\textrm{EOp}}(f^*,S') - \Gamma^{\textrm{EOp}}(f^*,S) > t/2).
\end{align*}
Taking expectation with respect to $S'$:
\begin{align*}
\mathbbm{1}(\Gamma^{\textrm{EOp}}(f^{*}) - \Gamma^{\textrm{EOp}}(f^*,S) \geq t)\mathbb{P}_{S'}(\Gamma^{\textrm{EOp}}(f^*) - \Gamma^{\textrm{EOp}}(f^*,S') < t/2) \leq \mathbb{P}_{S'}(\Gamma^{\textrm{EOp}}(f^*,S') - \Gamma^{\textrm{EOp}}(f^*,S) > t/2).
\end{align*}
Now using (\ref{eqn:actual_hoeffding_bound_opp}):
\begin{align*}
\mathbb{P}_{S'}(\Gamma^{\textrm{EOp}}(f^*) - \Gamma^{\textrm{EOp}}(f^*,S') \geq t/2) \leq 6\exp\left(-\frac{t^2NP^2}{32}\right) \leq \frac{1}{2},
\end{align*}
so:
\begin{align*}
\frac{1}{2}\mathbbm{1}(\Gamma^{\textrm{EOp}}(f^{*}) - \Gamma^{\textrm{EOp}}(f^*,S) \geq t) \leq \mathbb{P}_{S'}(\Gamma^{\textrm{EOp}}(f^*,S') - \Gamma^{\textrm{EOp}}(f^*,S) > t/2).
\end{align*}
Taking expectation with respect to $S$:
\begin{align*}
\mathbb{P}_S(\Gamma^{\textrm{EOp}}(f^{*}) - \Gamma^{\textrm{EOp}}(f^*,S) \geq t) & \leq 2 \mathbb{P}_{S, S'}(\Gamma^{\textrm{EOp}}(f^*,S') - \Gamma^{\textrm{EOp}}(f^*,S) > t/2) \\
& \leq 2 \mathbb{P}_{S, S'}(\sup_{f\in\mathcal{F}}(\Gamma^{\textrm{EOp}}(f,S') - \Gamma^{\textrm{EOp}}(f,S)) \geq t/2).
\end{align*}
\end{proof}
\noindent Given a set of $n$ input datapoints $z_1, \ldots, z_n$ with $z_i = (q_i, d_i, r(q_i, d_i))$, consider:
\begin{equation}
\mathcal{F}_{z_1, \ldots, z_n} = \{(f(q_1, d_1), \ldots, f(q_n, d_n)): f\in\mathcal{F}\}
\end{equation}
Then the growth function of $\mathcal{F}$ is defined as:
\begin{equation}
S_{\mathcal{F}}(n) = \sup_{(z_1, \ldots, z_n)}|\mathcal{F}_{z_1, \ldots, z_n}|
\end{equation}

\addtocounter{theorem}{-2}

We can now present a proof of Theorem \ref{thm:uniform_bound}:
\begin{theorem}
\label{thm:uniform_bound}
Suppose that $v = VC(\mathcal{F}) \geq 1$ and that $2Nm > v$. Then for any $\delta \in (0,1)$:
\begin{align}
\label{eqn:result_vc_opp}
\mathbb{P}_S \left(\sup_{f\in\mathcal{F}} (\Gamma^{\textrm{EOp}}(f) - \Gamma^{\textrm{EOp}}(f,S)) \geq 8 \sqrt{2\frac{d\log(\frac{2eNm}{d}) + \log(\frac{24}{\delta})}{NP^2}}\right) \leq \delta \\
\label{eqn:result_vc_DP}
\mathbb{P}_S \left(\sup_{f\in\mathcal{F}} (\Gamma^{\textrm{DP}}(f) - \Gamma^{\textrm{DP}}(f,S)) \geq 8 \sqrt{2\frac{d\log(\frac{2eNm}{d}) + \log(\frac{24}{\delta})}{NQ^2}}\right) \leq \delta \\
\label{eqn:result_vc_odds}
\mathbb{P}_S \left(\sup_{f\in\mathcal{F}} (\Gamma^{\textrm{EOd}}(f) - \Gamma^{\textrm{EOd}}(f,S)) \geq 8 \sqrt{2\frac{d\log(\frac{2eNm}{d}) + \log(\frac{48}{\delta})}{NP^2}}\right) \leq \delta 
\end{align}
\end{theorem}
\begin{proof}
Again we present the proof for equal opportunity, with the other inequalities following in an identical manner.

Note that given sets $S$ and $S'$, the values of $\Gamma^{\textrm{EOp}}(f,S)$ and $\Gamma^{\textrm{EOp}}(f,S')$ are completely determined by the values of $f$ on $S$ and $S'$ respectively. Therefore, for any $t\in \left(4\sqrt{\frac{2\log(12)}{NP^2}},1\right)$ using Lemma \ref{lemma:symmetrization} and the union bound:
\begin{align*}
\mathbb{P}_S \left(\sup_{f\in\mathcal{F}} (\Gamma^{\textrm{EOp}}(f) - \Gamma^{\textrm{EOp}}(f,S)) \geq t\right) & \leq 2\mathbb{P}_{S, S'}\left(\sup_{f\in\mathcal{F}} (\Gamma^{\textrm{EOp}}(f,S') - \Gamma^{\textrm{EOp}}(f,S)) \geq t/2\right) \\ & \leq 2S_{\mathcal{F}}(2Nm)\mathbb{P}_{S, S'}\left(\Gamma^{\textrm{EOp}}(f,S') - \Gamma^{\textrm{EOp}}(f,S) \geq t/2\right) \\ & \leq 2S_{\mathcal{F}}(2Nm)\mathbb{P}_{S, S'}\left((|\Gamma^{\textrm{EOp}}(f,S') - \Gamma^{\textrm{EOp}}(f)| \geq t/4) \right. \\
& \lor \left. (|\Gamma^{\textrm{EOp}}(f) - \Gamma^{\textrm{EOp}}(f,S)| \geq t/4)\right) \\ & \leq 4S_{\mathcal{F}}(2Nm) \mathbb{P}_S\left(|\Gamma^{\textrm{EOp}}(f) - \Gamma^{\textrm{EOp}}(f,S)| \geq t/4\right) \\ & \leq 24S_{\mathcal{F}}(2Nm)\exp\left(-\frac{t^2NP^2}{128}\right)
\end{align*}
In particular, if $d = VC(\mathcal{F})$, by Sauer's lemma $S_{\mathcal{F}}(2Nm) \leq \left(\frac{2eNm}{d}\right)^d$ whenever $2Nm > d$, so:
\begin{align*}
\mathbb{P}_S \left(\sup_{f\in\mathcal{F}} (\Gamma^{\textrm{EOp}}(f) - \Gamma^{\textrm{EOp}}(f,S)) \geq t\right) \leq 24 \left(\frac{2eNm}{d}\right)^d\exp\left(-\frac{t^2NP^2}{128}\right)
\end{align*}
It follows that:
\begin{align}
\label{eqn:results_vc_proof}
\mathbb{P}_S \left(\sup_{f\in\mathcal{F}} (\Gamma^{\textrm{EOp}}(f) - \Gamma^{\textrm{EOp}}(f,S)) \geq 8 \sqrt{2\frac{d\log(\frac{2eNm}{d}) + \log(\frac{24}{\delta})}{NP^2}}\right) \leq \delta
\end{align}
whenever:
\begin{align*}
1 > 8 \sqrt{2\frac{d\log(\frac{2eNm}{d}) + \log(\frac{24}{\delta})}{NP^2}} \geq 4\sqrt{\frac{2\log(12)}{NP^2}} 
\end{align*}
It is easy to see that the right inequality holds whenever $d\geq 1$, $2Nm \geq d$ and $\delta < 1$. In addition, inequality (\ref{eqn:results_vc_proof}) trivially holds if the left inequality is not fulfilled. Hence the result follows. 
\end{proof}

\subsection{Discussion}
\label{sec:app_discussion}

Theorem~\ref{thm:uniform_bound} bounds the fairness violation
on future data by the fairness on the training set plus an 
explicit complexity term, uniformly over all item selection
functions. 
Consequently, any item selection function with low fairness 
violation on the training set will have a similarly low fairness 
violation on new data, provided that enough data was used 
for training. Indeed, the complexity term decreases like $\sqrt{\log N/N}$ as $N\to \infty$, which is the 
expected behavior for a VC-based bound.

The same scaling behavior does not hold with respect to the 
number of items per query, $m$. 
This is unfortunate, but unavoidable, given the weak assumptions 
we make on the data generation process: because we do not restrict 
how the per-query item sets are created, each of them could 
simply consist of many copies of a single item. In that case, 
even arbitrary large $m$ would provide only as much 
information as $m=1$.  
In the current form, $m$ appears even logarithmically 
in the numerator of the complexity term. We believe this to be 
an artifact of our proof technique, and expect that a more refined
analysis will allow us to remove this dependence in the future.

Note that for real data, we do expect larger $m$ to be have a 
beneficial effect on generalization. This is the reason that 
we prefer to present the bound as it is in the theorem, \ie with 
the empirical fairness estimated from all available data, rather 
than any alternative formulation, \eg subsampling the training 
set to $m=1$, which would recover an \iid setting.  
Finding an assumption on the generating process of real-world LTR data that does allow bounds that decrease with respect to $m$ is an interesting topic for future work. 

In addition, we expect that more advanced techniques from learning theory, \eg analysis based on Rademacher complexities \cite{bartlett2002rademacher}, can be applied to obtain sharper, data-dependent guarantees. Indeed, there has been work on extending the classic Rademacher complexity generalization bounds to the case of dependent data, \eg \cite{usunier2005generalization}, and we deem the application of such techniques in the context of fair LTR an interesting direction for future work.

%% file: appendix_experimental_details.tex
Here we present further details about our experiments, in particular the construction of the feature embeddings $\phi:\mathcal{Q}\times\mathcal{D}\to\mathbb{R}^D$ and a discussion on the computation costs of our experiments.

\subsection{Feature extraction}

\myparagraph{TREC\footnote{\url{https://fair-trec.github.io/2019/index.html}}}
Inspired by the learning to rank approach for the TREC track of \cite{bonartfair}, we pre-compute $9$-dimensional embeddings of every query-paper pair by using the following hardcrafted features: 
\begin{itemize}
\item the BM25 score of the query with the title, abstract, authors, topics and publication venue of the paper (5 values),
\item the number of in- and out-citations (2 values),
\item the publication year of the paper (1 value)
\item the character length of the query (1 value). 
\end{itemize}
Each feature is normalized by substracting the mean of the feature 
over the dataset and dividing by its standard deviation.

\myparagraph{MSMARCO\footnote{\url{https://microsoft.github.io/msmarco/}}} 
We use pretrained $768$-dimensional BERT feature embeddings \cite{devlin2019bert} for representing the query-passage pairs. Specifically, we follow the embedding procedure described in \cite{nogueira2019passage, han2020learning}, where each query-passage pair is represented as the following token sequence:
$$ [CLS] \text{ query text } [SEP] \text{ passage text } [CLS]$$
This sequence is then processed through a pre-trained BERT model\footnote{\url{https://tfhub.dev/tensorflow/bert_en_uncased_L-12_H-768_A-12/1}} from Tensorflow Hub \cite{tensorflow2015-whitepaper}, with maximum sequence length set to $200$, and the hidden units of the first $[CLS]$ token are used as a representation of the query-passage pair.

\subsection{Computational costs}

For both datasets running a single experiment consists of training and evaluating a single ML model, with a fixed choice of training method (algorithm), type of fairness, split into protected groups and value of the regularization parameter. In both cases, training an individual model was relatively cheap, with the main computational considerations coming from the large number of experiments to be ran in total.

For \textbf{TREC} an individual experiment takes about \textit{$60$ seconds} on a CPU for our algorithm and for FA*IR and around \textit{$11$ and $17$ minutes} for DELTR and the per-query baseline respectively. Each experiment requires less than $512$MB of RAM. In total, we ran $1170$ experiments for each of these algorithms, one for every fixed type of fairness (equality of opportunity, demographic parity or equalized odds), split into protected groups (3 types of splits, corresponding to $t = 3,4,5$), value of the regularization parameter (a total of $26$ values were used), and a choice of a train-test split (a total of $5$ independent splits were used).

For \textbf{MSMARCO} an individual experiment takes about \textit{$9$ minutes} on a CPU for our algorithm and for FA*IR and about \textit{$40$ minutes} for the per-query method. Each experiment requires less than $64$GB of RAM. In total, we ran $960$ experiments for each of those models, one for every fixed type of fairness (equality of opportunity, demographic parity or equalized odds), split into protected groups (2 types of splits, \textit{com} and \textit{ext}), value of the regularization parameter (a total of $16$ values were used), and a choice of a train-test split (a total of $10$ random seeds were used to compute a random split).

Therefore, approximately $585$ and $928$ CPU hours are needed for the TREC and the MSMARCO experiments in total.

%% file: appendix_results.tex
We report on multiple additional metrics and experiments on the TREC and MSMARCO data, that were deferred to the supplementary material for space reasons. 

\subsection{Results with $\Pk$}
\label{app:results_pk}
 
We first present plots from the same experiments as in Figure \ref{fig:plots}, but with Precision@k as a metric for model performance. Specifically, Figure \ref{fig:plots-precision} shows the results when imposing different amounts of the equal opportunity fairness notions in typical settings for TREC ($t=3, k=3$; top row) and MSMARCO (\textit{com}, $k=3$; bottom row), both for our method and for the baselines. We see a very similar picture as with the NDCG metric, with no loss in precision for our method on the TREC data and little to no effect for MSMARCO, for small to medium values of $\alpha$. Again, the baselines are not able to consistently improve equal opportunity.

\begin{figure*}[h]
\begin{subfigure}{.24\textwidth}
  \centering
  \includegraphics[width=.95\linewidth]{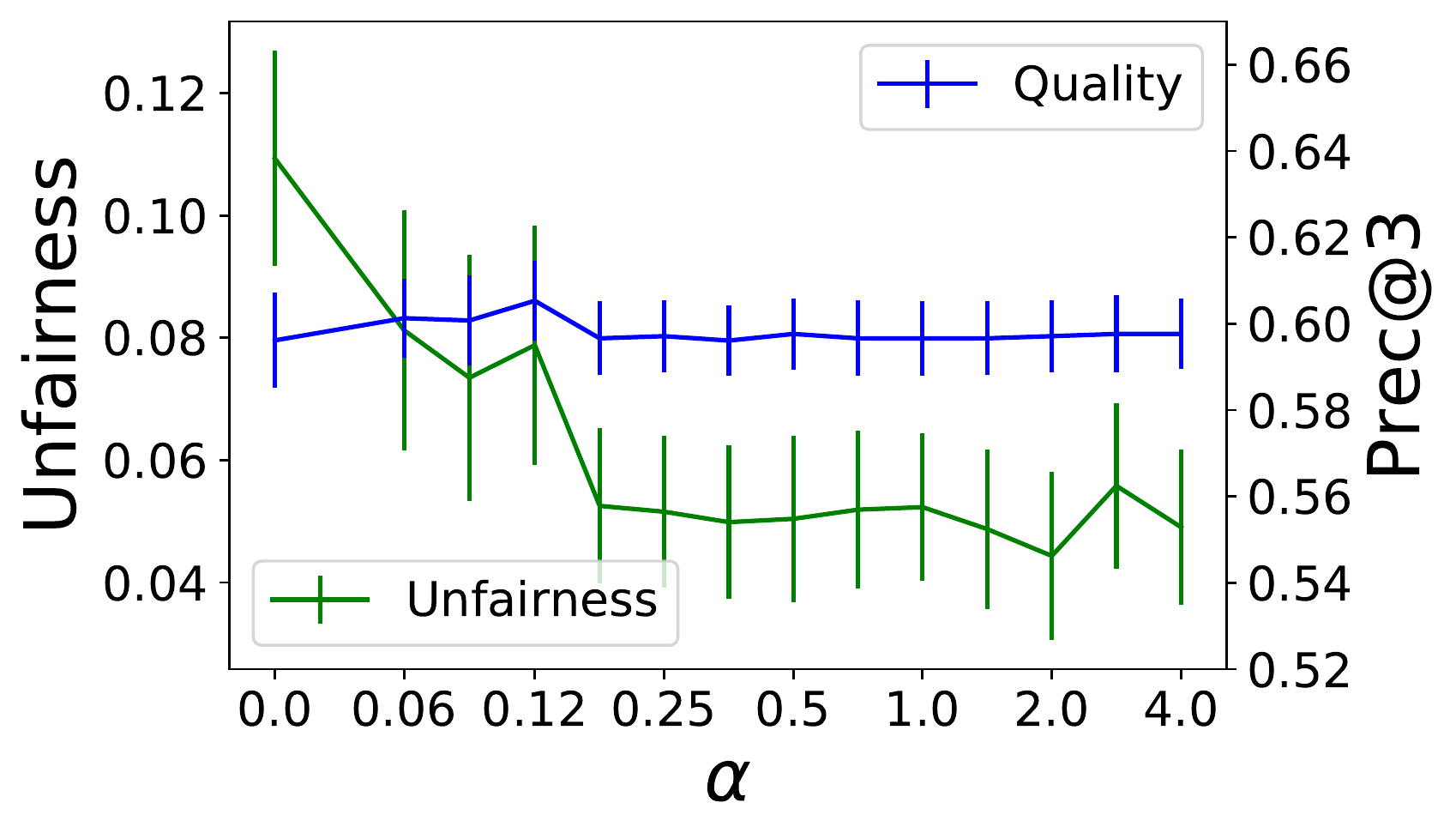}
  \caption{Ours, TREC}
  \label{fig:trec_amortized_demographic_parity}
\end{subfigure}%
\hfill
\begin{subfigure}{.24\textwidth}
  \centering
  \includegraphics[width=.95\linewidth]{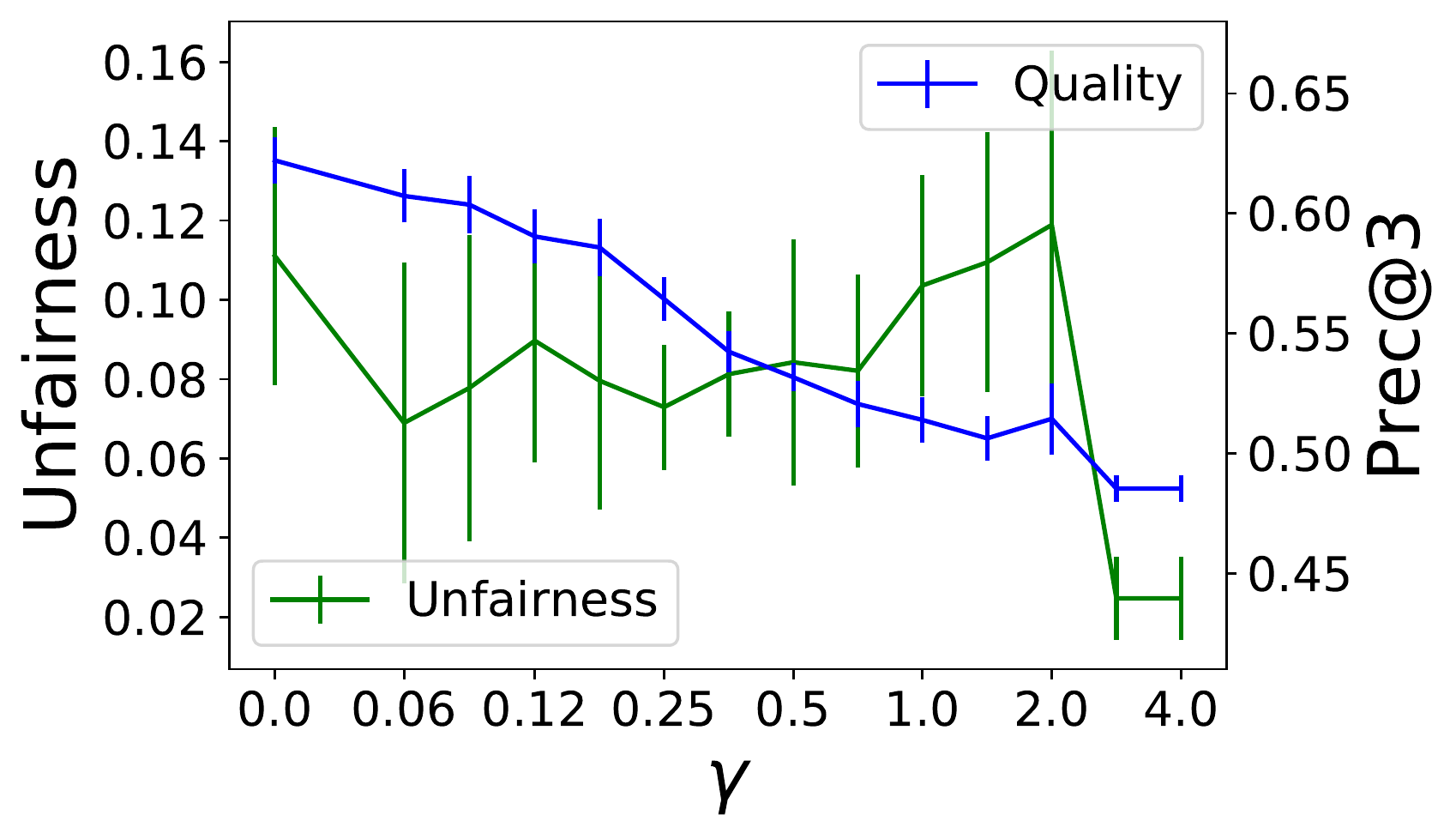}
  \caption{DELTR, TREC}
  \label{fig:trec_amortized_equal_odds}
\end{subfigure}%
\hfill
\begin{subfigure}{.24\textwidth}
  \centering
  \includegraphics[width=.95\linewidth]{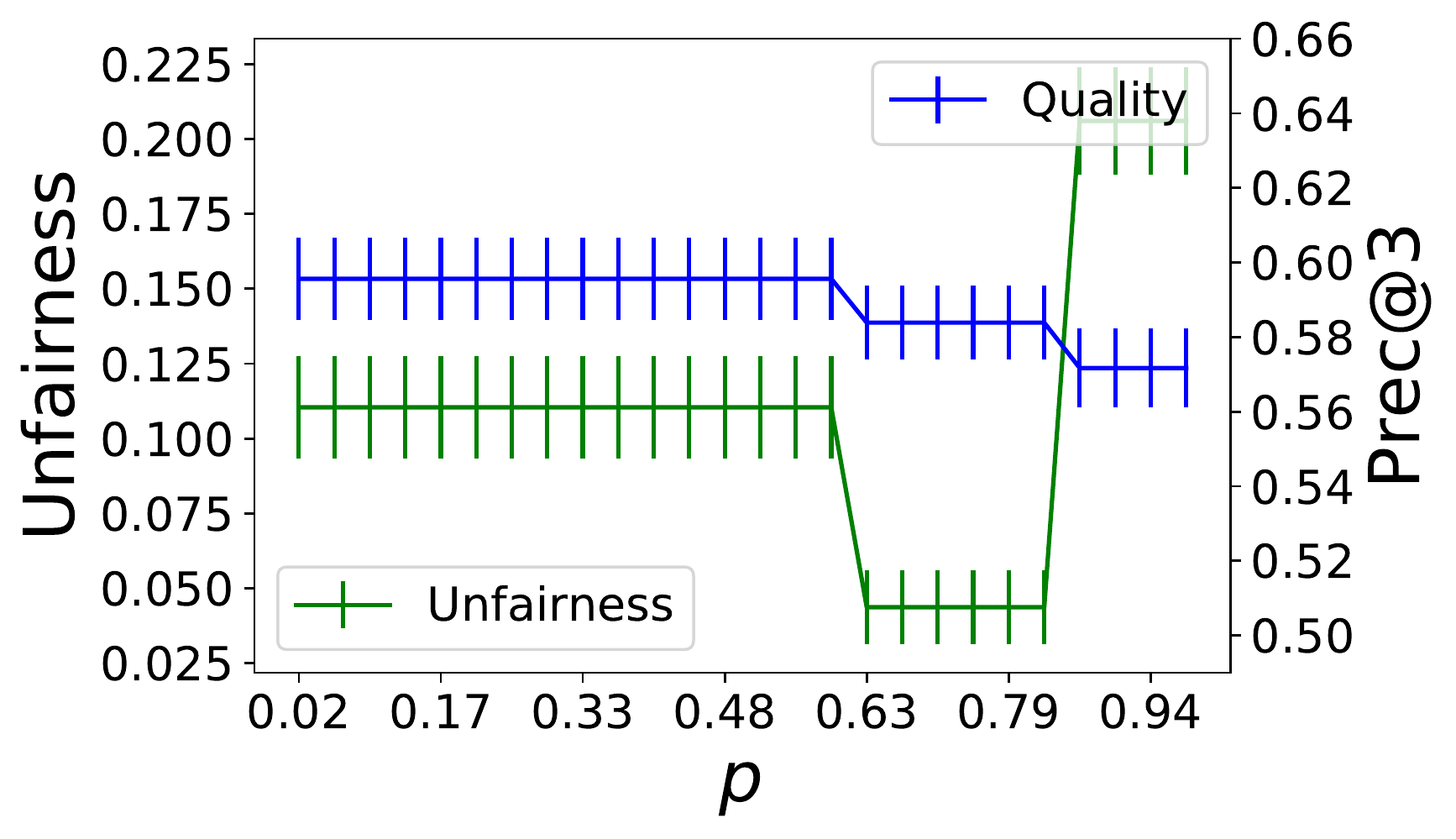}
  \caption{FA*IR, TREC}
  \label{fig:trec_amortized_equal_opp}
\end{subfigure}
\begin{subfigure}{.24\textwidth}
  \centering
  \includegraphics[width=.95\linewidth]{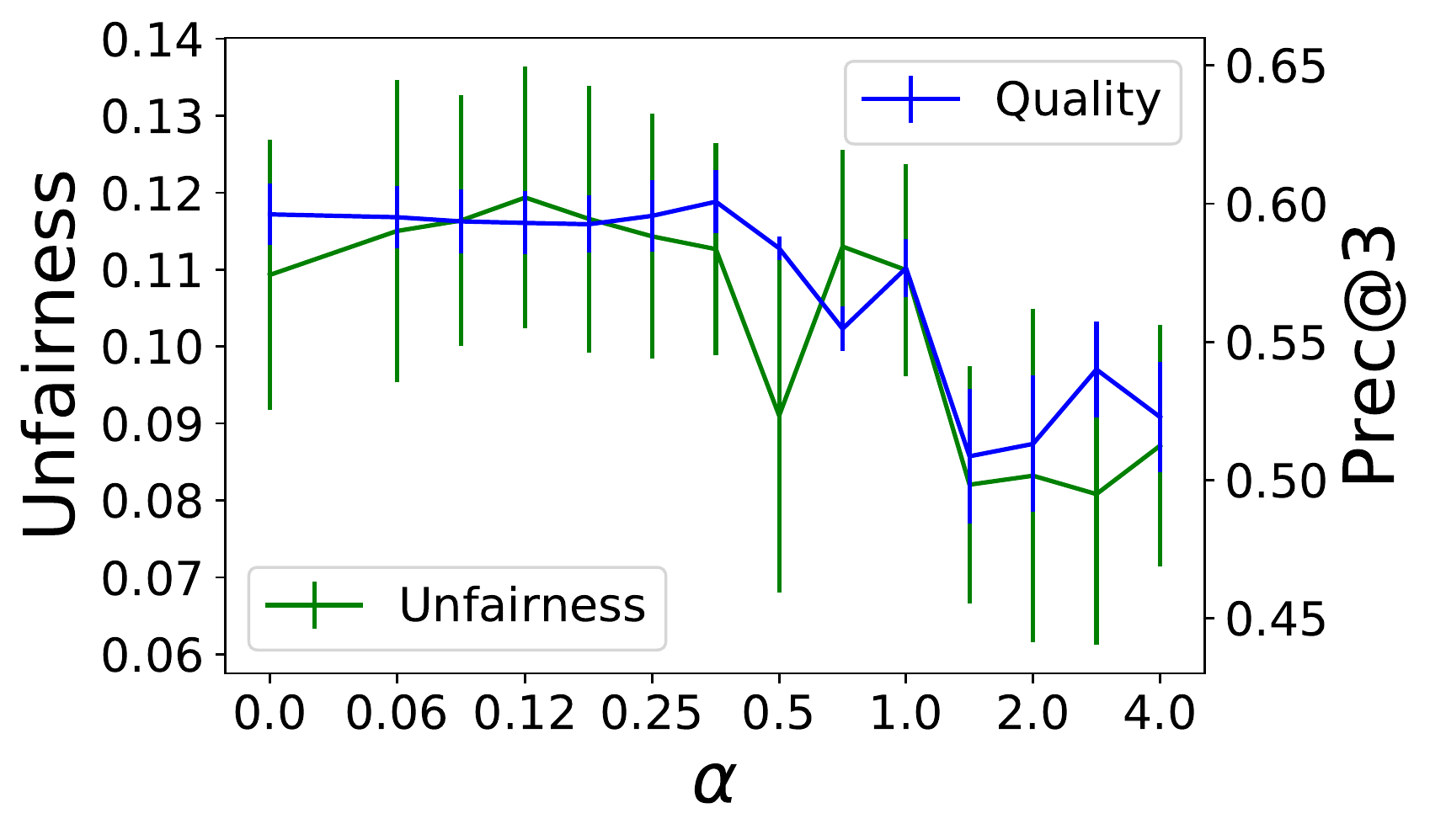}
  \caption{Per-query, TREC}
  \label{fig:trec_amortized_demographic_parity}
\end{subfigure}%
\vskip\baselineskip
\begin{subfigure}{.33\textwidth}
  \centering
  \includegraphics[width=.691\linewidth]{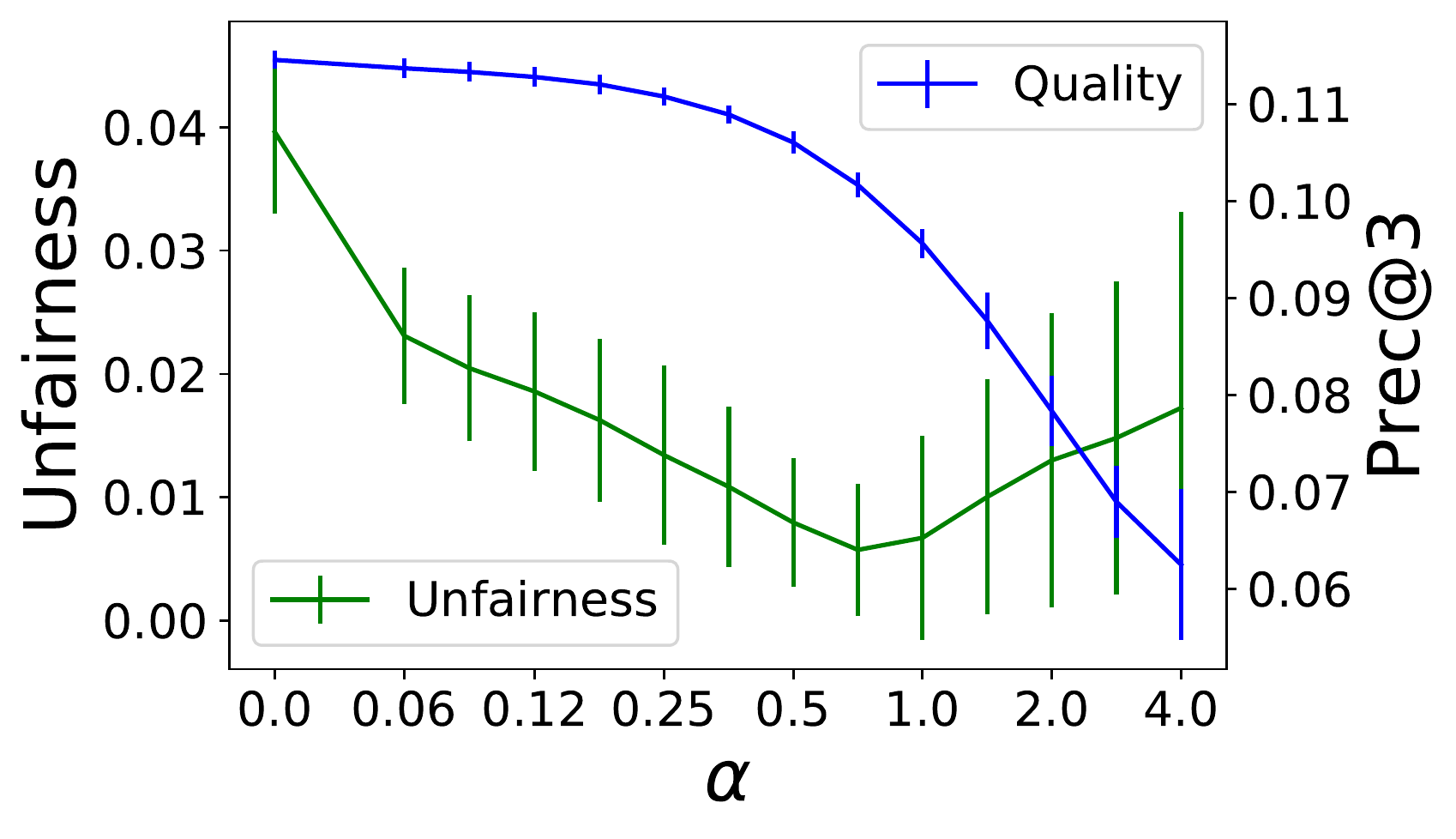}
  \caption{Ours, MSMARCO}
  \label{fig:trec_amortized_demographic_parity}
\end{subfigure}%
\hfill
\begin{subfigure}{.33\textwidth}
  \centering
  \includegraphics[width=.691\linewidth]{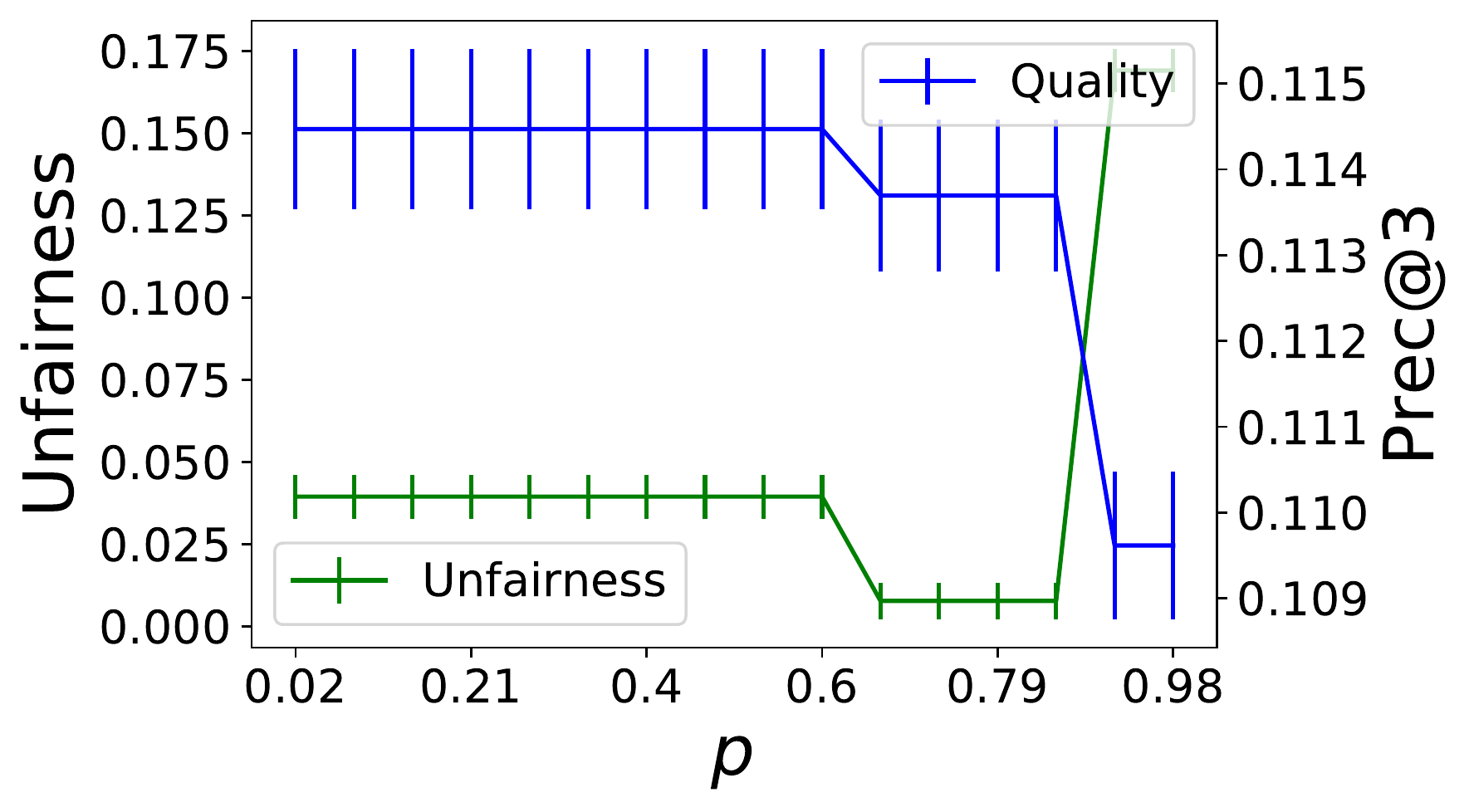}
  \caption{FA*IR, MSMARCO}
  \label{fig:trec_amortized_equal_opp}
\end{subfigure}
\begin{subfigure}{.33\textwidth}
  \centering
  \includegraphics[width=.691\linewidth]{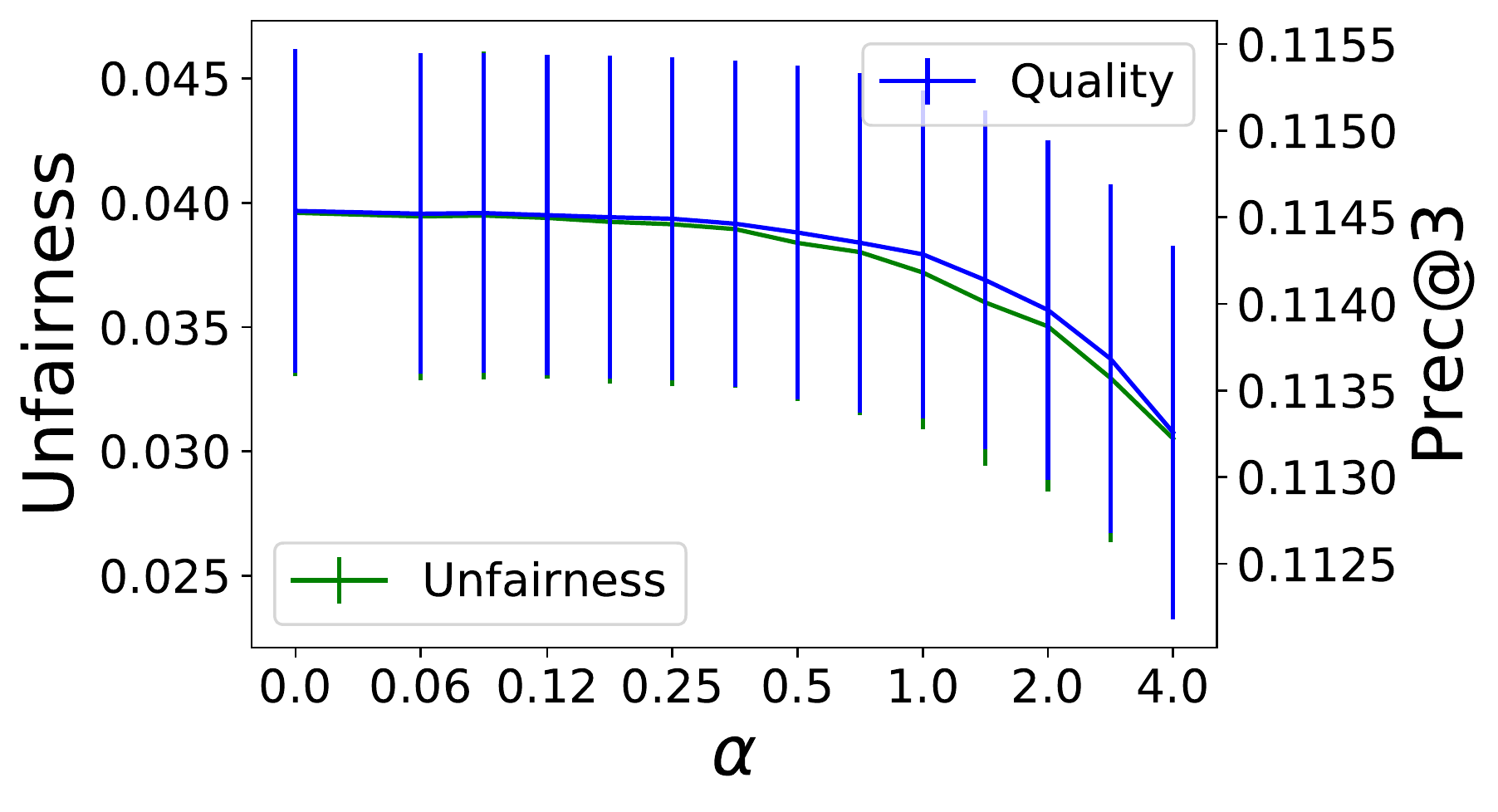}
  \caption{Per-query, MSMARCO}
  \label{fig:trec_amortized_demographic_parity}
\end{subfigure}%
\caption{Test-time performance of fair rankers with equal opportunity fairness, achieved by our algorithm and the baselines: unfairness (left $y$-axes) and Prec@3 ranking 
quality (right $y$-axes); after training with different regularization 
strengths ($x$-axis).}
\label{fig:plots-precision}
\end{figure*}

\subsection{Plots for all three fairness measures}
\label{sec:app_plots_for_three_measures}

Here we present all results in the typical settings for TREC ($t=3, k=3$) on  Figure \ref{fig:plots-trec-all} and for MSMARCO (\textit{com}, $k=3$) on Figure \ref{fig:plots-msmarco-all}. These plots complement Figure \ref{fig:plots} from the main body of the paper by showing the results for the other two fairness measures, demographic parity and equalized odds, as well. We see that also for these measures our algorithm effectively improves fairness at little to no cost in ranking performance, for small values of $\alpha$. On the other hand, the baselines perform erratically and inconsistently across the measures and the two datasets.

\begin{figure*}[h]
\begin{subfigure}{.33\textwidth}
  \centering
  \includegraphics[width=.95\linewidth]{TREC_NDCG_equal_opp_amortized_mean_i10_3_3}
  \caption{Ours, equal opportunity}
\end{subfigure}%
\hfill
\begin{subfigure}{.33\textwidth}
  \centering
  \includegraphics[width=.95\linewidth]{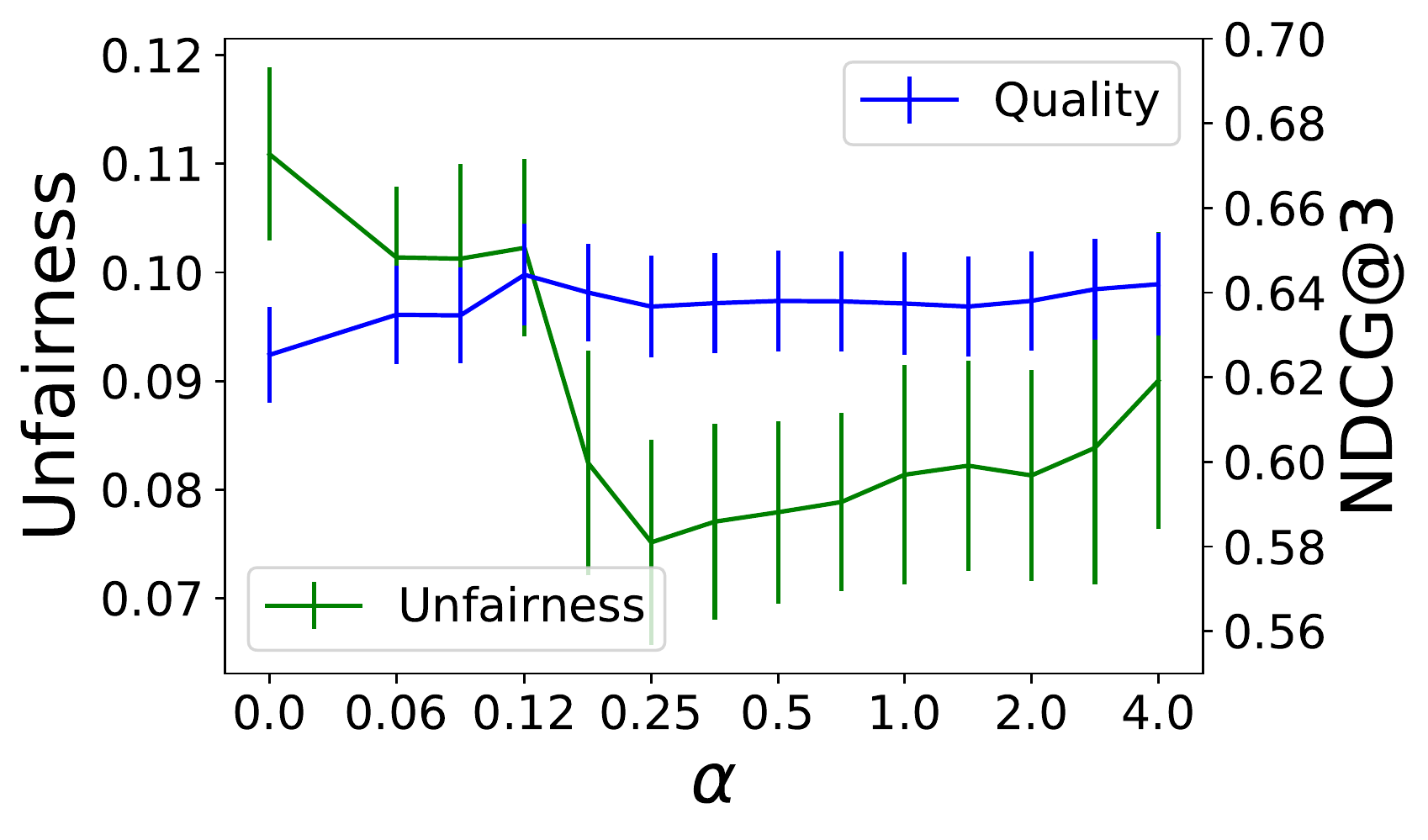}
  \caption{Ours, demographic parity}
\end{subfigure}%
\hfill
\begin{subfigure}{.33\textwidth}
  \centering
  \includegraphics[width=.95\linewidth]{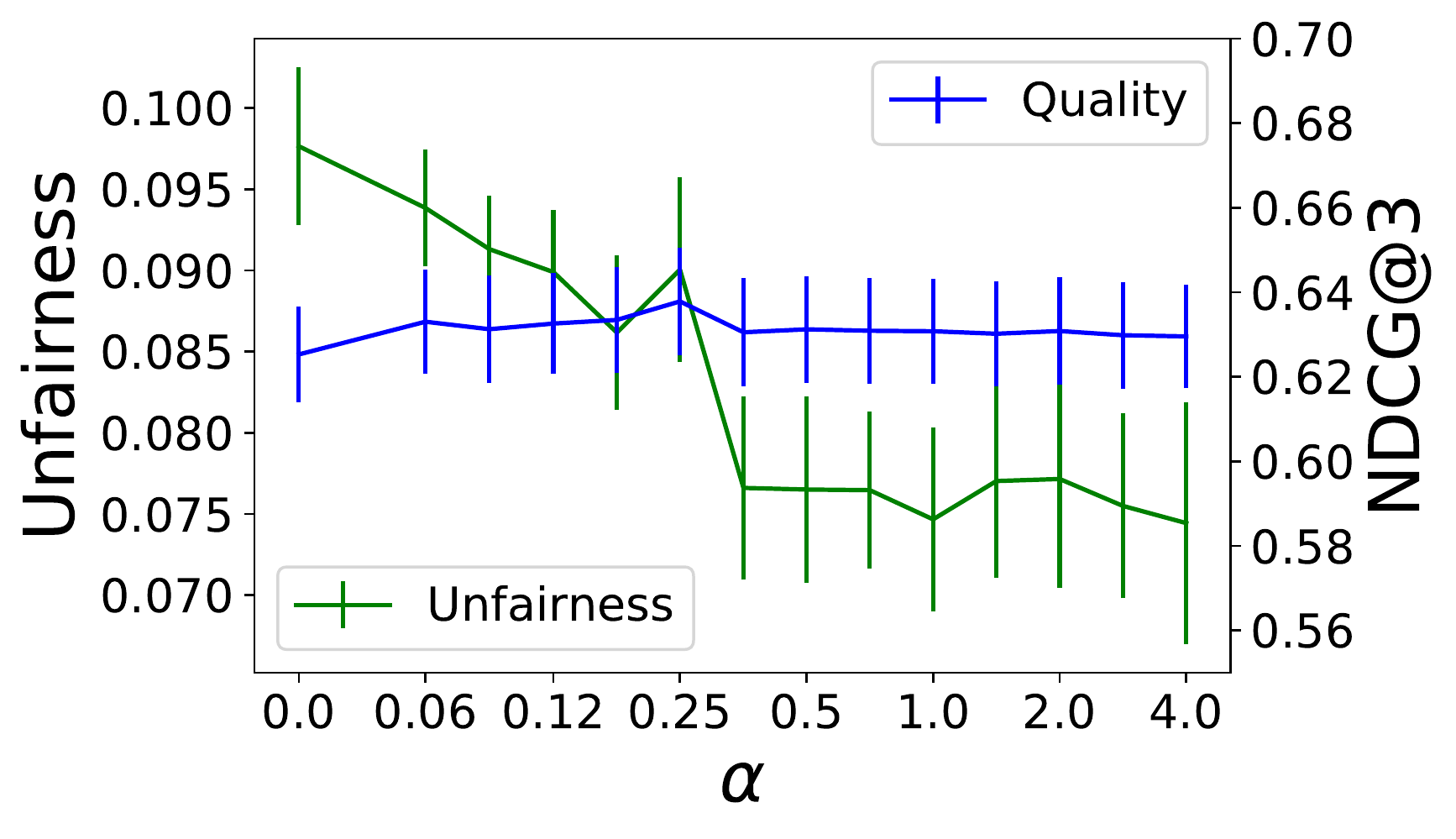}
  \caption{Ours, equal odds}
\end{subfigure}%
\vskip\baselineskip
\begin{subfigure}{.33\textwidth}
  \centering
  \includegraphics[width=.95\linewidth]{baseline_TREC_NDCG_equal_opp_amortized_mean_i10_3_3}
  \caption{DELTR, equal opportunity}
\end{subfigure}%
\hfill
\begin{subfigure}{.33\textwidth}
  \centering
  \includegraphics[width=.95\linewidth]{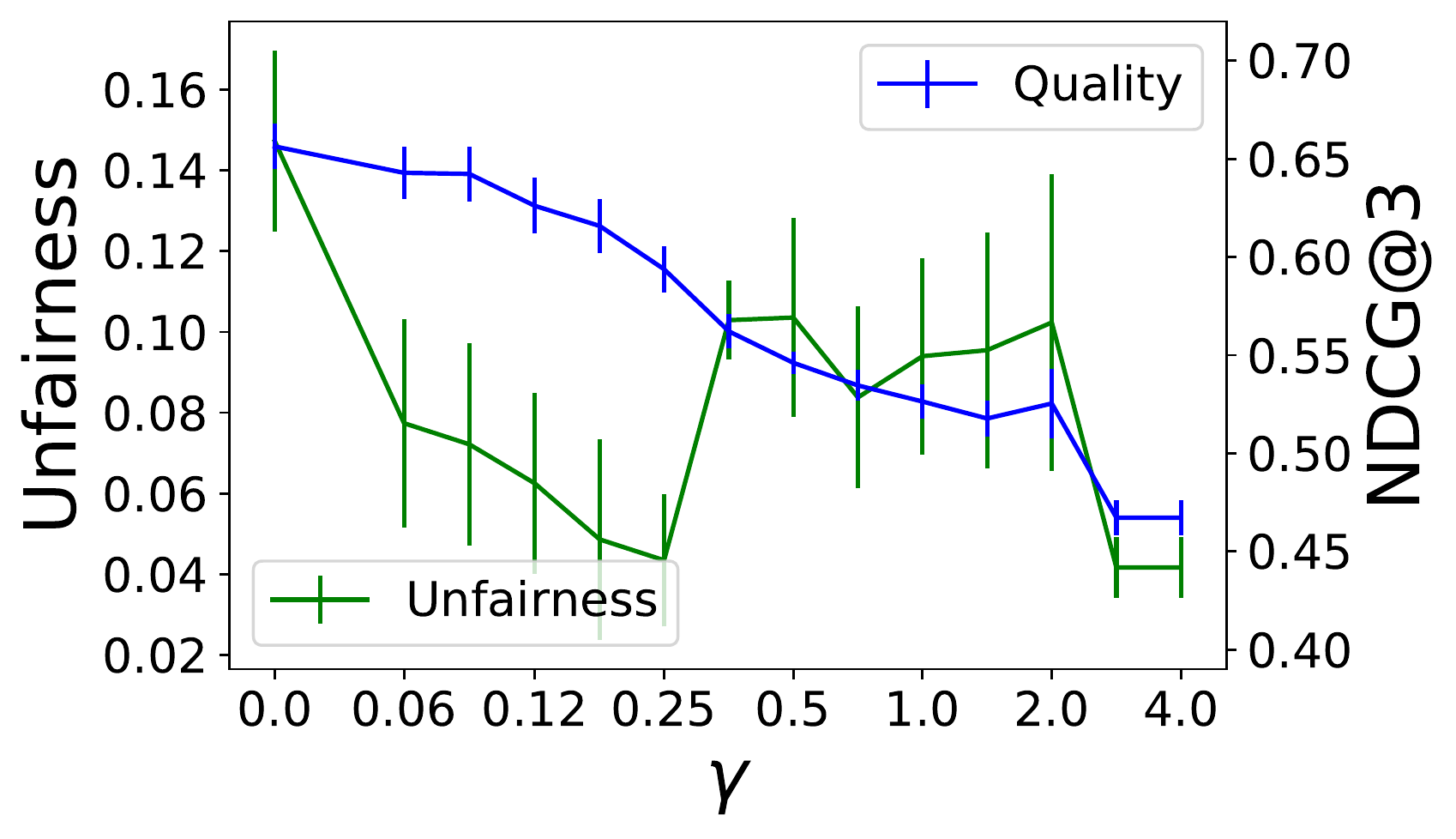}
  \caption{DELTR, demographic parity}
\end{subfigure}%
\hfill
\begin{subfigure}{.33\textwidth}
  \centering
  \includegraphics[width=.95\linewidth]{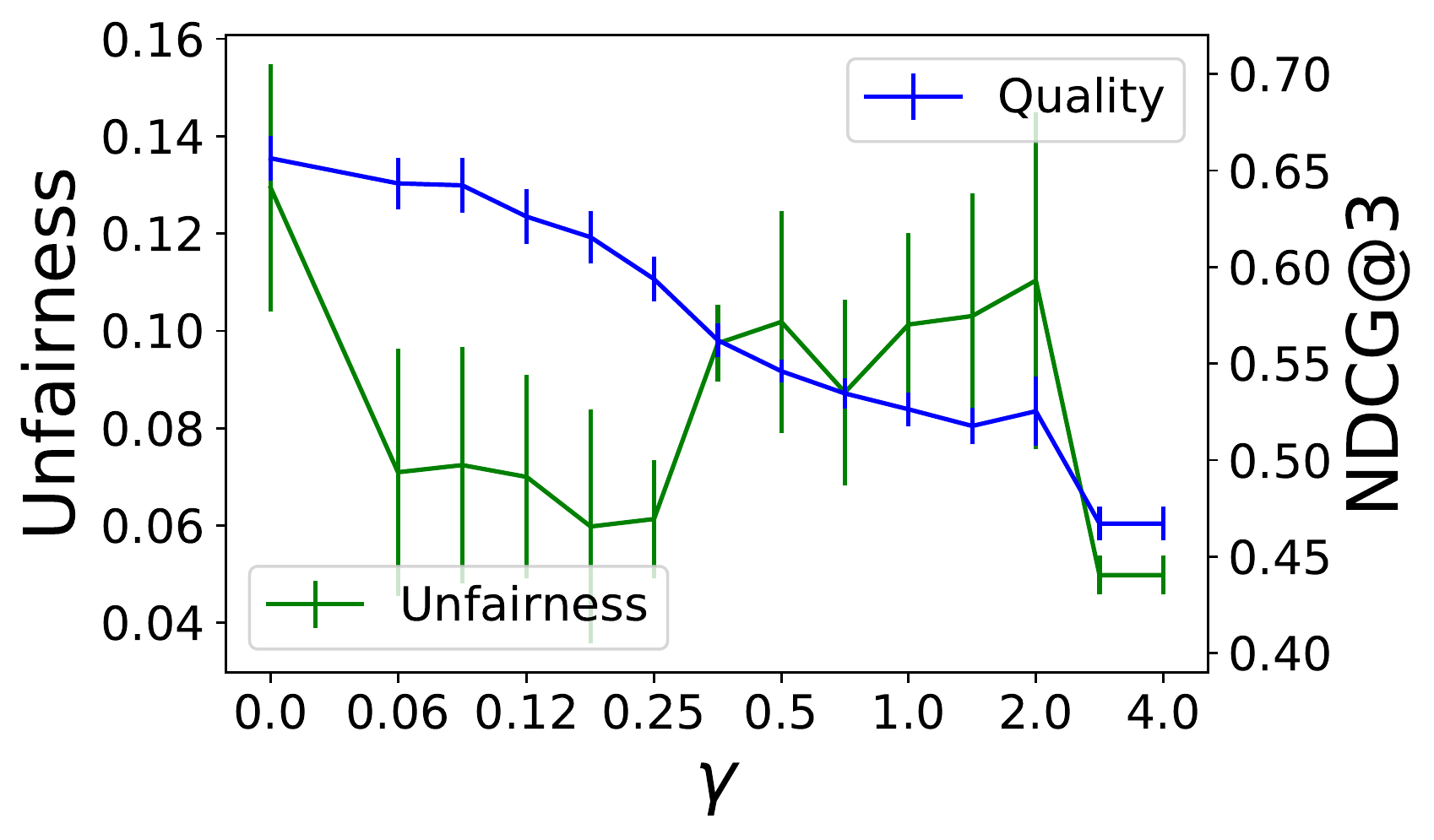}
  \caption{DELTR, equal odds}
\end{subfigure}%
\vskip\baselineskip
\begin{subfigure}{.33\textwidth}
  \centering
  \includegraphics[width=.95\linewidth]{baseline2_TREC_NDCG_equal_opp_amortized_mean_i10_3_3}
  \caption{FA*IR, equal opportunity}
\end{subfigure}%
\hfill
\begin{subfigure}{.33\textwidth}
  \centering
  \includegraphics[width=.95\linewidth]{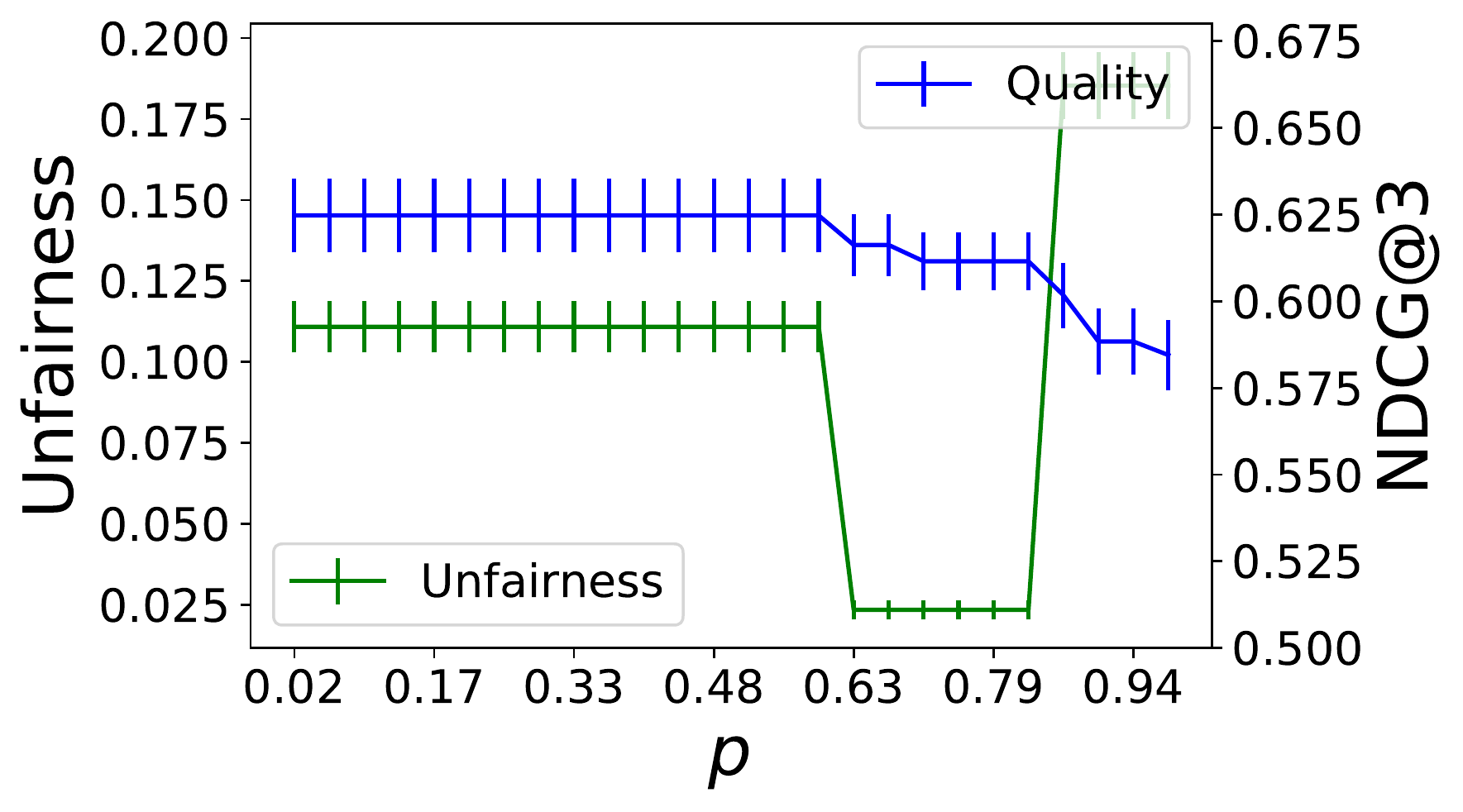}
  \caption{FA*IR, demographic parity}
\end{subfigure}%
\hfill
\begin{subfigure}{.33\textwidth}
  \centering
  \includegraphics[width=.95\linewidth]{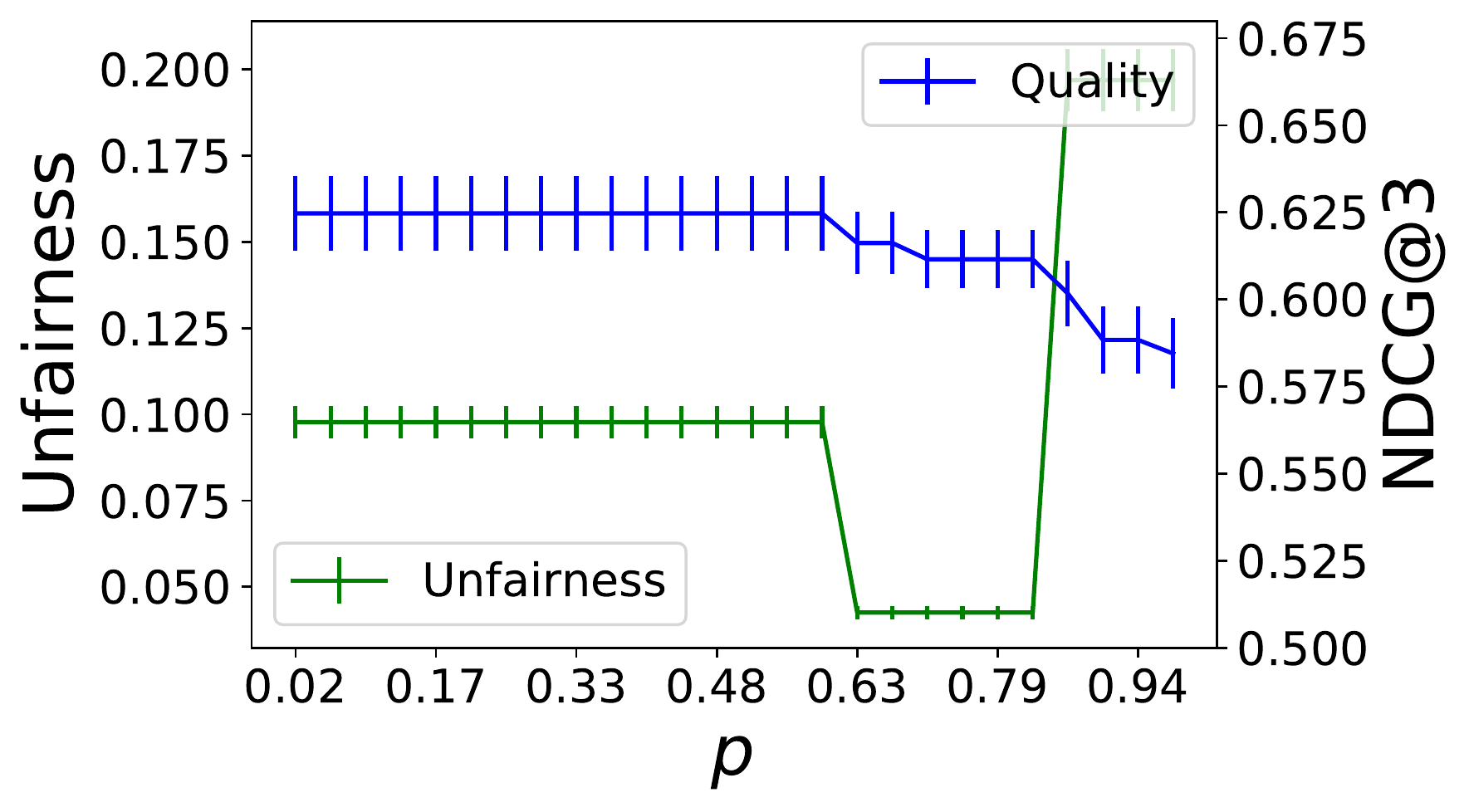}
  \caption{FA*IR, equal odds}
\end{subfigure}%
\vskip\baselineskip
\begin{subfigure}{.33\textwidth}
  \centering
  \includegraphics[width=.95\linewidth]{TREC_NDCG_equal_opp_per_query_mean_i10_3_3}
  \caption{Per-query, equal opportunity}
\end{subfigure}%
\hfill
\begin{subfigure}{.33\textwidth}
  \centering
  \includegraphics[width=.95\linewidth]{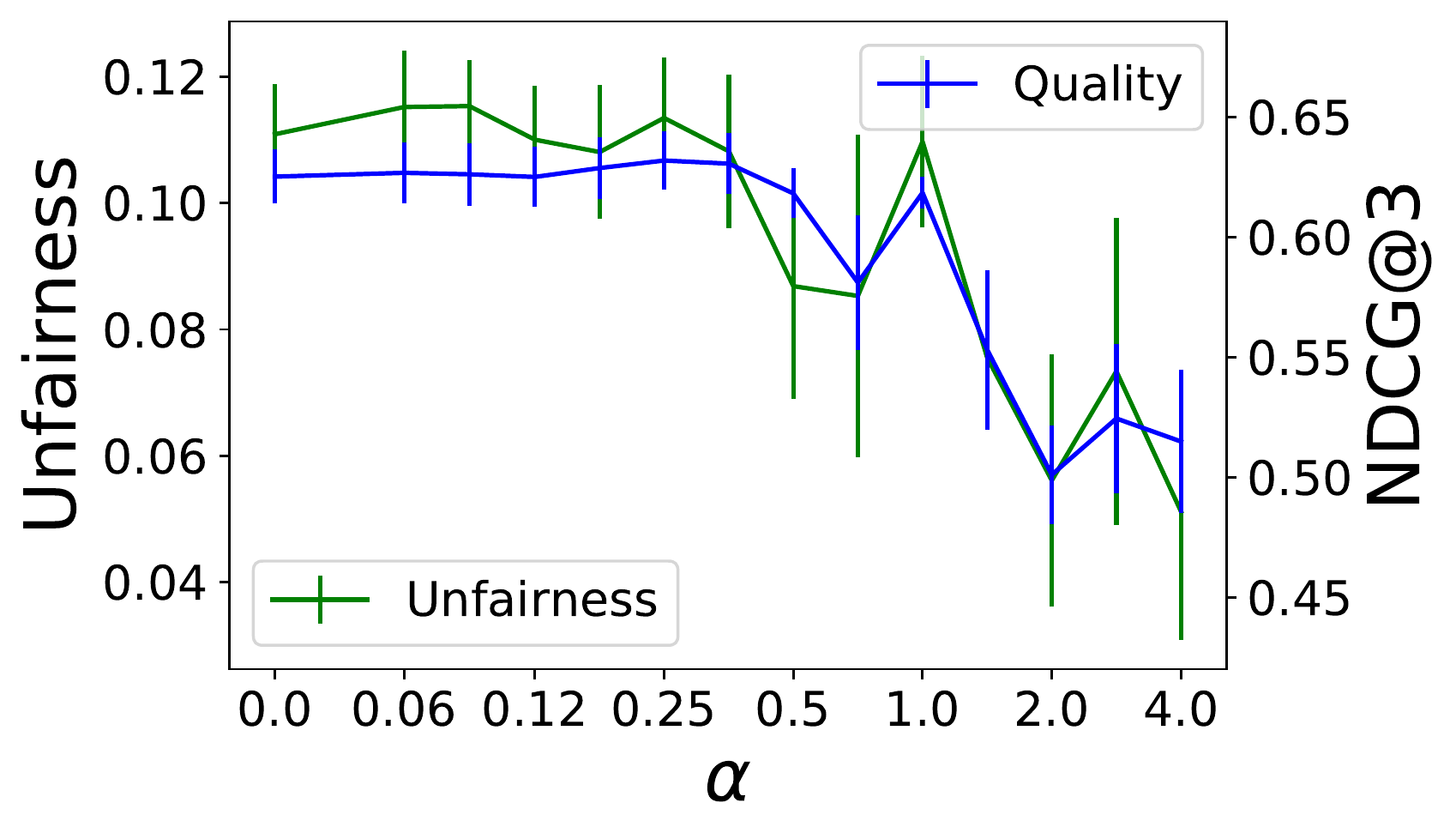}
  \caption{Per-query, demographic parity}
\end{subfigure}%
\hfill
\begin{subfigure}{.33\textwidth}
  \centering
  \includegraphics[width=.95\linewidth]{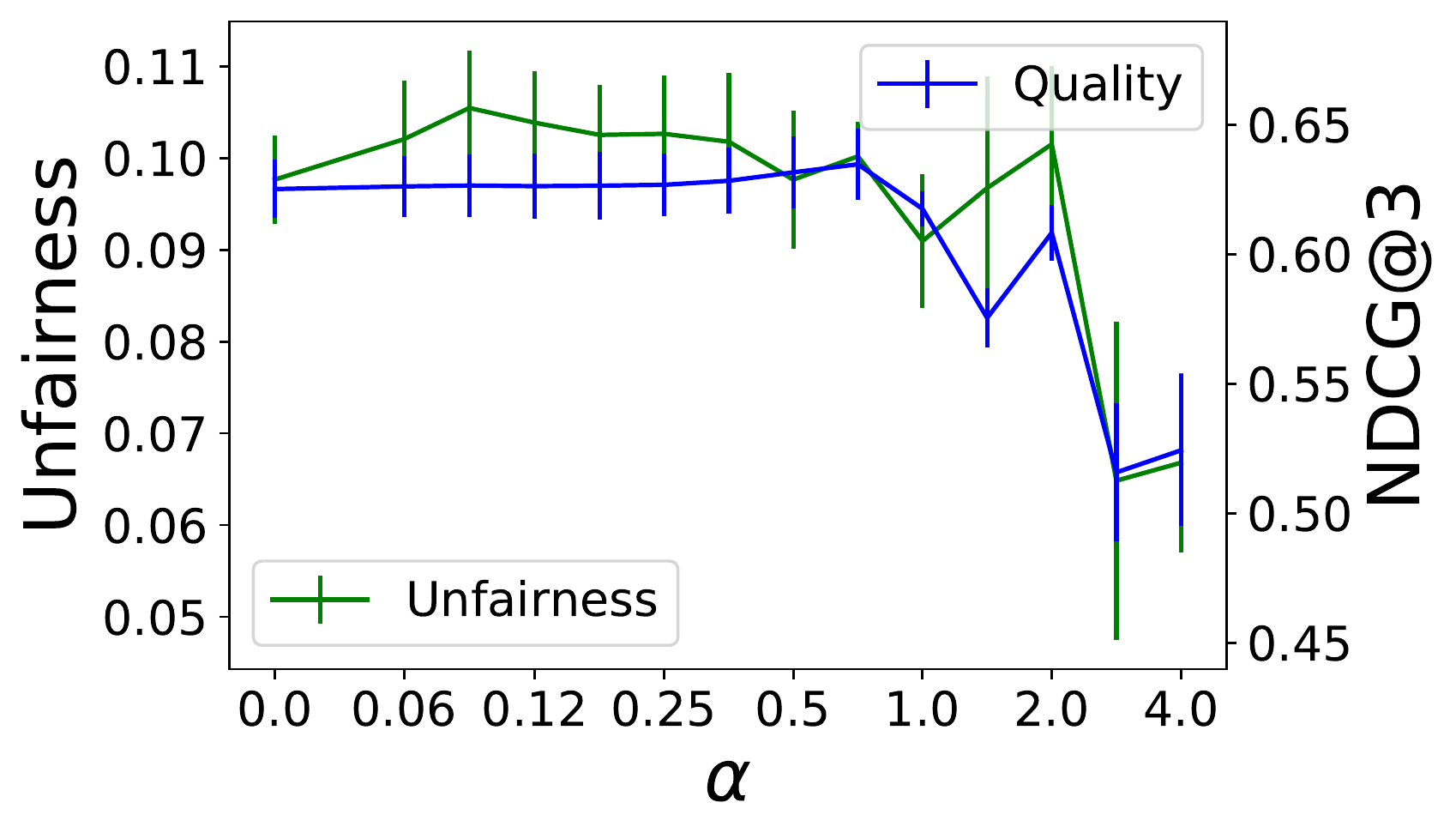}
  \caption{Per-query, equal odds}
\end{subfigure}
\vskip\baselineskip
\caption{TREC: Test-time performance of fair rankers with equal opportunity, demographic parity and equalized odds fairness, achieved by our algorithm and the baselines: unfairness (left $y$-axes) and NDCG@3 ranking 
quality (right $y$-axes); after training with different regularization 
strengths ($x$-axis).}
\label{fig:plots-trec-all}
\end{figure*}

\begin{figure*}[h]
\begin{subfigure}{.33\textwidth}
  \centering
  \includegraphics[width=.95\linewidth]{MSMARCO_NDCG_equal_opp_amortized_com_0_3}
  \caption{Ours, equal opportunity}
\end{subfigure}%
\hfill
\begin{subfigure}{.33\textwidth}
  \centering
  \includegraphics[width=.95\linewidth]{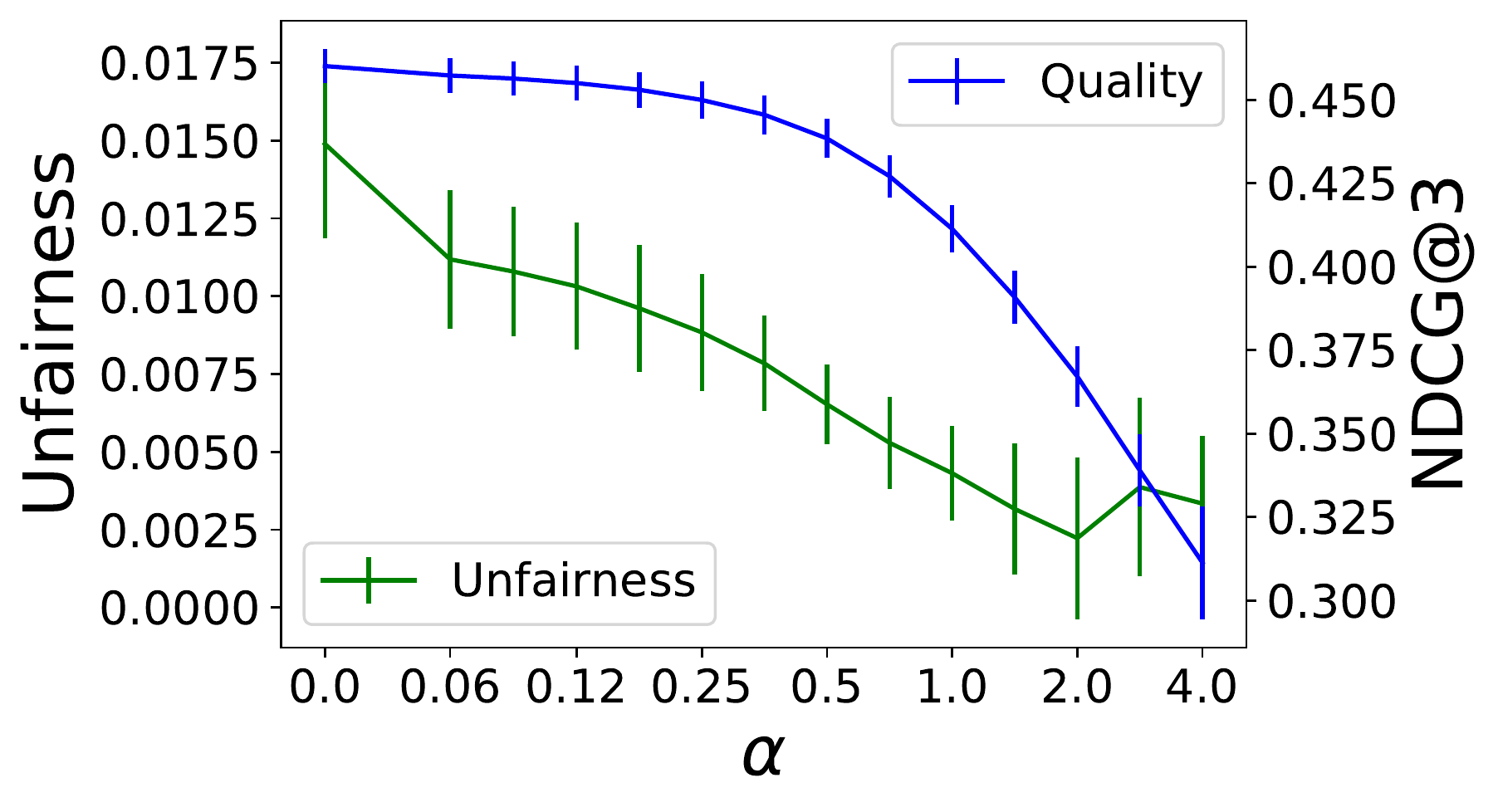}
  \caption{Ours, demographic parity}
\end{subfigure}%
\hfill
\begin{subfigure}{.33\textwidth}
  \centering
  \includegraphics[width=.95\linewidth]{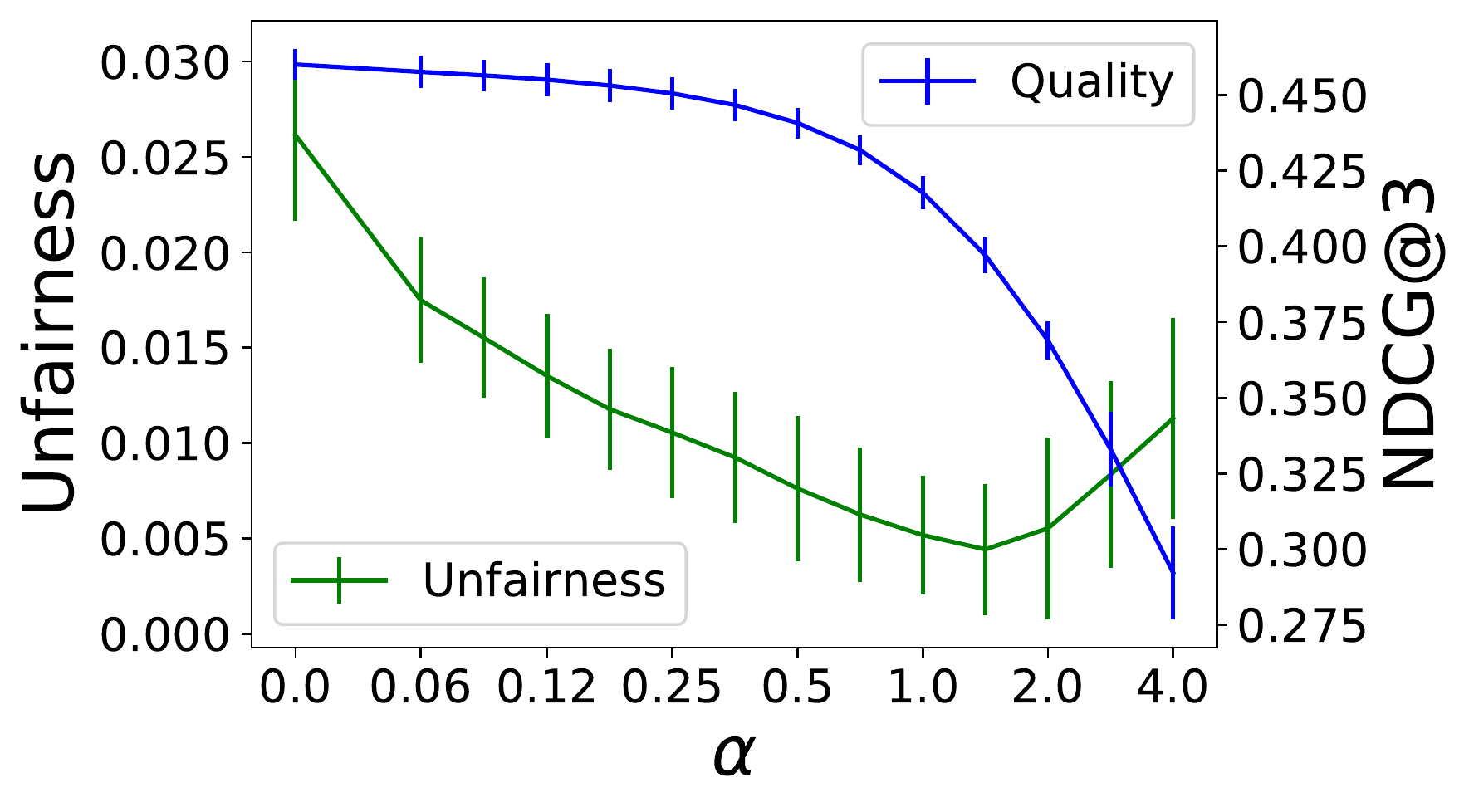}
  \caption{Ours, equal odds}
\end{subfigure}%
\vskip\baselineskip
\begin{subfigure}{.33\textwidth}
  \centering
  \includegraphics[width=.95\linewidth]{baseline2_MSMARCO_NDCG_equal_opp_amortized_com_0_3}
  \caption{FA*IR, equal opportunity}
\end{subfigure}%
\hfill
\begin{subfigure}{.33\textwidth}
  \centering
  \includegraphics[width=.95\linewidth]{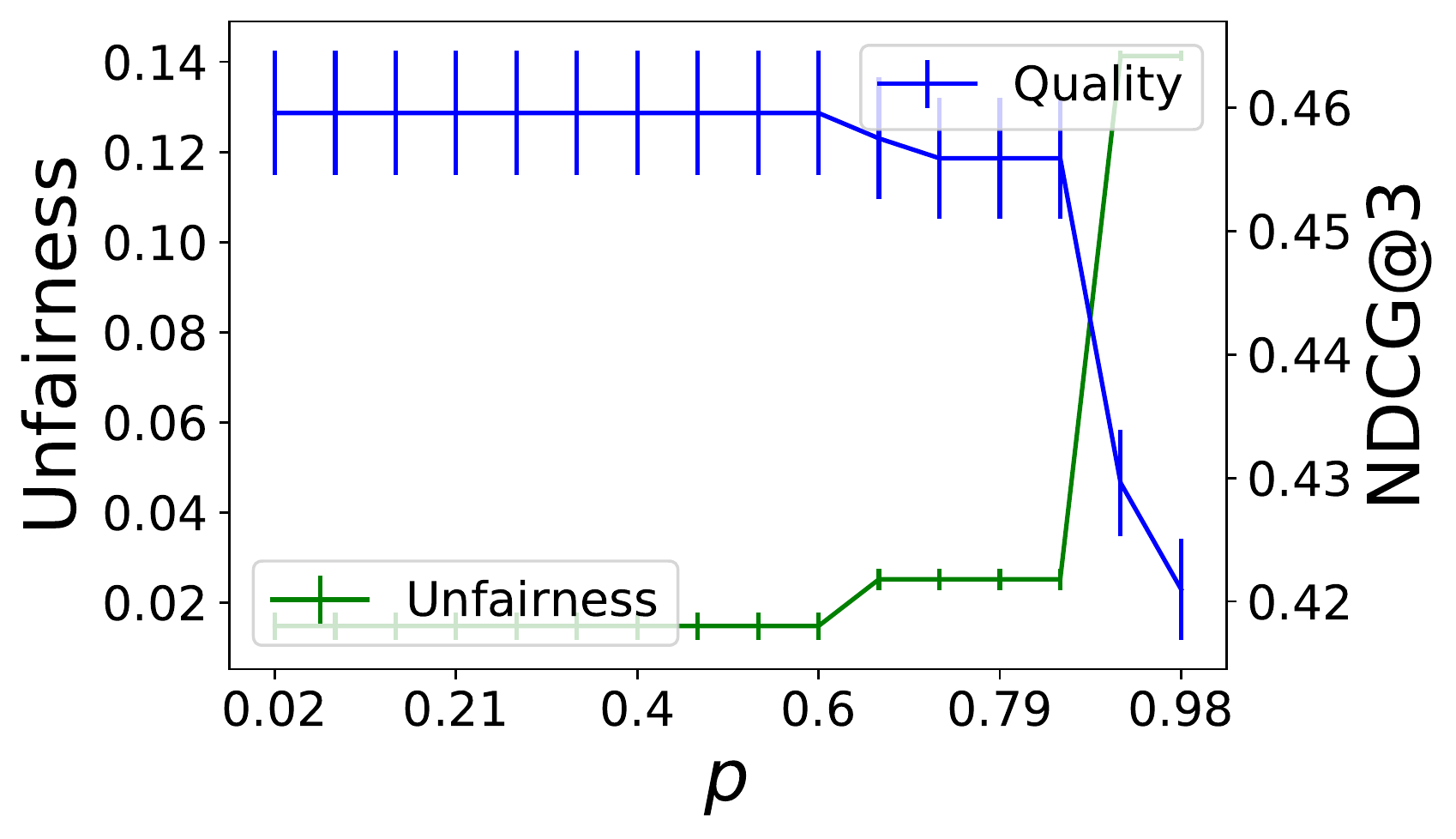}
  \caption{FA*IR, demographic parity}
\end{subfigure}%
\hfill
\begin{subfigure}{.33\textwidth}
  \centering
  \includegraphics[width=.95\linewidth]{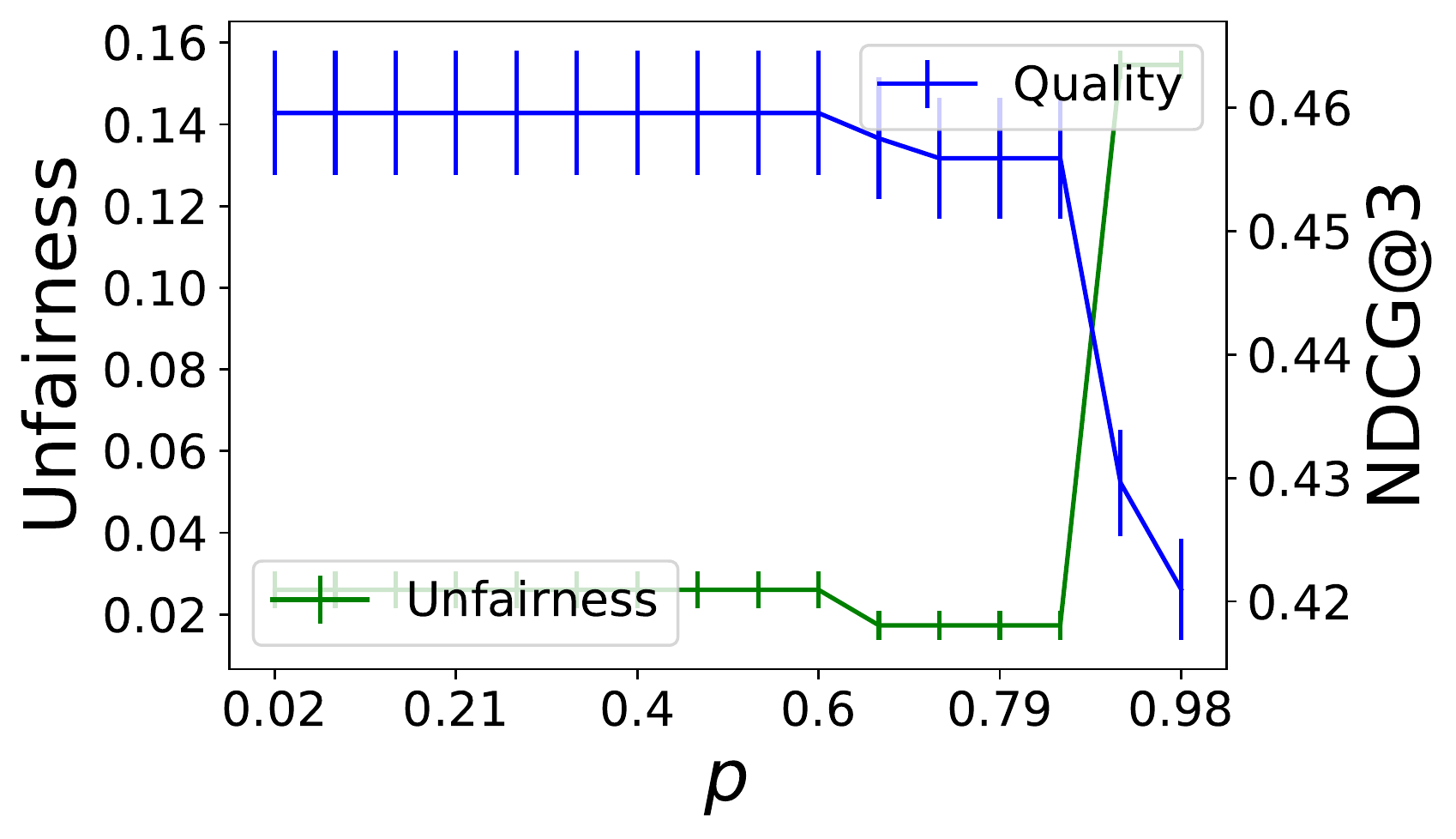}
  \caption{FA*IR, equal odds}
\end{subfigure}%
\vskip\baselineskip
\begin{subfigure}{.33\textwidth}
  \centering
  \includegraphics[width=.95\linewidth]{MSMARCO_NDCG_equal_opp_per_query_com_0_3}
  \caption{Per-query, equal opportunity}
\end{subfigure}%
\hfill
\begin{subfigure}{.33\textwidth}
  \centering
  \includegraphics[width=.95\linewidth]{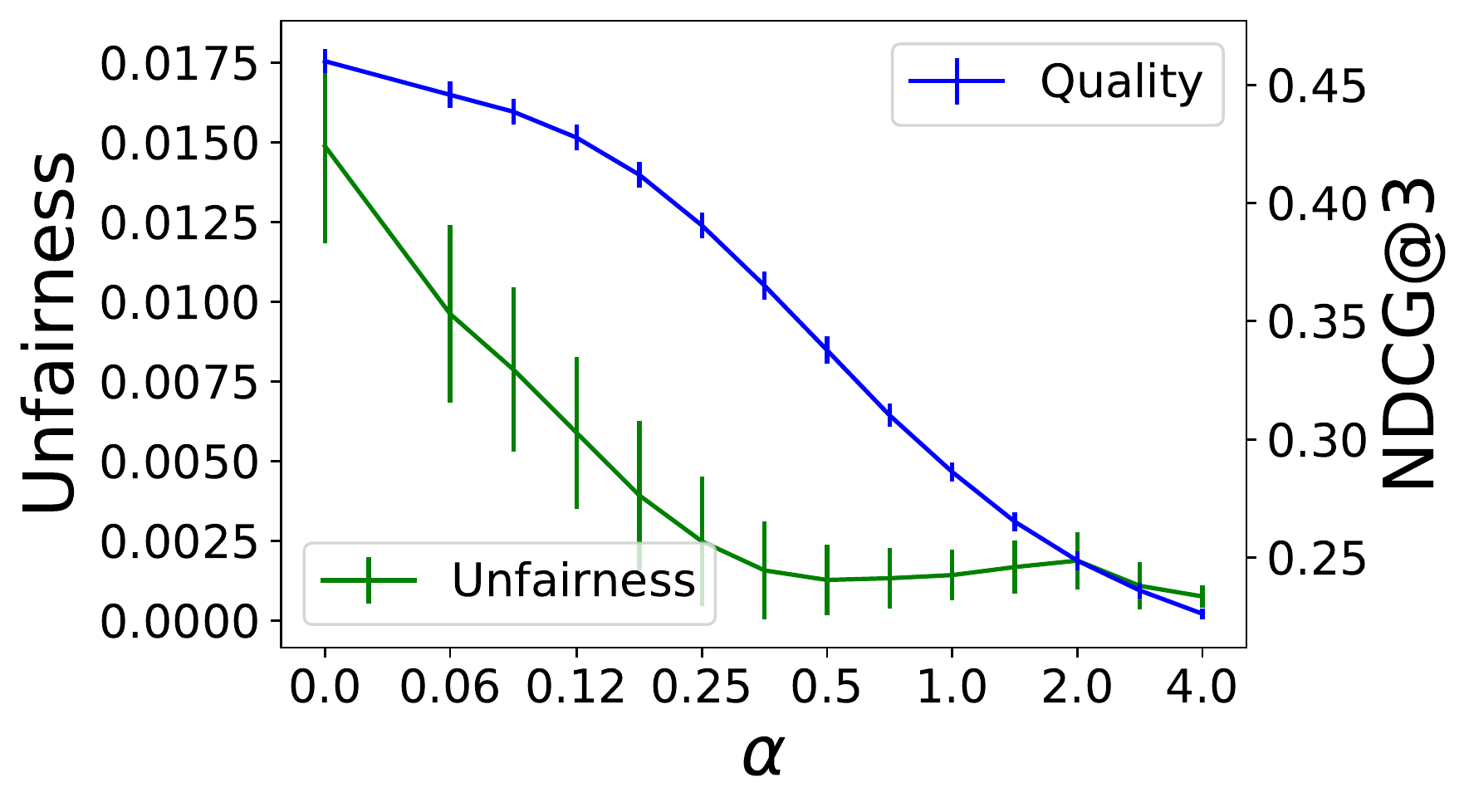}
  \caption{Per-query, demographic parity}
\end{subfigure}%
\hfill
\begin{subfigure}{.33\textwidth}
  \centering
  \includegraphics[width=.95\linewidth]{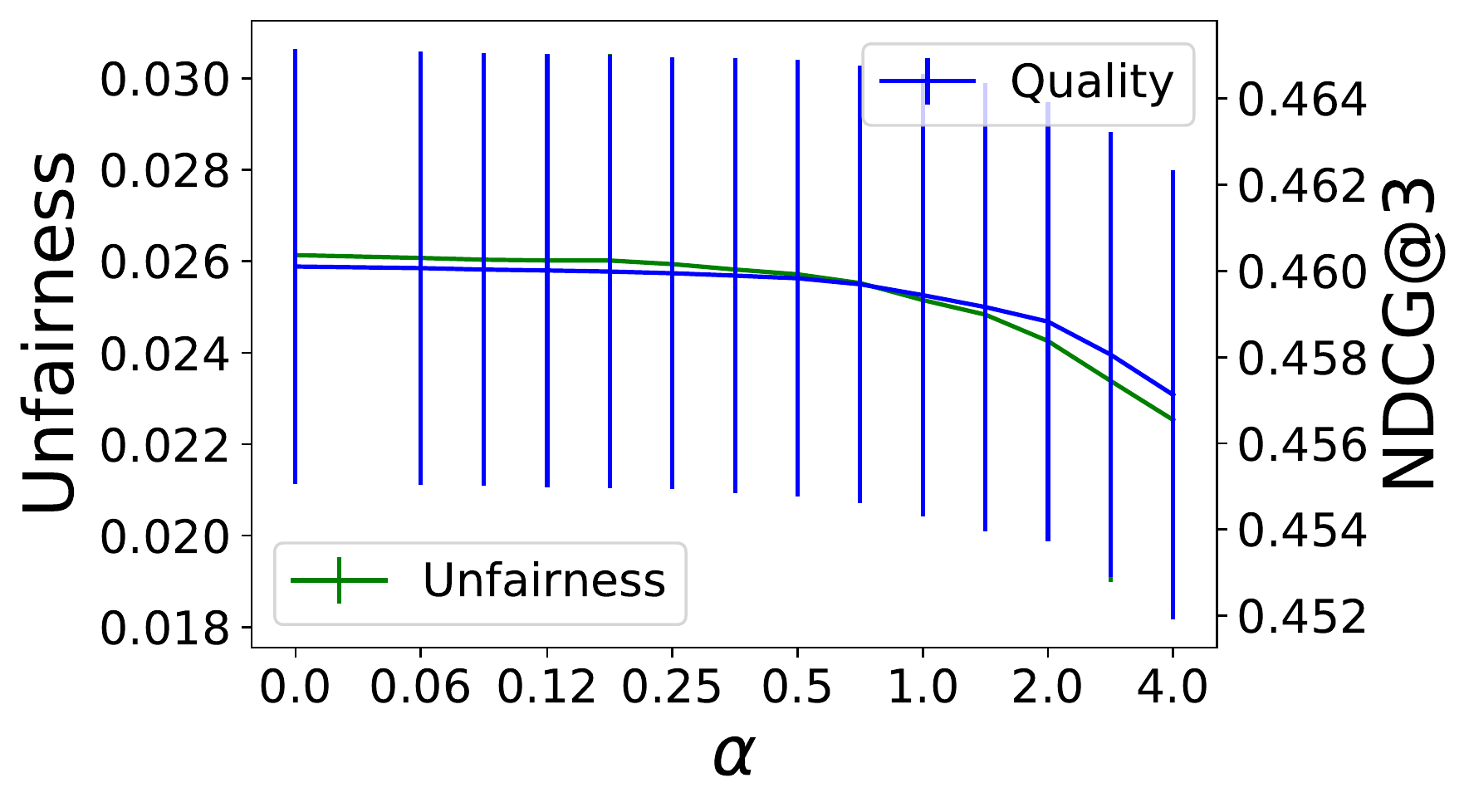}
  \caption{Per-query, equal odds}
\end{subfigure}
\vskip\baselineskip
\caption{MSMARCO: Test-time performance of fair rankers with equal opportunity, demographic parity and equalized odds fairness, achieved by our algorithm and the baselines: unfairness (left $y$-axes) and NDCG@3 ranking 
quality (right $y$-axes); after training with different regularization 
strengths ($x$-axis).}
\label{fig:plots-msmarco-all}
\end{figure*}

\subsection{Mean and maximal improvements of fairness}
\label{sec:results_split_over_k}

Next we provide estimates on how much improvement of fairness is achievable without sacrificing ranking quality, both for our algorithm and the baselines. Essentially, we provide analogs of the results in Table \ref{table:results_ndcg_averages}, but with detailed splits according to the values of $k$.

Specifically, Tables \ref{table:results_max_ours}, \ref{table:results_max_deltr}, \ref{table:results_max_fair}, \ref{table:results_max_per_query} report the maximal reduction of the 
fairness violation measure over the range of values of the trade-off parameter for which the 
corresponding model's prediction quality is not significantly 
worse than for a model trained without a fairness regularizer 
(\ie $\alpha = 0$), for our method, DELTR, FAIR and the per-query fairness variant respectively. Recall that we call a model significantly worse than another 
if the difference of the mean quality values of the two models is 
larger than the sum of the standard errors/deviations, for TREC/MSMARCO respectively, around 
those averages (that is, if the error bars, as in Figure~\ref{fig:plots}, 
would not intersect). Each of the entries in a table reports the improvement possible by the algorithm under consideration.

Next, Tables \ref{table:results_mean_ours}, \ref{table:results_mean_deltr}, \ref{table:results_mean_fair}, \ref{table:results_mean_per_query} report the mean reduction of the 
fairness violation measure over the range of values of the trade-off parameter for which the 
corresponding model's prediction quality is not significantly 
worse than for a model trained without a fairness regularizer 
(\ie $\alpha = 0$), again for our method, DELTR, FAIR and the per-query fairness variant respectively.


\begin{table}[h!]\footnotesize
\caption{Maximal relative fairness increase without a significant decrease of ranking quality for our method. See main text for details.}
\centering\setlength{\tabcolsep}{3.3pt}
 \begin{tabular}{| c c || c | c | c | c | c || c |} 
 \hline
 \multirow{2}{*}{TREC} & & \multicolumn{5}{c||}{$k$} & average \\ 
  & & 1 & 2 & 3 & 4 & 5 & over $k$ \\ 
 \hline\hline 
\multirow{3}{*}{\parbox{.1\textwidth}{equality of \newline opportunity}} & $t = 3$ & 52\% & 58\% & 59\% & 37\% & 32\% & 48\% \\
& $t = 4$ & 51\% & 46\% & 42\% & 56\% & 36\% & 46\% \\
& $t = 5$ & 53\% & 55\% & 48\% & 48\% & 27\% & 46\% \\
 \hline
\multirow{3}{*}{\parbox{.1\textwidth}{demographic \newline parity}} & $t = 3$ & 41\% & 31\% & 32\% & 17\% & 12\% & 27\% \\
& $t = 4$ & 65\% & 41\% & 45\% & 42\% & 30\% & 44\% \\
& $t = 5$ & 62\% & 43\% & 67\% & 61\% & 50\% & 57\% \\
 \hline
\multirow{3}{*}{\parbox{.1\textwidth}{equalized \newline odds}} & $t = 3$ & 23\% & 24\% & 24\% & 13\% & 16\% & 20\% \\
& $t = 4$ & 32\% & 25\% & 31\% & 38\% & 24\% & 30\% \\
& $t = 5$ & 34\% & 10\% & 38\% & 42\% & 21\% & 29\% \\
\hline
\multicolumn{2}{|c|}{average over settings} & 46\% & 37\% & 43\% & 39\% & 27\% & 39\% \\
\hline
\end{tabular}
\begin{tabular}{| c c || c | c | c | c | c || c |} 
 \hline
 \multirow{2}{*}{MSMARCO} & & \multicolumn{5}{c||}{$k$} & average \\ 
 & & 1 & 2 & 3 & 4 & 5 & over $k$ \\ 
 \hline\hline 
\multirow{2}{*}{\parbox{.1\textwidth}{equality of \newline opportunity}} & \emph{com} & 63\% & 58\% & 53\% & 49\% & 50\% & 55\% \\
& \emph{ext} & 25\% & 24\% & 19\% & 19\% & 9\% & 19\% \\
\hline
\multirow{2}{*}{\parbox{.1\textwidth}{demographic \newline parity}} & \emph{com} & 38\% & 41\% & 41\% & 46\% & 43\% & 42\% \\
& \emph{ext} & 14\% & 16\% & 19\% & 22\% & 28\% & 20\% \\
\hline
\multirow{2}{*}{\parbox{.1\textwidth}{equalized \newline odds}} & \emph{com} & 67\% & 63\% & 60\% & 58\% & 59\% & 61\% \\
& \emph{ext} & 27\% & 29\% & 27\% & 30\% & 30\% & 28\% \\
\hline
\multicolumn{2}{|c|}{average over settings} & 39\% & 38\% & 36\% & 37\% & 36\% & 37\% \\
\hline
\end{tabular}
\label{table:results_max_ours}
\end{table}

\begin{table}[h!]\footnotesize
\caption{Maximal relative fairness increase without a significant decrease of ranking quality: DELTR method. See main text for details.}
\centering\setlength{\tabcolsep}{3.3pt}
 \begin{tabular}{| c c || c | c | c | c | c || c |} 
 \hline
 \multirow{2}{*}{TREC} & & \multicolumn{5}{c||}{$k$} & average \\ 
  & & 1 & 2 & 3 & 4 & 5 & over $k$ \\ 
 \hline\hline 
\multirow{3}{*}{\parbox{.1\textwidth}{equality of \newline opportunity}} & $t = 3$ & 26\% & 43\% & 38\% & 49\% & 41\% & 39\% \\
& $t = 4$ & 10\% & 0\% & 0\% & 0\% & 0\% & 2\% \\
& $t = 5$ & 28\% & 20\% & 23\% & 18\% & 0\% & 18\% \\
 \hline
\multirow{3}{*}{\parbox{.1\textwidth}{demographic \newline parity}} & $t = 3$ & 59\% & 66\% & 51\% & 50\% & 50\% & 55\% \\
& $t = 4$ & 22\% & 22\% & 11\% & 7\% & 0\% & 12\% \\
& $t = 5$ & 40\% & 30\% & 28\% & 20\% & 11\% & 26\% \\
 \hline
\multirow{3}{*}{\parbox{.1\textwidth}{equalized \newline odds}} & $t = 3$ & 39\% & 53\% & 45\% & 54\% & 49\% & 48\% \\
& $t = 4$ & 19\% & 18\% & 4\% & 3\% & 0\% & 9\% \\
& $t = 5$ & 33\% & 26\% & 23\% & 18\% & 6\% & 21\% \\
\hline
\multicolumn{2}{|c|}{average over settings} & 31\% & 31\% & 25\% & 24\% & 17\% & 26\% \\
\hline
\end{tabular}
\label{table:results_max_deltr}
\end{table}

\begin{table}[h!]\footnotesize
\caption{Maximal relative fairness increase without a significant decrease of ranking quality: FA*IR method. See main text for details.}
\centering\setlength{\tabcolsep}{3.3pt}
 \begin{tabular}{| c c || c | c | c || c |} 
 \hline
 \multirow{2}{*}{TREC} & & \multicolumn{3}{c||}{$k$} & average \\ 
  & & 3 & 4 & 5 & over $k$ \\ 
 \hline\hline 
\multirow{3}{*}{\parbox{.1\textwidth}{equality of \newline opportunity}} & $t = 3$ & 60\% & 56\% & 53\% & 56\% \\
& $t = 4$ & 23\% & 78\% & 68\% & 56\% \\
& $t = 5$ & 0\% & 5\% & 13\% & 6\% \\
 \hline
\multirow{3}{*}{\parbox{.1\textwidth}{demographic \newline parity}} & $t = 3$ & 79\% & 77\% & 93\% & 83\% \\
& $t = 4$ & 26\% & 67\% & 75\% & 56\% \\
& $t = 5$ & 0\% & 0\% & 33\% & 11\% \\
 \hline
\multirow{3}{*}{\parbox{.1\textwidth}{equalized \newline odds}} & $t = 3$ & 56\% & 54\% & 60\% & 57\% \\
& $t = 4$ & 10\% & 48\% & 62\% & 40\% \\
& $t = 5$ & 0\% & 0\% & 1\% & 0\% \\
\hline
\multicolumn{2}{|c|}{average over settings} & 28\% & 43\% & 51\% & 41\% \\
\hline
\end{tabular}
\begin{tabular}{| c c || c | c | c || c |} 
 \hline
 \multirow{2}{*}{MSMARCO} & & \multicolumn{3}{c||}{$k$} & average \\ 
 & & 3 & 4 & 5 & over $k$ \\ 
 \hline\hline 
\multirow{2}{*}{\parbox{.1\textwidth}{equality of \newline opportunity}} & \emph{com} & 80\% & 66\% & 44\% & 64\% \\
& \emph{ext} & 0\% & 0\% & 0\% & 0\% \\
\hline
\multirow{2}{*}{\parbox{.1\textwidth}{demographic \newline parity}}& \emph{com} & 0\% & 36\% & 80\% & 39\% \\
& \emph{ext} & 0\% & 0\% & 0\% & 0\% \\
\hline
\multirow{2}{*}{\parbox{.1\textwidth}{equalized \newline odds}} & \emph{com} & 33\% & 51\% & 51\% & 45\% \\
& \emph{ext} & 0\% & 0\% & 0\% & 0\% \\
\hline
\multicolumn{2}{|c|}{average over settings} & 19\% & 26\% & 29\% & 25\% \\
\hline
\end{tabular}
\label{table:results_max_fair}
\end{table}

\begin{table}[h!]\footnotesize
\caption{Maximal relative fairness increase without a significant decrease of ranking quality: per query fairness. See main text for details.}
\centering\setlength{\tabcolsep}{3.3pt}
 \begin{tabular}{| c c || c | c | c | c | c || c |} 
 \hline
 \multirow{2}{*}{TREC} & & \multicolumn{5}{c||}{$k$} & average \\ 
  & & 1 & 2 & 3 & 4 & 5 & over $k$ \\ 
 \hline\hline 
\multirow{3}{*}{\parbox{.1\textwidth}{equality of \newline opportunity}} & $t = 3$ & 28\% & 19\% & 17\% & 20\% & 9\% & 19\% \\
& $t = 4$ & 21\% & 0\% & 2\% & 29\% & 19\% & 14\% \\
& $t = 5$ & 9\% & 14\% & 10\% & 6\% & 0\% & 8\% \\
 \hline
\multirow{3}{*}{\parbox{.1\textwidth}{demographic \newline parity}} & $t = 3$ & 4\% & 14\% & 22\% & 17\% & 17\% & 15\% \\
& $t = 4$ & 32\% & 26\% & 26\% & 25\% & 24\% & 27\% \\
& $t = 5$ & 48\% & 15\% & 35\% & 8\% & 15\% & 24\% \\
 \hline
\multirow{3}{*}{\parbox{.1\textwidth}{equalized \newline odds}} & $t = 3$ & 16\% & 14\% & 7\% & 22\% & 11\% & 14\% \\
& $t = 4$ & 11\% & 3\% & 19\% & 14\% & 16\% & 13\% \\
& $t = 5$ & 7\% & 8\% & 11\% & 11\% & 0\% & 7\% \\
\hline
\multicolumn{2}{|c|}{average over settings} & 19\% & 13\% & 17\% & 17\% & 12\% & 16\% \\
\hline
\end{tabular}
\begin{tabular}{| c c || c | c | c | c | c || c |} 
 \hline
 \multirow{2}{*}{MSMARCO} & & \multicolumn{5}{c||}{$k$} & average \\ 
 & & 1 & 2 & 3 & 4 & 5 & over $k$ \\ 
 \hline\hline 
\multirow{2}{*}{\parbox{.1\textwidth}{equality of \newline opportunity}} & \emph{com} & 28\% & 25\% & 23\% & 19\% & 23\% & 24\% \\
& \emph{ext} & 3\% & 1\% & 5\% & 1\% & 2\% & 2\% \\
\multirow{2}{*}{\parbox{.1\textwidth}{demographic \newline parity}}& \emph{com} & 0\% & 0\% & 0\% & 0\% & 0\% & 0\% \\
& \emph{ext} & 4\% & 6\% & 8\% & 6\% & 9\% & 7\% \\
\multirow{2}{*}{\parbox{.1\textwidth}{equalized \newline odds}} & \emph{com} & 16\% & 15\% & 14\% & 13\% & 14\% & 14\% \\
& \emph{ext} & 1\% & 1\% & 2\% & 1\% & 2\% & 1\% \\
\hline
\multicolumn{2}{|c|}{average over settings} & 9\% & 8\% & 9\% & 7\% & 8\% & 8\% \\
\hline
\end{tabular}
\label{table:results_max_per_query}
\end{table}


\begin{table}[b]\footnotesize
\caption{Mean relative fairness increase without a significant decrease of ranking quality for our method.}
\centering\setlength{\tabcolsep}{3.3pt}
 \begin{tabular}{| c c || c | c | c | c | c || c |} 
 \hline
 \multirow{2}{*}{TREC} & & \multicolumn{5}{c||}{$k$} & average \\ 
  & & 1 & 2 & 3 & 4 & 5 & over $k$ \\ 
 \hline\hline 
\multirow{3}{*}{\parbox{.1\textwidth}{equality of \newline opportunity}} & $t = 3$ & 29\% & 42\% & 44\% & 27\% & 26\% & 34\% \\
& $t = 4$ & 41\% & 38\% & 33\% & 44\% & 28\% & 37\% \\
& $t = 5$ & 35\% & 28\% & 41\% & 40\% & 18\% & 32\% \\
 \hline
\multirow{3}{*}{\parbox{.1\textwidth}{demographic \newline parity}}  & $t = 3$ & 30\% & 20\% & 21\% & 8\% & 9\% & 17\% \\
& $t = 4$ & 49\% & 26\% & 32\% & 30\% & 23\% & 32\% \\
& $t = 5$ & 43\% & 21\% & 56\% & 42\% & 39\% & 40\% \\
 \hline
\multirow{3}{*}{\parbox{.1\textwidth}{equalized \newline odds}} & $t = 3$ & 16\% & 14\% & 15\% & 7\% & 11\% & 13\% \\
& $t = 4$ & 25\% & 18\% & 20\% & 27\% & 17\% & 21\% \\
& $t = 5$ & 23\% & 6\% & 28\% & 29\% & 17\% & 21\% \\
\hline
\multicolumn{2}{|c|}{average over settings} & 32\% & 24\% & 32\% & 28\% & 21\% & 27\% \\
\hline
\end{tabular}
\begin{tabular}{| c c || c | c | c | c | c || c |} 
 \hline
 \multirow{2}{*}{MSMARCO} & & \multicolumn{5}{c||}{$k$} & average \\ 
 & & 1 & 2 & 3 & 4 & 5 & over $k$ \\ 
 \hline\hline 
\multirow{2}{*}{\parbox{.1\textwidth}{equality of \newline opportunity}} & \emph{com} & 42\% & 39\% & 36\% & 33\% & 33\% & 36\% \\
& \emph{ext} & 12\% & 12\% & 9\% & 10\% & 5\% & 10\% \\
\hline
\multirow{2}{*}{\parbox{.1\textwidth}{demographic \newline parity}} & \emph{com} & 24\% & 26\% & 27\% & 30\% & 29\% & 27\% \\
& \emph{ext} & 10\% & 11\% & 12\% & 13\% & 17\% & 13\% \\
\hline
\multirow{2}{*}{\parbox{.1\textwidth}{equalized \newline odds}} & \emph{com} & 44\% & 41\% & 39\% & 38\% & 40\% & 41\% \\
& \emph{ext} & 16\% & 17\% & 15\% & 18\% & 18\% & 17\% \\
\hline
\multicolumn{2}{|c|}{average over settings} & 25\% & 24\% & 23\% & 24\% & 24\% & 24\% \\
\hline
 \end{tabular}
 \label{table:results_mean_ours}
\end{table}

\begin{table}[b]\footnotesize
\caption{Mean relative fairness increase without a significant decrease of ranking quality: DELTR method.}
\centering\setlength{\tabcolsep}{3.3pt}
 \begin{tabular}{| c c || c | c | c | c | c || c |} 
 \hline
 \multirow{2}{*}{TREC} & & \multicolumn{5}{c||}{$k$} & average \\ 
  & & 1 & 2 & 3 & 4 & 5 & over $k$ \\ 
 \hline\hline 
\multirow{3}{*}{\parbox{.1\textwidth}{equality of \newline opportunity}} & $t = 3$ & 10\% & 26\% & 23\% & 28\% & 27\% & 23\% \\
& $t = 4$ & 5\% & -2\% & -9\% & -12\% & -23\% & -8\% \\
& $t = 5$ & 15\% & 10\% & 11\% & 9\% & -11\% & 7\% \\
 \hline
\multirow{3}{*}{\parbox{.1\textwidth}{demographic \newline parity}}  & $t = 3$ & 43\% & 45\% & 33\% & 31\% & 30\% & 36\% \\
& $t = 4$ & 11\% & 11\% & 5\% & 3\% & -1\% & 6\% \\
& $t = 5$ & 25\% & 15\% & 14\% & 10\% & 5\% & 14\% \\
 \hline
\multirow{3}{*}{\parbox{.1\textwidth}{equalized \newline odds}} & $t = 3$ & 24\% & 40\% & 30\% & 33\% & 30\% & 31\% \\
& $t = 4$ & 9\% & 9\% & 2\% & 2\% & -3\% & 4\% \\
& $t = 5$ & 21\% & 13\% & 12\% & 9\% & 3\% & 12\% \\
\hline
\multicolumn{2}{|c|}{average over settings} & 18\% & 19\% & 13\% & 13\% & 7\% & 14\% \\
\hline
\end{tabular}
 \label{table:results_mean_deltr}
\end{table}

\begin{table}[b]\footnotesize
\caption{Mean relative fairness increase without a significant decrease of ranking quality: FA*IR method.}
\centering\setlength{\tabcolsep}{3.3pt}
 \begin{tabular}{| c c || c | c | c || c |} 
 \hline
 \multirow{2}{*}{TREC} & & \multicolumn{3}{c||}{$k$} & average \\ 
  & & 3 & 4 & 5 & over $k$ \\ 
 \hline\hline 
\multirow{3}{*}{\parbox{.1\textwidth}{equality of \newline opportunity}} & $t = 3$ & 16\% & 17\% & 20\% & 18\% \\
& $t = 4$ & -1\% & 7\% & 26\% & 11\% \\
& $t = 5$ & -24\% & -25\% & -2\% & -17\% \\
 \hline
\multirow{3}{*}{\parbox{.1\textwidth}{demographic \newline parity}}  & $t = 3$ & 22\% & 20\% & 29\% & 24\% \\
& $t = 4$ & 1\% & 3\% & 26\% & 10\% \\
& $t = 5$ & -37\% & -52\% & -16\% & -35\% \\
 \hline
\multirow{3}{*}{\parbox{.1\textwidth}{equalized \newline odds}} & $t = 3$ & 15\% & 11\% & 22\% & 16\% \\
& $t = 4$ & -5\% & -5\% & 19\% & 3\% \\
& $t = 5$ & -40\% & -44\% & -22\% & -35\% \\
\hline
\multicolumn{2}{|c|}{average over settings} & -6\% & -7\% & 11\% & -1\% \\
\hline
\end{tabular}
\begin{tabular}{| c c || c | c | c || c |} 
 \hline
 \multirow{2}{*}{MSMARCO} & & \multicolumn{3}{c||}{$k$} & average \\ 
 & & 3 & 4 & 5 & over $k$ \\ 
 \hline\hline 
\multirow{2}{*}{\parbox{.1\textwidth}{equality of \newline opportunity}} & \emph{com} & 23\% & -5\% & 13\% & 11\% \\
& \emph{ext} & -55\% & -114\% & -166\% & -112\% \\
\hline
\multirow{2}{*}{\parbox{.1\textwidth}{demographic \newline parity}} & \emph{com} & -20\% & -90\% & -39\% & -50\% \\
& \emph{ext} & -91\% & -173\% & -240\% & -168\% \\
\hline
\multirow{2}{*}{\parbox{.1\textwidth}{equalized \newline odds}} & \emph{com} & 10\% & -35\% & -9\% & -12\% \\
& \emph{ext} & -70\% & -144\% & -210\% & -142\% \\
\hline
\multicolumn{2}{|c|}{average over settings} & -34\% & -93\% & -109\% & -79\% \\
\hline
 \end{tabular}
 \label{table:results_mean_fair}
\end{table}

\begin{table}[b]\footnotesize
\caption{Mean relative fairness increase without a significant decrease of ranking quality: per-query fairness.}
\centering\setlength{\tabcolsep}{3.3pt}
 \begin{tabular}{| c c || c | c | c | c | c || c |} 
 \hline
 \multirow{2}{*}{TREC} & & \multicolumn{5}{c||}{$k$} & average \\ 
  & & 1 & 2 & 3 & 4 & 5 & over $k$ \\ 
 \hline\hline 
\multirow{3}{*}{\parbox{.1\textwidth}{equality of \newline opportunity}} & $t = 3$ & -14\% & -8\% & -2\% & 1\% & -1\% & -5\% \\
& $t = 4$ & -14\% & -10\% & -6\% & 2\% & 7\% & -4\% \\
& $t = 5$ & -27\% & -7\% & -4\% & -9\% & -16\% & -13\% \\
 \hline
\multirow{3}{*}{\parbox{.1\textwidth}{demographic \newline parity}}  & $t = 3$ & -4\% & -1\% & 2\% & -1\% & -1\% & -1\% \\
& $t = 4$ & 4\% & 3\% & 6\% & 7\% & 6\% & 5\% \\
& $t = 5$ & -3\% & -2\% & 11\% & -5\% & 2\% & 1\% \\
 \hline
\multirow{3}{*}{\parbox{.1\textwidth}{equalized \newline odds}} & $t = 3$ & -3\% & -2\% & -3\% & -3\% & -3\% & -3\% \\
& $t = 4$ & -7\% & -3\% & 2\% & -1\% & 4\% & -1\% \\
& $t = 5$ & -11\% & -8\% & 4\% & 4\% & -9\% & -4\% \\
\hline
\multicolumn{2}{|c|}{average over settings} & -9\% & -4\% & 1\% & -1\% & -1\% & -3\% \\
\hline
\end{tabular}
\begin{tabular}{| c c || c | c | c | c | c || c |} 
 \hline
 \multirow{2}{*}{MSMARCO} & & \multicolumn{5}{c||}{$k$} & average \\ 
 & & 1 & 2 & 3 & 4 & 5 & over $k$ \\ 
 \hline\hline 
\multirow{2}{*}{\parbox{.1\textwidth}{equality of \newline opportunity}} & \emph{com} & 7\% & 6\% & 6\% & 5\% & 6\% & 6\% \\
& \emph{ext} & 0\% & 0\% & 1\% & 0\% & 0\% & 0\% \\
\hline
\multirow{2}{*}{\parbox{.1\textwidth}{demographic \newline parity}} & \emph{com} & 0\% & 0\% & 0\% & 0\% & 0\% & 0\% \\
& \emph{ext} & 2\% & 3\% & 4\% & 3\% & 5\% & 3\% \\
\hline
\multirow{2}{*}{\parbox{.1\textwidth}{equalized \newline odds}} & \emph{com} & 4\% & 4\% & 3\% & 3\% & 4\% & 4\% \\
& \emph{ext} & 0\% & 0\% & 0\% & 0\% & 0\% & 0\% \\
\hline
\multicolumn{2}{|c|}{average over settings} & 2\% & 2\% & 2\% & 2\% & 2\% & 2\% \\
\hline
 \end{tabular}
 \label{table:results_mean_per_query}
\end{table}

\clearpage
\subsection{Plots for other values of $k$ and other splits into protected groups}
\label{sec:all_plots}

All plots below show NDCG@$k$ and fairness achieved by our method, for the three fairness notions on every row.

\subsubsection{TREC results}

Different rows correspond to different values of $k$ and $t$ (the threshold for the i$10$ index).


\begin{figure*}[h]
\begin{subfigure}{.33\textwidth}
  \centering
  \includegraphics[width=.95\linewidth]{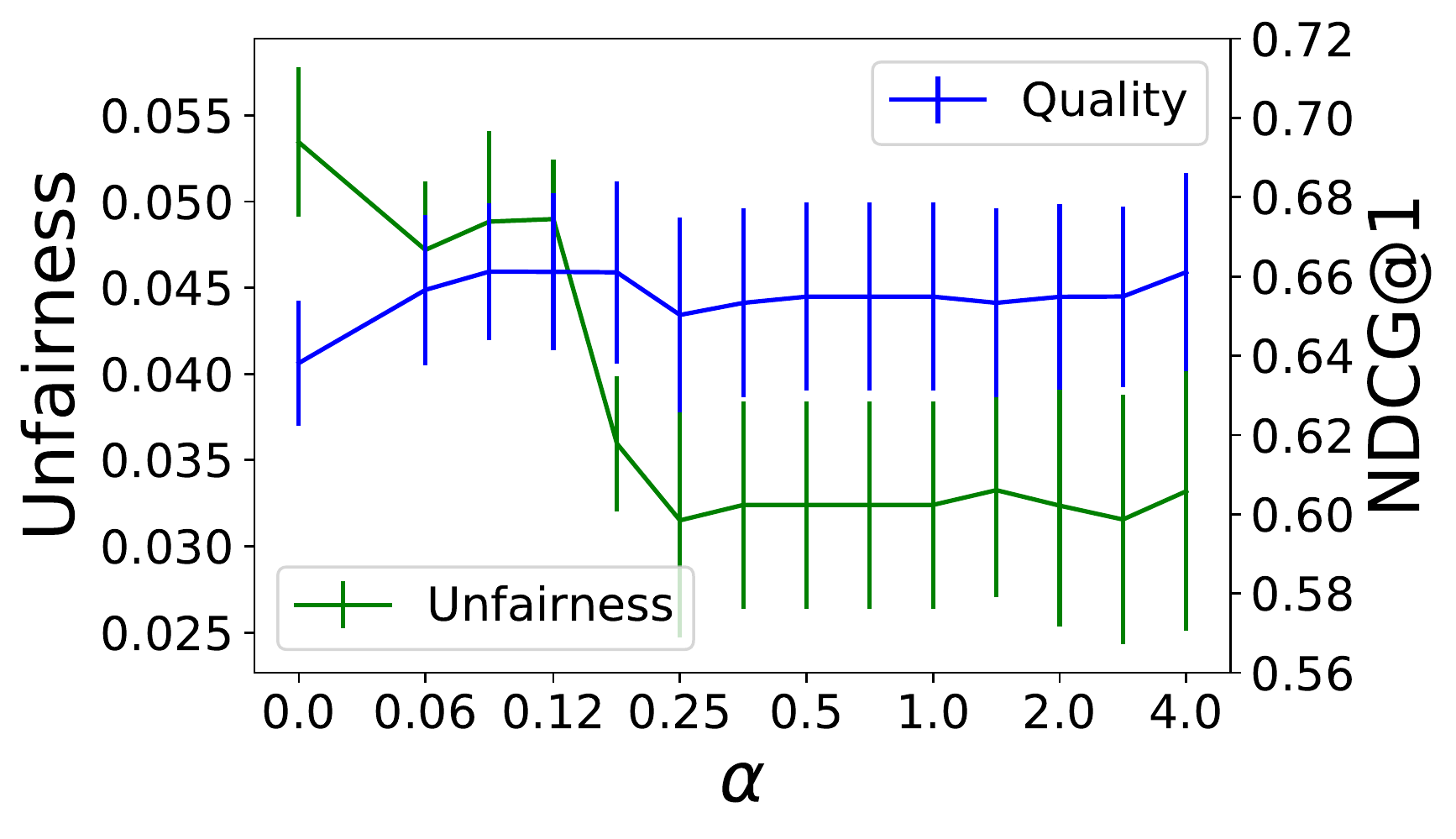}
  \caption{Demographic parity, TREC data}
\end{subfigure}%
\hfill
\begin{subfigure}{.33\textwidth}
  \centering
  \includegraphics[width=.95\linewidth]{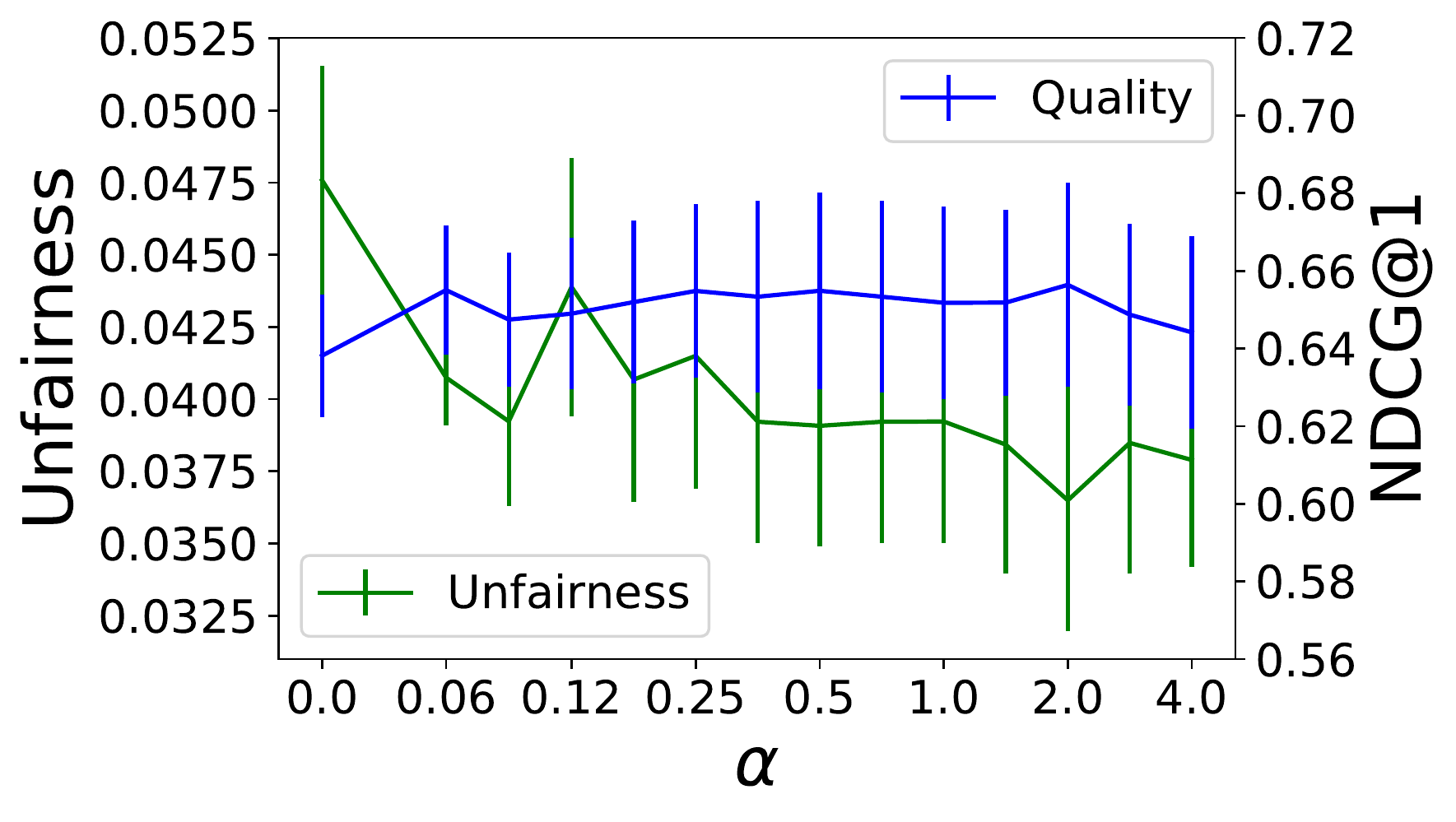}
  \caption{Equalized odds, TREC data}
\end{subfigure}%
\hfill
\begin{subfigure}{.33\textwidth}
  \centering
  \includegraphics[width=.95\linewidth]{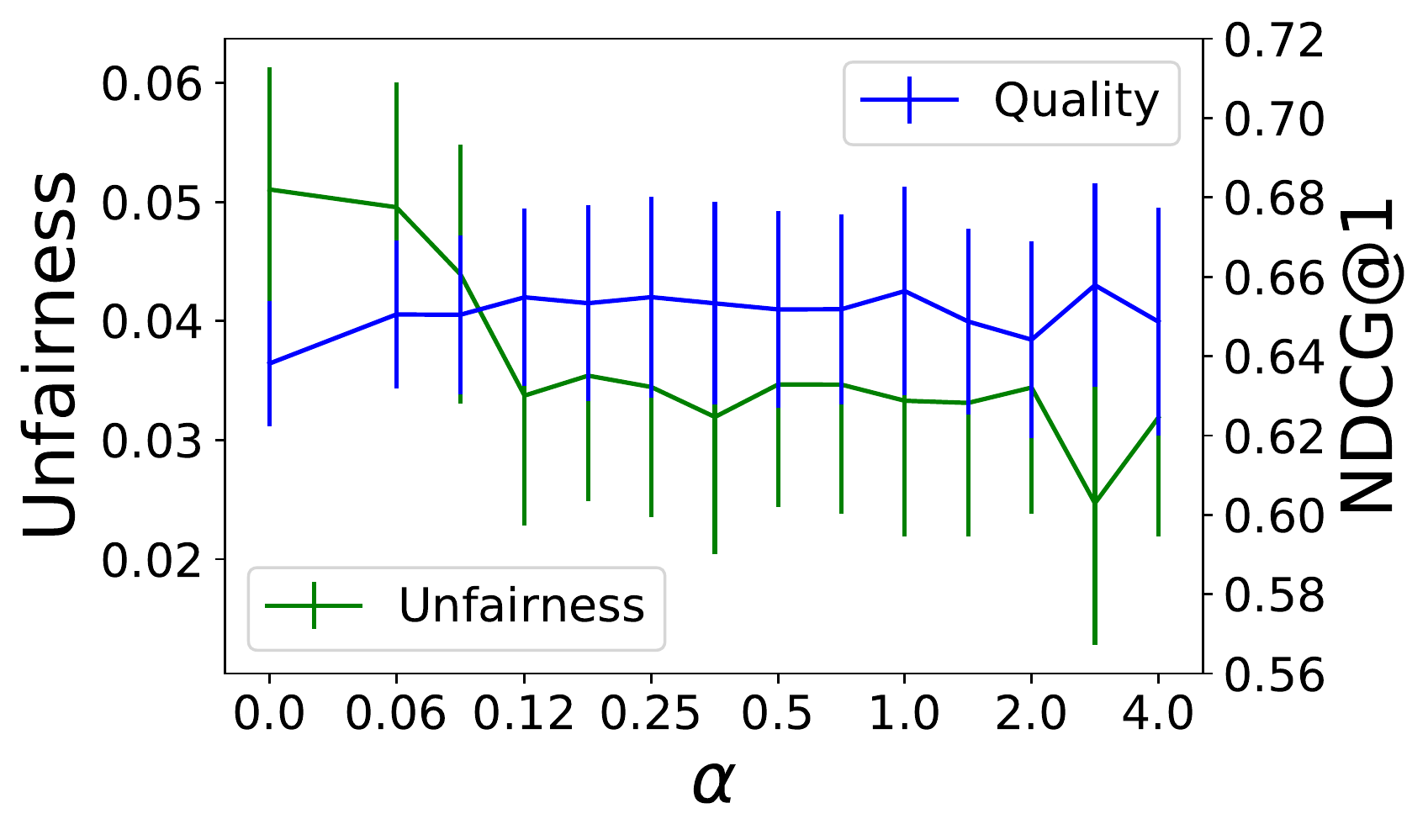}
  \caption{Equality of opportunity, TREC data}
\end{subfigure}
\caption{$k = 1, t = 3$}
\end{figure*}

\begin{figure*}[h]
\begin{subfigure}{.33\textwidth}
  \centering
  \includegraphics[width=.95\linewidth]{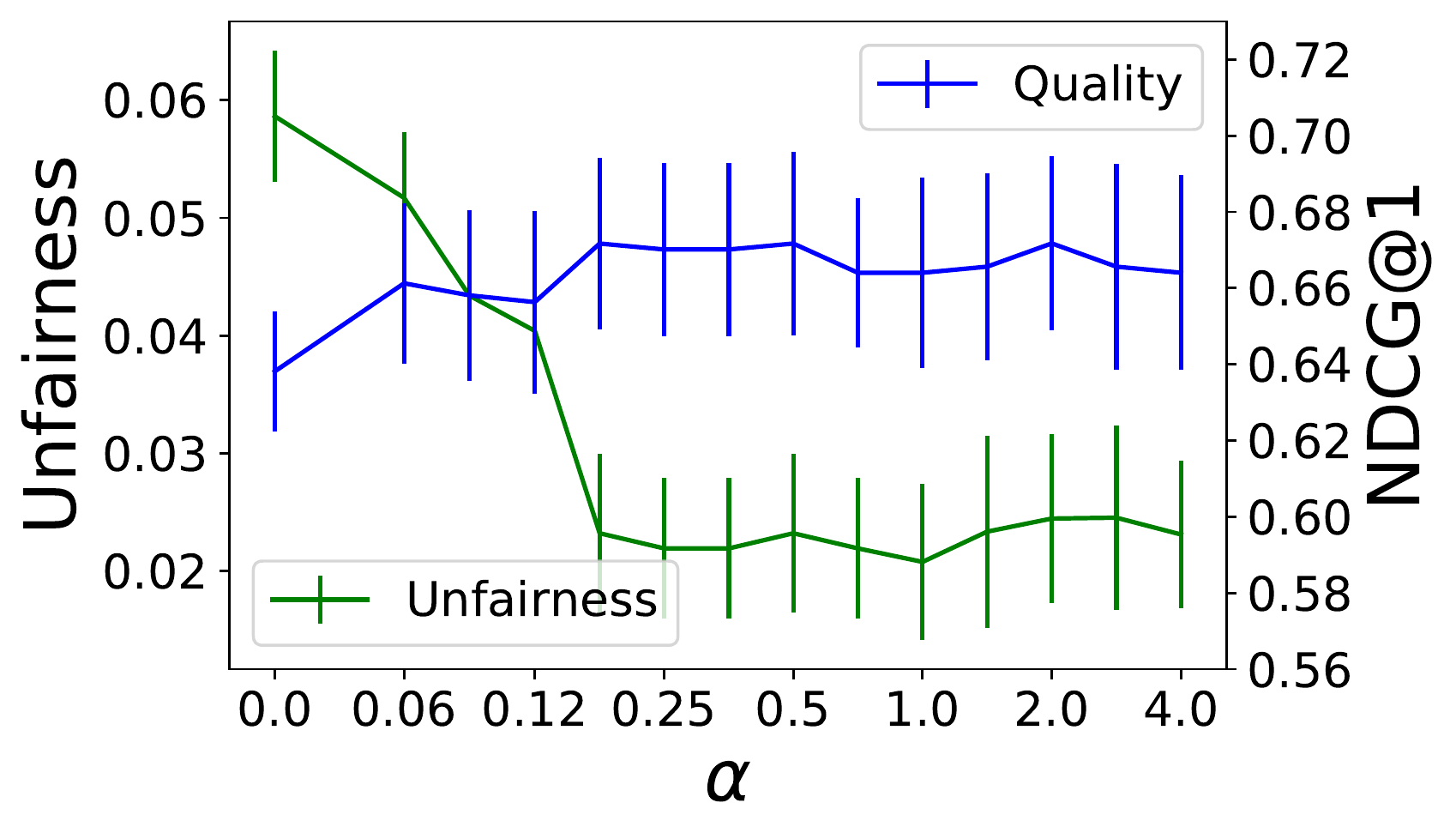}
  \caption{Demographic parity, TREC data}
\end{subfigure}%
\hfill
\begin{subfigure}{.33\textwidth}
  \centering
  \includegraphics[width=.95\linewidth]{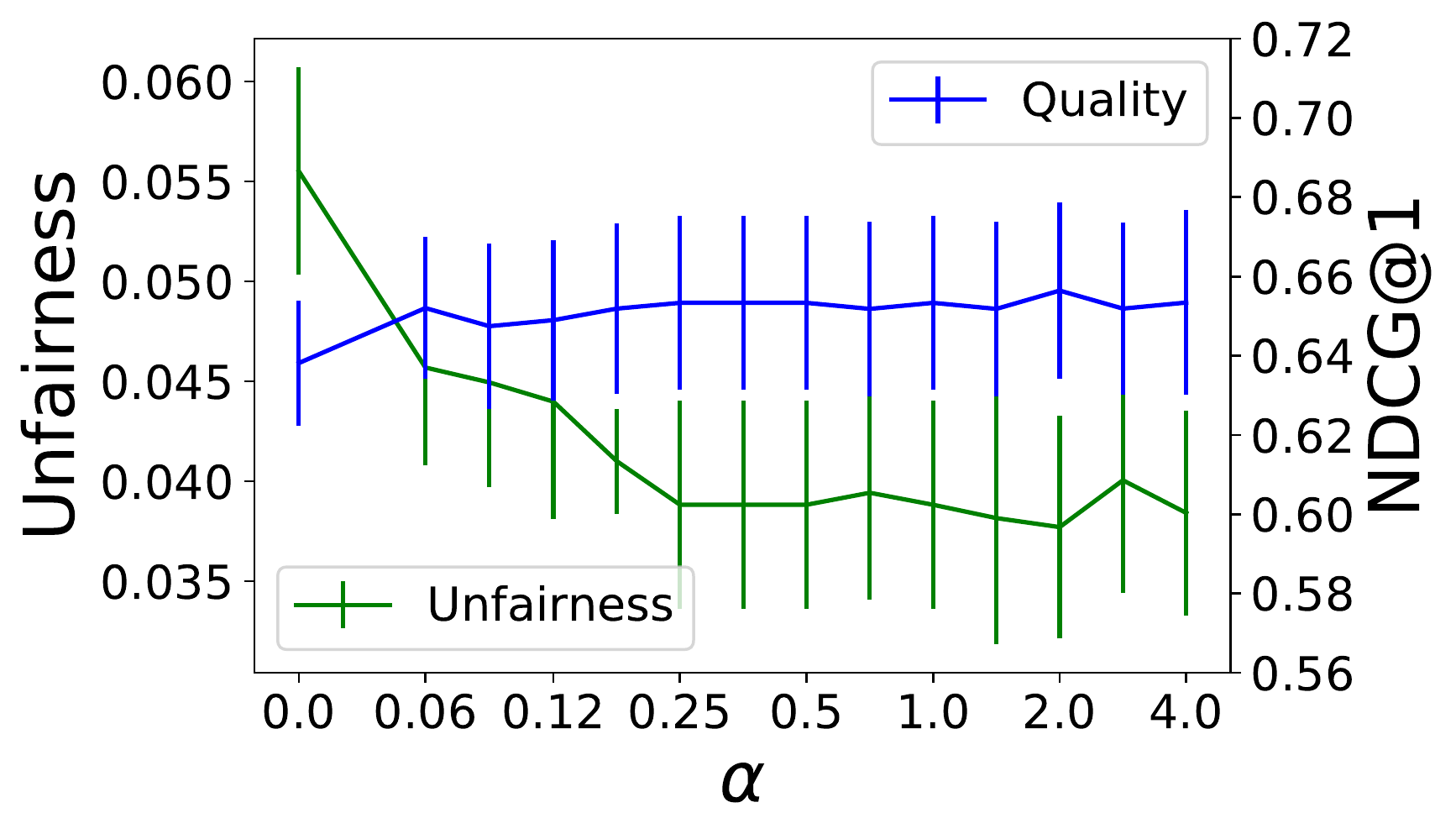}
  \caption{Equalized odds, TREC data}
\end{subfigure}%
\hfill
\begin{subfigure}{.33\textwidth}
  \centering
  \includegraphics[width=.95\linewidth]{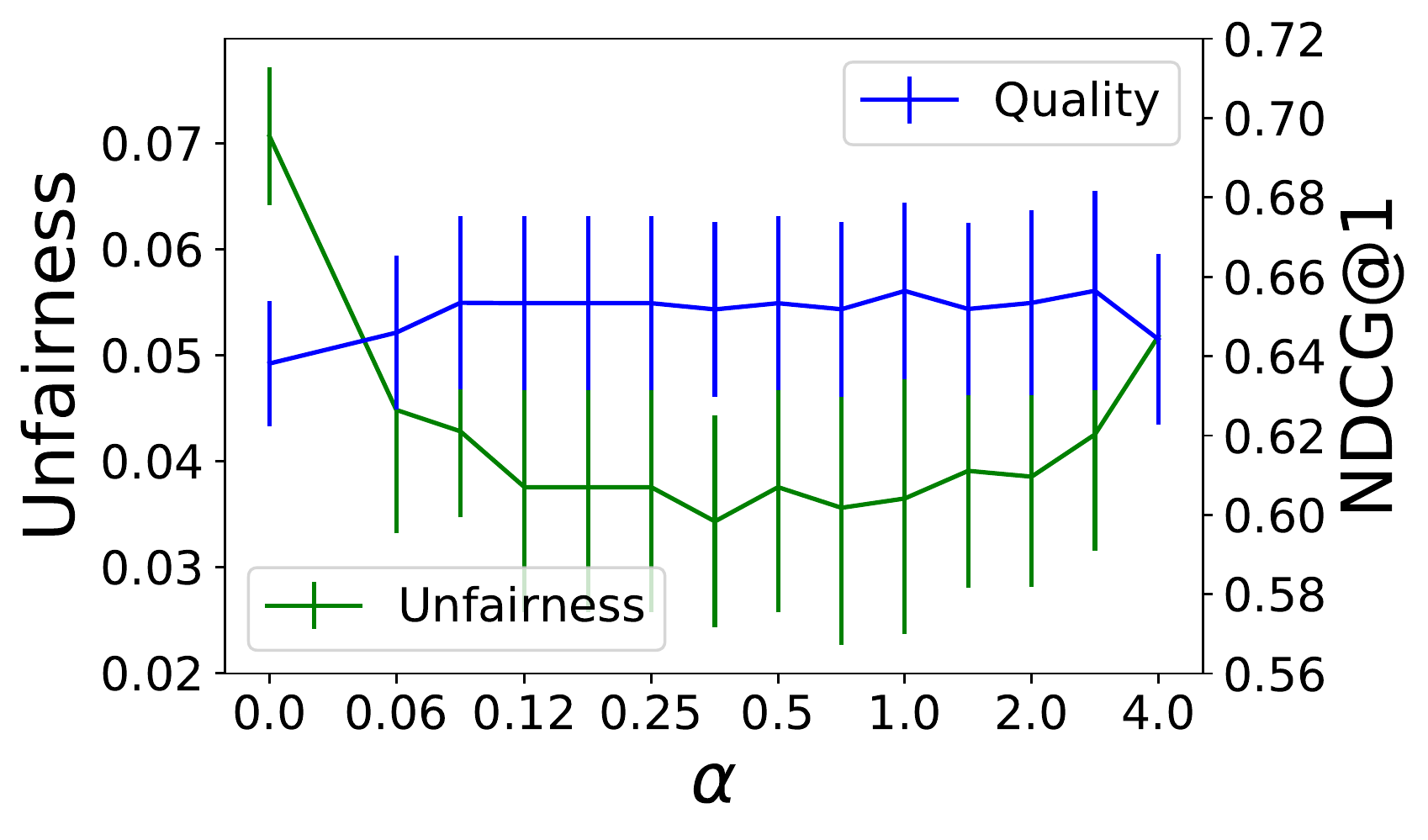}
  \caption{Equality of opportunity, TREC data}
\end{subfigure}
\caption{$k = 1, t = 4$}
\end{figure*}

\begin{figure*}[h]
\begin{subfigure}{.33\textwidth}
  \centering
  \includegraphics[width=.95\linewidth]{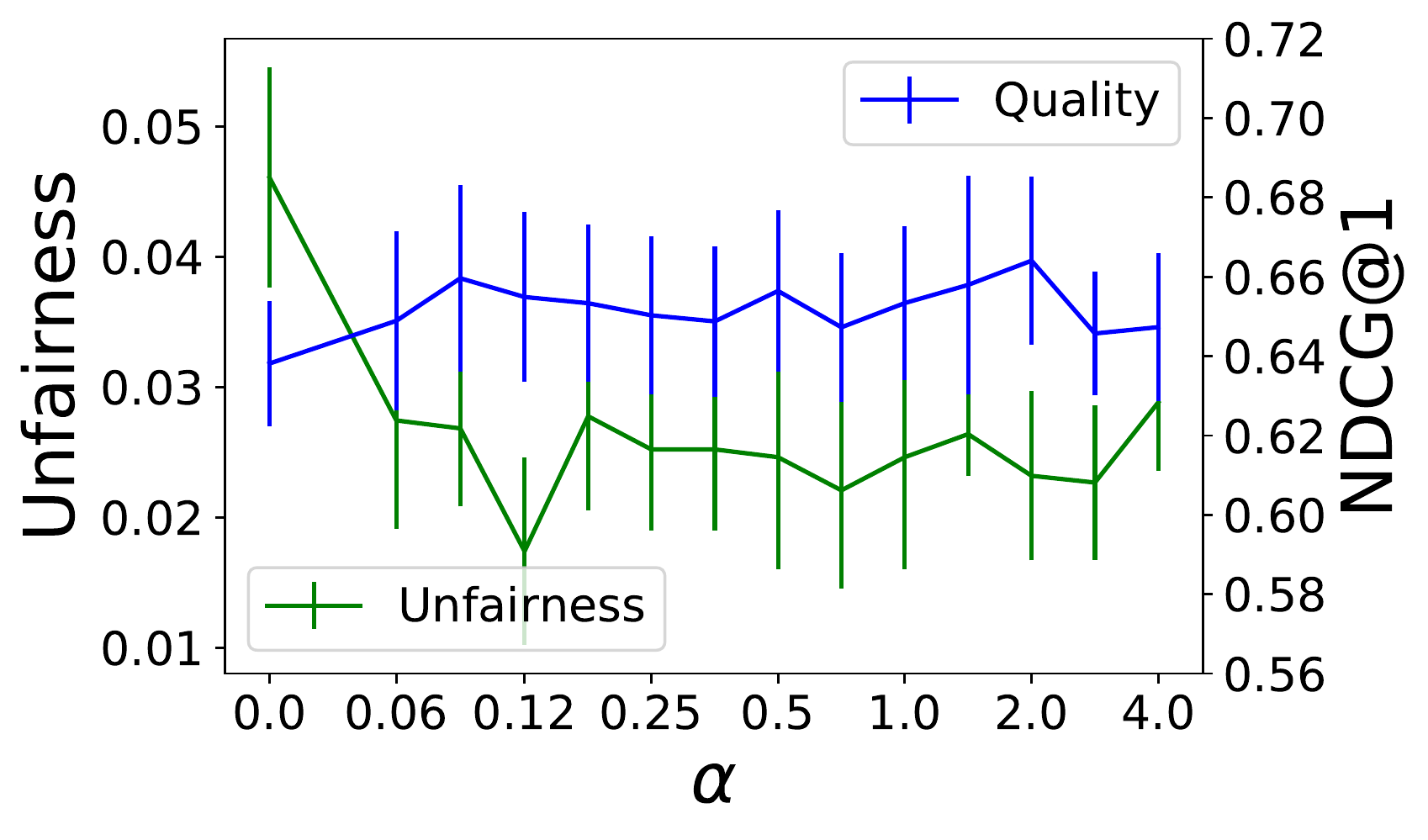}
  \caption{Demographic parity, TREC data}
\end{subfigure}%
\hfill
\begin{subfigure}{.33\textwidth}
  \centering
  \includegraphics[width=.95\linewidth]{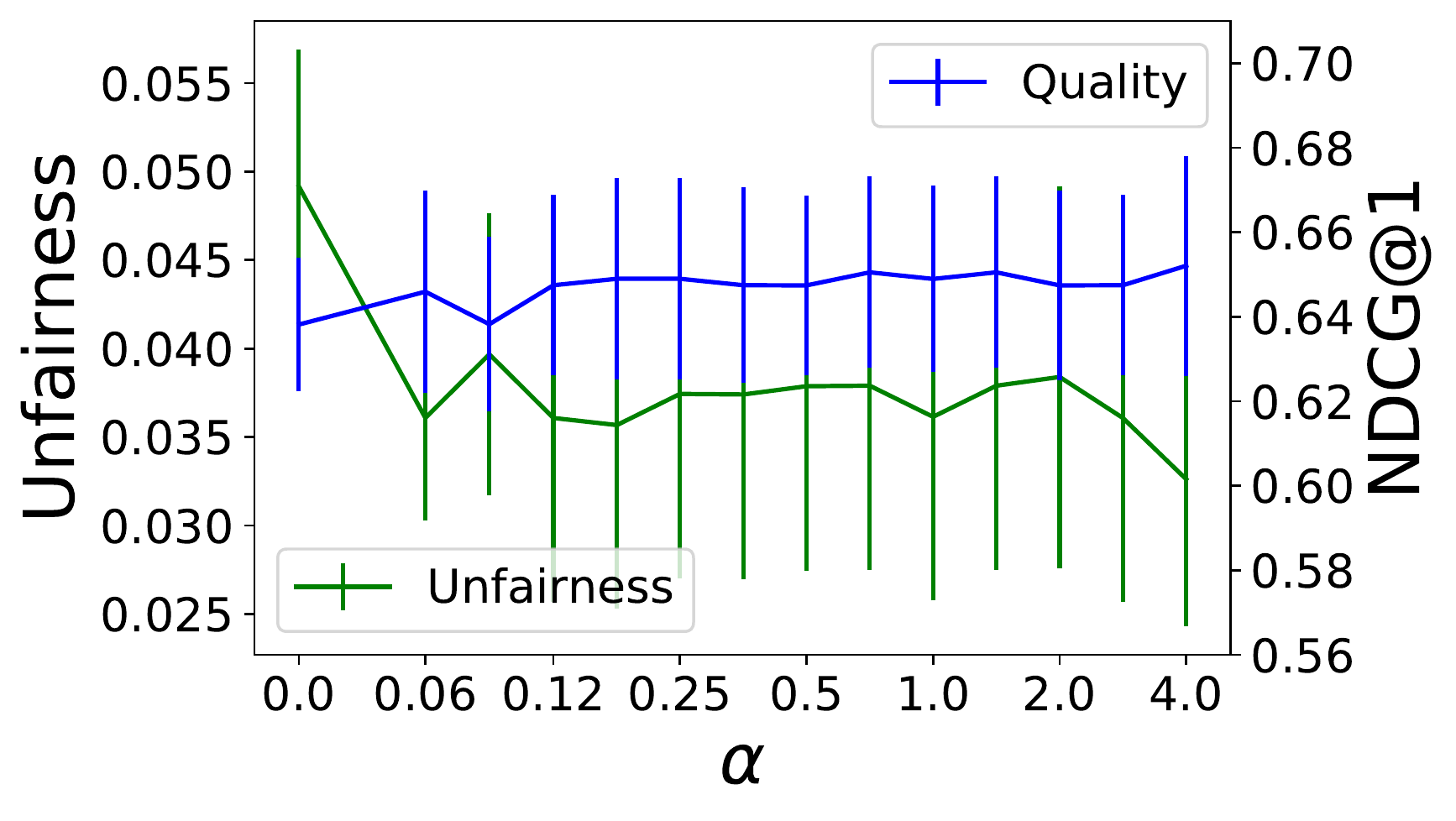}
  \caption{Equalized odds, TREC data}
\end{subfigure}%
\hfill
\begin{subfigure}{.33\textwidth}
  \centering
  \includegraphics[width=.95\linewidth]{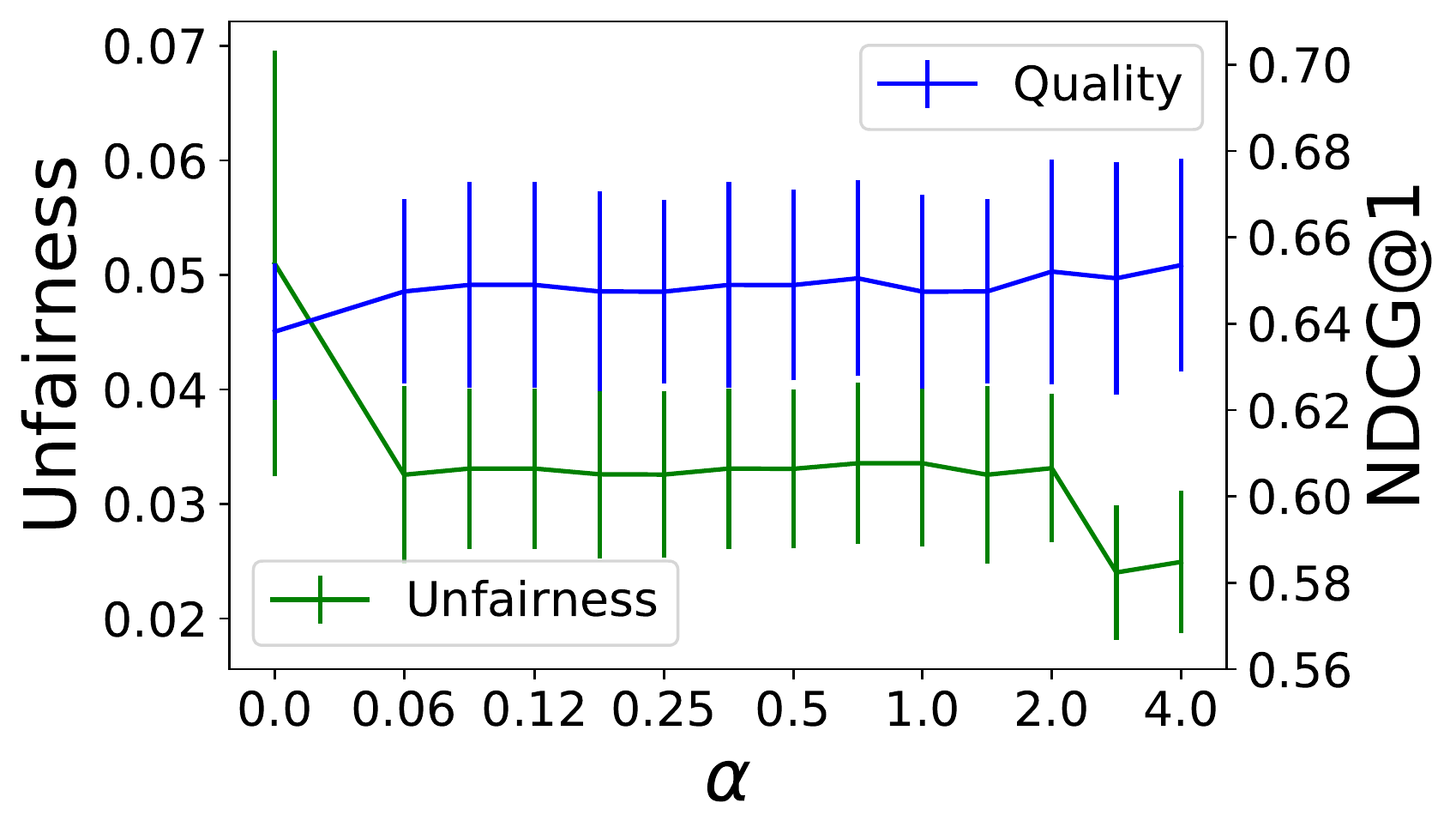}
  \caption{Equality of opportunity, TREC data}
\end{subfigure}
\caption{$k = 1, t = 5$}
\end{figure*}


\begin{figure*}[h]
\begin{subfigure}{.33\textwidth}
  \centering
  \includegraphics[width=.95\linewidth]{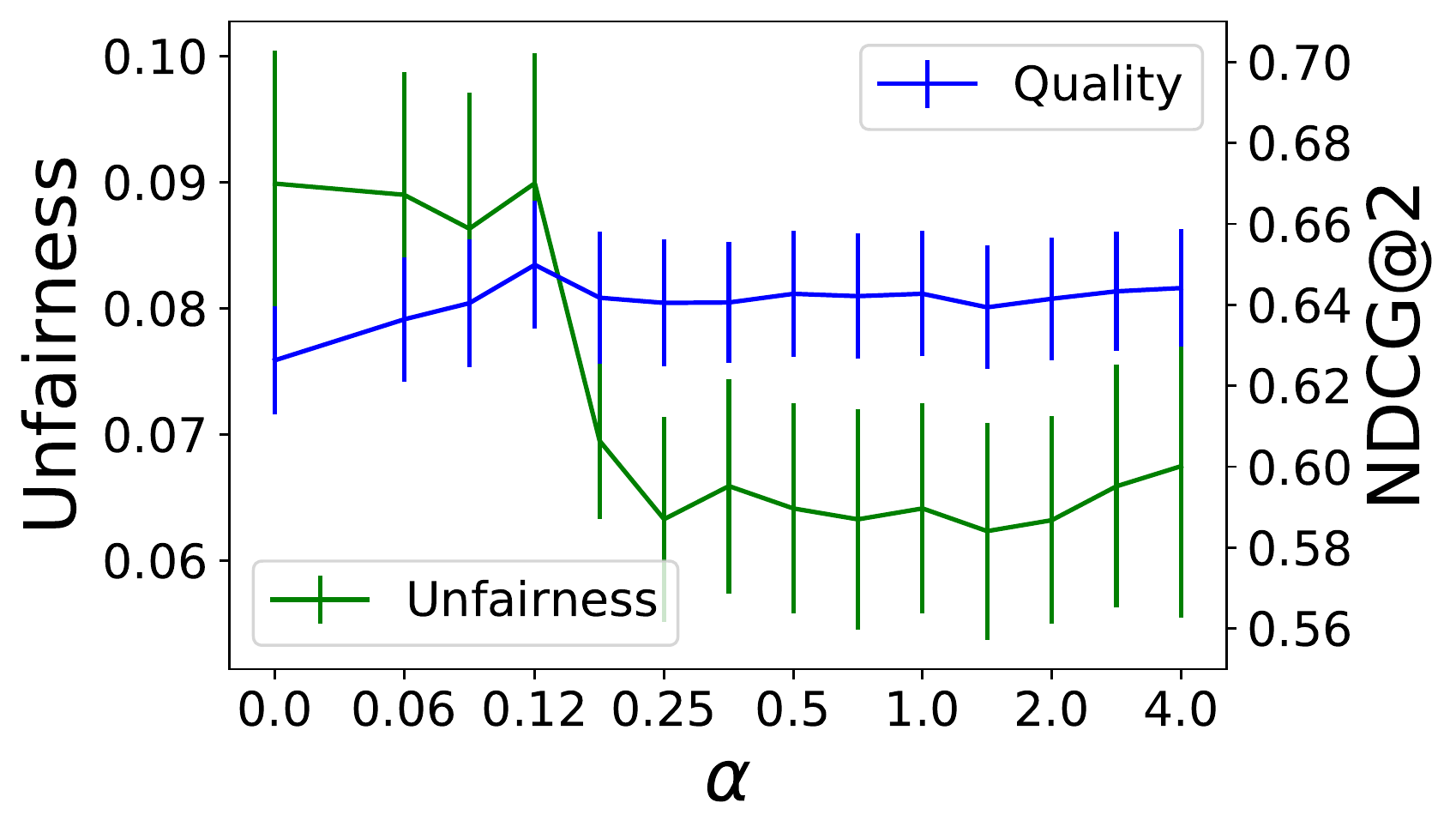}
  \caption{Demographic parity, TREC data}
\end{subfigure}%
\hfill
\begin{subfigure}{.33\textwidth}
  \centering
  \includegraphics[width=.95\linewidth]{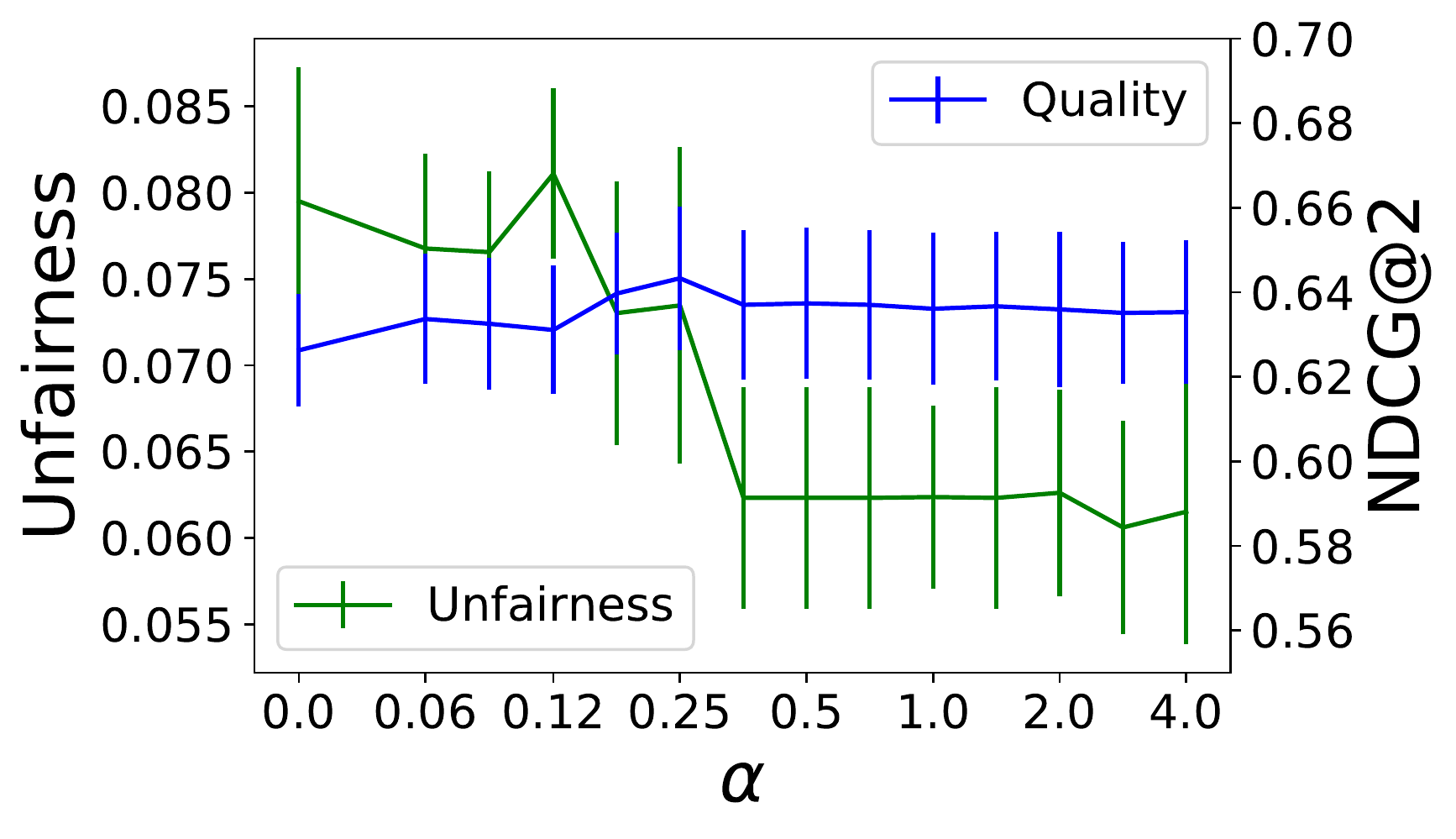}
  \caption{Equalized odds, TREC data}
\end{subfigure}%
\hfill
\begin{subfigure}{.33\textwidth}
  \centering
  \includegraphics[width=.95\linewidth]{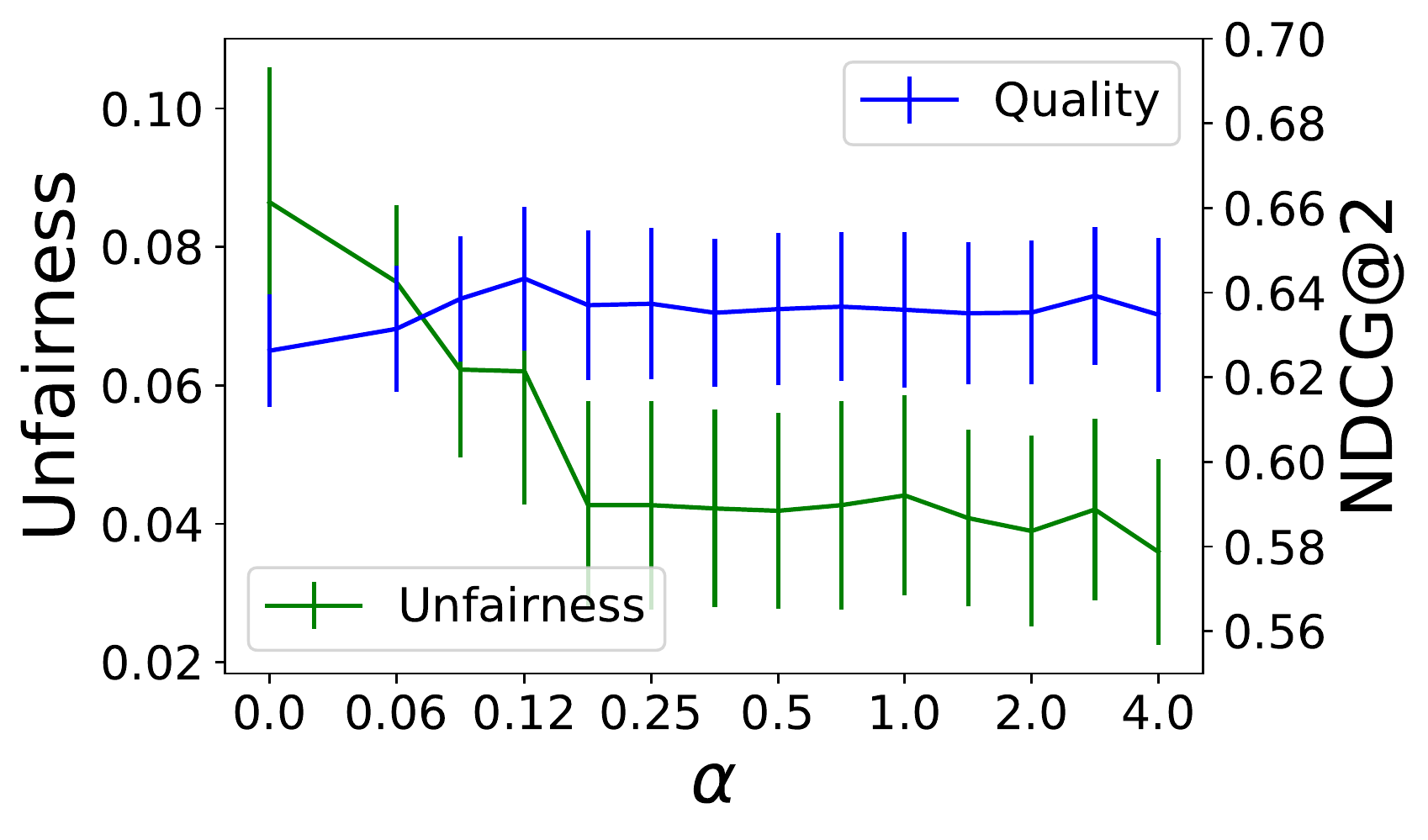}
  \caption{Equality of opportunity, TREC data}
\end{subfigure}
\caption{$k = 2, t = 3$}
\end{figure*}

\begin{figure*}[h]
\begin{subfigure}{.33\textwidth}
  \centering
  \includegraphics[width=.95\linewidth]{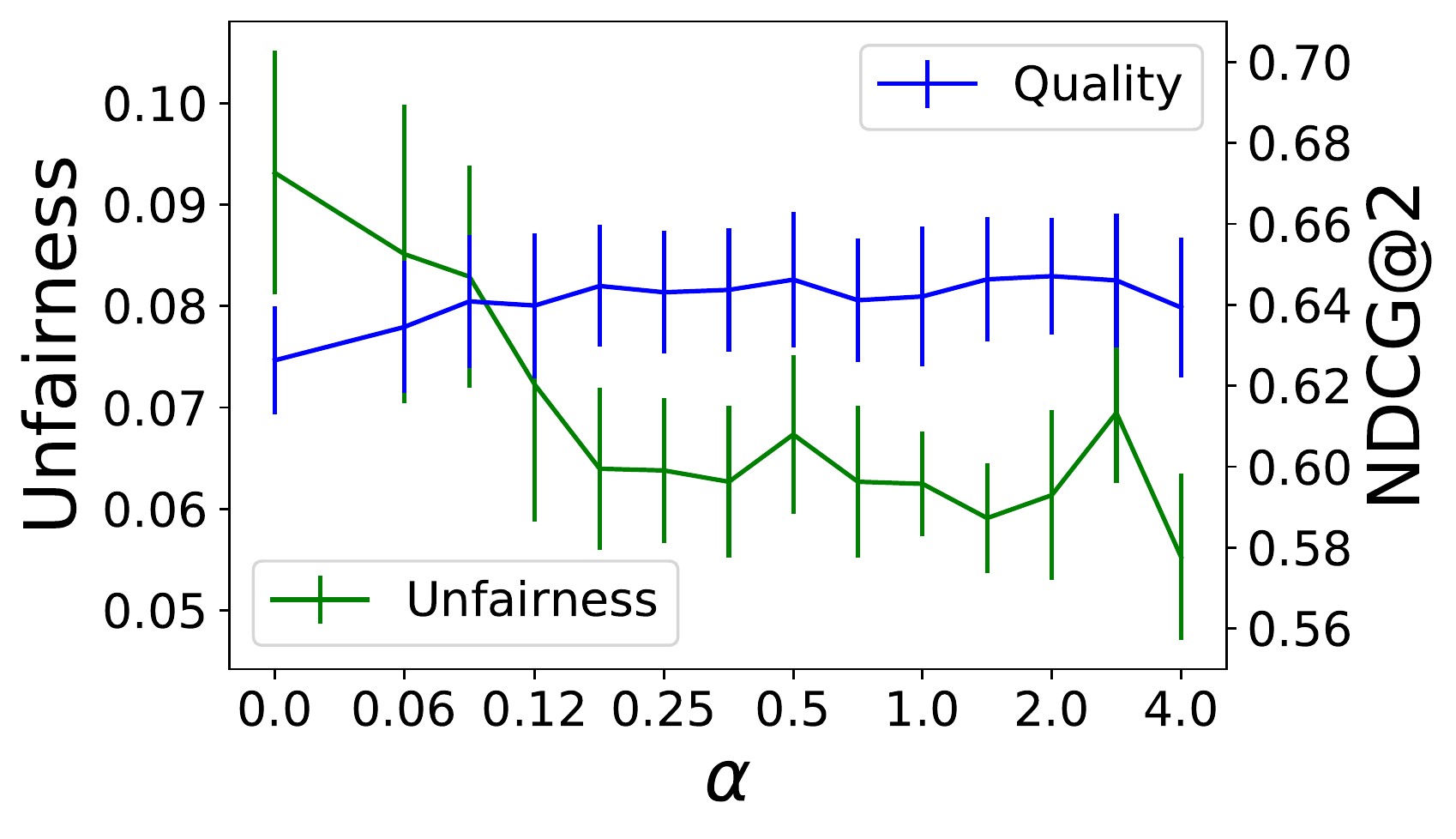}
  \caption{Demographic parity, TREC data}
\end{subfigure}%
\hfill
\begin{subfigure}{.33\textwidth}
  \centering
  \includegraphics[width=.95\linewidth]{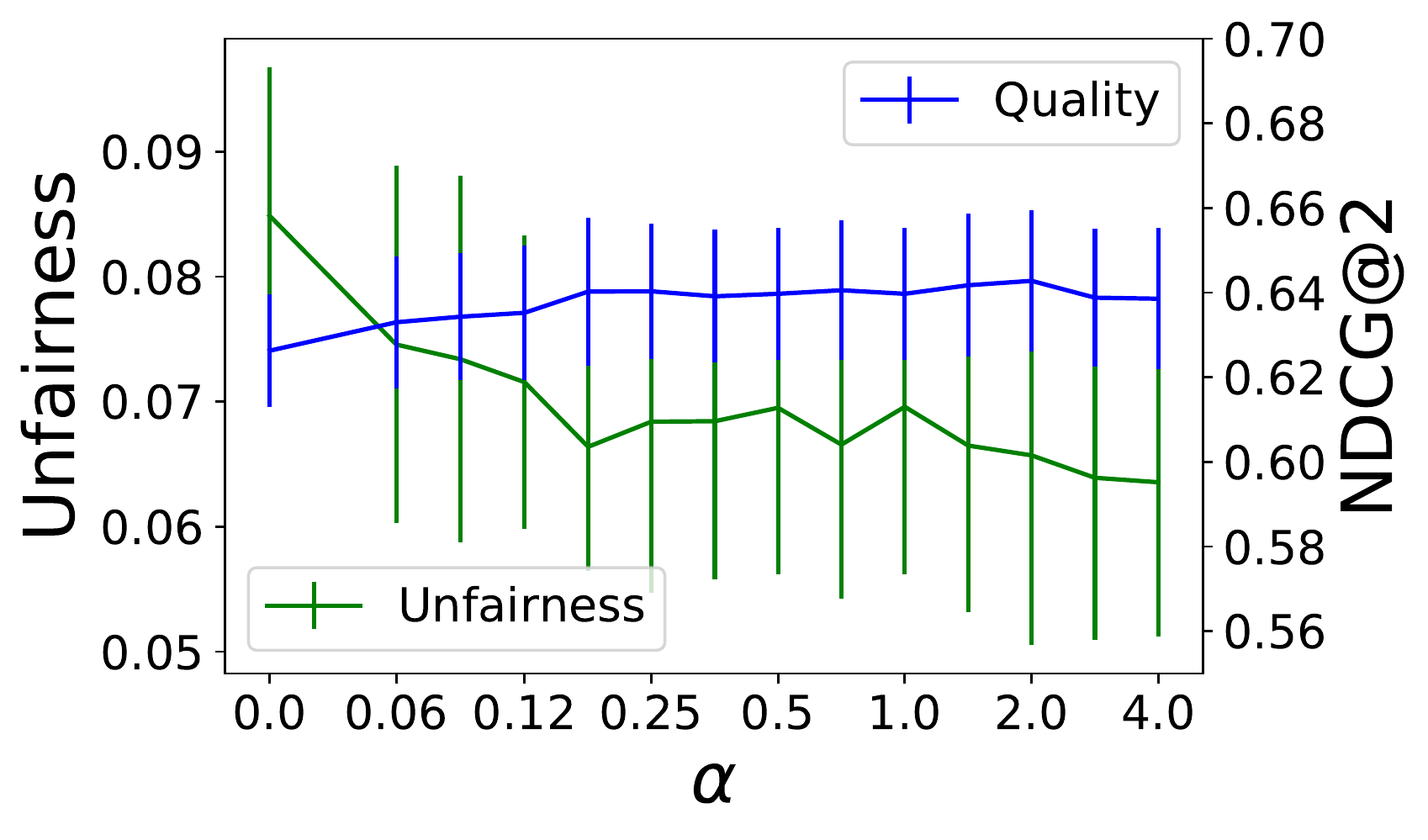}
  \caption{Equalized odds, TREC data}
\end{subfigure}%
\hfill
\begin{subfigure}{.33\textwidth}
  \centering
  \includegraphics[width=.95\linewidth]{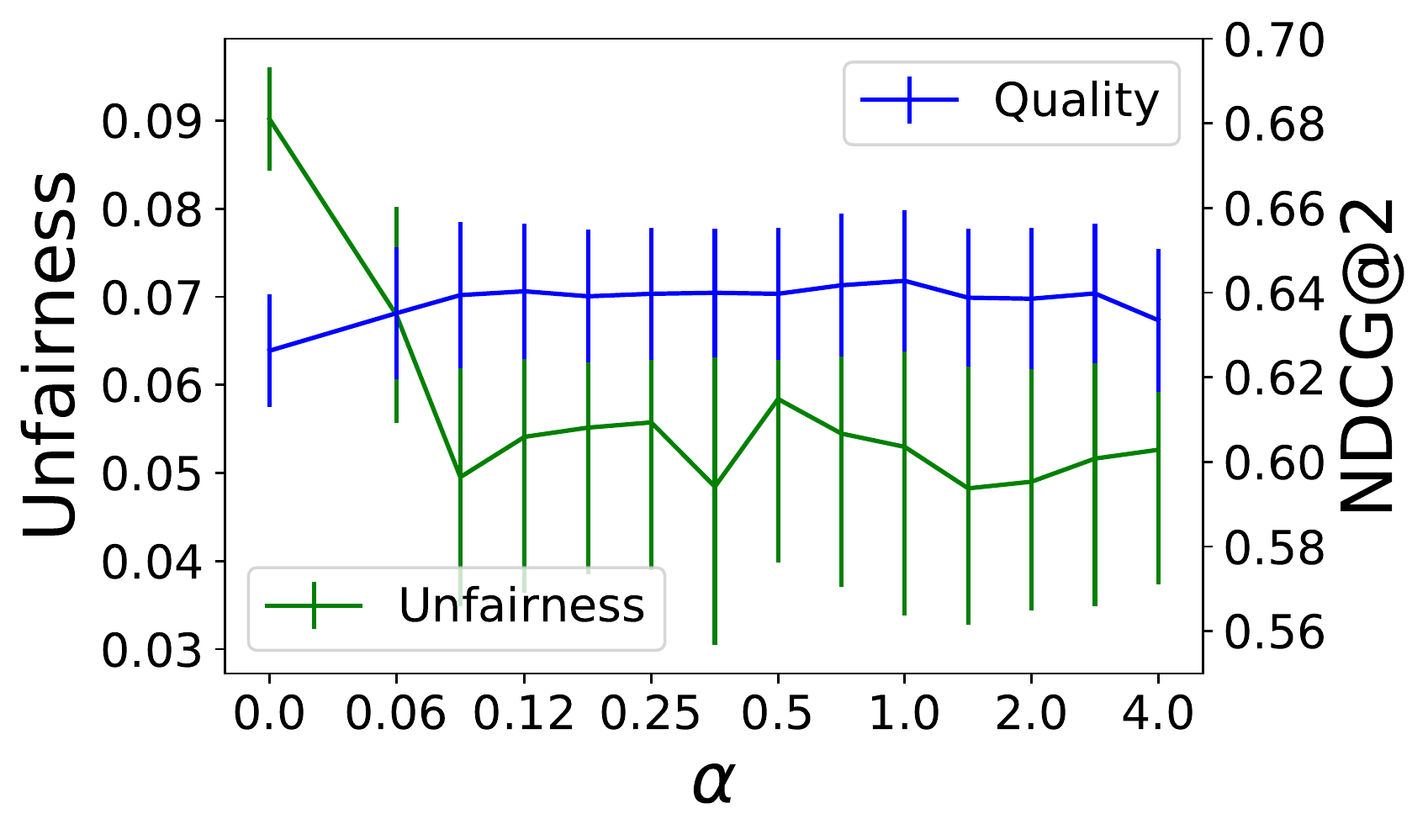}
  \caption{Equality of opportunity, TREC data}
\end{subfigure}
\caption{$k = 2, t = 4$}
\end{figure*}

\begin{figure*}[h]
\begin{subfigure}{.33\textwidth}
  \centering
  \includegraphics[width=.95\linewidth]{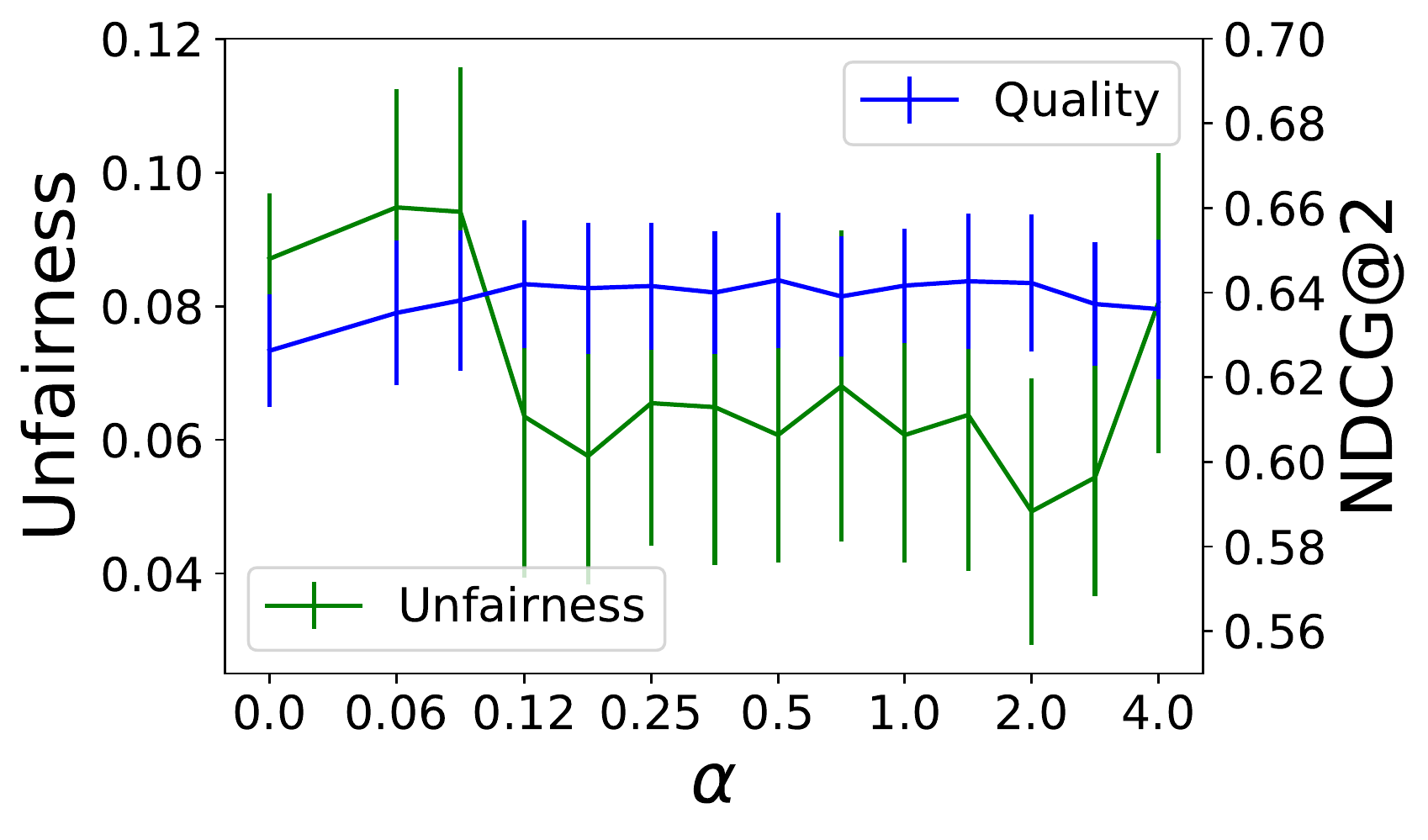}
  \caption{Demographic parity, TREC data}
\end{subfigure}%
\hfill
\begin{subfigure}{.33\textwidth}
  \centering
  \includegraphics[width=.95\linewidth]{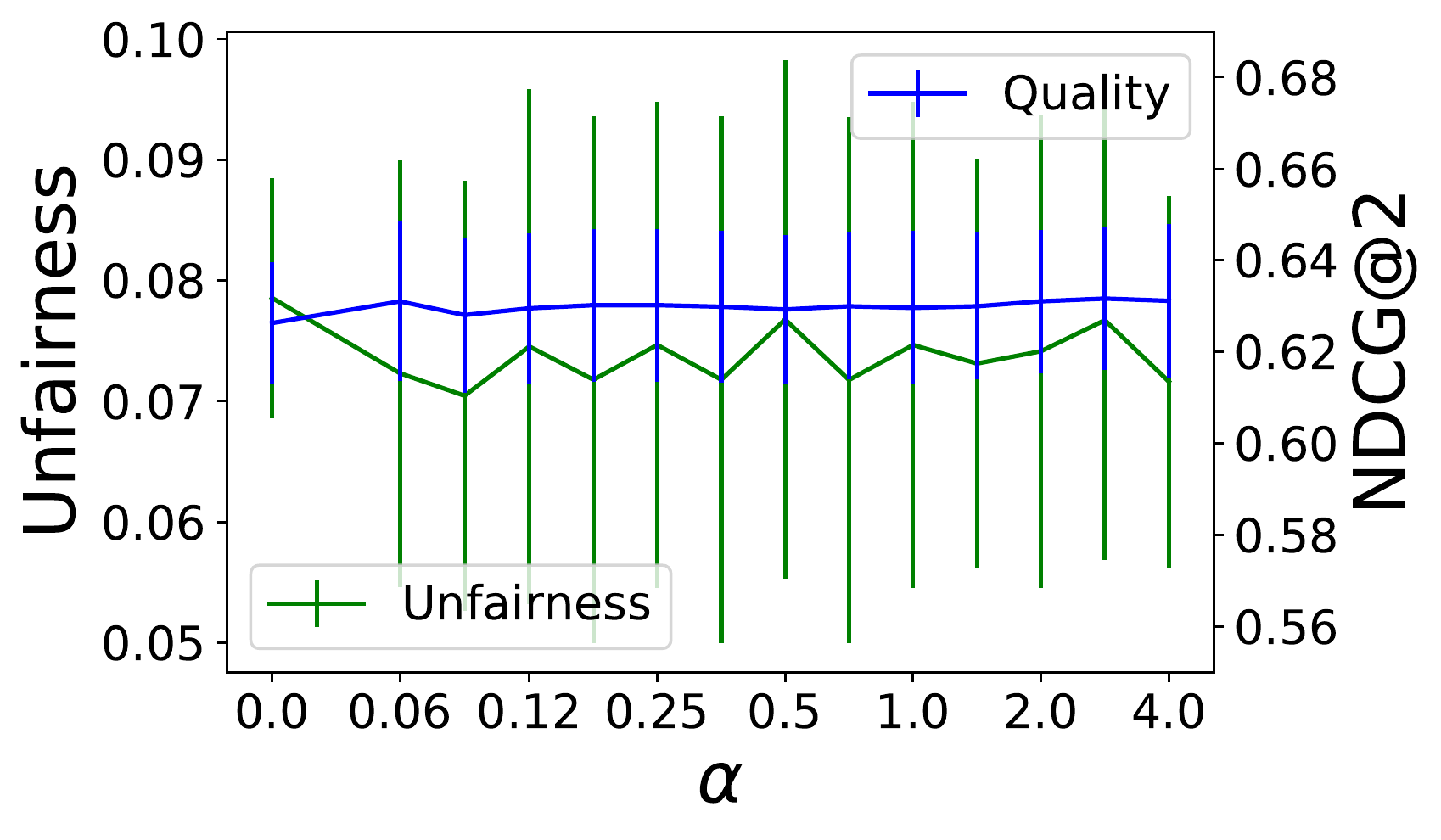}
  \caption{Equalized odds, TREC data}
\end{subfigure}%
\hfill
\begin{subfigure}{.33\textwidth}
  \centering
  \includegraphics[width=.95\linewidth]{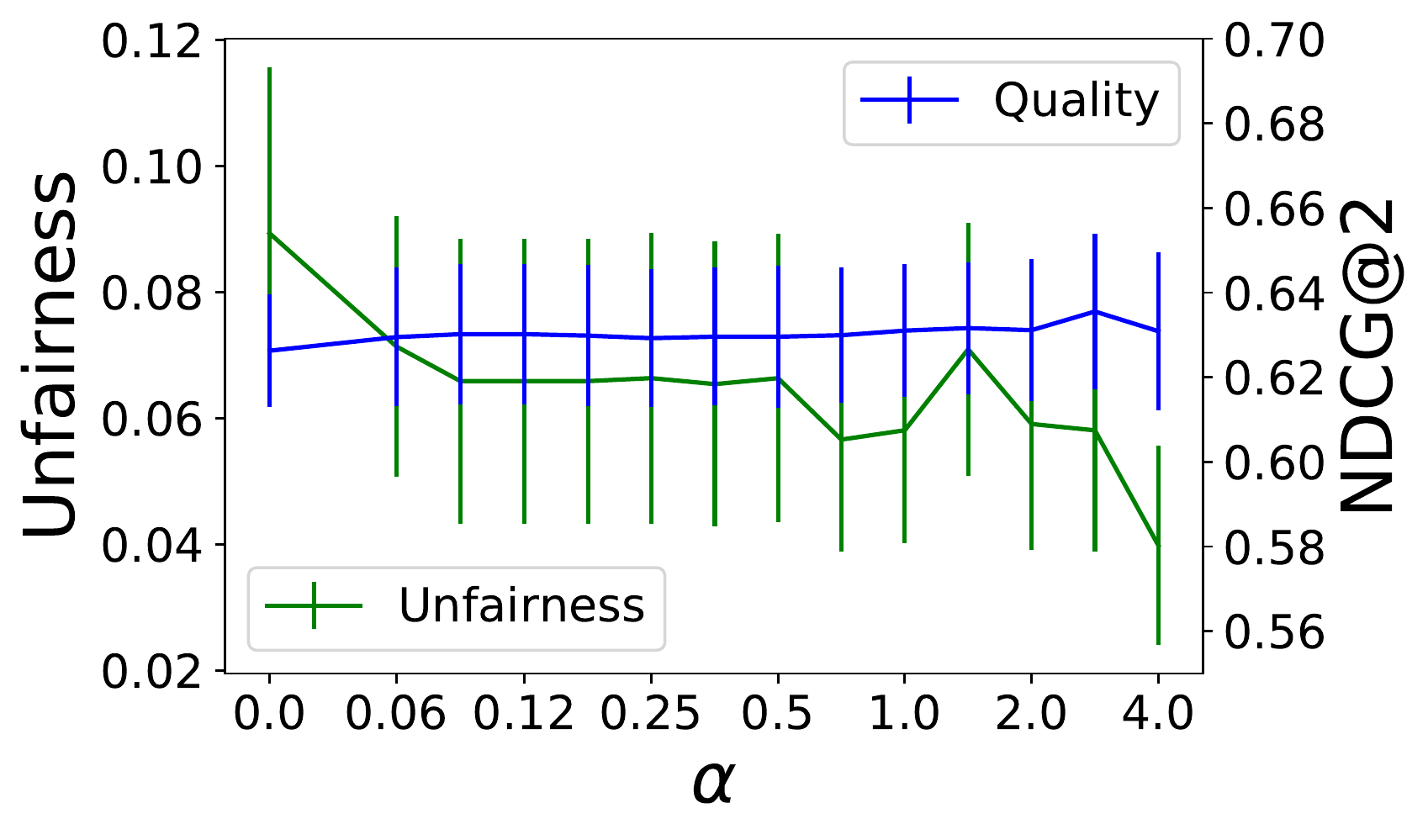}
  \caption{Equality of opportunity, TREC data}
\end{subfigure}
\caption{$k = 2, t = 5$}
\end{figure*}


\begin{figure*}[h]
\begin{subfigure}{.33\textwidth}
  \centering
  \includegraphics[width=.95\linewidth]{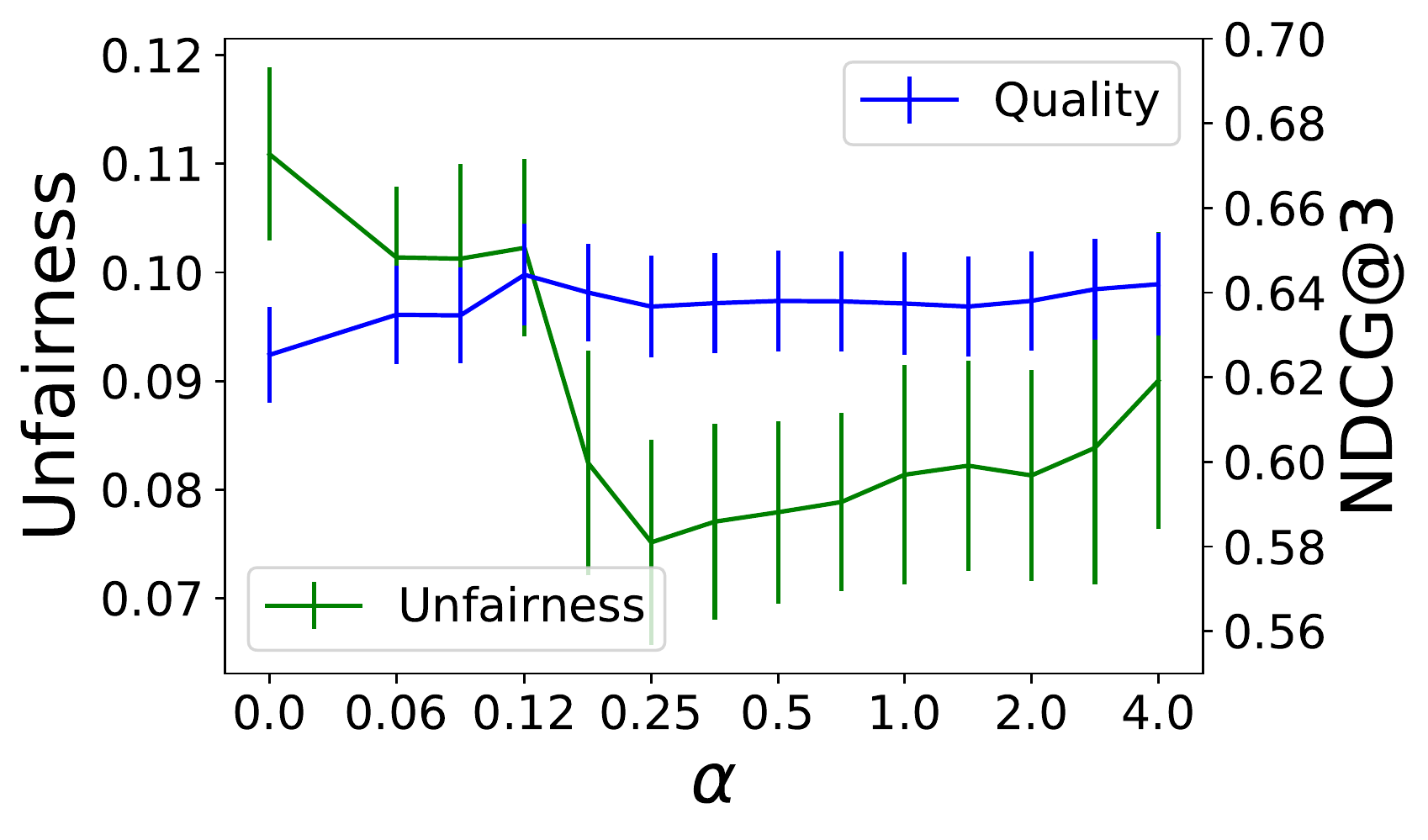}
  \caption{Demographic parity, TREC data}
\end{subfigure}%
\hfill
\begin{subfigure}{.33\textwidth}
  \centering
  \includegraphics[width=.95\linewidth]{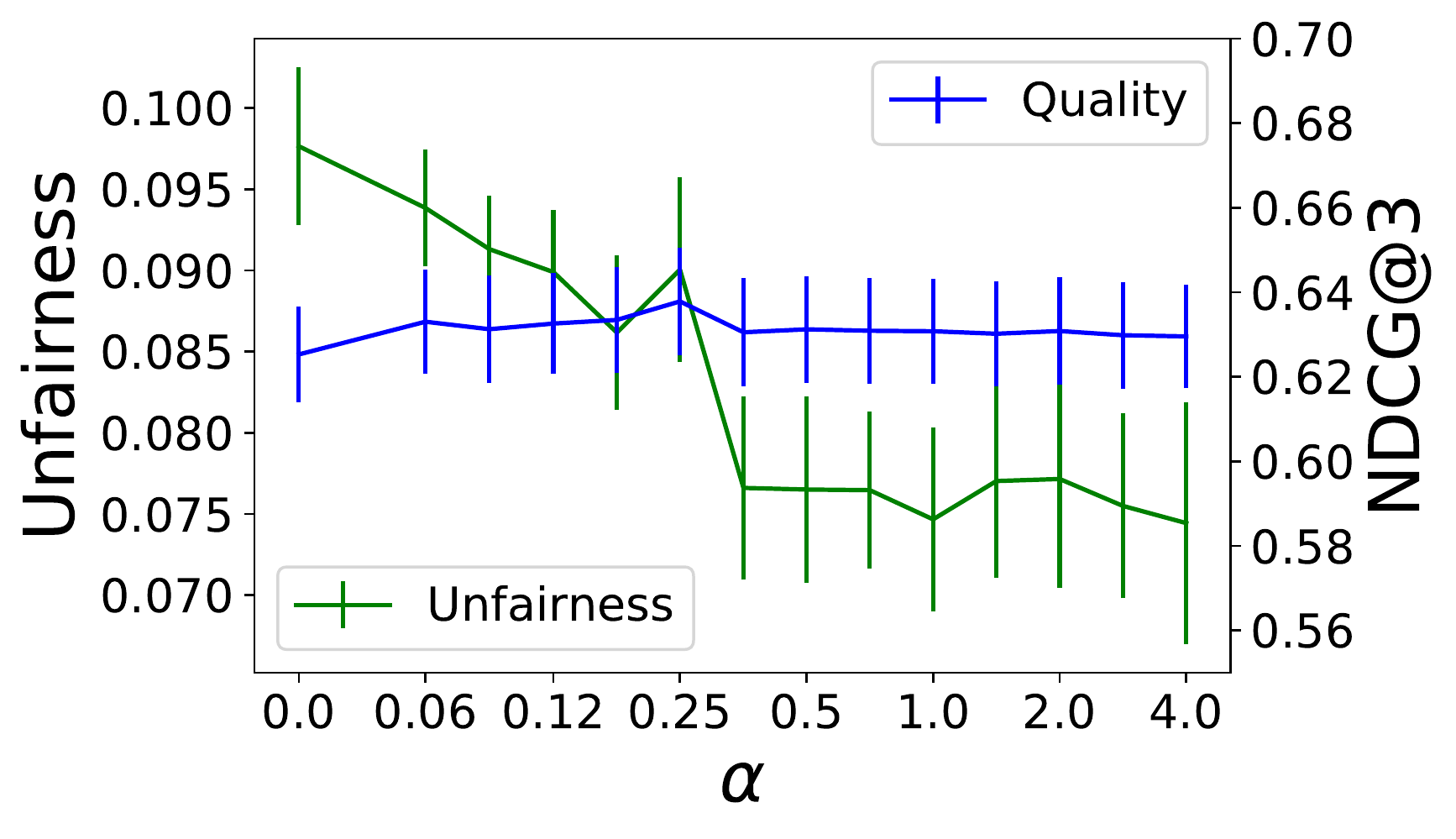}
  \caption{Equalized odds, TREC data}
\end{subfigure}%
\hfill
\begin{subfigure}{.33\textwidth}
  \centering
  \includegraphics[width=.95\linewidth]{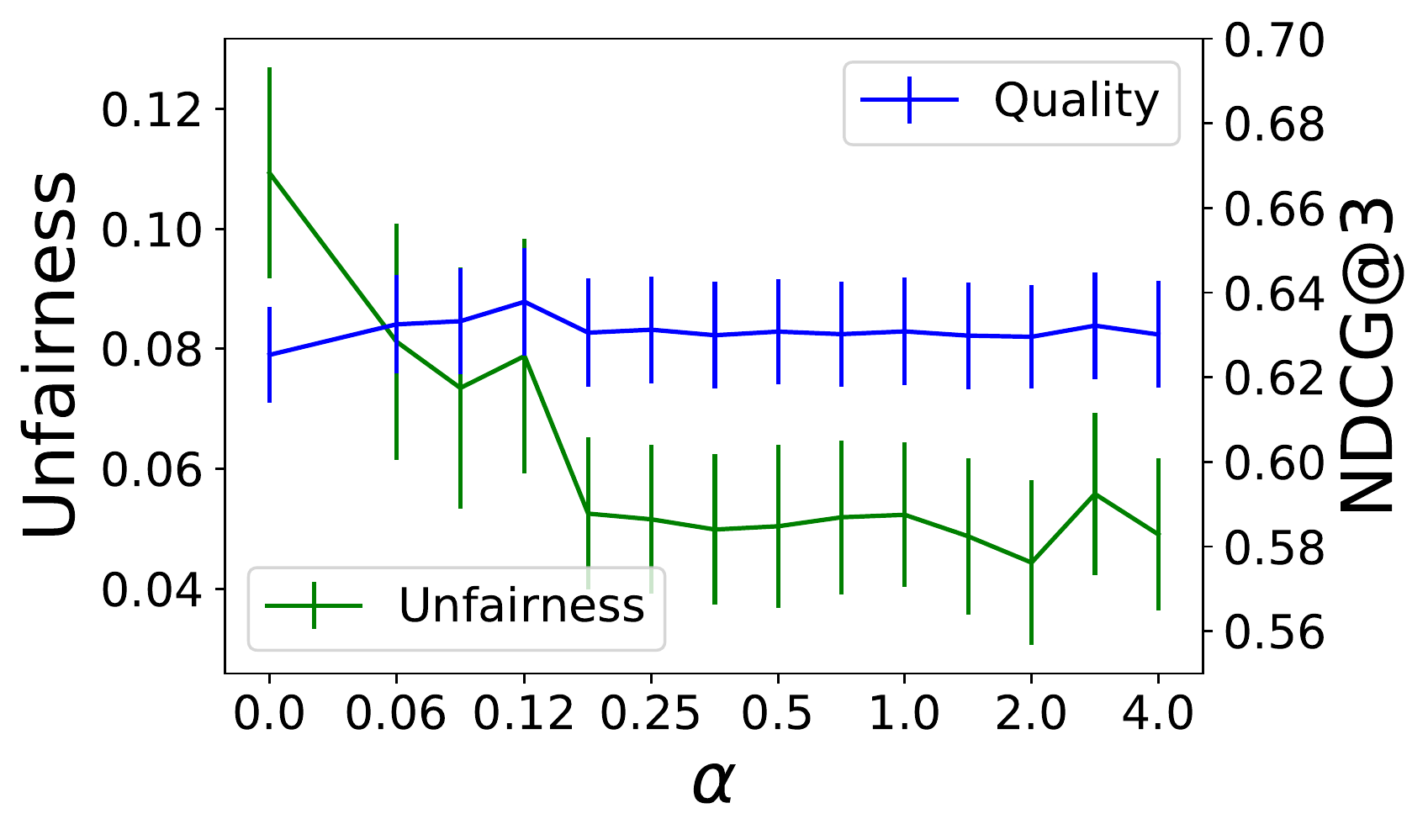}
  \caption{Equality of opportunity, TREC data}
\end{subfigure}
\caption{$k = 3, t = 3$}
\end{figure*}

\begin{figure*}[h]
\begin{subfigure}{.33\textwidth}
  \centering
  \includegraphics[width=.95\linewidth]{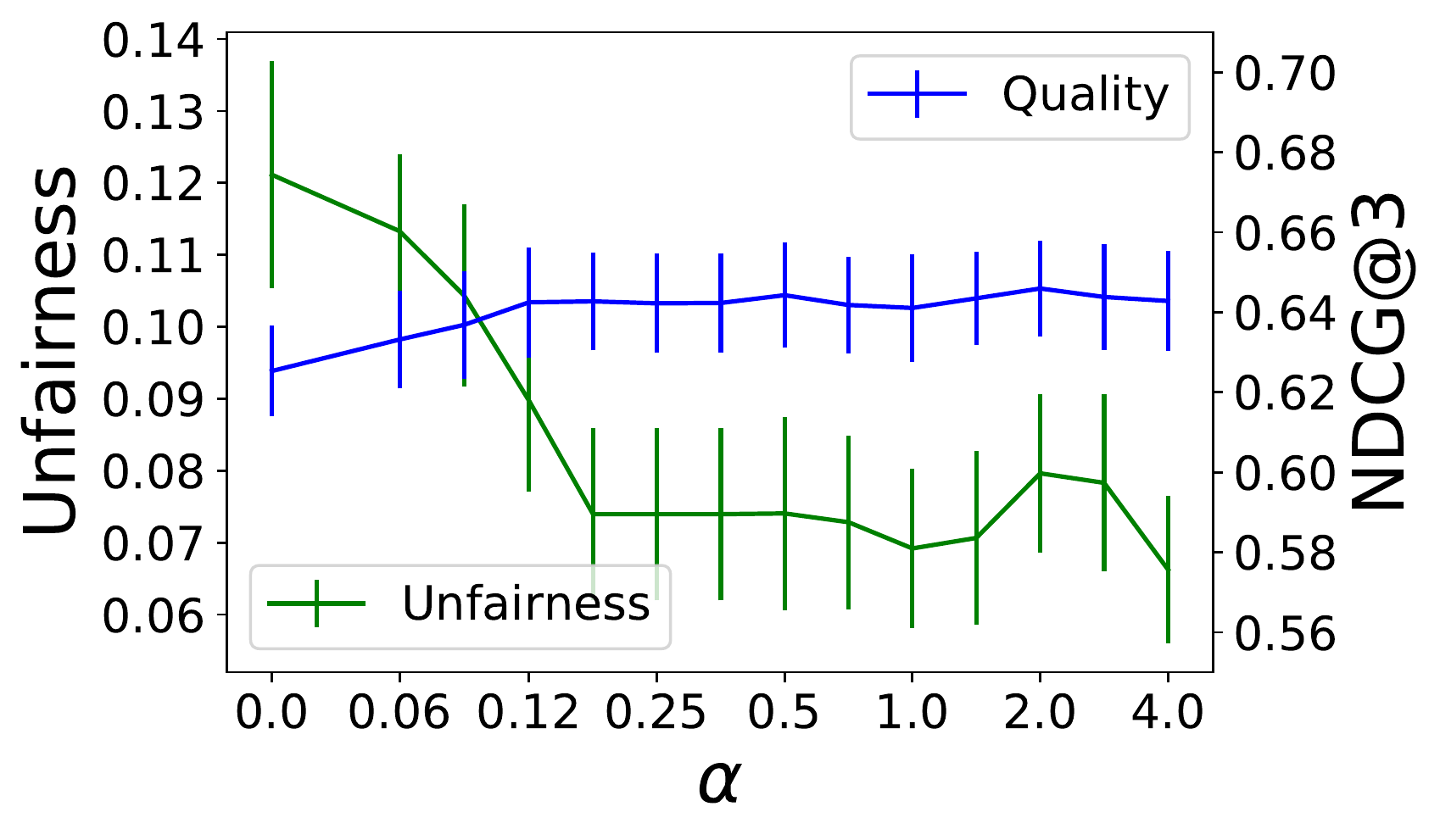}
  \caption{Demographic parity, TREC data}
\end{subfigure}%
\hfill
\begin{subfigure}{.33\textwidth}
  \centering
  \includegraphics[width=.95\linewidth]{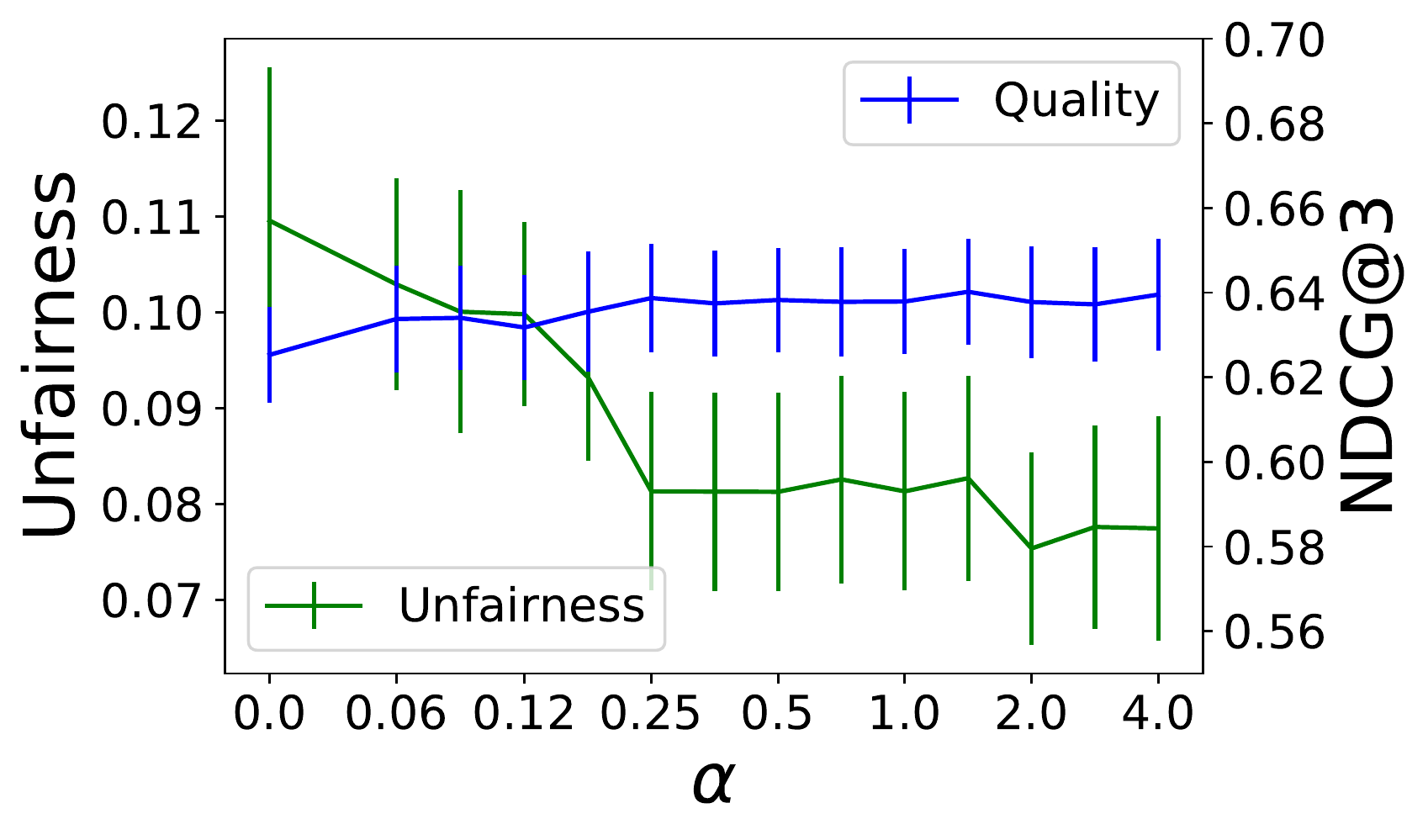}
  \caption{Equalized odds, TREC data}
\end{subfigure}%
\hfill
\begin{subfigure}{.33\textwidth}
  \centering
  \includegraphics[width=.95\linewidth]{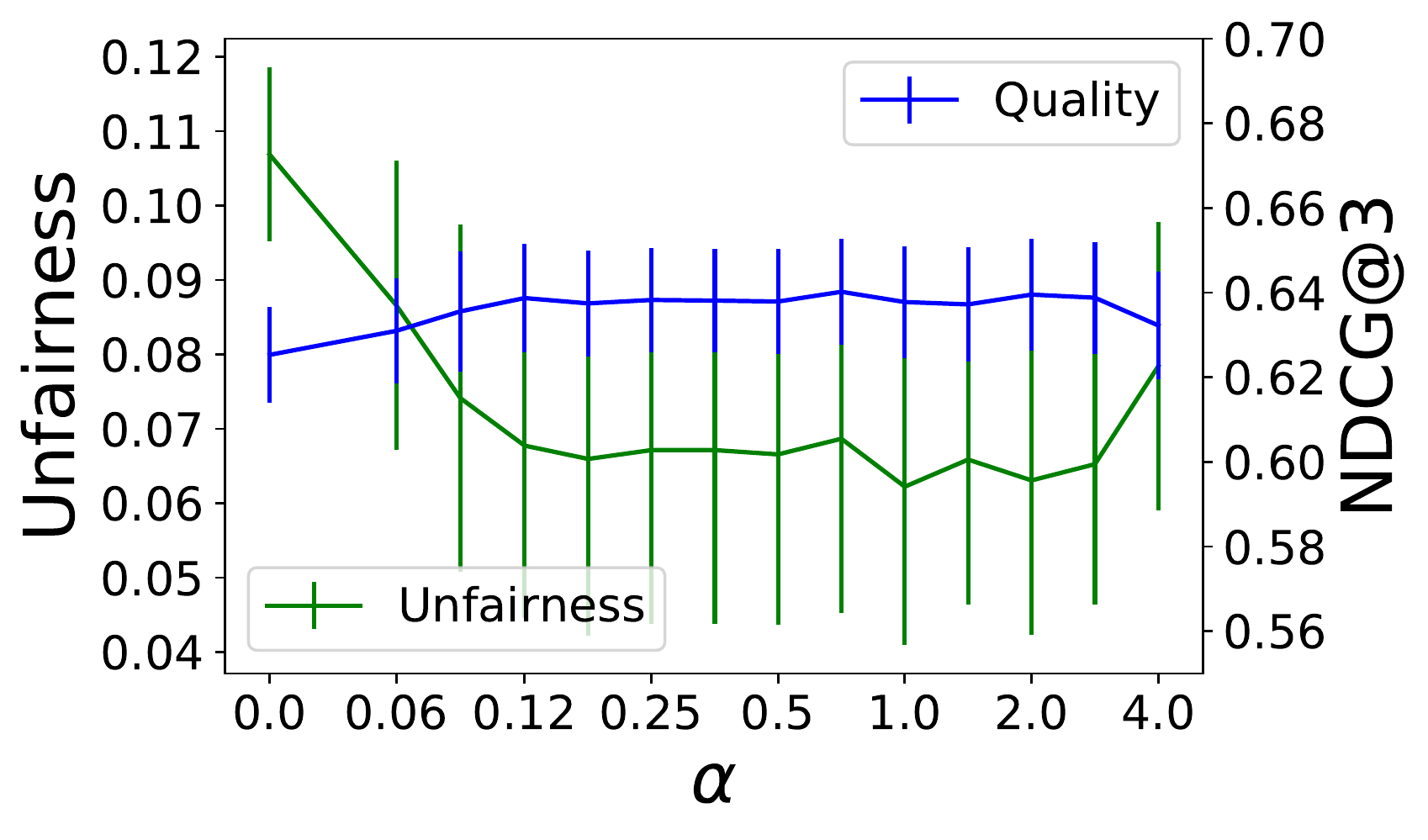}
  \caption{Equality of opportunity, TREC data}
\end{subfigure}
\caption{$k = 3, t = 4$}
\end{figure*}

\begin{figure*}[h]
\begin{subfigure}{.33\textwidth}
  \centering
  \includegraphics[width=.95\linewidth]{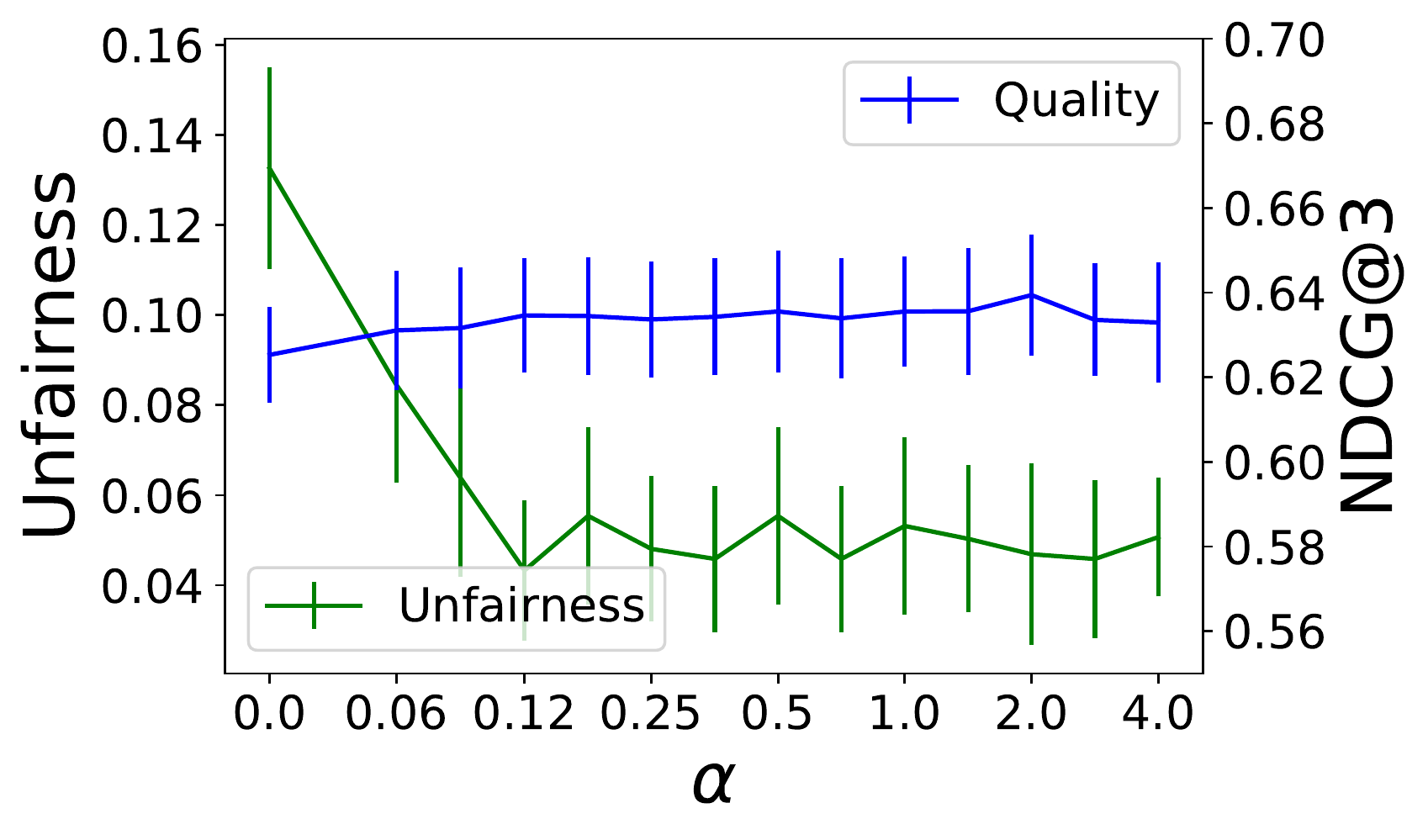}
  \caption{Demographic parity, TREC data}
\end{subfigure}%
\hfill
\begin{subfigure}{.33\textwidth}
  \centering
  \includegraphics[width=.95\linewidth]{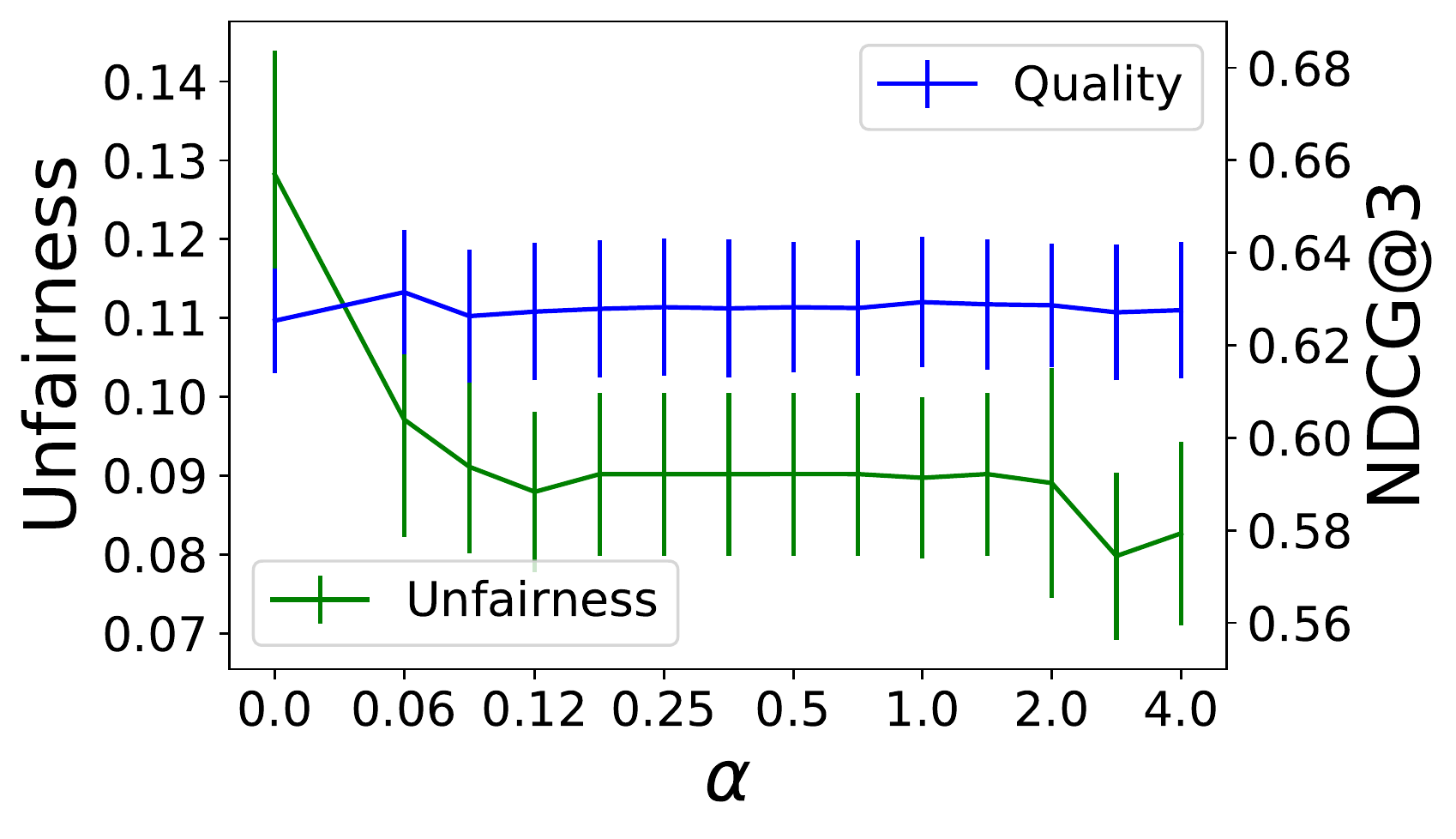}
  \caption{Equalized odds, TREC data}
\end{subfigure}%
\hfill
\begin{subfigure}{.33\textwidth}
  \centering
  \includegraphics[width=.95\linewidth]{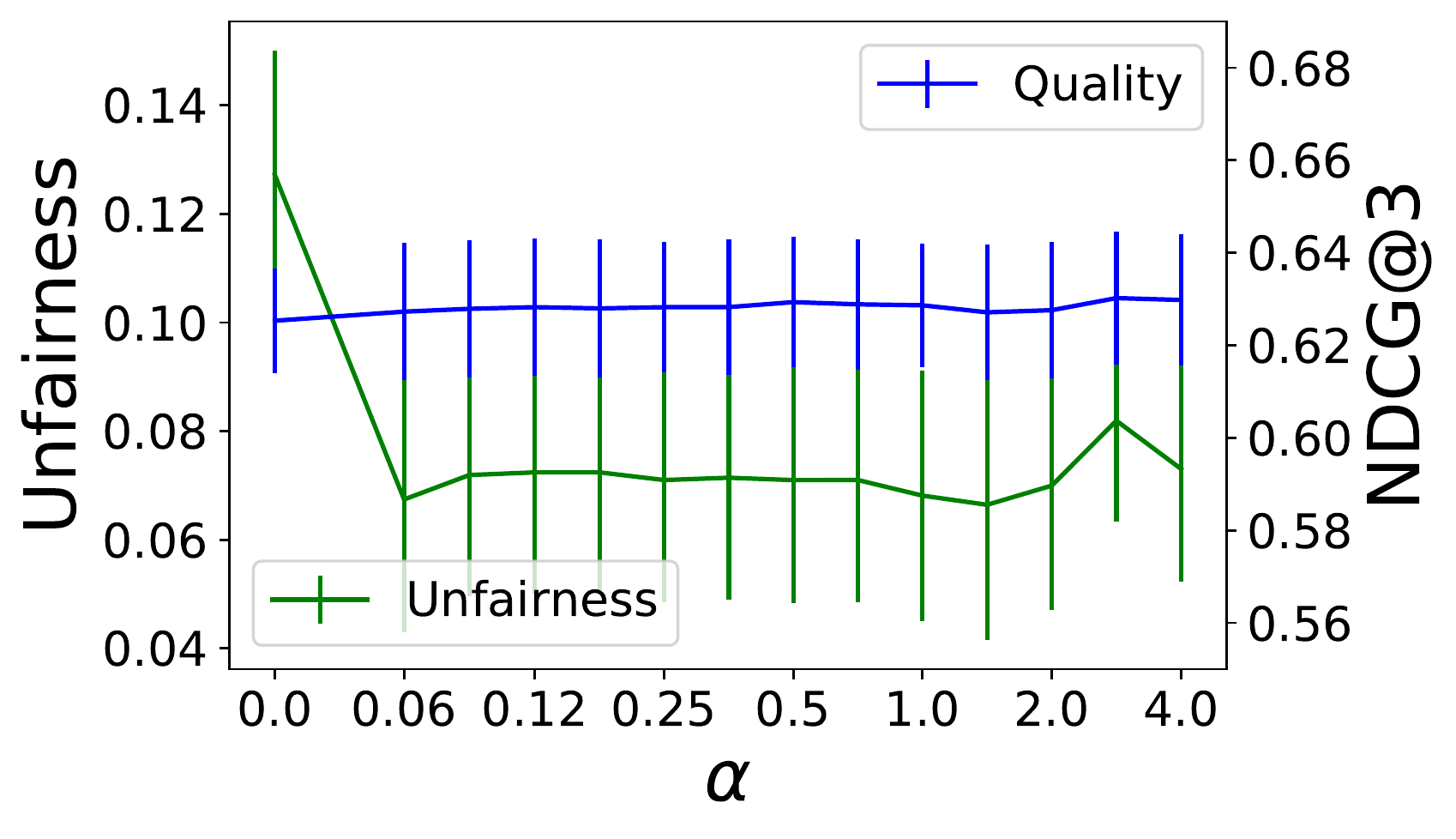}
  \caption{Equality of opportunity, TREC data}
\end{subfigure}
\caption{$k = 3, t = 5$}
\end{figure*}


\begin{figure*}[h]
\begin{subfigure}{.33\textwidth}
  \centering
  \includegraphics[width=.95\linewidth]{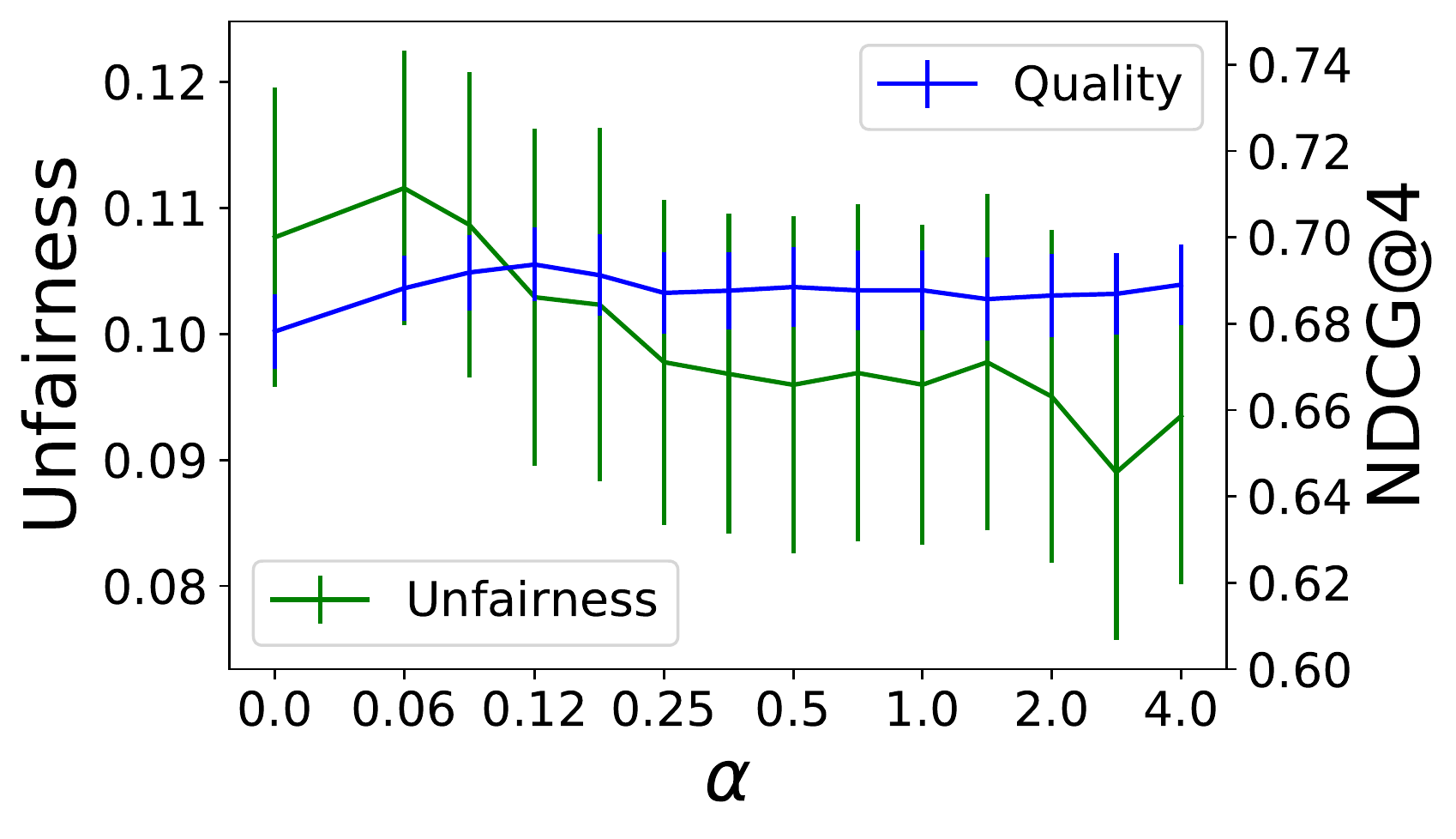}
  \caption{Demographic parity, TREC data}
\end{subfigure}%
\hfill
\begin{subfigure}{.33\textwidth}
  \centering
  \includegraphics[width=.95\linewidth]{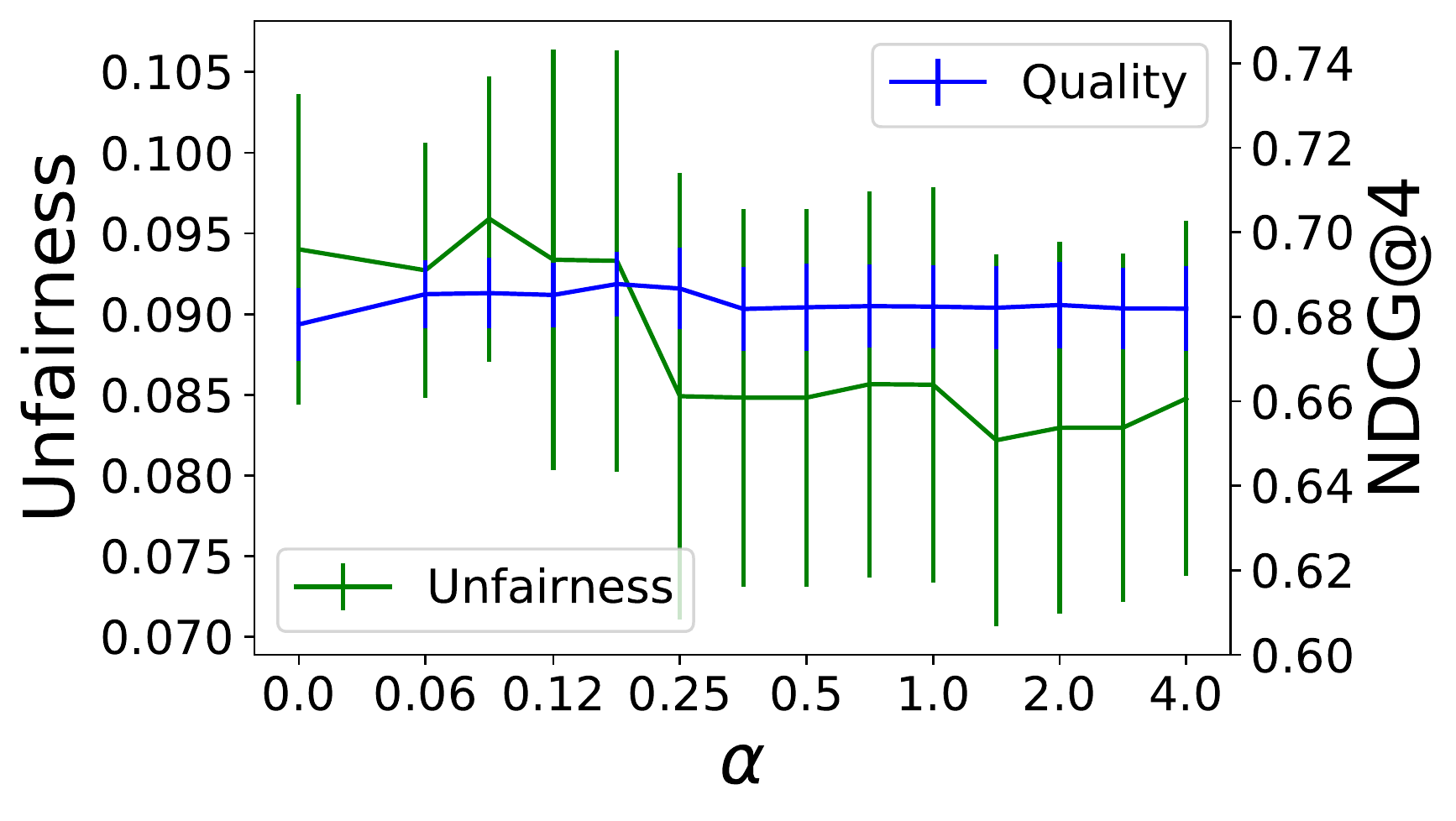}
  \caption{Equalized odds, TREC data}
\end{subfigure}%
\hfill
\begin{subfigure}{.33\textwidth}
  \centering
  \includegraphics[width=.95\linewidth]{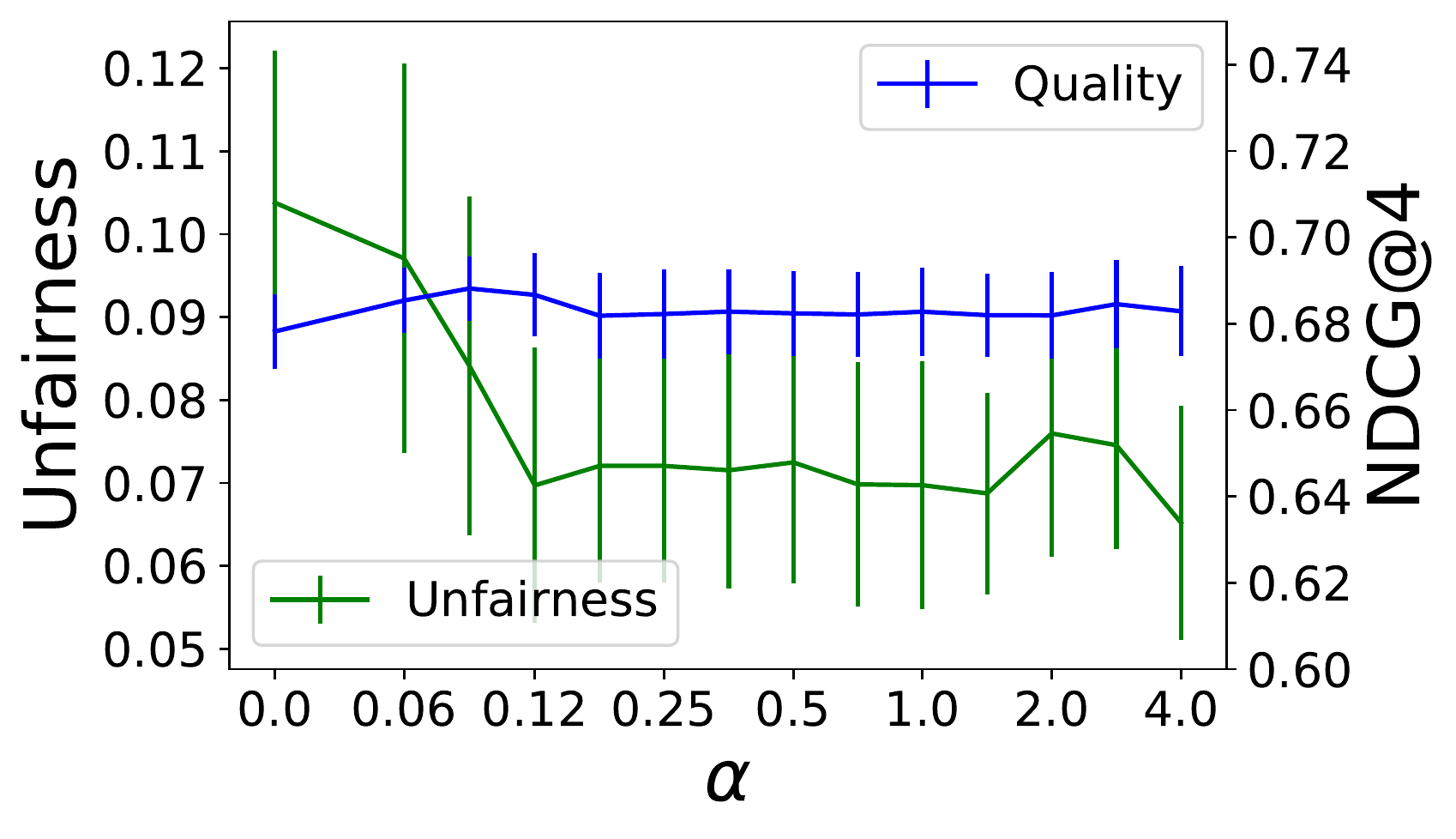}
  \caption{Equality of opportunity, TREC data}
\end{subfigure}
\caption{$k = 4, t = 3$}
\end{figure*}

\begin{figure*}[h]
\begin{subfigure}{.33\textwidth}
  \centering
  \includegraphics[width=.95\linewidth]{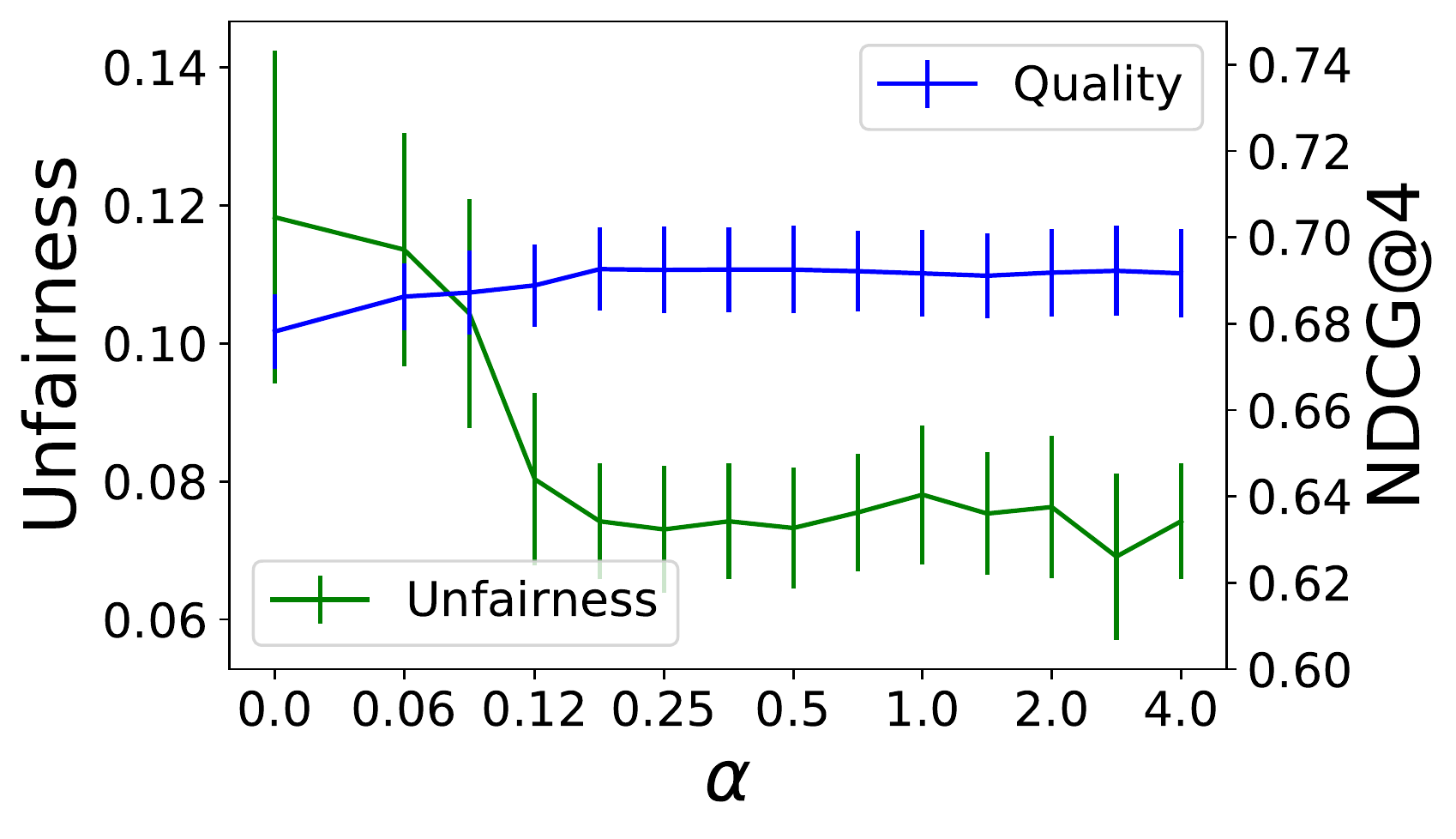}
  \caption{Demographic parity, TREC data}
\end{subfigure}%
\hfill
\begin{subfigure}{.33\textwidth}
  \centering
  \includegraphics[width=.95\linewidth]{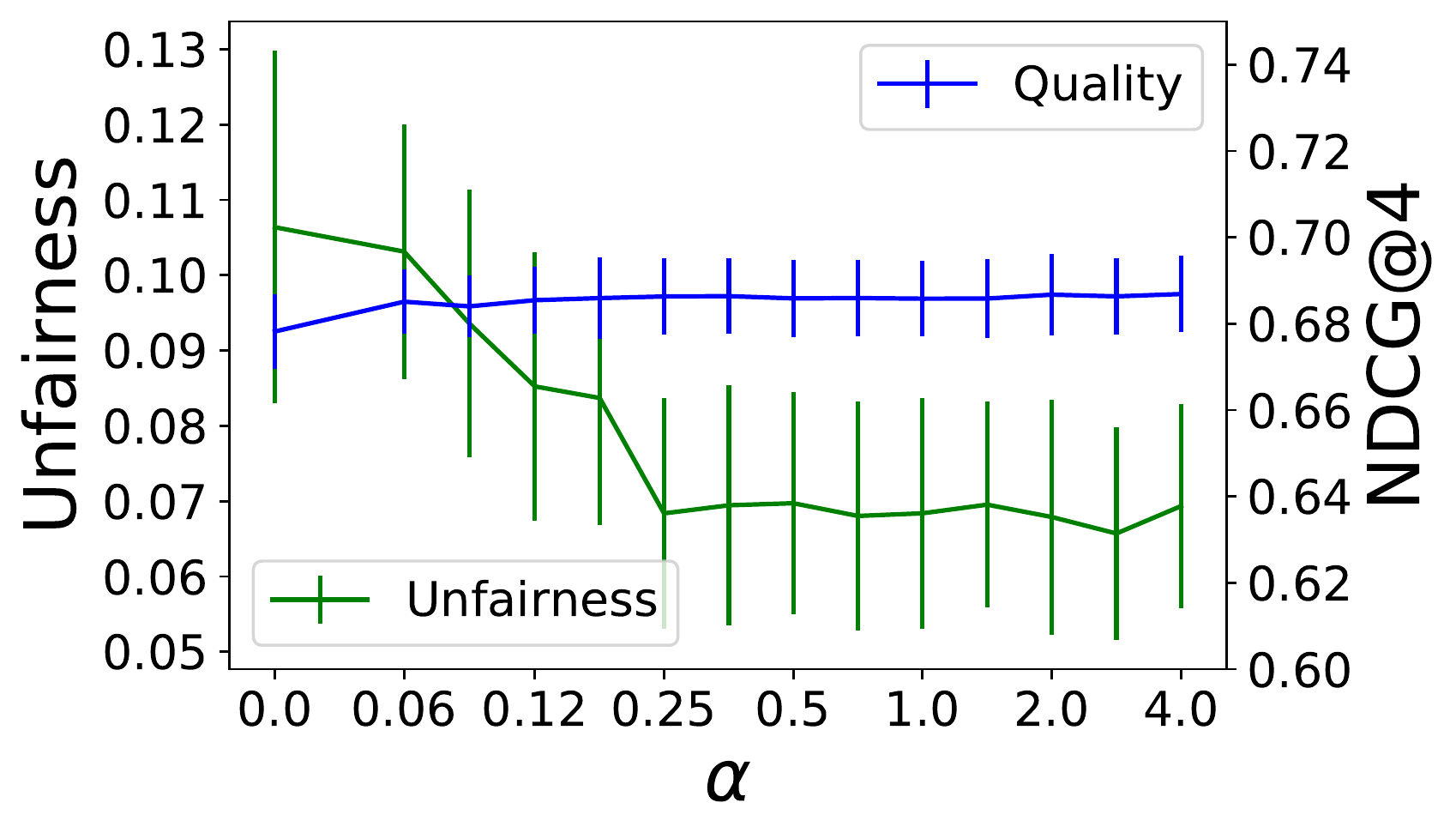}
  \caption{Equalized odds, TREC data}
\end{subfigure}%
\hfill
\begin{subfigure}{.33\textwidth}
  \centering
  \includegraphics[width=.95\linewidth]{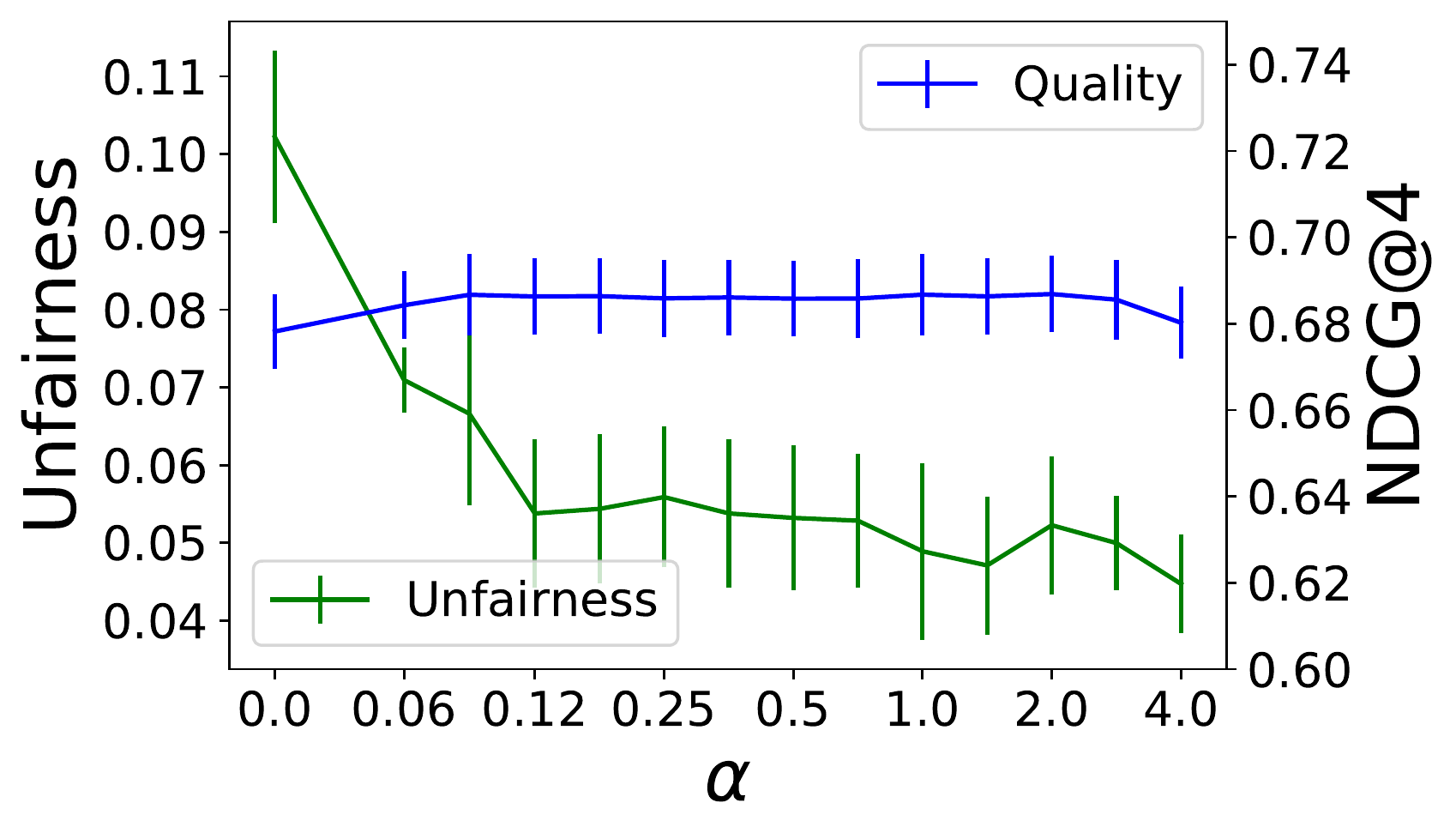}
  \caption{Equality of opportunity, TREC data}
\end{subfigure}
\caption{$k = 4, t = 4$}
\end{figure*}

\begin{figure*}[h]
\begin{subfigure}{.33\textwidth}
  \centering
  \includegraphics[width=.95\linewidth]{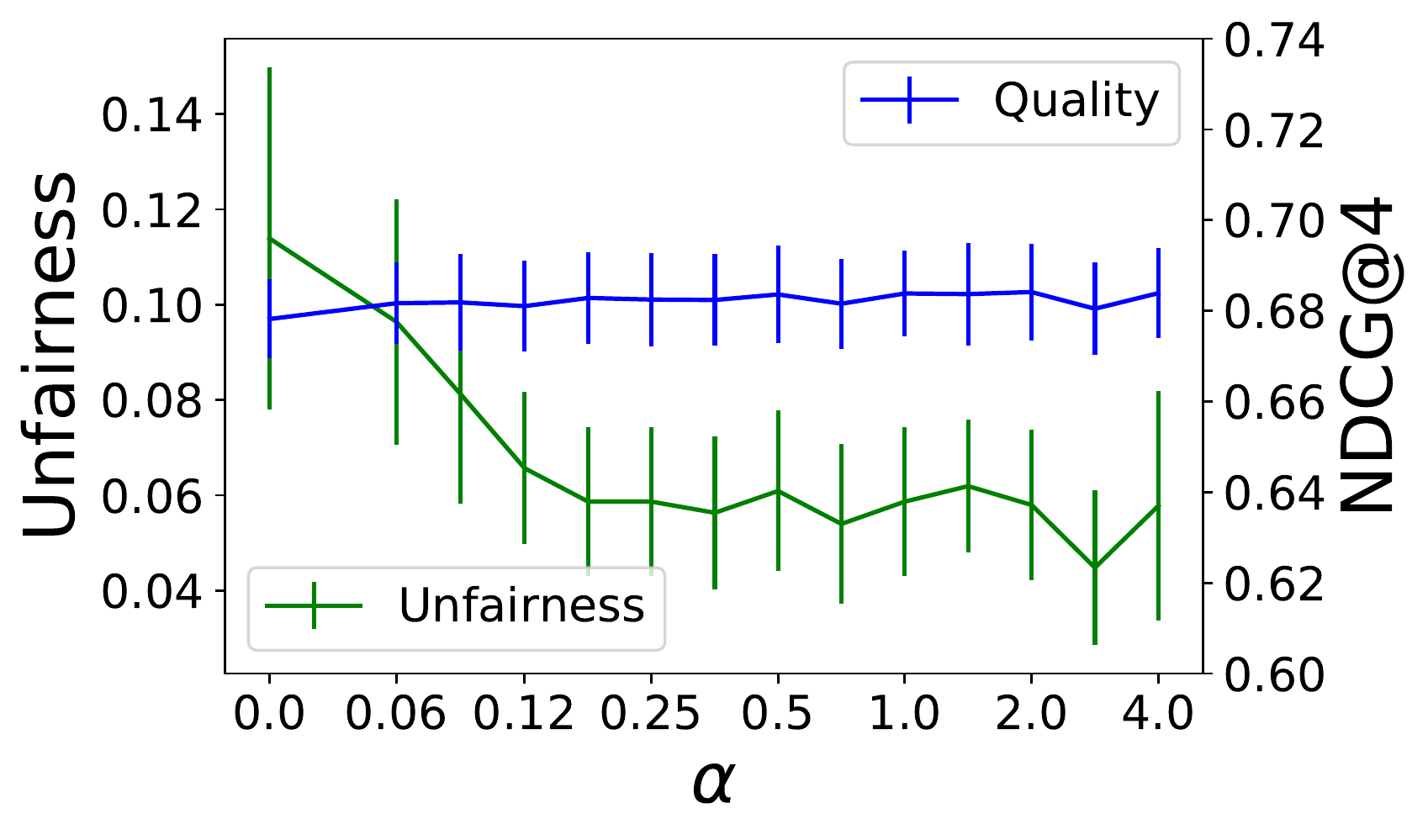}
  \caption{Demographic parity, TREC data}
\end{subfigure}%
\hfill
\begin{subfigure}{.33\textwidth}
  \centering
  \includegraphics[width=.95\linewidth]{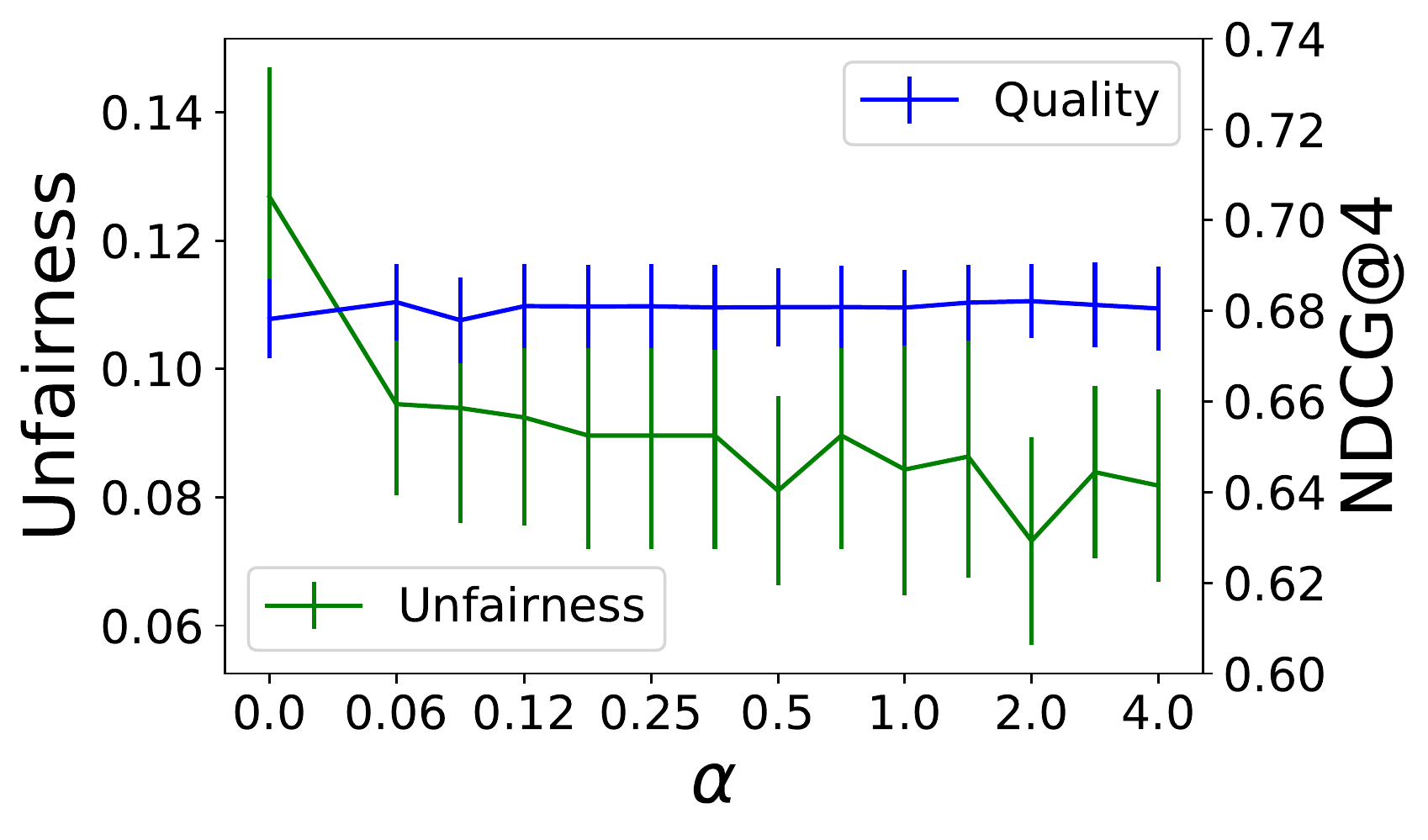}
  \caption{Equalized odds, TREC data}
\end{subfigure}%
\hfill
\begin{subfigure}{.33\textwidth}
  \centering
  \includegraphics[width=.95\linewidth]{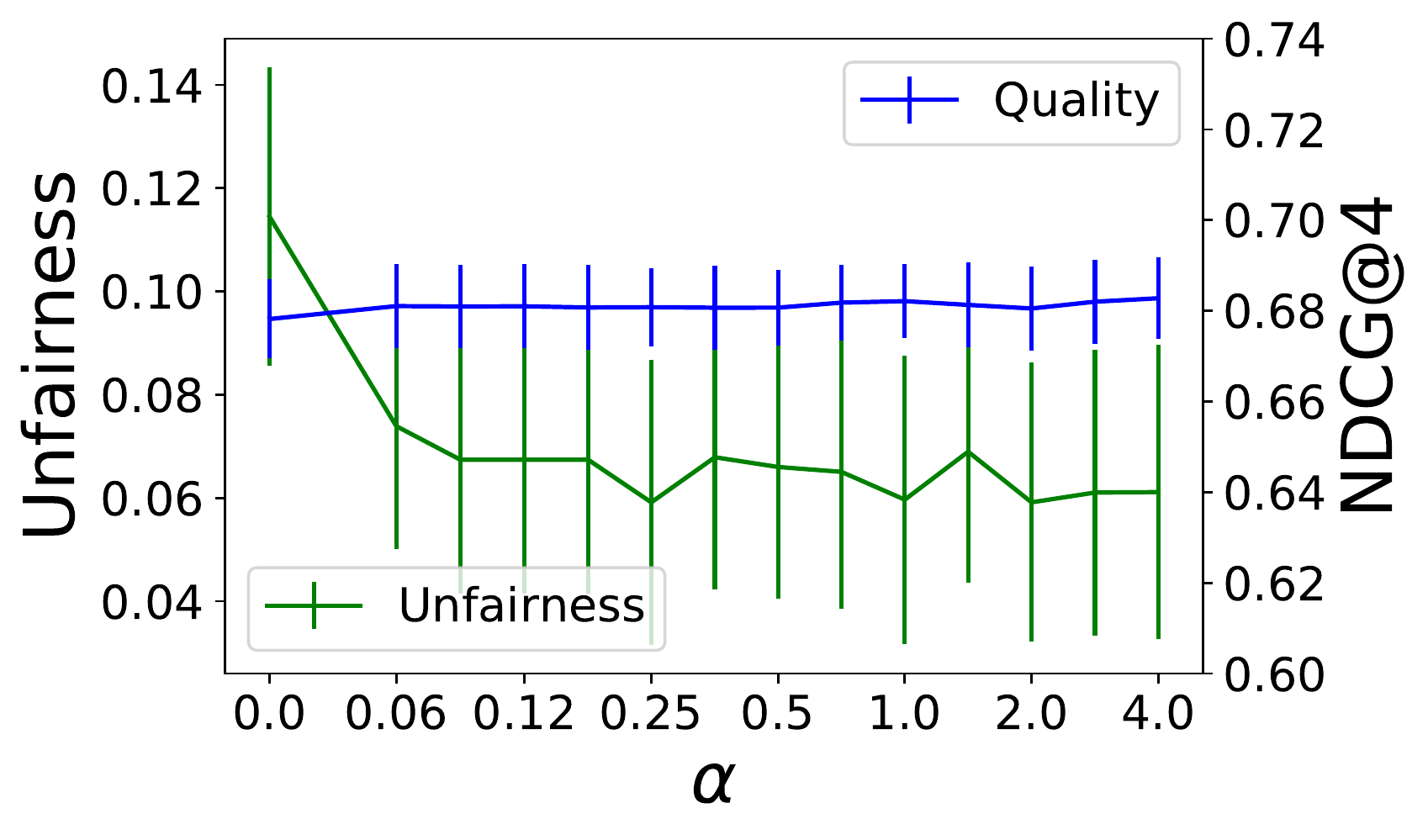}
  \caption{Equality of opportunity, TREC data}
\end{subfigure}
\caption{$k = 4, t = 5$}
\end{figure*}


\begin{figure*}[h]
\begin{subfigure}{.33\textwidth}
  \centering
  \includegraphics[width=.95\linewidth]{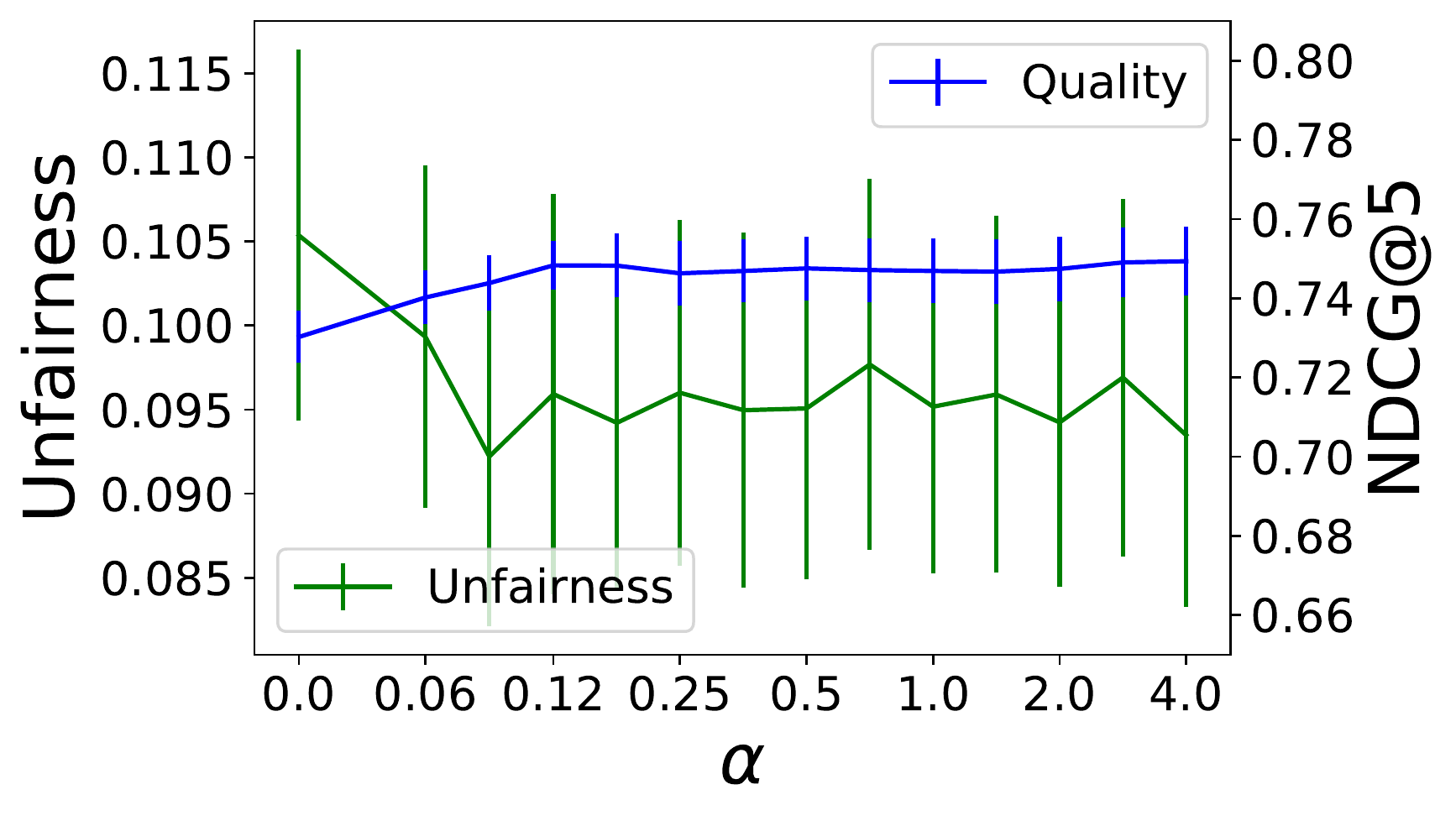}
  \caption{Demographic parity, TREC data}
\end{subfigure}%
\hfill
\begin{subfigure}{.33\textwidth}
  \centering
  \includegraphics[width=.95\linewidth]{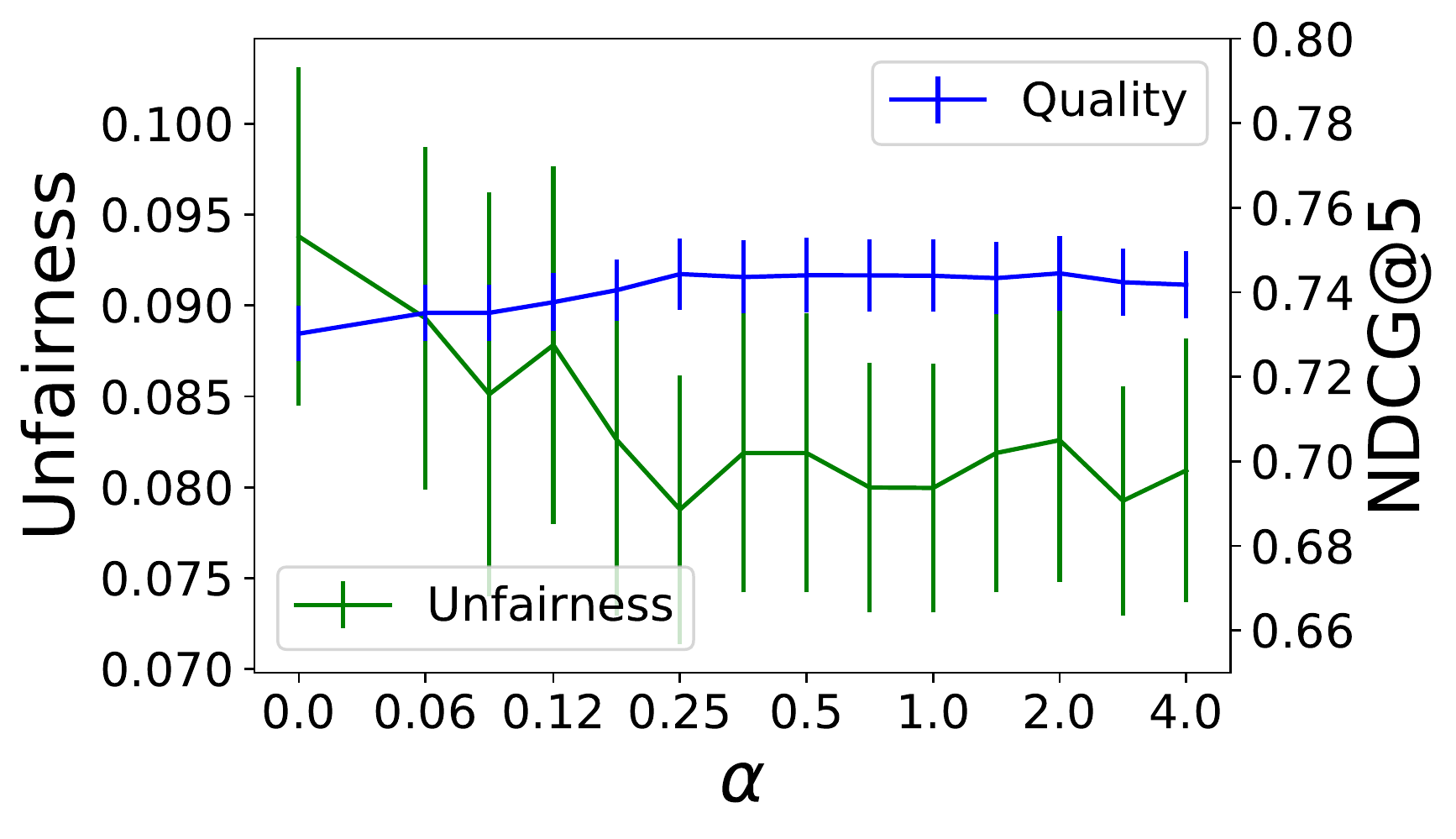}
  \caption{Equalized odds, TREC data}
\end{subfigure}%
\hfill
\begin{subfigure}{.33\textwidth}
  \centering
  \includegraphics[width=.95\linewidth]{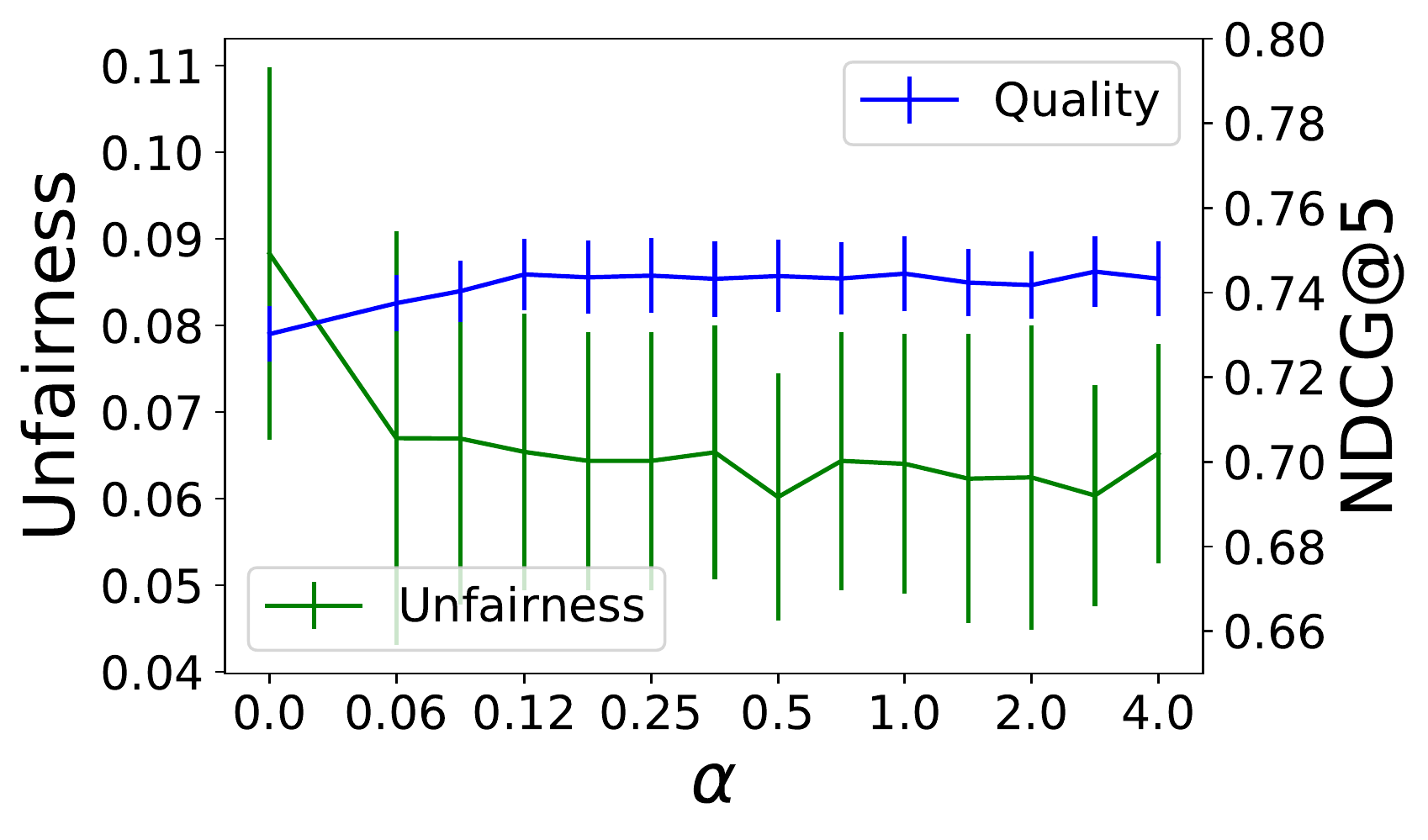}
  \caption{Equality of opportunity, TREC data}
\end{subfigure}
\caption{$k = 5, t = 3$}
\end{figure*}

\begin{figure*}[h]
\begin{subfigure}{.33\textwidth}
  \centering
  \includegraphics[width=.95\linewidth]{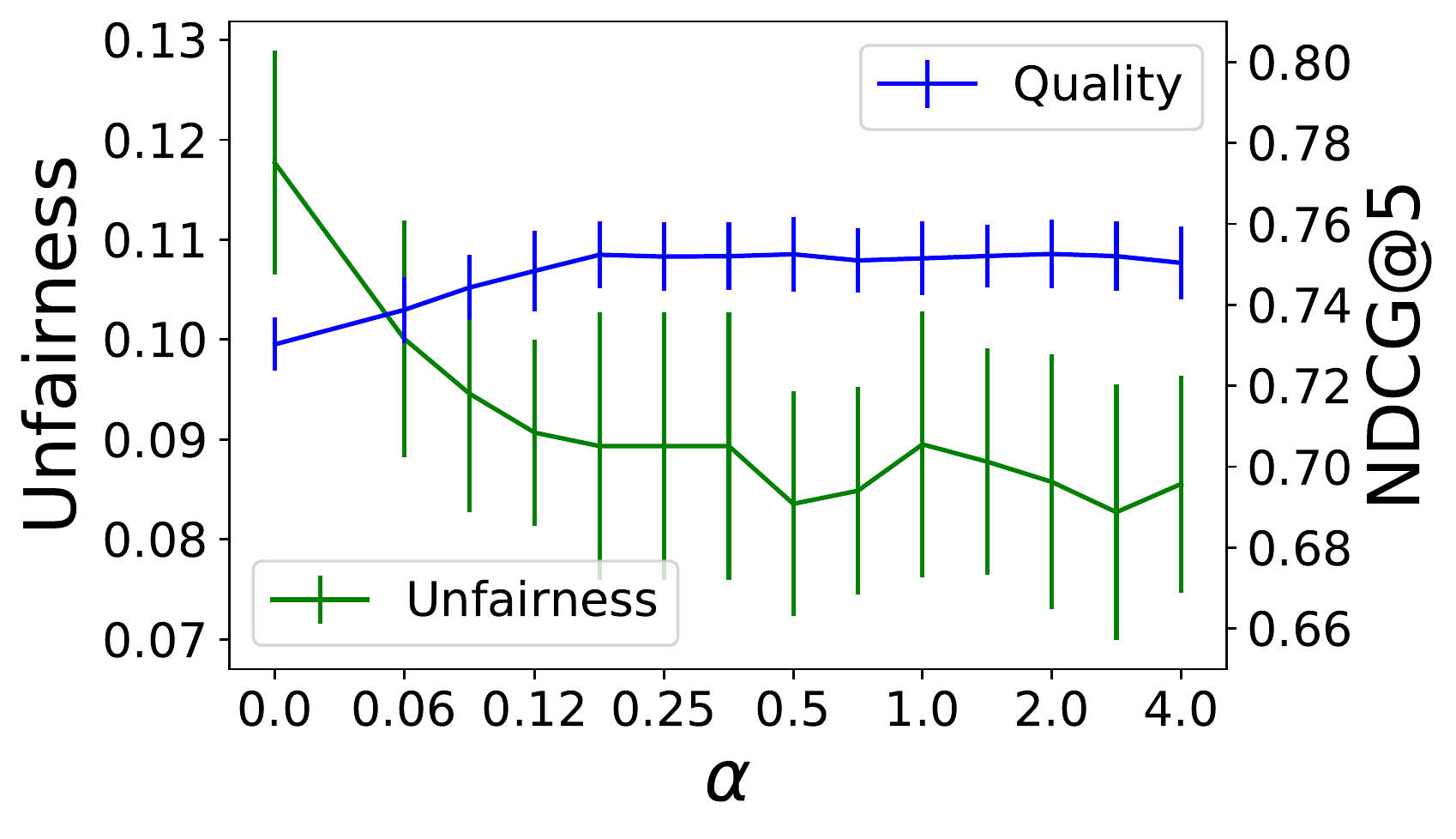}
  \caption{Demographic parity, TREC data}
\end{subfigure}%
\hfill
\begin{subfigure}{.33\textwidth}
  \centering
  \includegraphics[width=.95\linewidth]{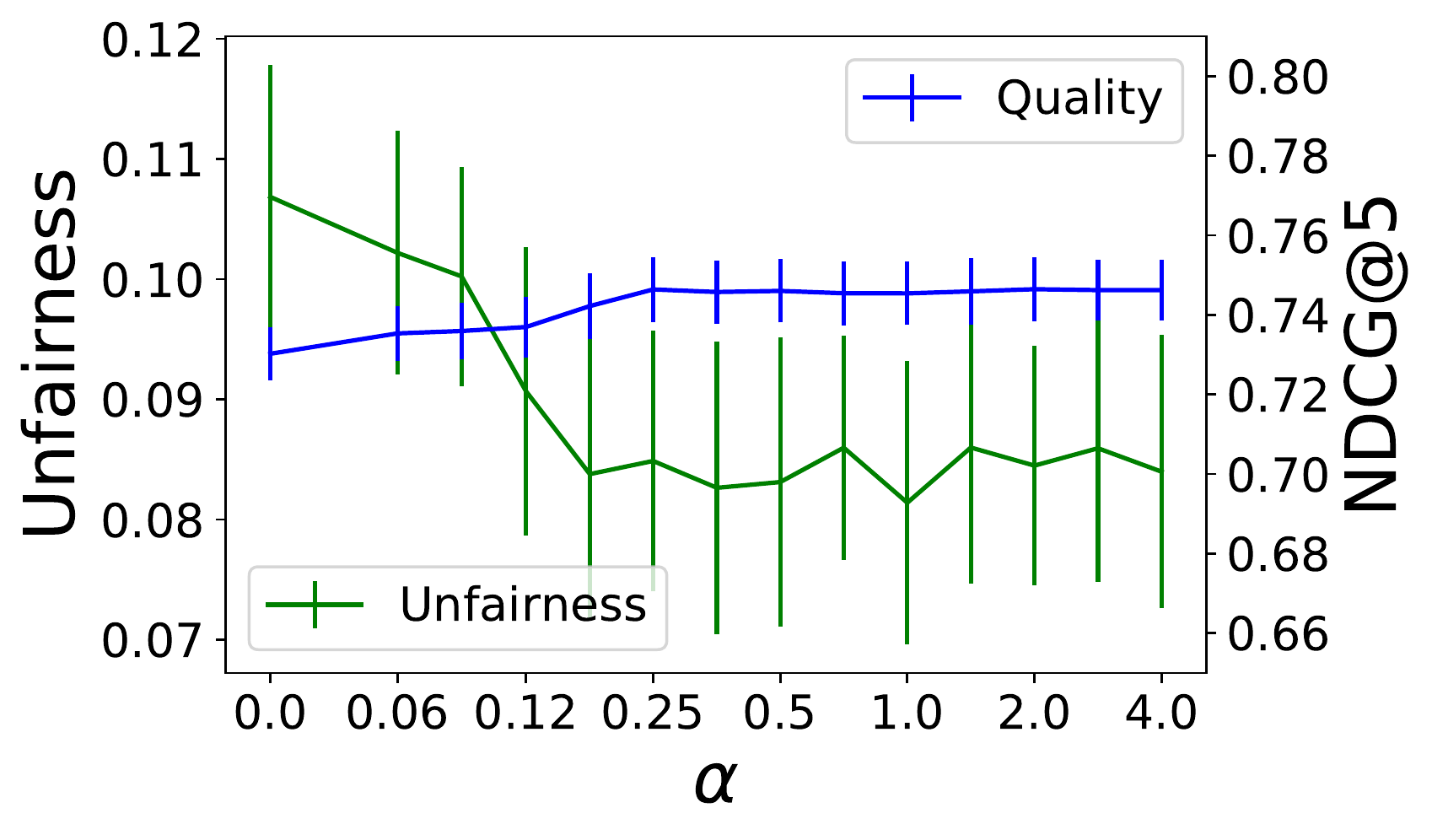}
  \caption{Equalized odds, TREC data}
\end{subfigure}%
\hfill
\begin{subfigure}{.33\textwidth}
  \centering
  \includegraphics[width=.95\linewidth]{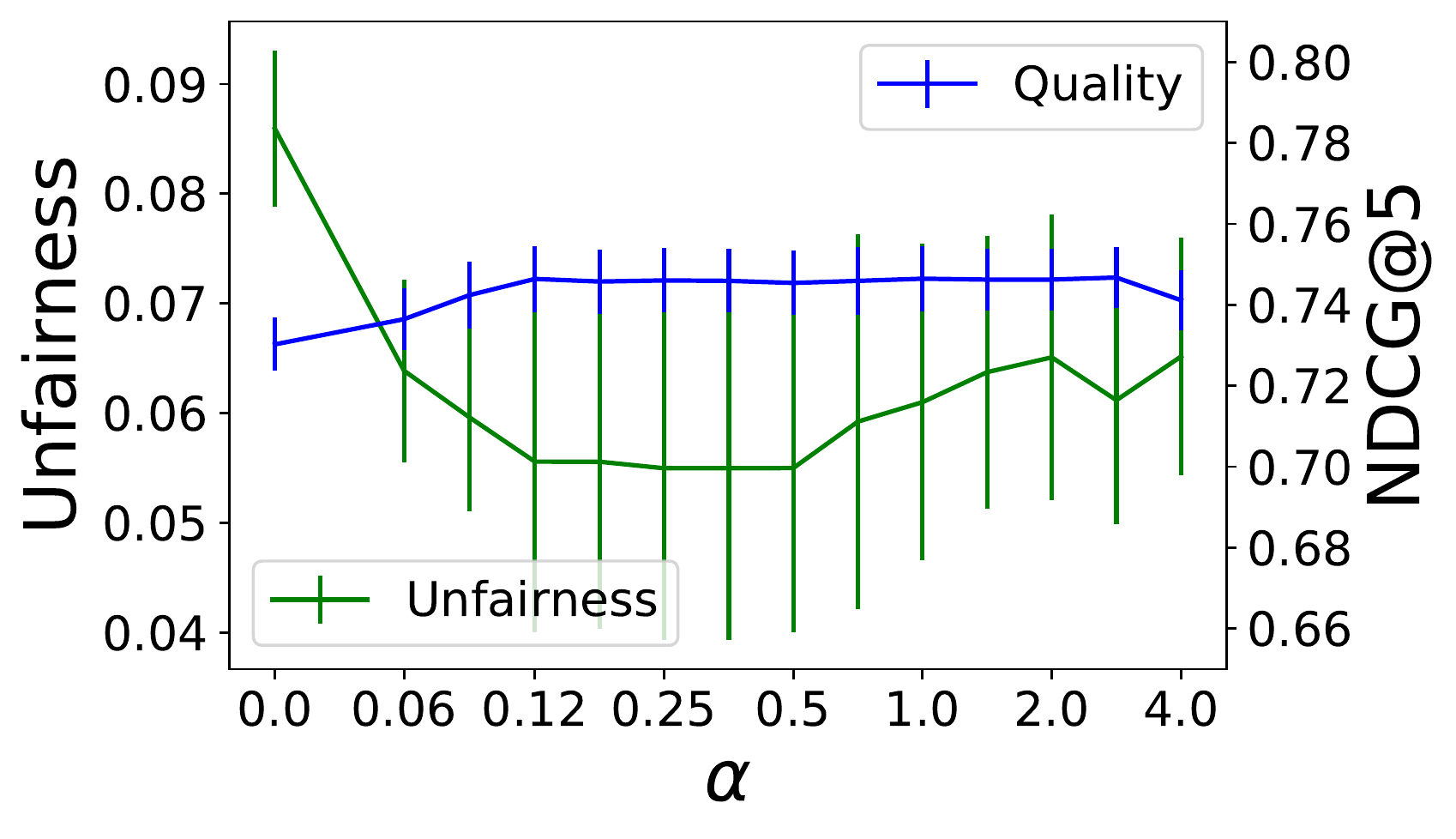}
  \caption{Equality of opportunity, TREC data}
\end{subfigure}
\caption{$k = 5, t = 4$}
\end{figure*}

\begin{figure*}[h]
\begin{subfigure}{.33\textwidth}
  \centering
  \includegraphics[width=.95\linewidth]{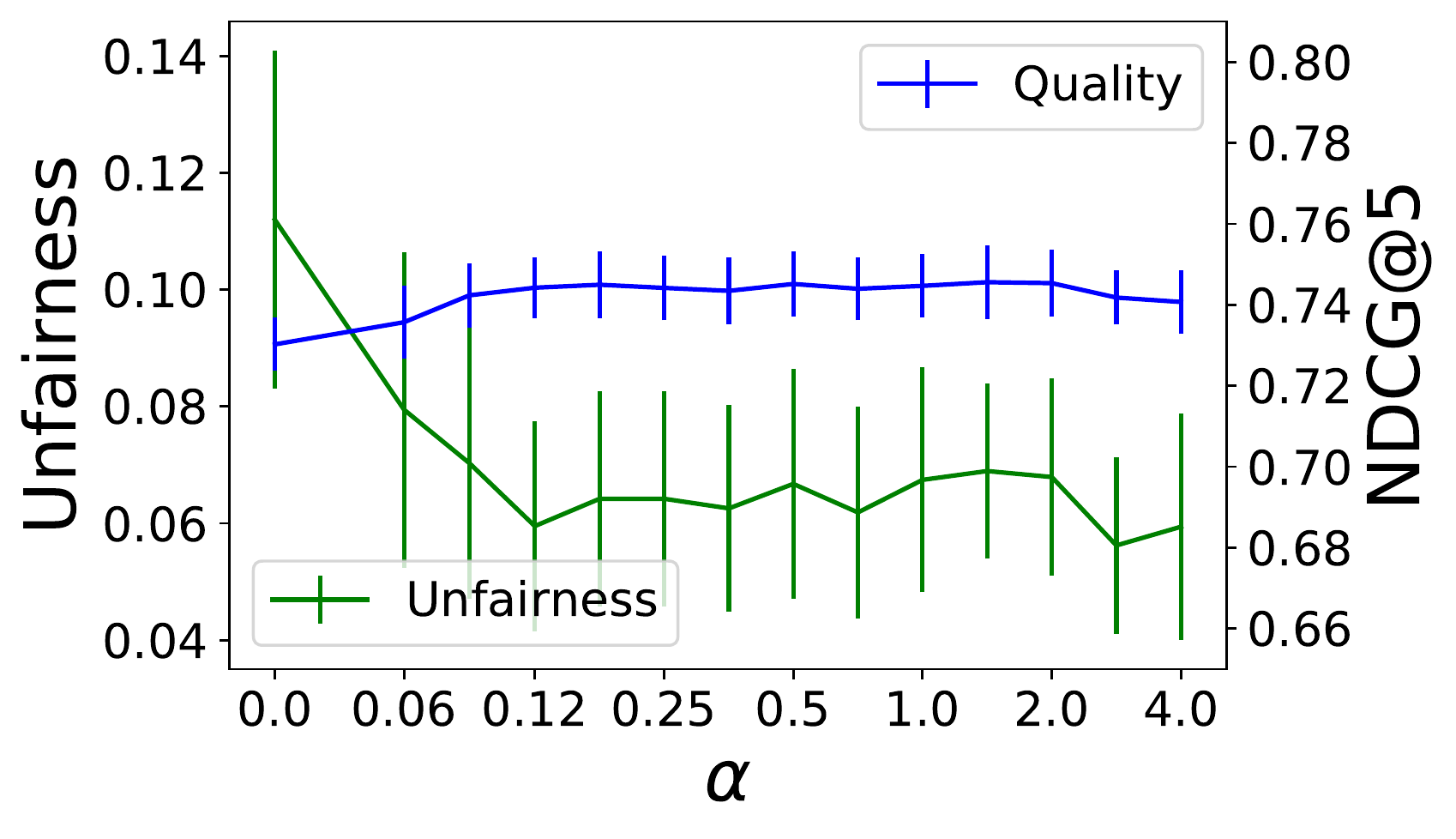}
  \caption{Demographic parity, TREC data}
\end{subfigure}%
\hfill
\begin{subfigure}{.33\textwidth}
  \centering
  \includegraphics[width=.95\linewidth]{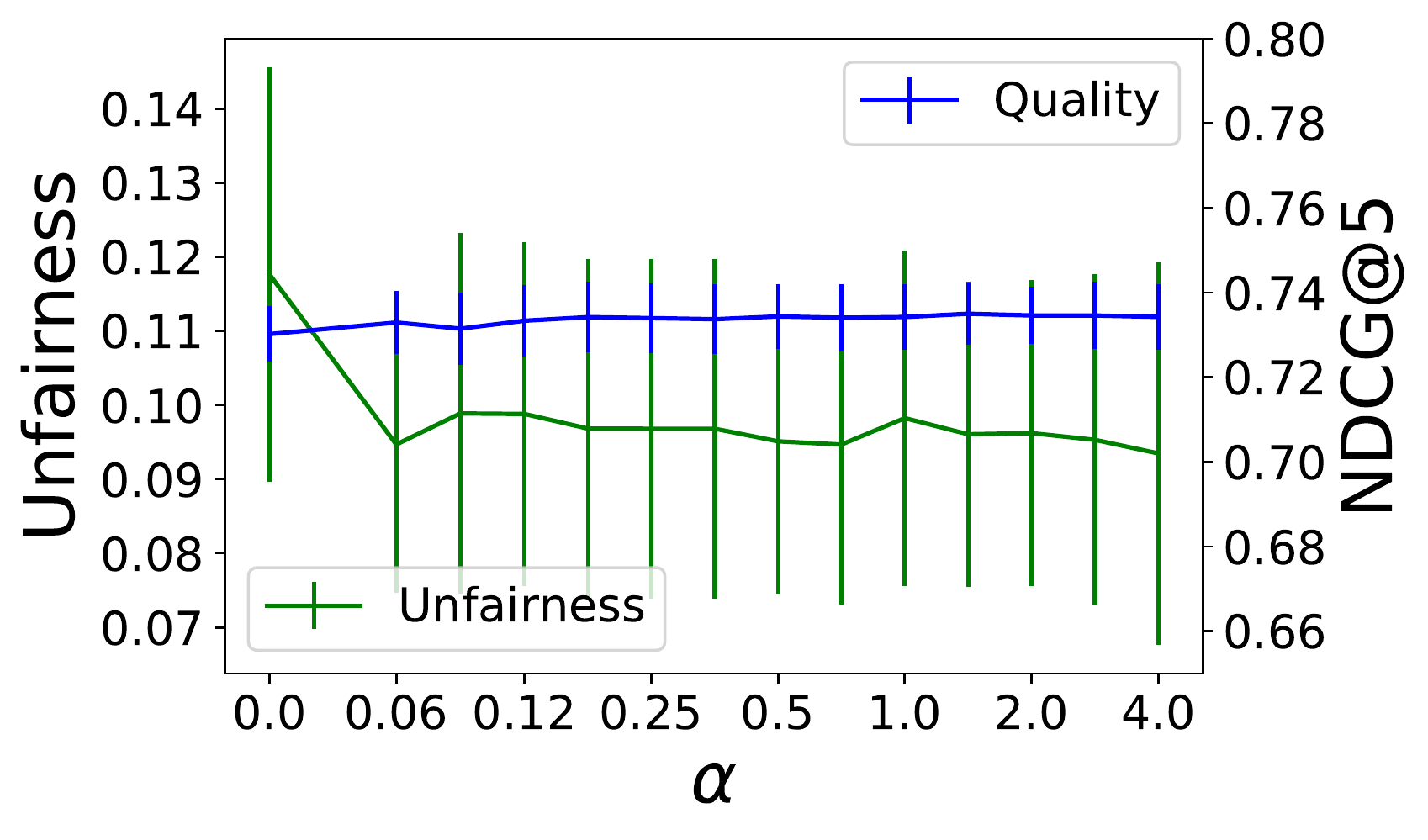}
  \caption{Equalized odds, TREC data}
\end{subfigure}%
\hfill
\begin{subfigure}{.33\textwidth}
  \centering
  \includegraphics[width=.95\linewidth]{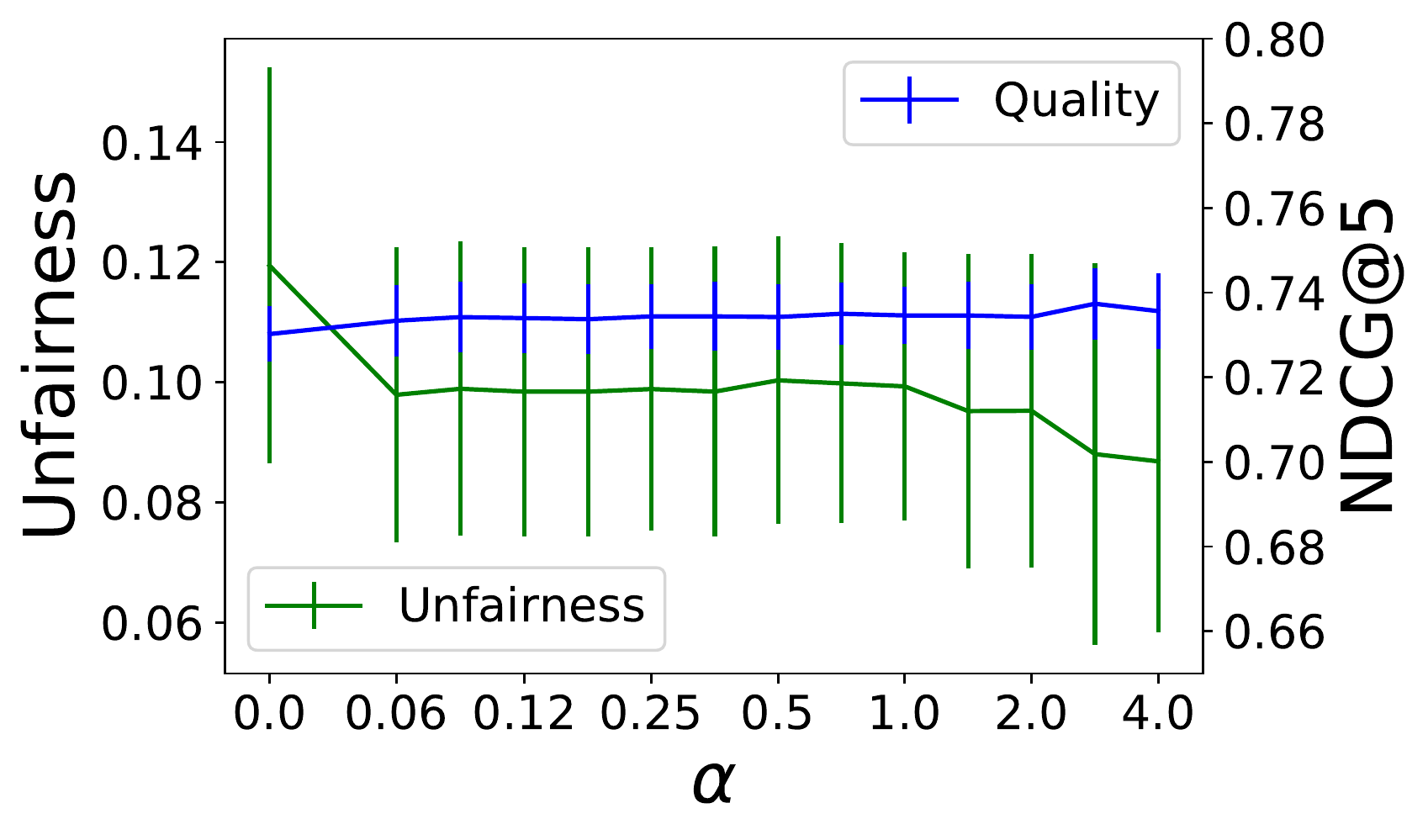}
  \caption{Equality of opportunity, TREC data}
\end{subfigure}
\caption{$k = 5, t = 5$}
\end{figure*}

\clearpage
\subsubsection{MSMARCO results}

Different rows correspond to different values of $k$ and the two different splits into protected groups (\textit{com} and \textit{ext}).


\begin{figure*}[h]
\begin{subfigure}{.33\textwidth}
  \centering
  \includegraphics[width=.95\linewidth]{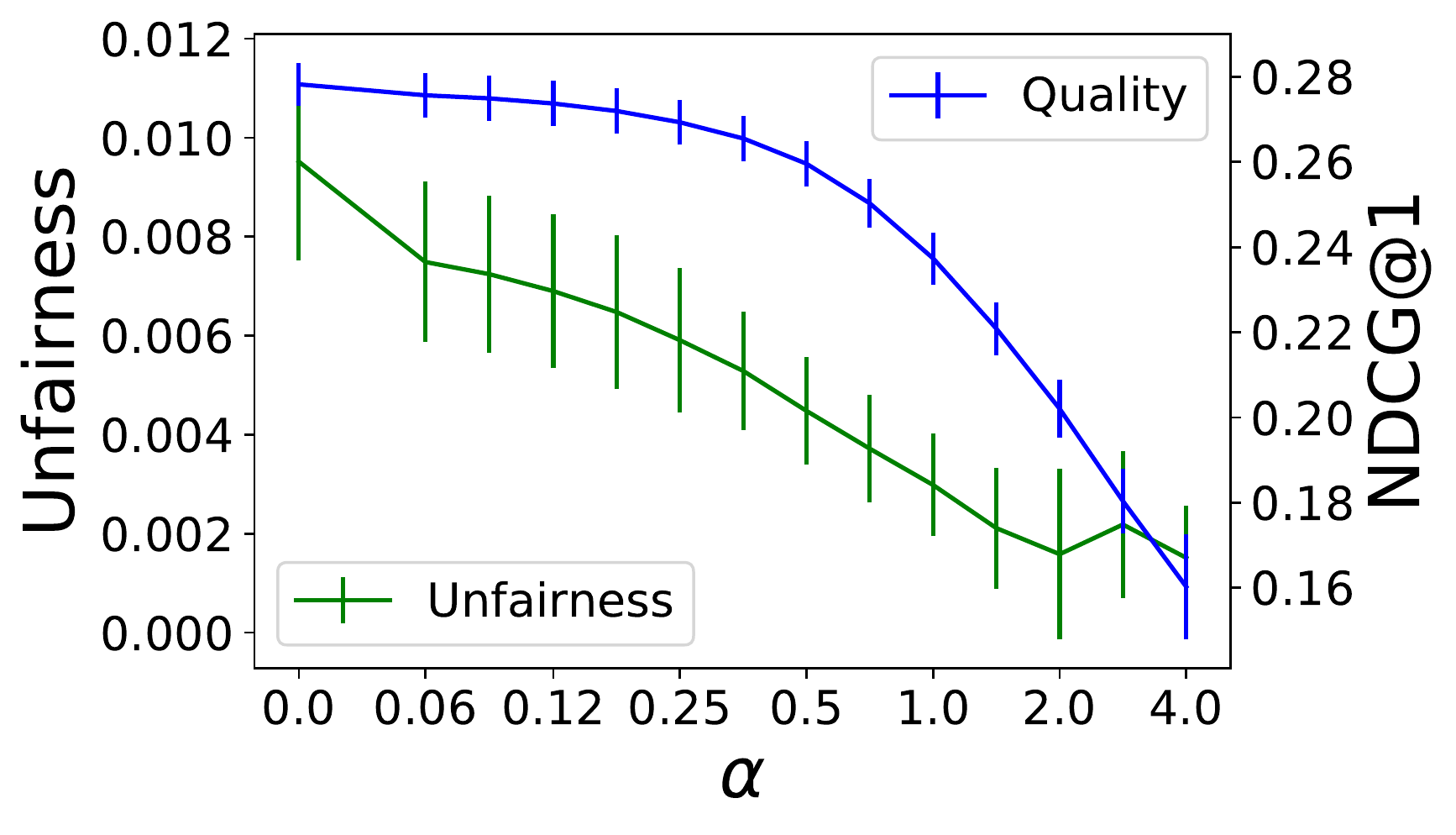}
  \caption{Demographic parity, MSMARCO}
\end{subfigure}%
\hfill
\begin{subfigure}{.33\textwidth}
  \centering
  \includegraphics[width=.95\linewidth]{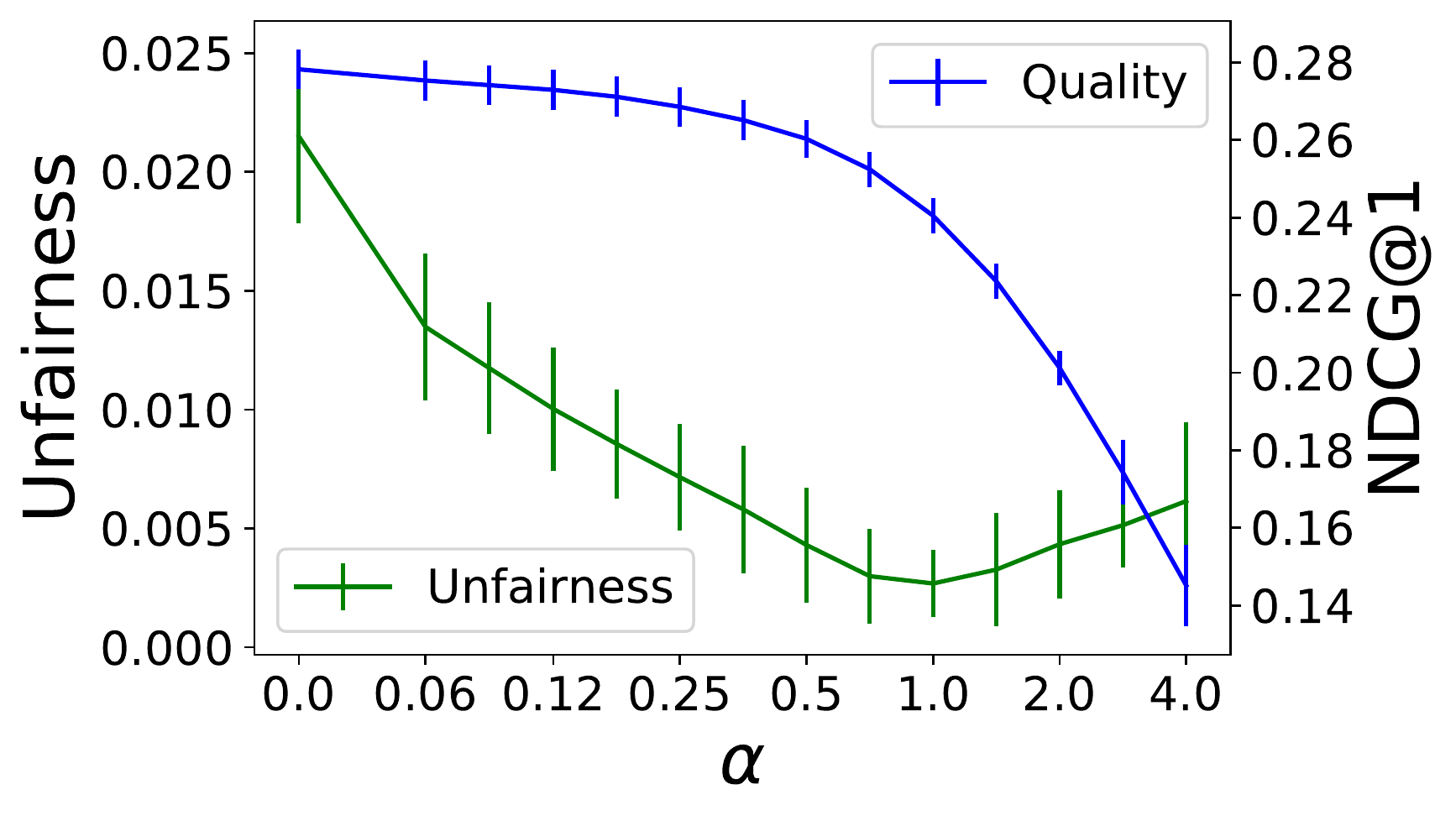}
  \caption{Equalized odds, MSMARCO}
\end{subfigure}%
\hfill
\begin{subfigure}{.33\textwidth}
  \centering
  \includegraphics[width=.95\linewidth]{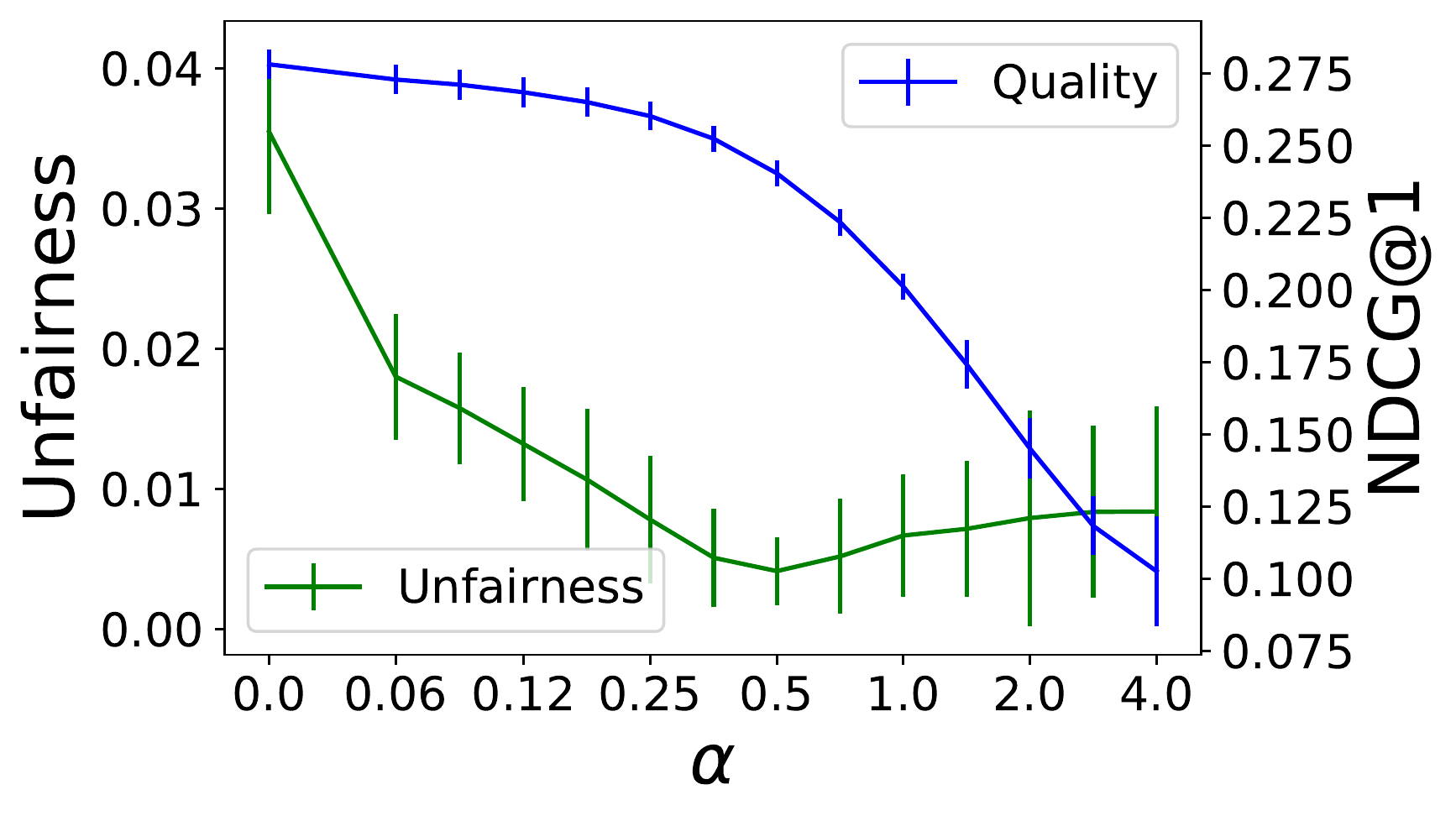}
  \caption{Equality of opportunity, MSMARCO}
\end{subfigure}
\caption{$k = 1$, \textit{com}}
\end{figure*}

\begin{figure*}[h]
\begin{subfigure}{.33\textwidth}
  \centering
  \includegraphics[width=.95\linewidth]{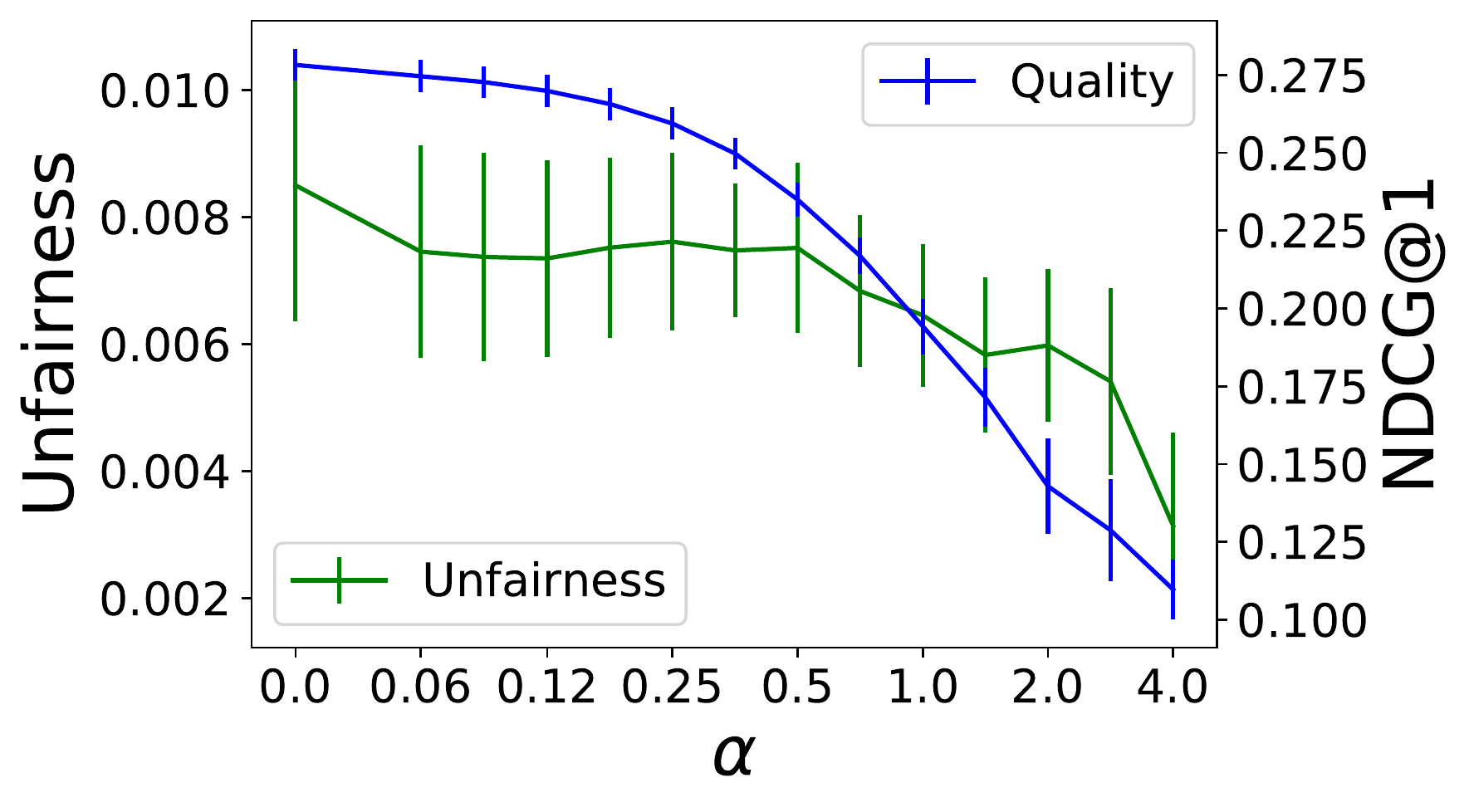}
  \caption{Demographic parity, MSMARCO}
\end{subfigure}%
\hfill
\begin{subfigure}{.33\textwidth}
  \centering
  \includegraphics[width=.95\linewidth]{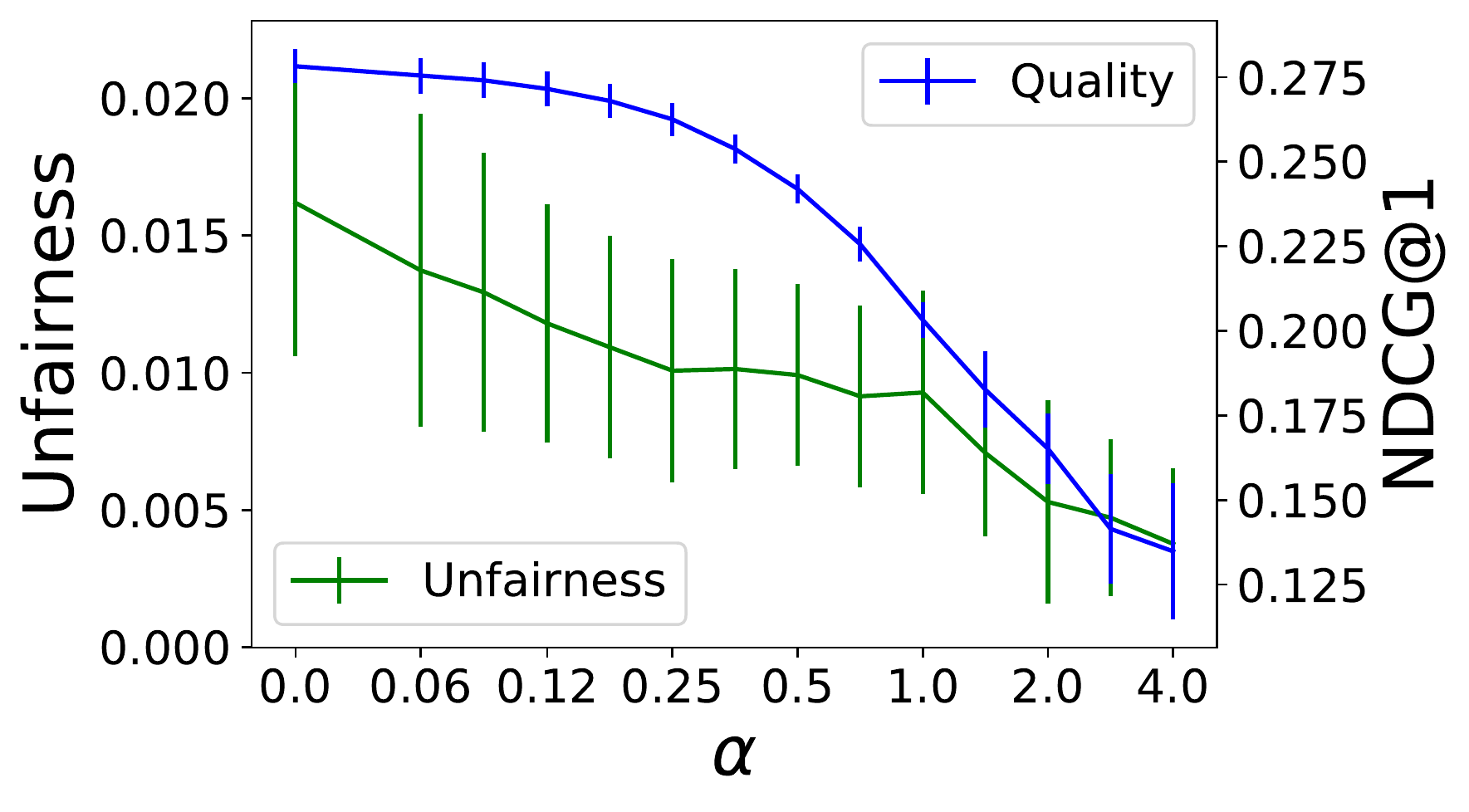}
  \caption{Equalized odds, MSMARCO}
\end{subfigure}%
\hfill
\begin{subfigure}{.33\textwidth}
  \centering
  \includegraphics[width=.95\linewidth]{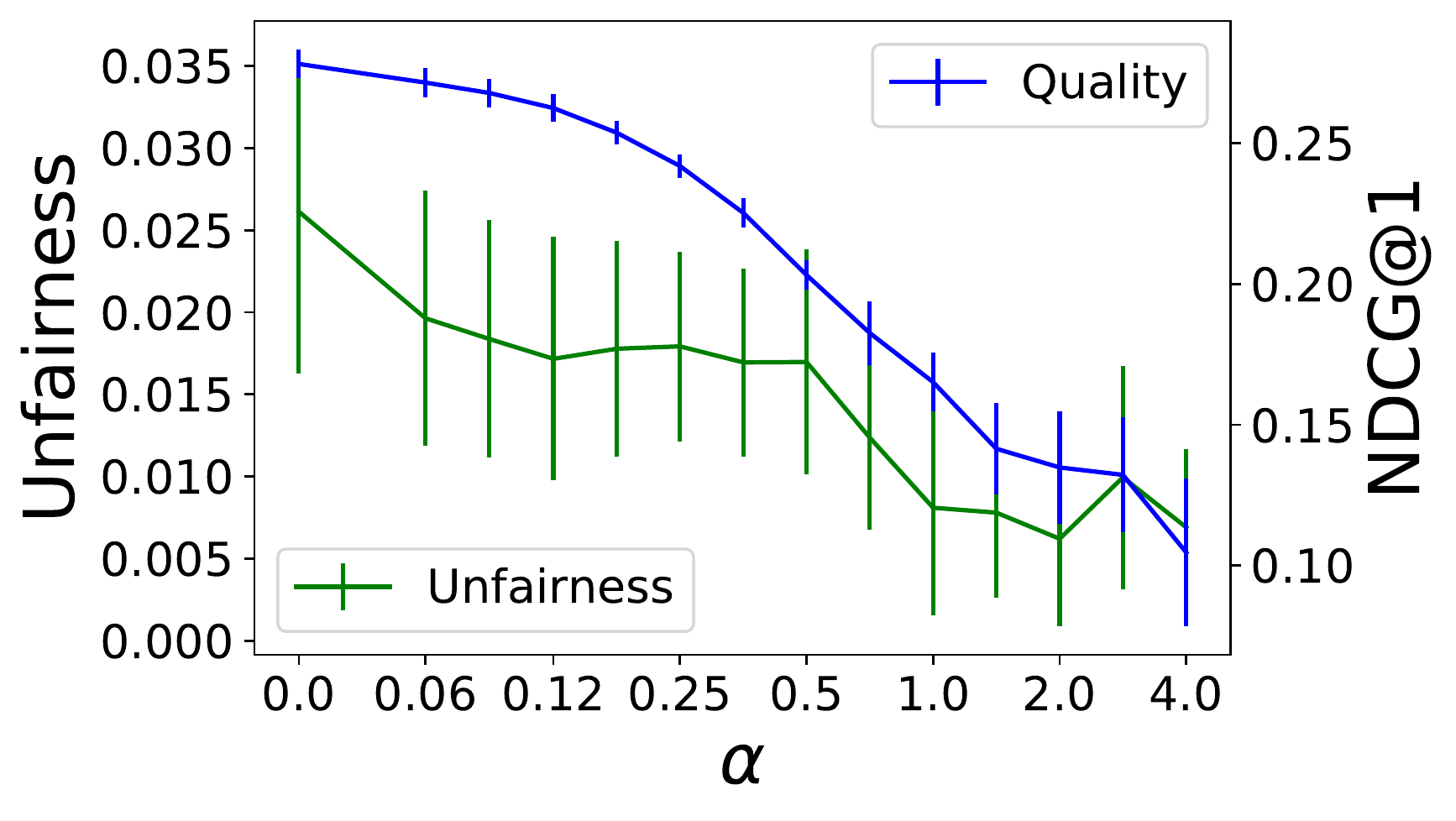}
  \caption{Equality of opportunity, MSMARCO}
\end{subfigure}
\caption{$k = 1$, \textit{ext}}
\end{figure*}


\begin{figure*}[h]
\begin{subfigure}{.33\textwidth}
  \centering
  \includegraphics[width=.95\linewidth]{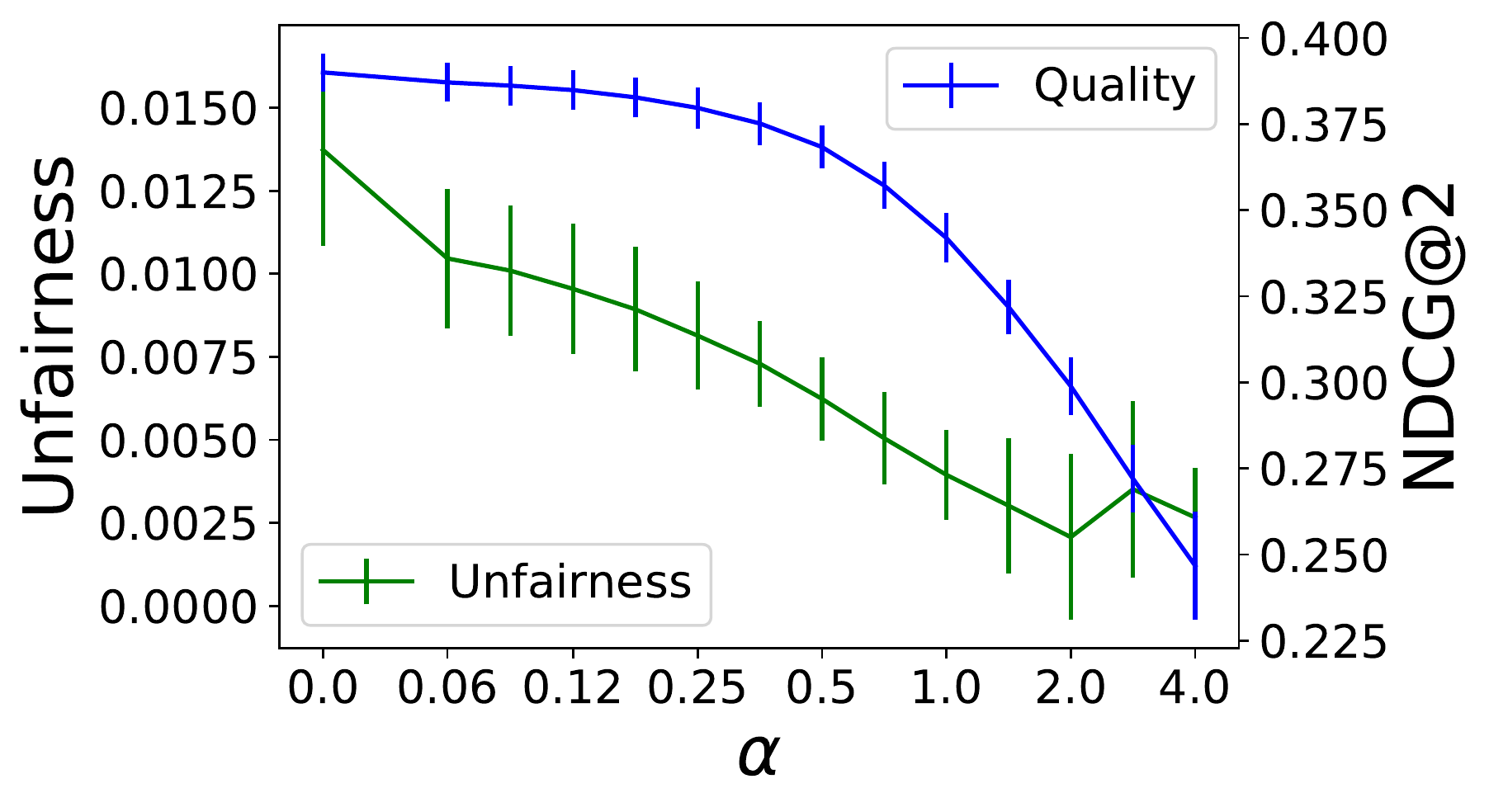}
  \caption{Demographic parity, MSMARCO}
\end{subfigure}%
\hfill
\begin{subfigure}{.33\textwidth}
  \centering
  \includegraphics[width=.95\linewidth]{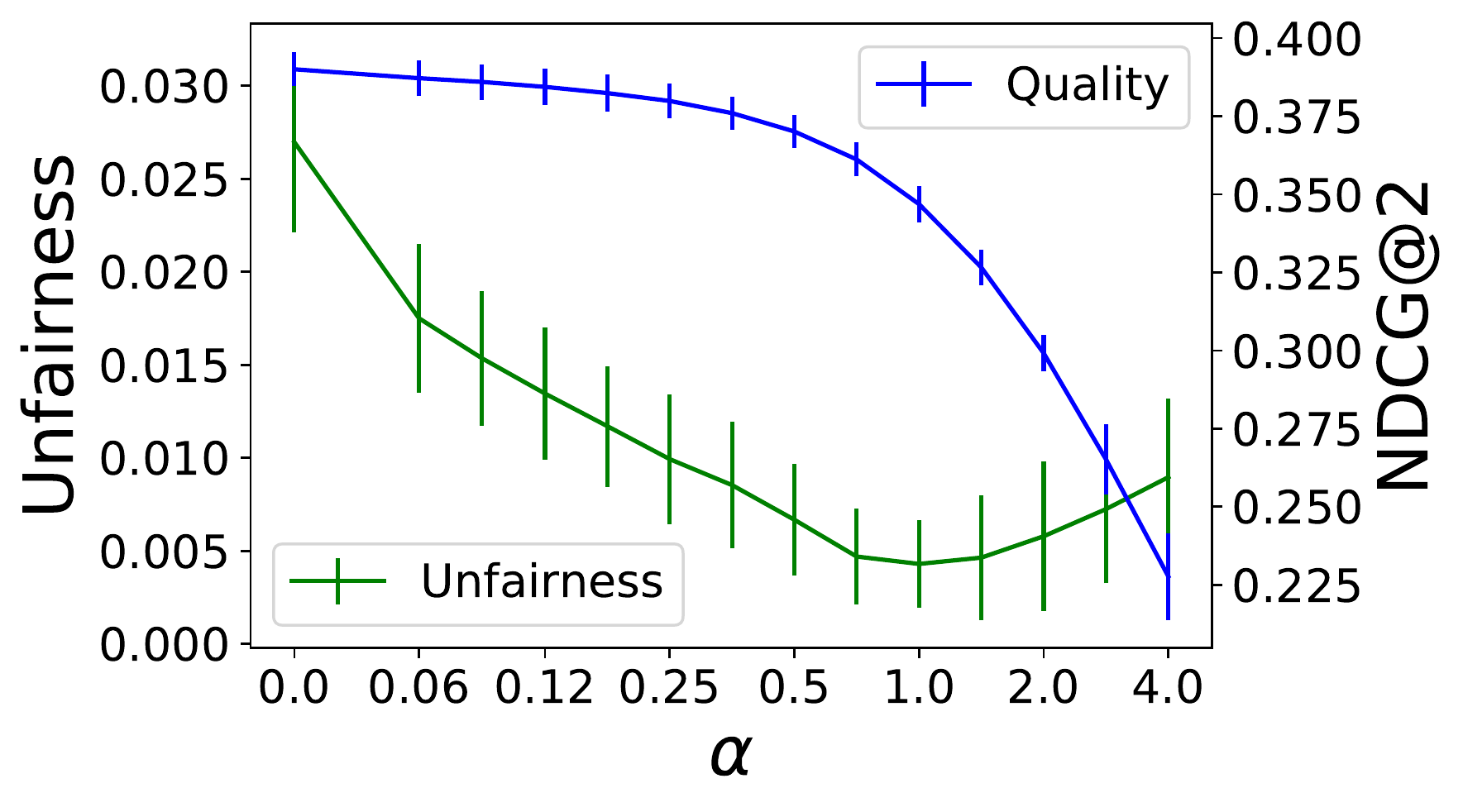}
  \caption{Equalized odds, MSMARCO}
\end{subfigure}%
\hfill
\begin{subfigure}{.33\textwidth}
  \centering
  \includegraphics[width=.95\linewidth]{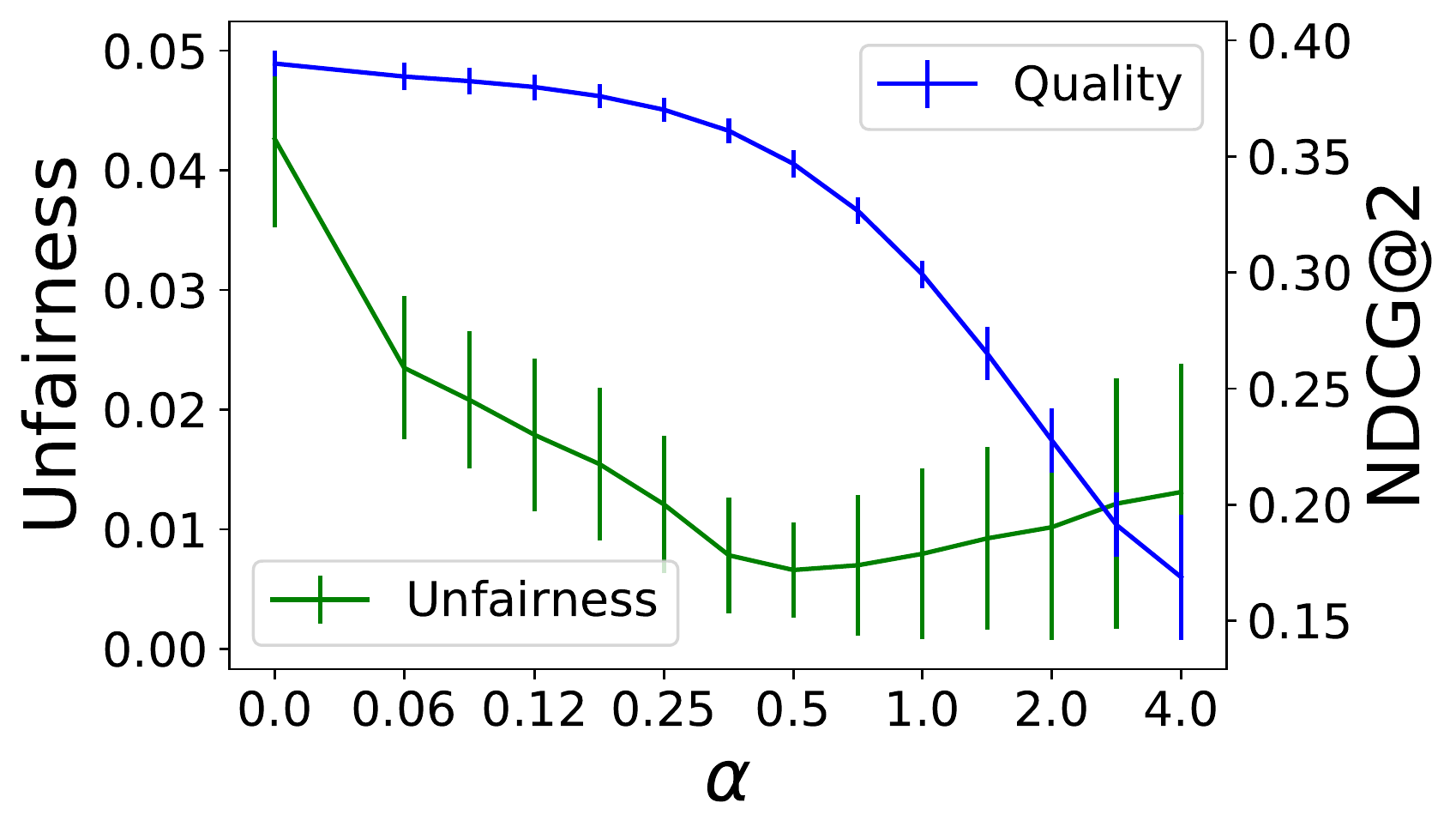}
  \caption{Equality of opportunity, MSMARCO}
\end{subfigure}
\caption{$k = 2$, \textit{com}}
\end{figure*}

\begin{figure*}[h]
\begin{subfigure}{.33\textwidth}
  \centering
  \includegraphics[width=.95\linewidth]{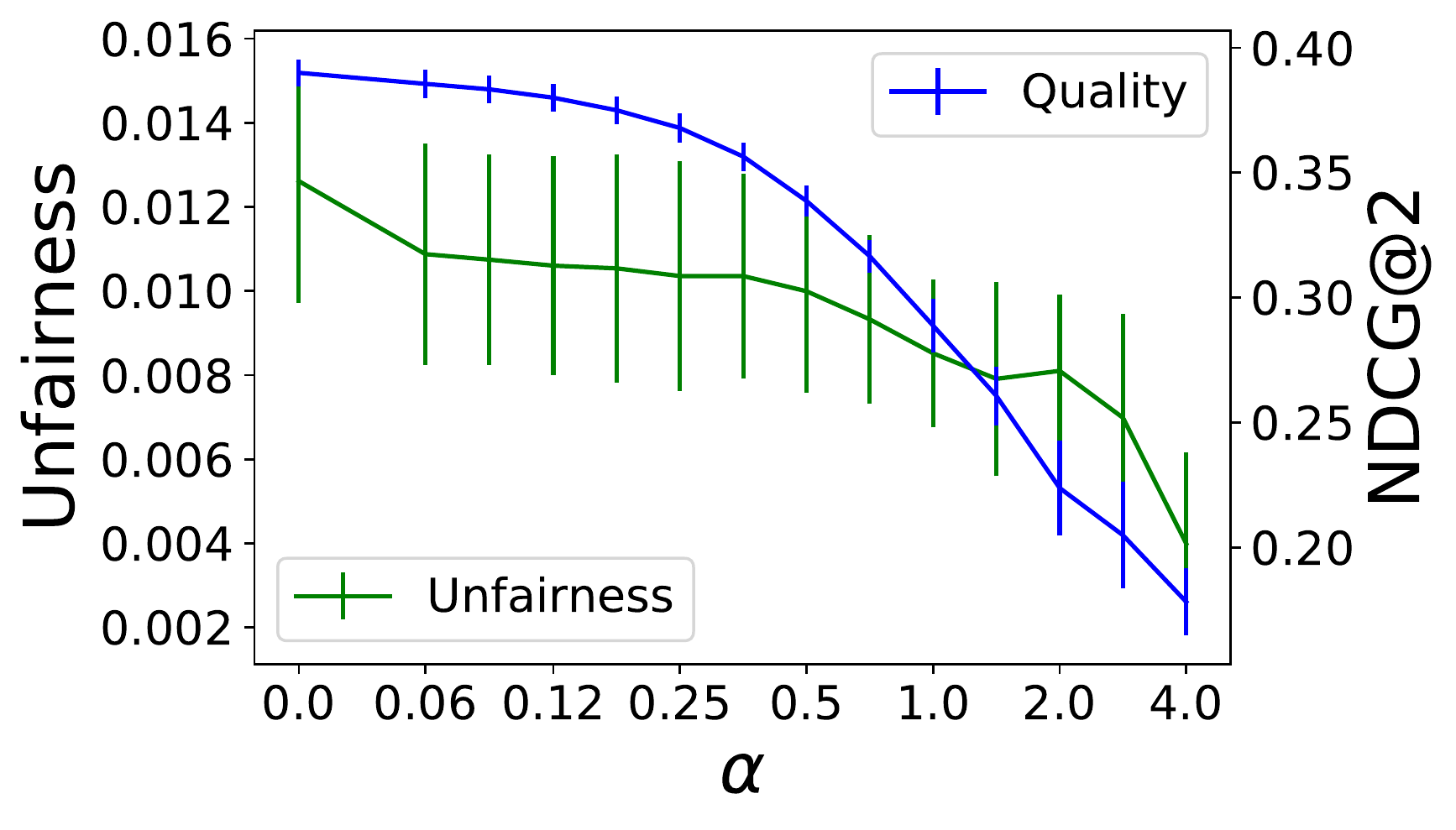}
  \caption{Demographic parity, MSMARCO}
\end{subfigure}%
\hfill
\begin{subfigure}{.33\textwidth}
  \centering
  \includegraphics[width=.95\linewidth]{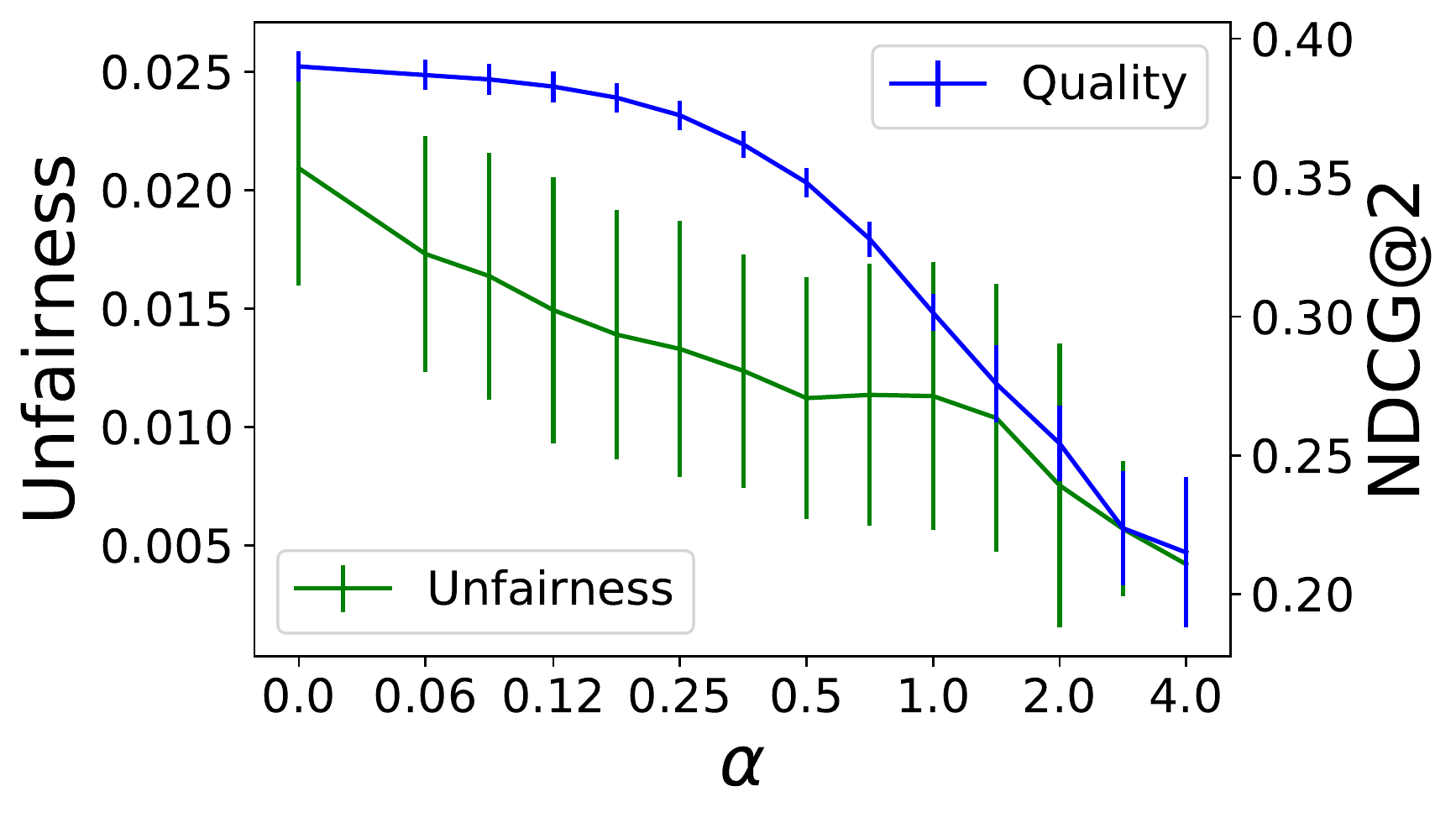}
  \caption{Equalized odds, MSMARCO}
\end{subfigure}%
\hfill
\begin{subfigure}{.33\textwidth}
  \centering
  \includegraphics[width=.95\linewidth]{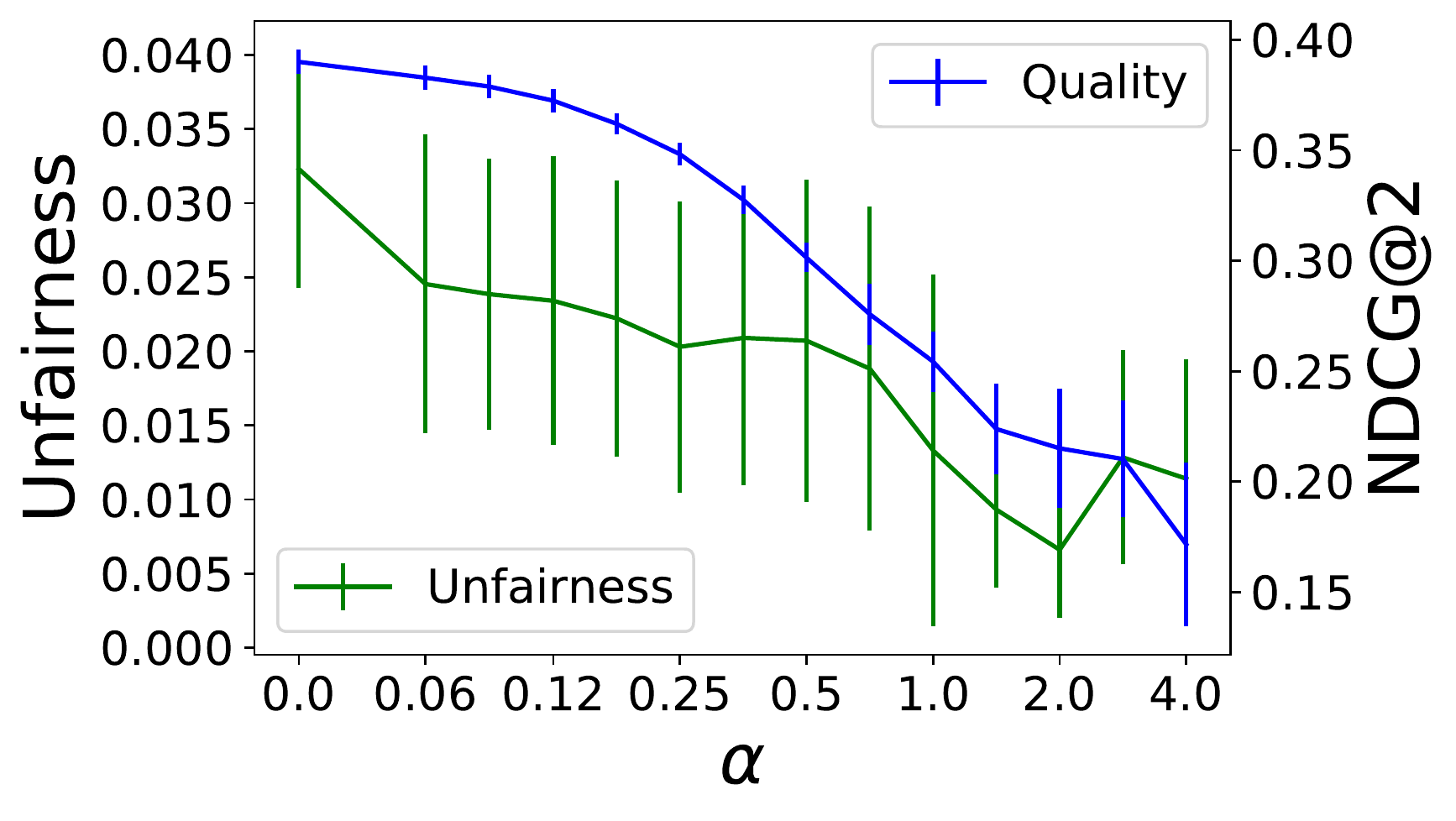}
  \caption{Equality of opportunity, MSMARCO}
\end{subfigure}
\caption{$k = 2$, \textit{ext}}
\end{figure*}


\begin{figure*}[h]
\begin{subfigure}{.33\textwidth}
  \centering
  \includegraphics[width=.95\linewidth]{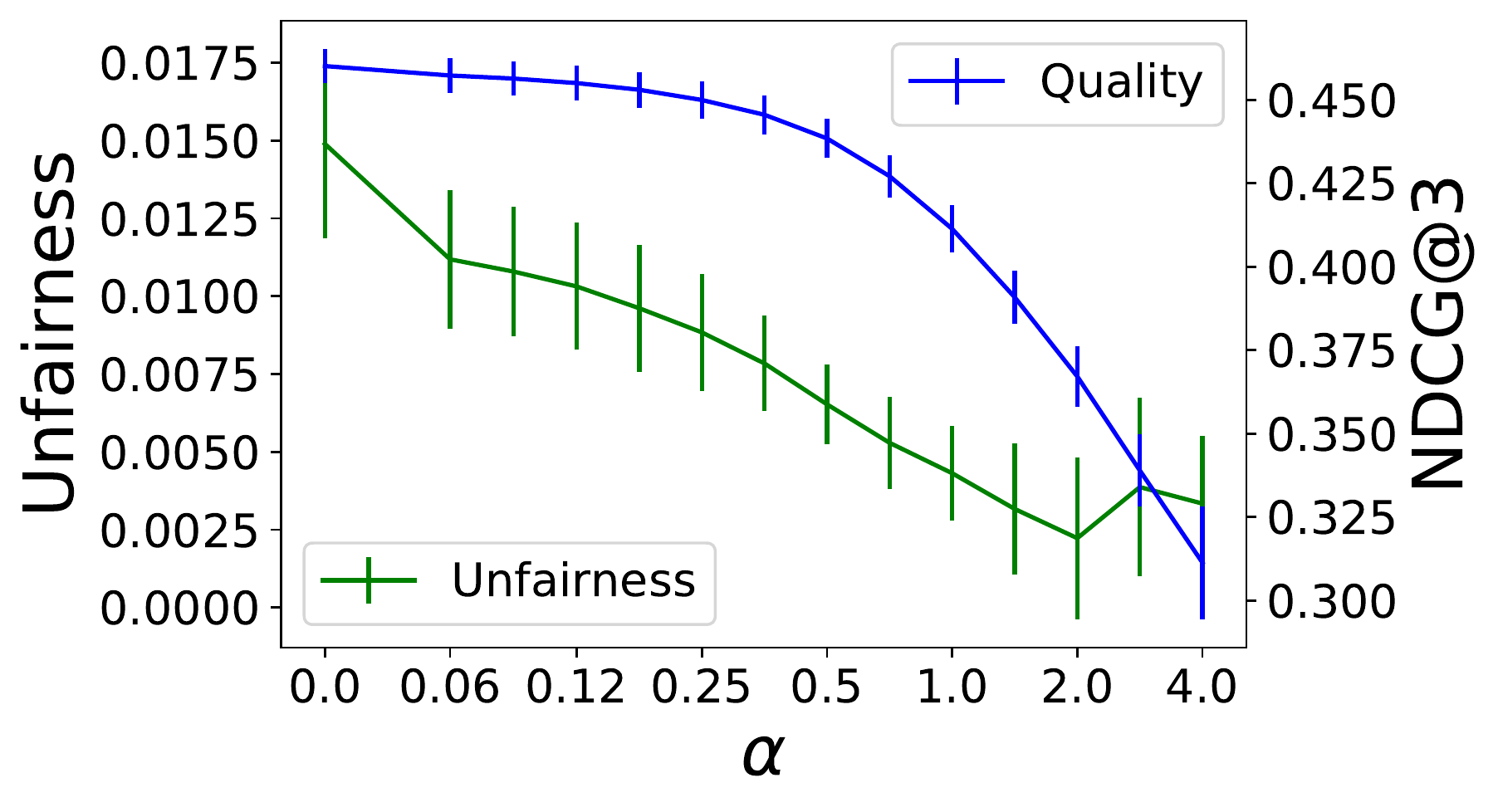}
  \caption{Demographic parity, MSMARCO}
\end{subfigure}%
\hfill
\begin{subfigure}{.33\textwidth}
  \centering
  \includegraphics[width=.95\linewidth]{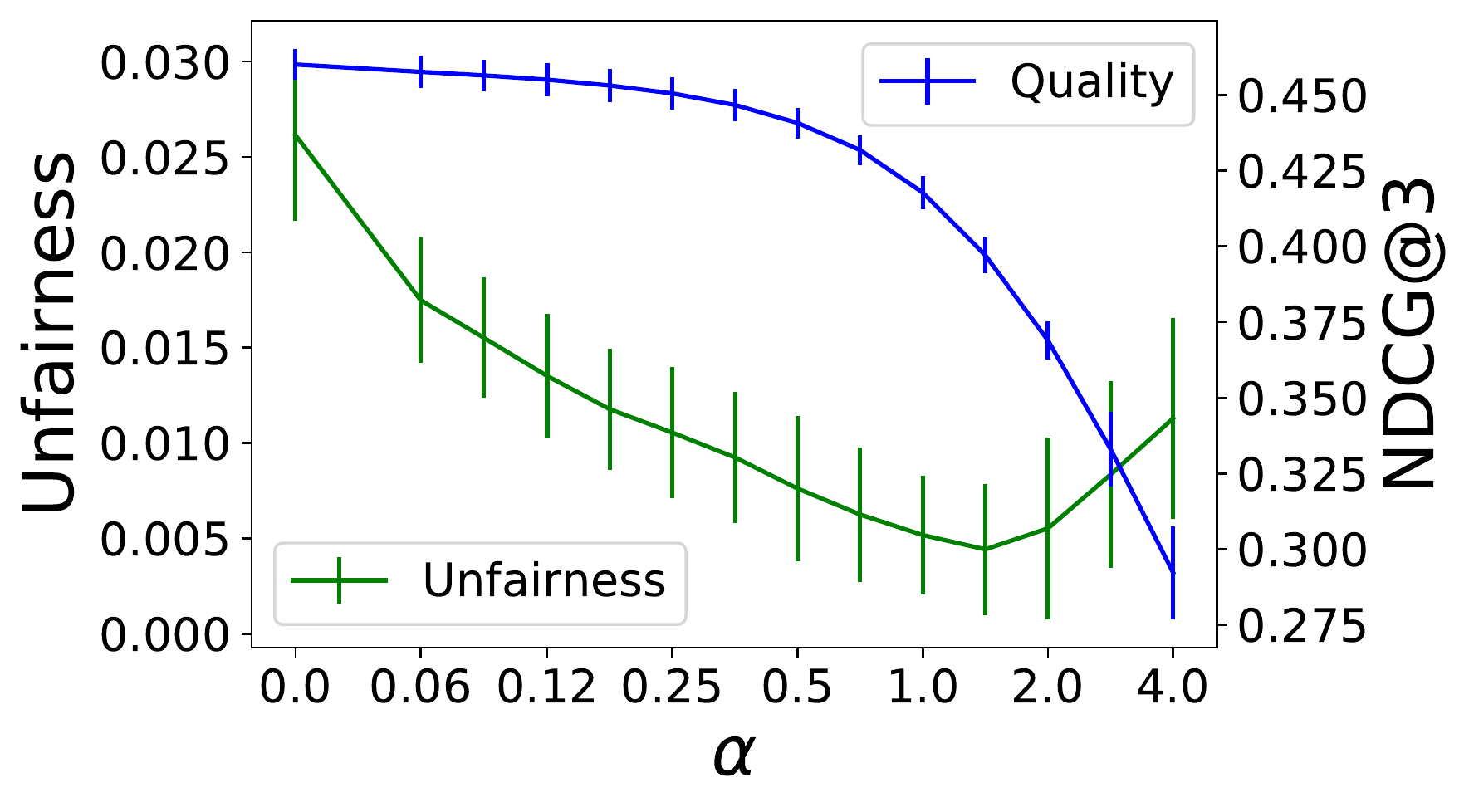}
  \caption{Equalized odds, MSMARCO}
\end{subfigure}%
\hfill
\begin{subfigure}{.33\textwidth}
  \centering
  \includegraphics[width=.95\linewidth]{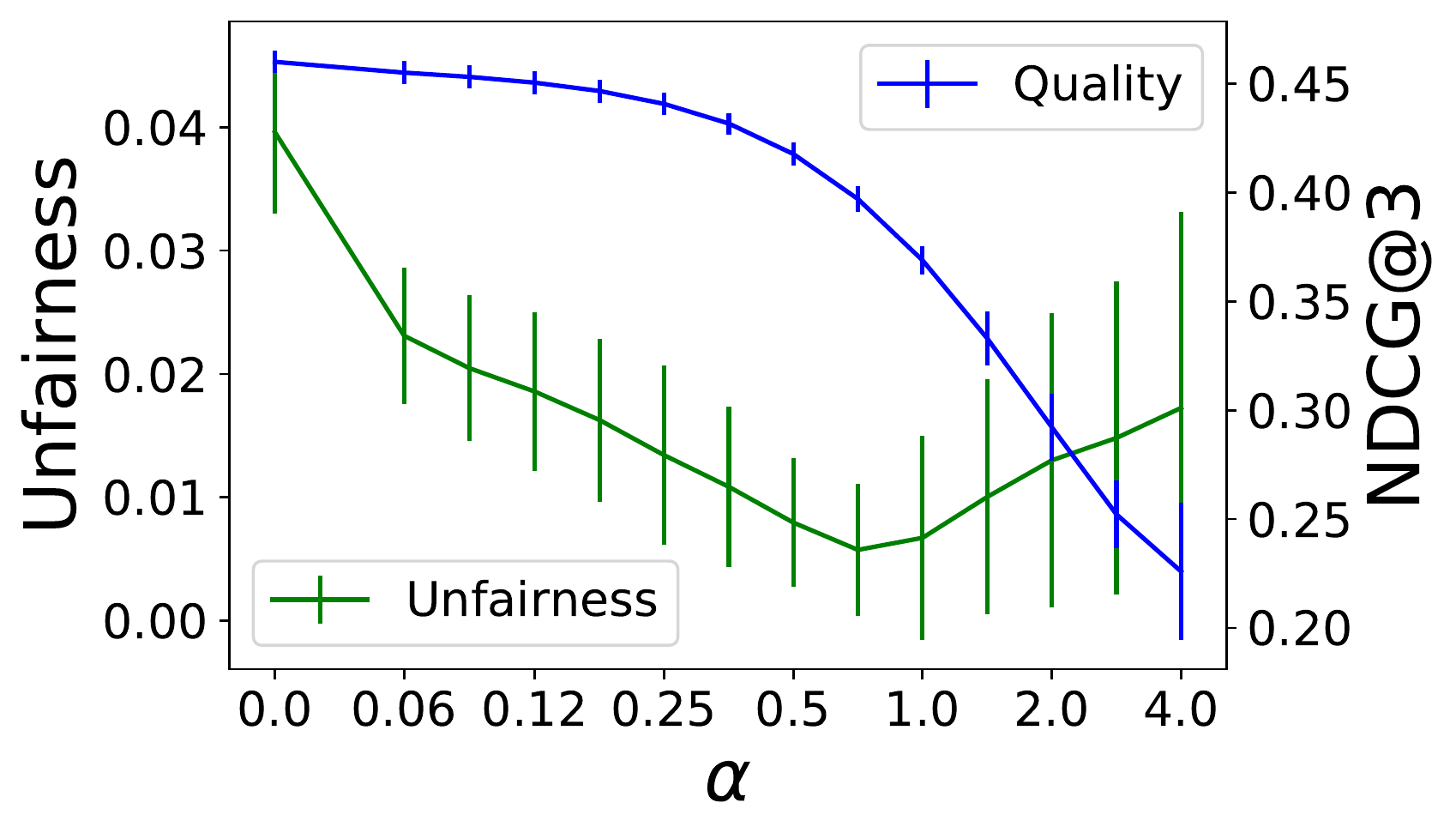}
  \caption{Equality of opportunity, MSMARCO}
\end{subfigure}
\caption{$k = 3$, \textit{com}}
\end{figure*}

\begin{figure*}[h]
\begin{subfigure}{.33\textwidth}
  \centering
  \includegraphics[width=.95\linewidth]{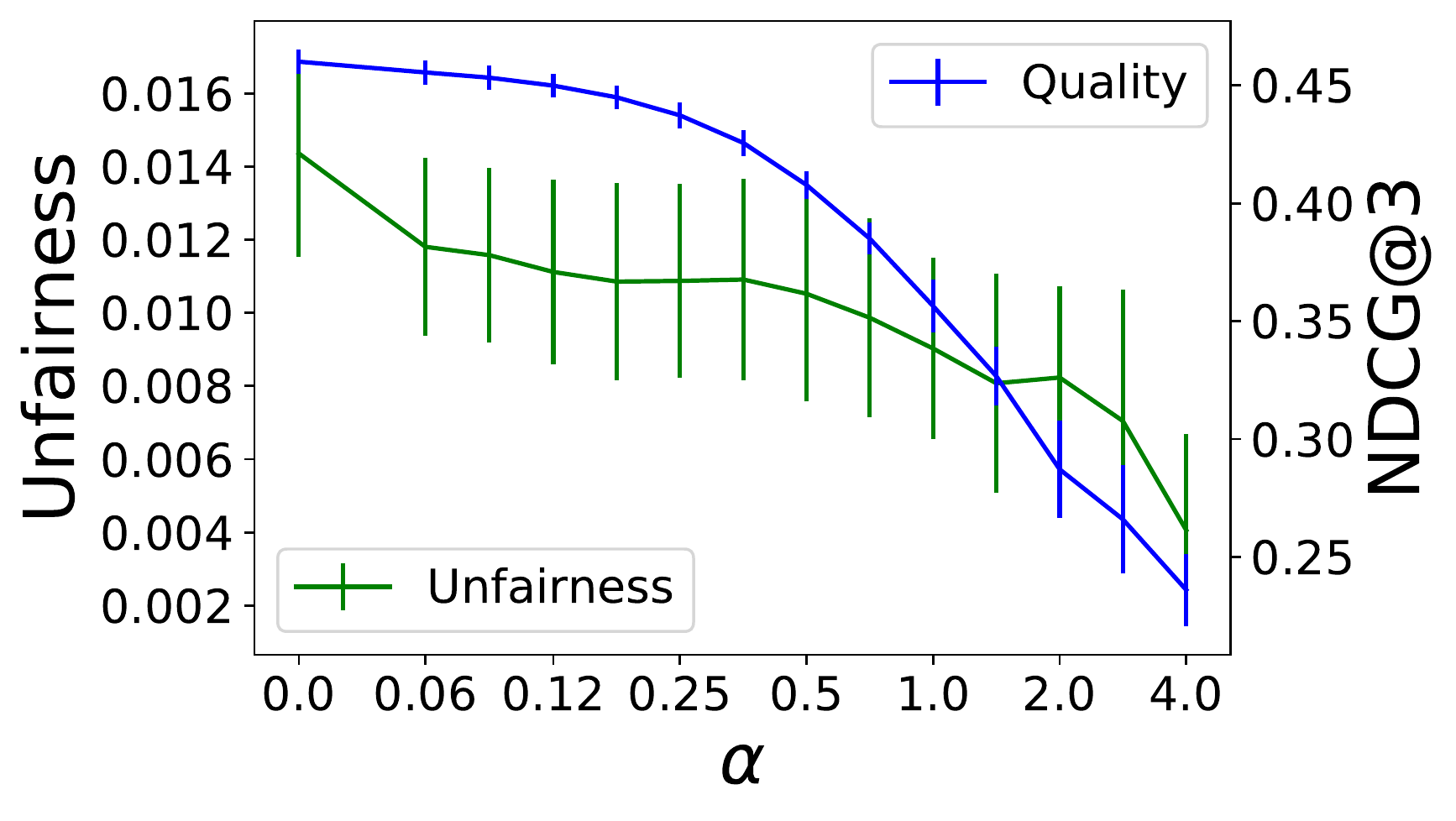}
  \caption{Demographic parity, MSMARCO}
\end{subfigure}%
\hfill
\begin{subfigure}{.33\textwidth}
  \centering
  \includegraphics[width=.95\linewidth]{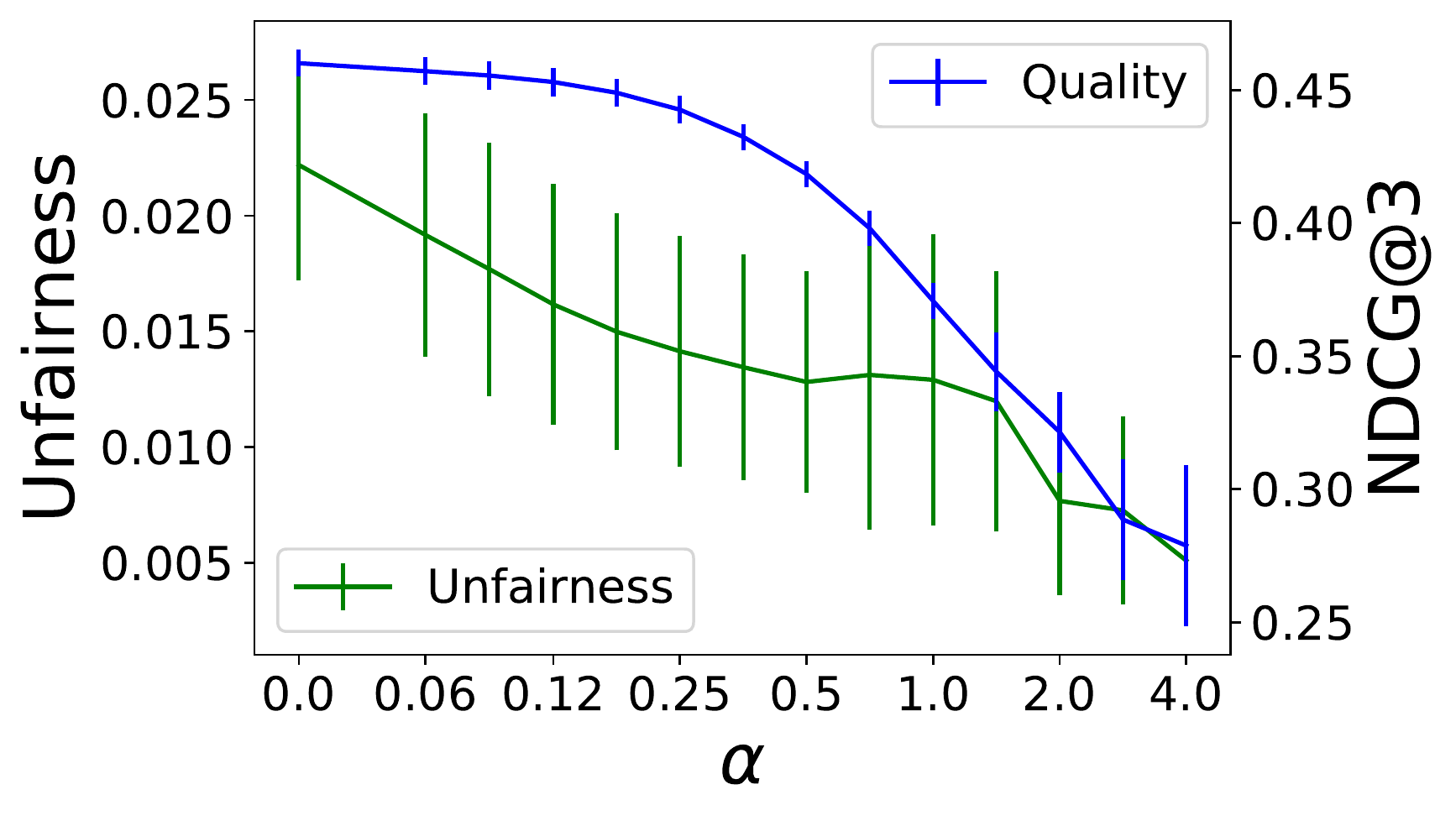}
  \caption{Equalized odds, MSMARCO}
\end{subfigure}%
\hfill
\begin{subfigure}{.33\textwidth}
  \centering
  \includegraphics[width=.95\linewidth]{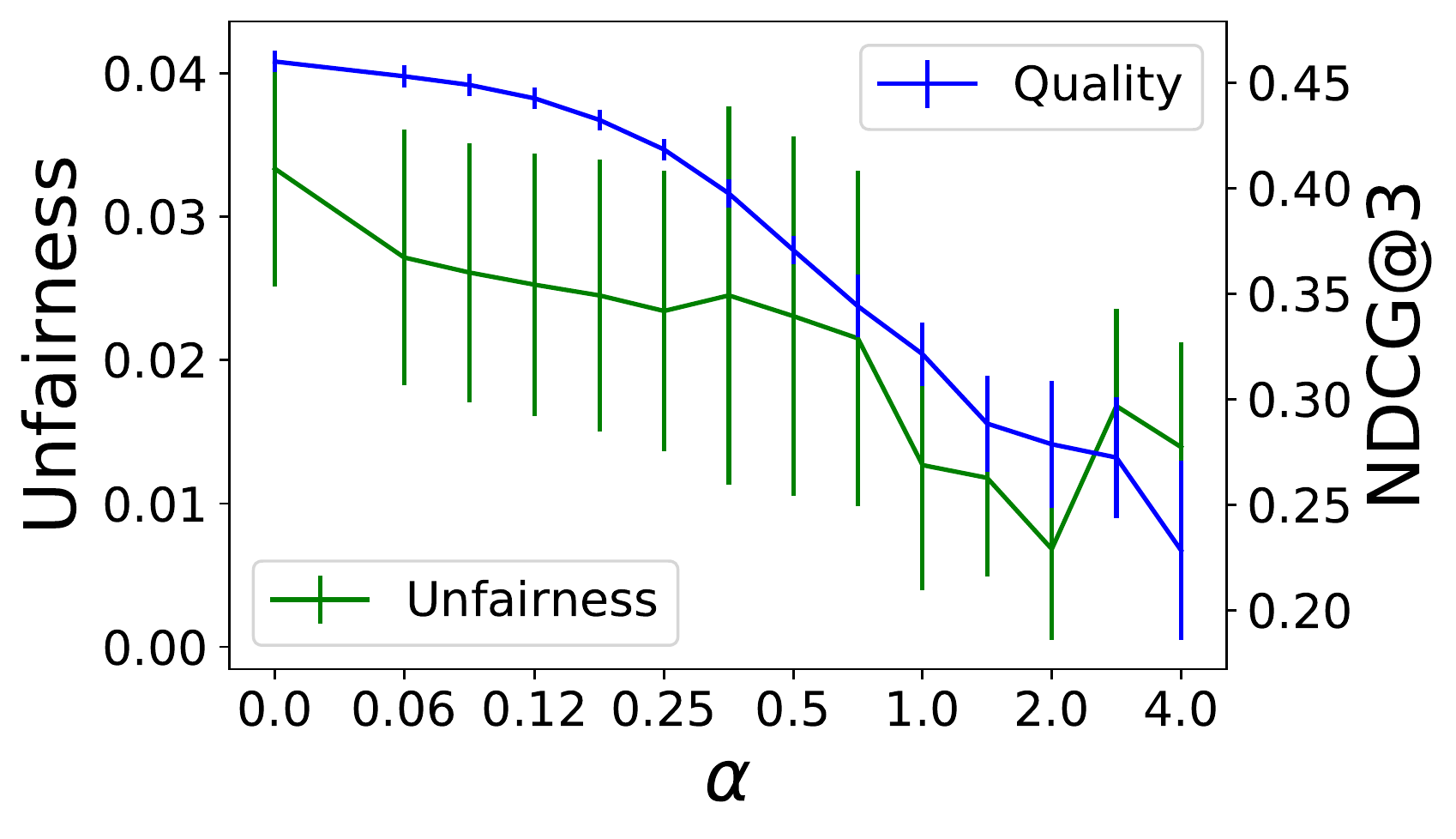}
  \caption{Equality of opportunity, MSMARCO}
\end{subfigure}
\caption{$k = 3$, \textit{ext}}
\end{figure*}


\begin{figure*}[h]
\begin{subfigure}{.33\textwidth}
  \centering
  \includegraphics[width=.95\linewidth]{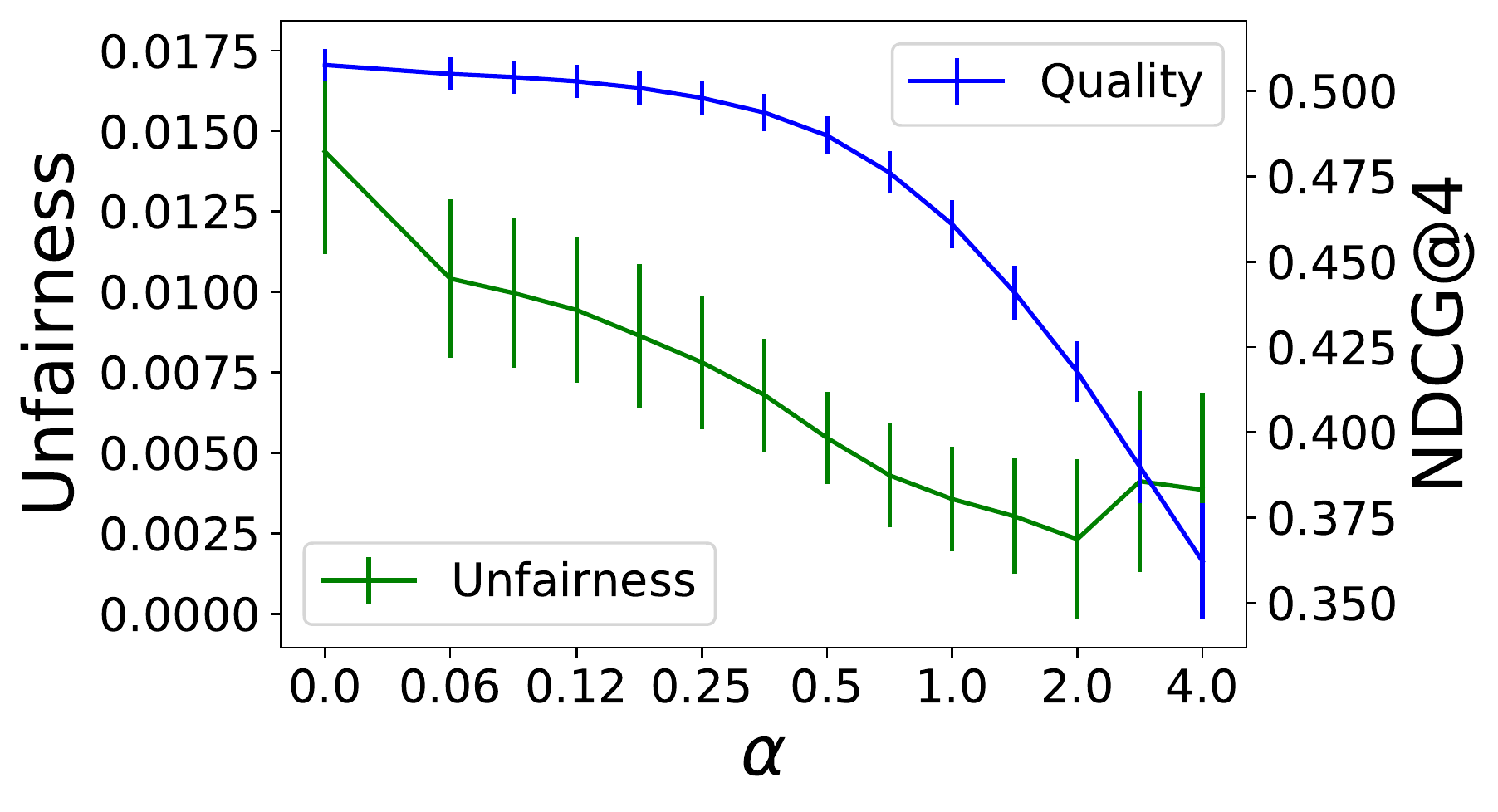}
  \caption{Demographic parity, MSMARCO}
\end{subfigure}%
\hfill
\begin{subfigure}{.33\textwidth}
  \centering
  \includegraphics[width=.95\linewidth]{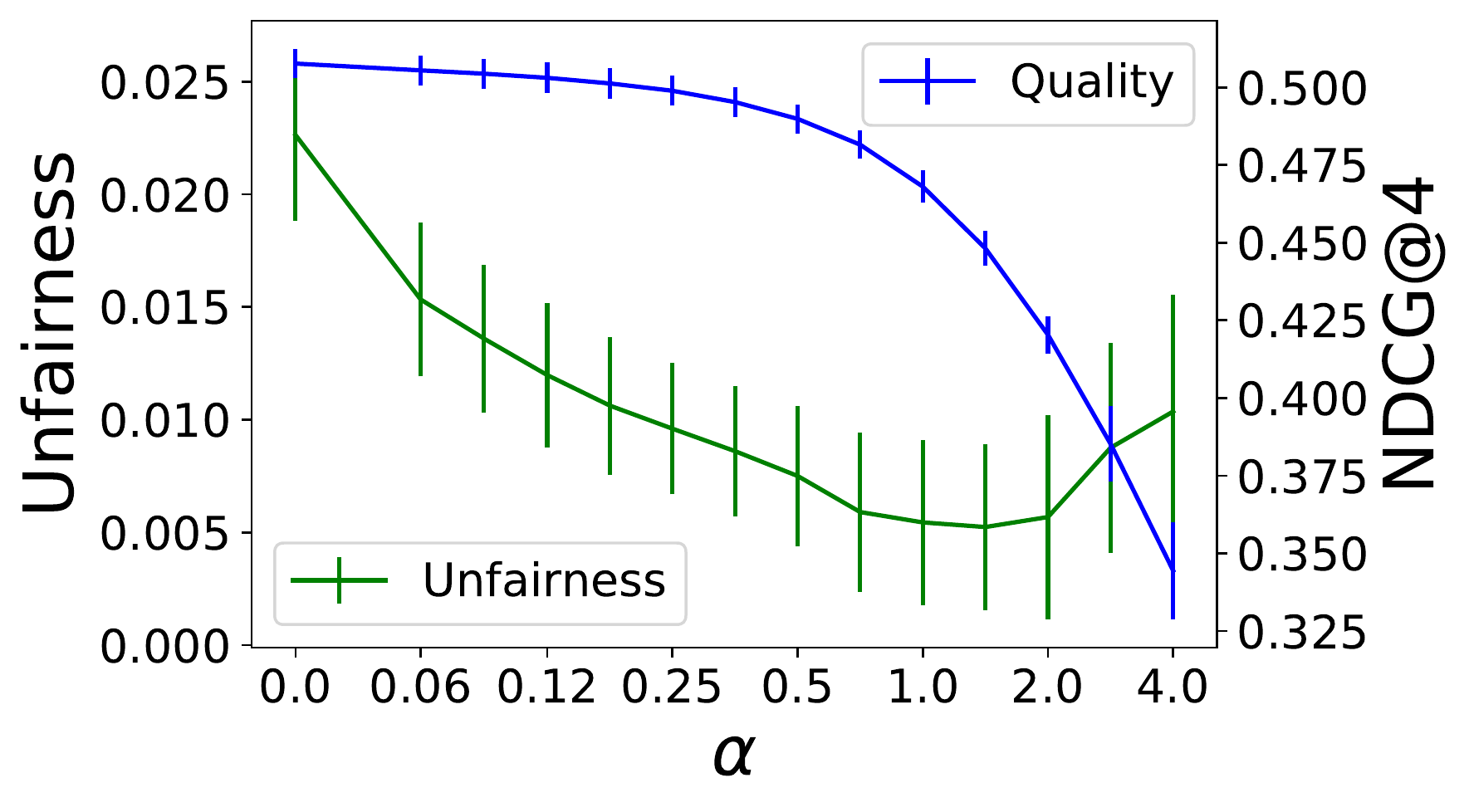}
  \caption{Equalized odds, MSMARCO}
\end{subfigure}%
\hfill
\begin{subfigure}{.33\textwidth}
  \centering
  \includegraphics[width=.95\linewidth]{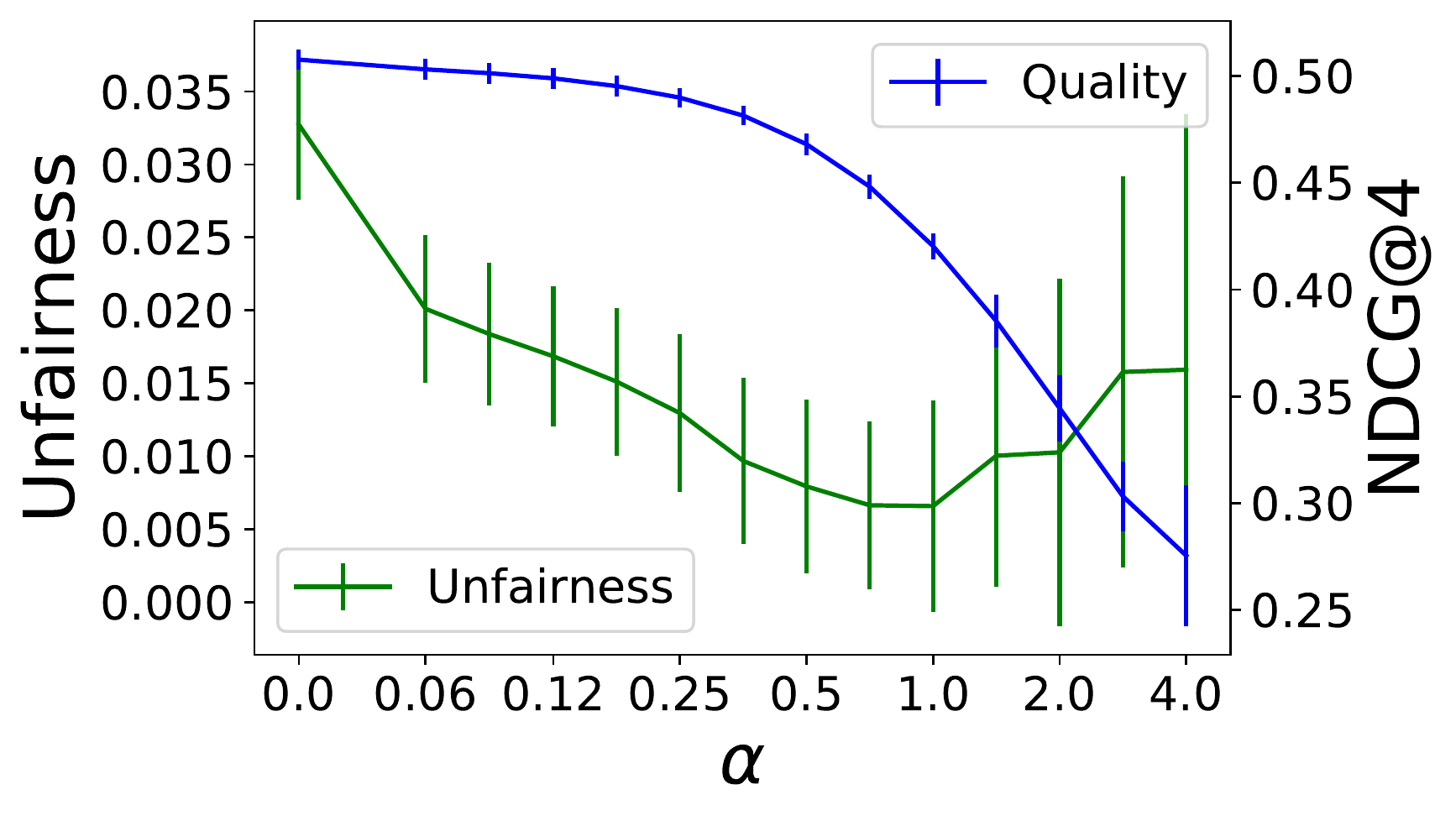}
  \caption{Equality of opportunity, MSMARCO}
\end{subfigure}
\caption{$k = 4$, \textit{com}}
\end{figure*}

\begin{figure*}[h]
\begin{subfigure}{.33\textwidth}
  \centering
  \includegraphics[width=.95\linewidth]{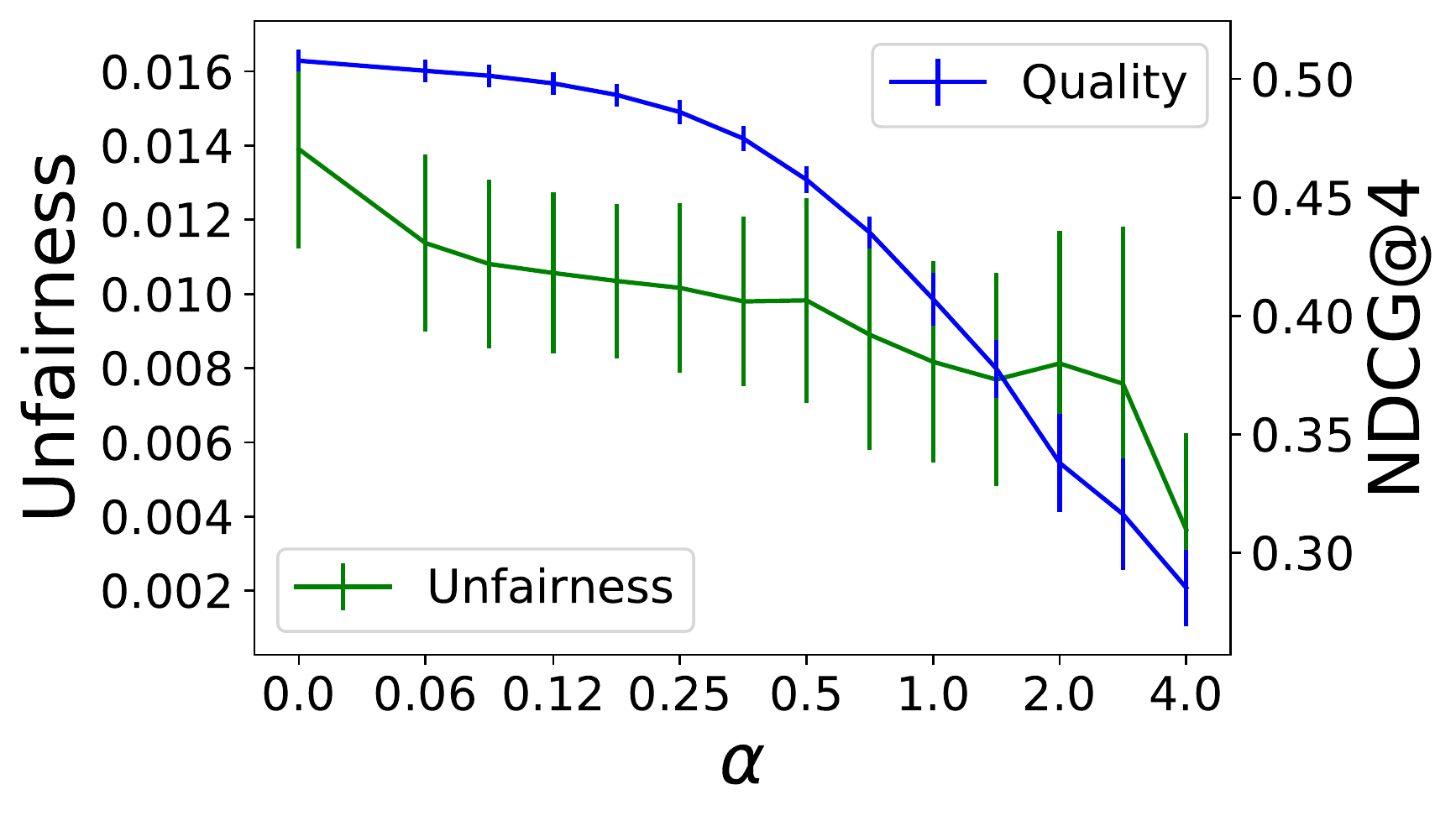}
  \caption{Demographic parity, MSMARCO}
\end{subfigure}%
\hfill
\begin{subfigure}{.33\textwidth}
  \centering
  \includegraphics[width=.95\linewidth]{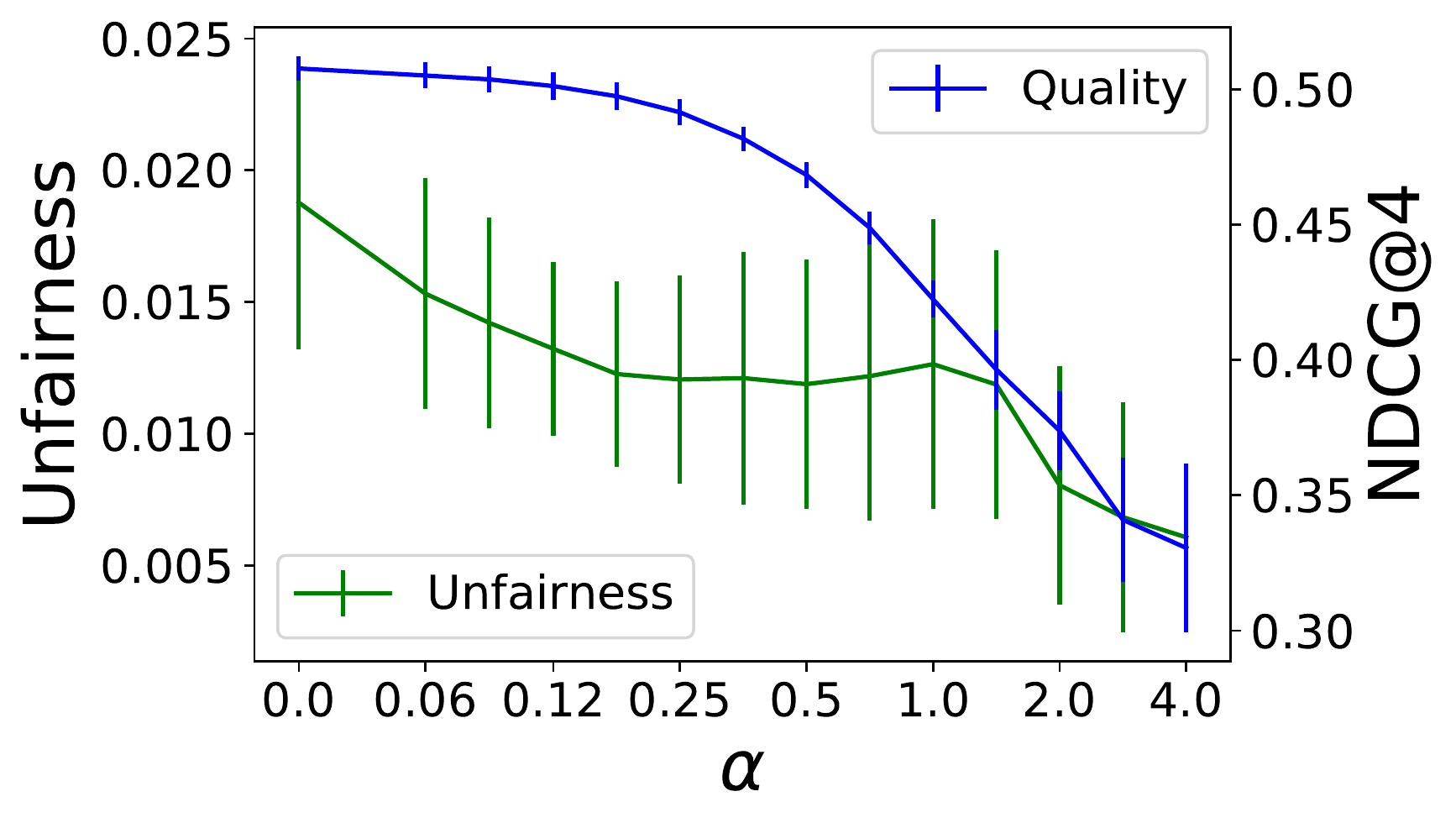}
  \caption{Equalized odds, MSMARCO}
\end{subfigure}%
\hfill
\begin{subfigure}{.33\textwidth}
  \centering
  \includegraphics[width=.95\linewidth]{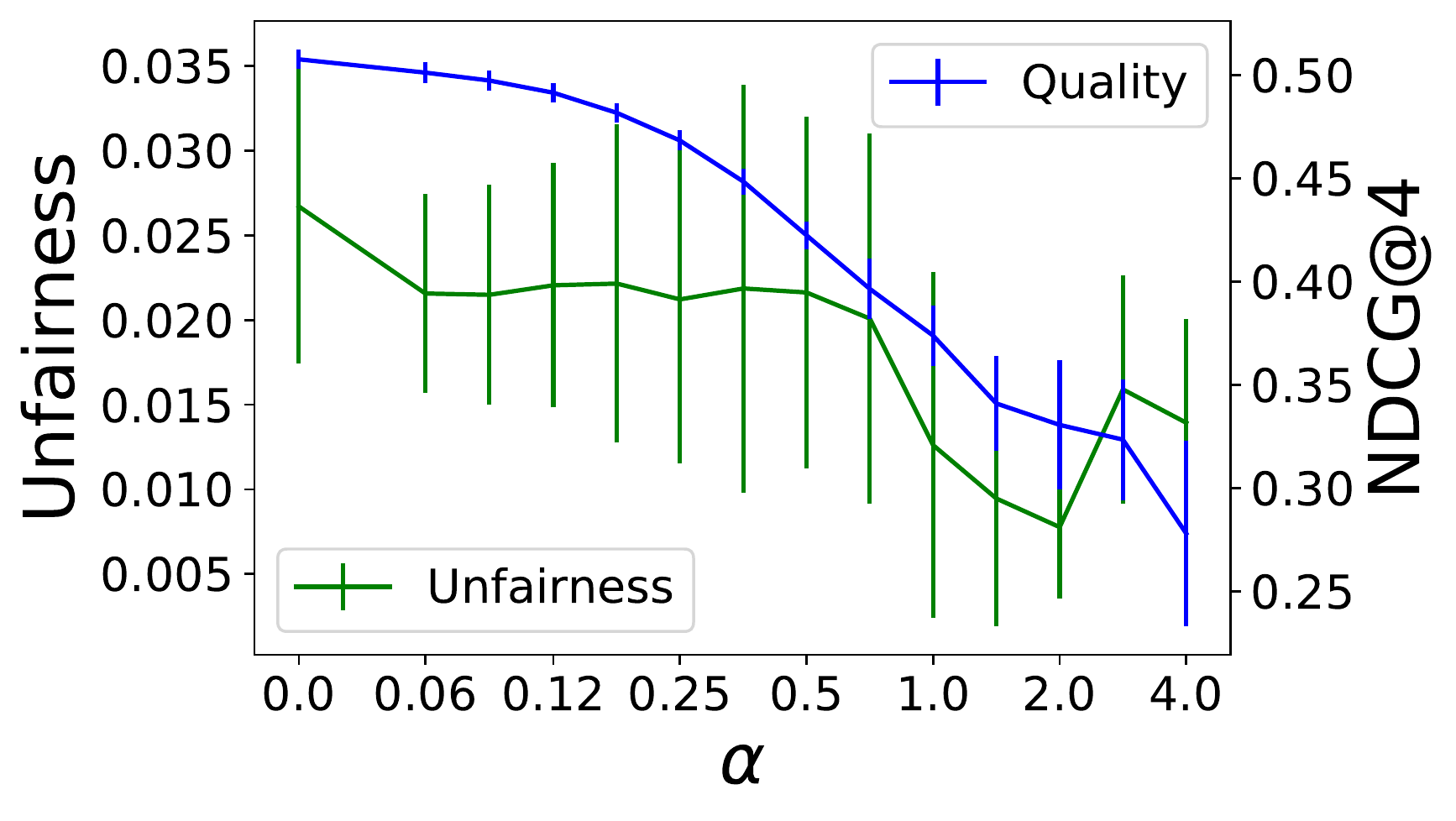}
  \caption{Equality of opportunity, MSMARCO}
\end{subfigure}
\caption{$k = 4$, \textit{ext}}
\end{figure*}


\begin{figure*}[h]
\begin{subfigure}{.33\textwidth}
  \centering
  \includegraphics[width=.95\linewidth]{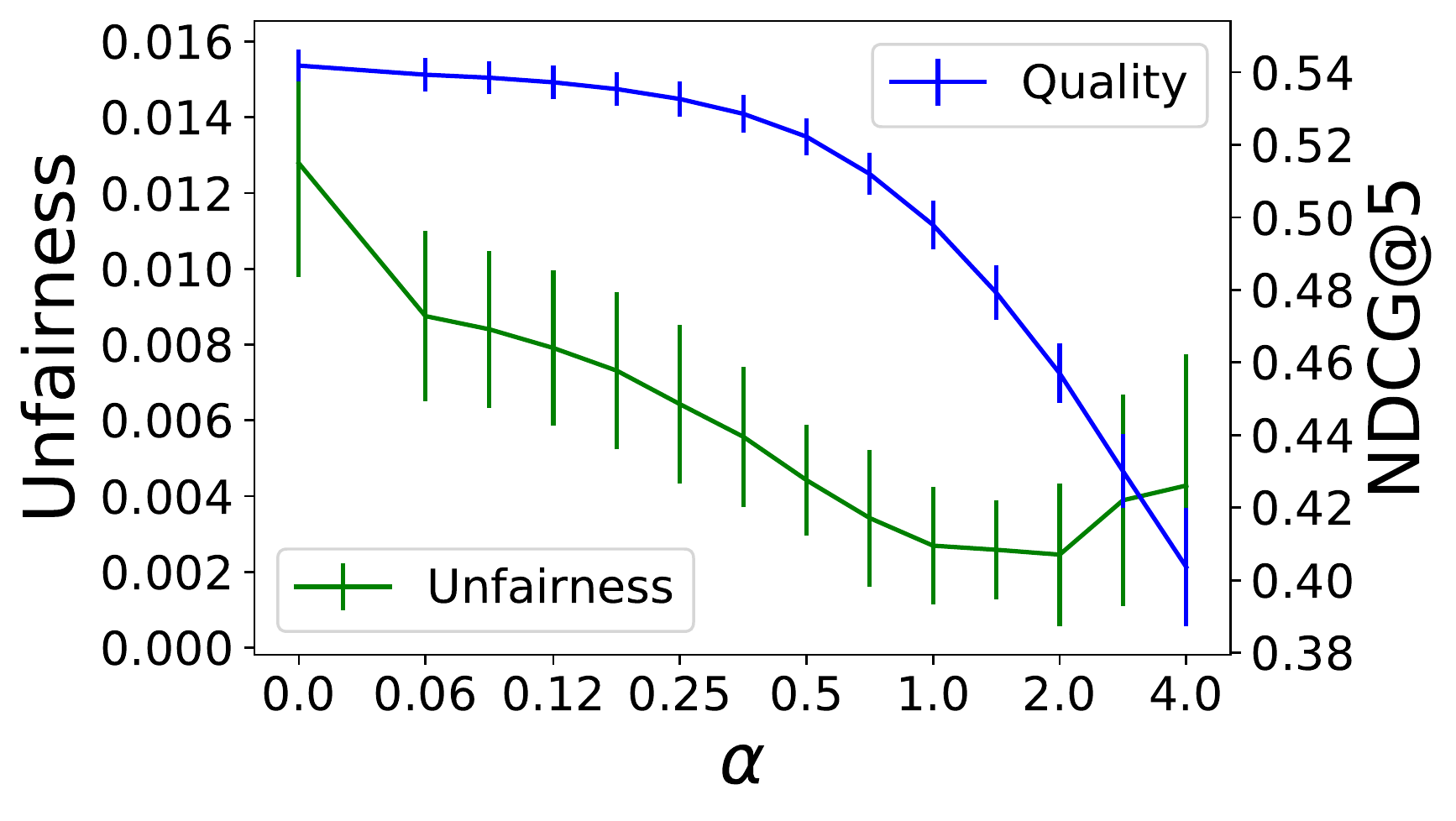}
  \caption{Demographic parity, MSMARCO}
\end{subfigure}%
\hfill
\begin{subfigure}{.33\textwidth}
  \centering
  \includegraphics[width=.95\linewidth]{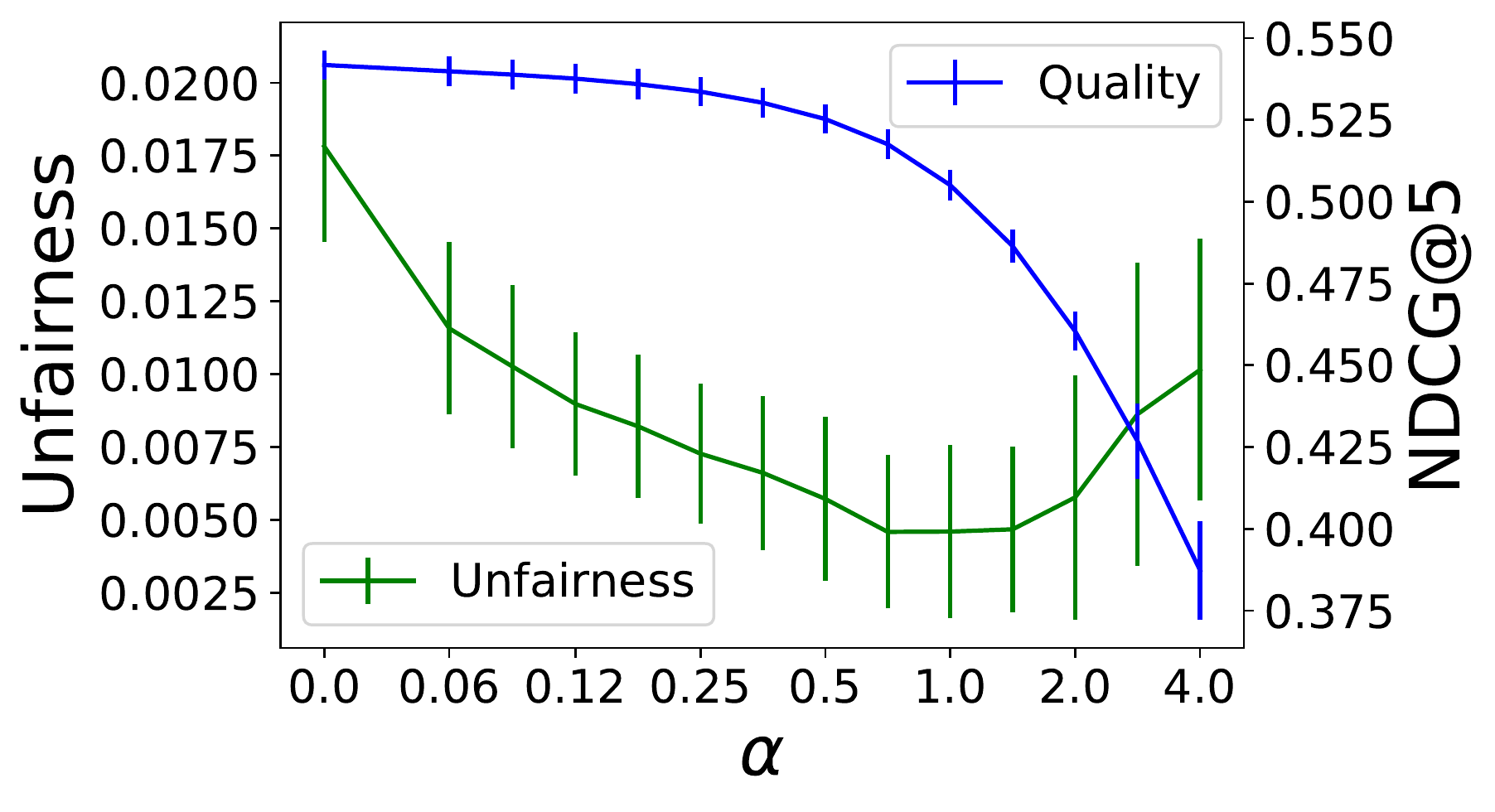}
  \caption{Equalized odds, MSMARCO}
\end{subfigure}%
\hfill
\begin{subfigure}{.33\textwidth}
  \centering
  \includegraphics[width=.95\linewidth]{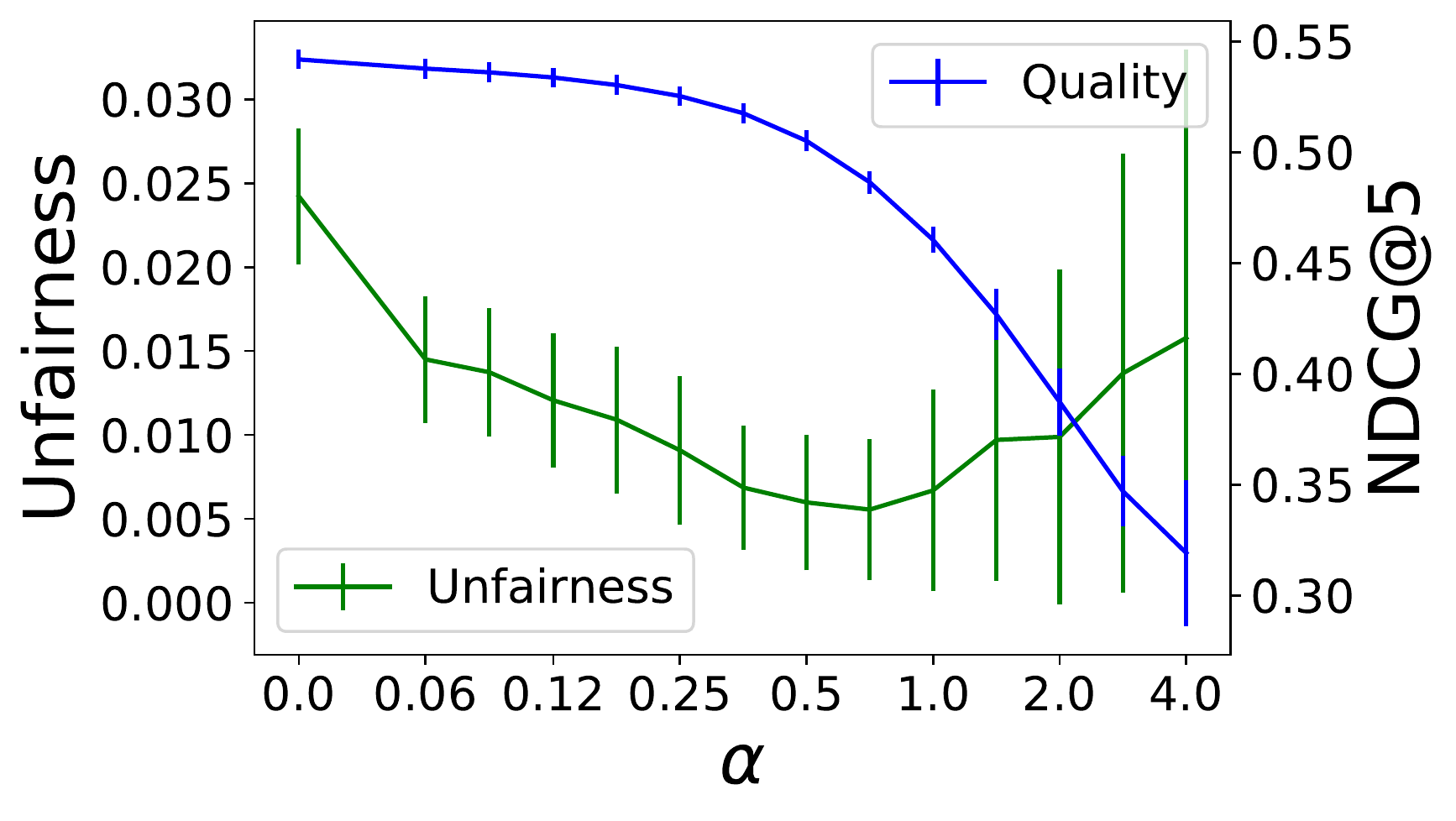}
  \caption{Equality of opportunity, MSMARCO}
\end{subfigure}
\caption{$k = 5$, \textit{com}}
\end{figure*}

\begin{figure*}[h]
\begin{subfigure}{.33\textwidth}
  \centering
  \includegraphics[width=.95\linewidth]{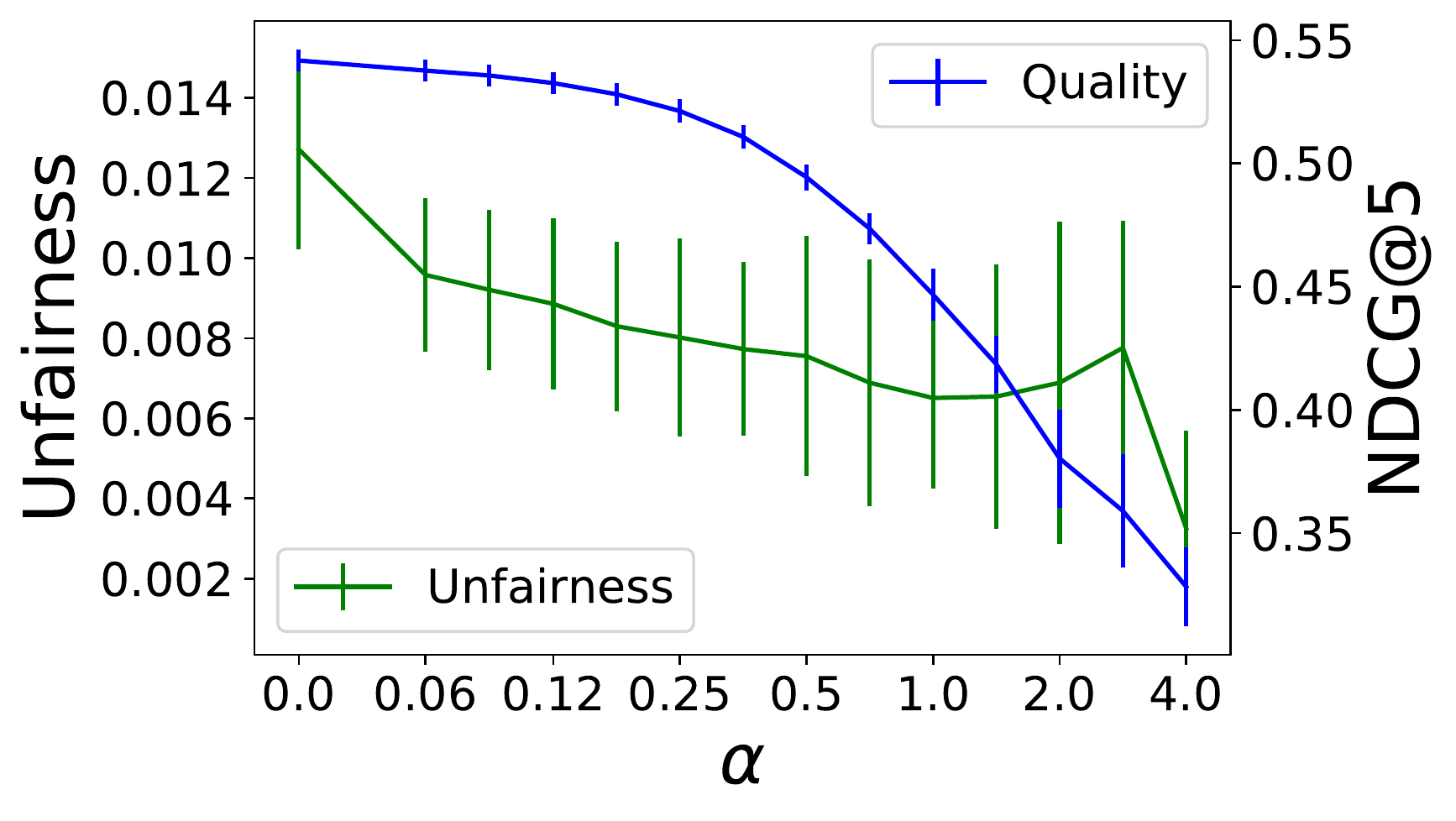}
  \caption{Demographic parity, MSMARCO}
\end{subfigure}%
\hfill
\begin{subfigure}{.33\textwidth}
  \centering
  \includegraphics[width=.95\linewidth]{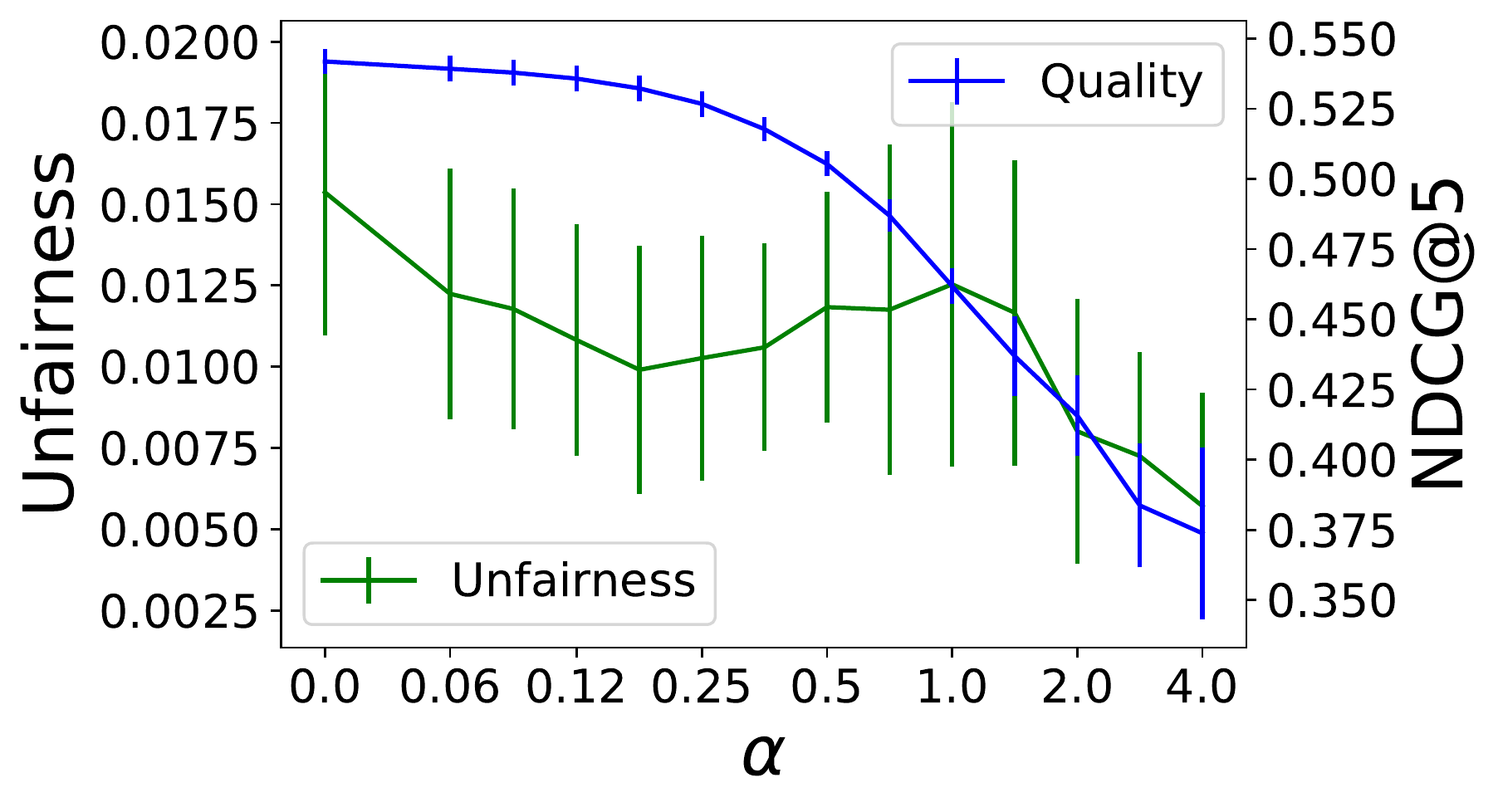}
  \caption{Equalized odds, MSMARCO}
\end{subfigure}%
\hfill
\begin{subfigure}{.33\textwidth}
  \centering
  \includegraphics[width=.95\linewidth]{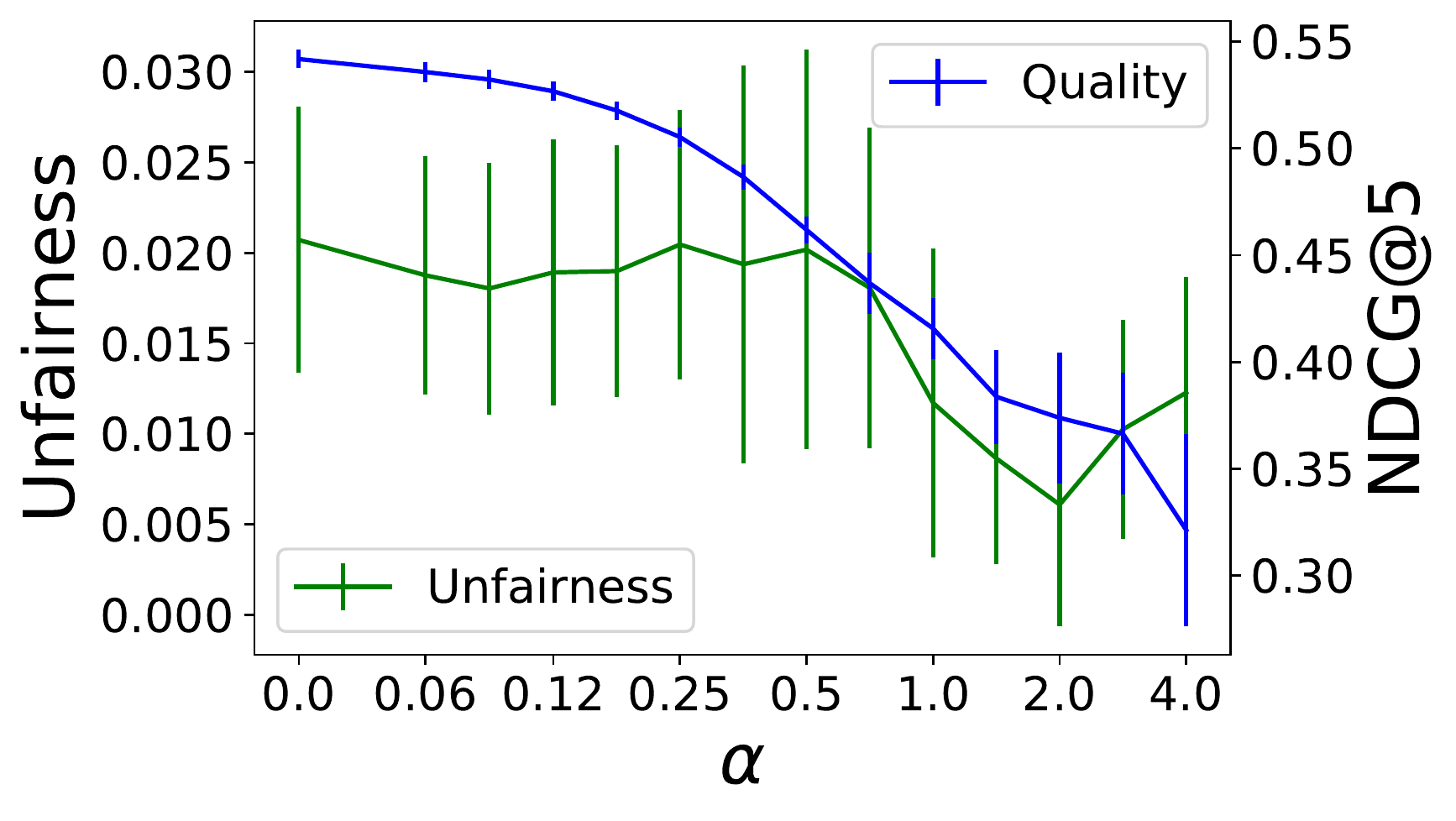}
  \caption{Equality of opportunity, MSMARCO}
\end{subfigure}
\caption{$k = 5$, \textit{ext}}
\end{figure*}

%% file: ms.bbl
\begin{thebibliography}{78}
\providecommand{\natexlab}[1]{#1}
\providecommand{\url}[1]{\texttt{#1}}
\expandafter\ifx\csname urlstyle\endcsname\relax
  \providecommand{\doi}[1]{doi: #1}\else
  \providecommand{\doi}{doi: \begingroup \urlstyle{rm}\Url}\fi

\bibitem[Abadi et~al.(2015)Abadi, Agarwal, Barham, Brevdo, Chen, Citro,
  Corrado, Davis, Dean, Devin, Ghemawat, Goodfellow, Harp, Irving, Isard, Jia,
  Jozefowicz, Kaiser, Kudlur, Levenberg, Man\'{e}, Monga, Moore, Murray, Olah,
  Schuster, Shlens, Steiner, Sutskever, Talwar, Tucker, Vanhoucke, Vasudevan,
  Vi\'{e}gas, Vinyals, Warden, Wattenberg, Wicke, Yu, and
  Zheng]{tensorflow2015-whitepaper}
M.~Abadi, A.~Agarwal, P.~Barham, E.~Brevdo, Z.~Chen, C.~Citro, G.~S. Corrado,
  A.~Davis, J.~Dean, M.~Devin, S.~Ghemawat, I.~Goodfellow, A.~Harp, G.~Irving,
  M.~Isard, Y.~Jia, R.~Jozefowicz, L.~Kaiser, M.~Kudlur, J.~Levenberg,
  D.~Man\'{e}, R.~Monga, S.~Moore, D.~Murray, C.~Olah, M.~Schuster, J.~Shlens,
  B.~Steiner, I.~Sutskever, K.~Talwar, P.~Tucker, V.~Vanhoucke, V.~Vasudevan,
  F.~Vi\'{e}gas, O.~Vinyals, P.~Warden, M.~Wattenberg, M.~Wicke, Y.~Yu, and
  X.~Zheng.
\newblock {TensorFlow}: Large-scale machine learning on heterogeneous systems,
  2015.
\newblock URL \url{https://www.tensorflow.org/}.
\newblock Software available from tensorflow.org.

\bibitem[Agarwal et~al.(2018)Agarwal, Beygelzimer, Dud{\'\i}k, Langford, and
  Wallach]{agarwal2018reductions}
A.~Agarwal, A.~Beygelzimer, M.~Dud{\'\i}k, J.~Langford, and H.~Wallach.
\newblock A reductions approach to fair classification.
\newblock In \emph{International Conference on Machine Learing (ICML)}, 2018.

\bibitem[Asudeh et~al.(2019)Asudeh, Jagadish, Stoyanovich, and
  Das]{asudeh2019designing}
A.~Asudeh, H.~Jagadish, J.~Stoyanovich, and G.~Das.
\newblock Designing fair ranking schemes.
\newblock In \emph{International Conference on Management of Data (COMAD)},
  2019.

\bibitem[Baharlouei et~al.(2019)Baharlouei, Nouiehed, Beirami, and
  Razaviyayn]{baharlouei2019renyi}
S.~Baharlouei, M.~Nouiehed, A.~Beirami, and M.~Razaviyayn.
\newblock R{\'e}nyi fair inference.
\newblock In \emph{International Conference on Learning Representations
  (ICLR)}, 2019.

\bibitem[Barocas et~al.(2019)Barocas, Hardt, and
  Narayanan]{barocas-hardt-narayanan}
S.~Barocas, M.~Hardt, and A.~Narayanan.
\newblock \emph{Fairness and Machine Learning}.
\newblock 2019.
\newblock \url{http://www.fairmlbook.org}.

\bibitem[Bartlett and Mendelson(2002)]{bartlett2002rademacher}
P.~L. Bartlett and S.~Mendelson.
\newblock Rademacher and gaussian complexities: Risk bounds and structural
  results.
\newblock \emph{Journal of Machine Learning Research (JMLR)}, 2002.

\bibitem[Bengio et~al.(2019)Bengio, Dembczynski, Joachims, Kloft, and
  Varma]{bengio2019extreme}
S.~Bengio, K.~Dembczynski, T.~Joachims, M.~Kloft, and M.~Varma.
\newblock Extreme classification.
\newblock In \emph{Dagstuhl Reports 18291}. Schloss Dagstuhl -- Leibniz Center
  for Informatics, 2019.

\bibitem[Beutel et~al.(2019)Beutel, Chen, Doshi, Qian, Wei, Wu, Heldt, Zhao,
  Hong, Chi, et~al.]{beutel2019fairness}
A.~Beutel, J.~Chen, T.~Doshi, H.~Qian, L.~Wei, Y.~Wu, L.~Heldt, Z.~Zhao,
  L.~Hong, E.~H. Chi, et~al.
\newblock Fairness in recommendation ranking through pairwise comparisons.
\newblock In \emph{Conference on Knowledge Discovery and Data Mining (KDD)},
  2019.

\bibitem[Biddle(2006)]{biddle2006adverse}
D.~Biddle.
\newblock \emph{Adverse impact and test validation: A practitioner's guide to
  valid and defensible employment testing}.
\newblock Gower Publishing, 2006.

\bibitem[Biega et~al.(2018)Biega, Gummadi, and Weikum]{biega2018equity}
A.~J. Biega, K.~P. Gummadi, and G.~Weikum.
\newblock Equity of attention: Amortizing individual fairness in rankings.
\newblock In \emph{International Conference on Research and Development in
  Information Retrieval (SIGIR)}, 2018.

\bibitem[Biega et~al.(2019)Biega, Diaz, Ekstrand, and
  Kohlmeier]{trec-fair-ranking-2019}
A.~J. Biega, F.~Diaz, M.~D. Ekstrand, and S.~Kohlmeier.
\newblock Overview of the {TREC} 2019 fair ranking track.
\newblock In \emph{The Twenty-Eighth Text REtrieval Conference ({TREC} 2019)
  Proceedings}, 2019.

\bibitem[Bonart(2019)]{bonartfair}
M.~Bonart.
\newblock Fair ranking in academic search, 2019.
\newblock URL \url{https://trec.nist.gov/pubs/trec28/papers/IR-Cologne.FR.pdf}.

\bibitem[Bower et~al.(2021)Bower, Eftekhari, Yurochkin, and
  Sun]{bower2021individually}
A.~Bower, H.~Eftekhari, M.~Yurochkin, and Y.~Sun.
\newblock Individually fair rankings.
\newblock In \emph{International Conference on Learning Representations
  (ICLR)}, 2021.

\bibitem[Burke(2017)]{burke2017multisided}
R.~Burke.
\newblock Multisided fairness for recommendation.
\newblock \emph{arXiv preprint arXiv:1707.00093}, 2017.

\bibitem[Burke et~al.(2018)Burke, Sonboli, and
  Ordonez-Gauger]{burke2018balanced}
R.~Burke, N.~Sonboli, and A.~Ordonez-Gauger.
\newblock Balanced neighborhoods for multi-sided fairness in recommendation.
\newblock In \emph{Conference on Fairness, Accountability and Transparency
  (FAccT)}, 2018.

\bibitem[Calders and Verwer(2010)]{calders2010three}
T.~Calders and S.~Verwer.
\newblock Three naive {B}ayes approaches for discrimination-free
  classification.
\newblock \emph{Data Mining and Knowledge Discovery (DMKD)}, 2010.

\bibitem[Calders et~al.(2009)Calders, Kamiran, and
  Pechenizkiy]{calders2009building}
T.~Calders, F.~Kamiran, and M.~Pechenizkiy.
\newblock Building classifiers with independency constraints.
\newblock In \emph{International Conference on Data Mining Workshops (IDCMW)},
  2009.

\bibitem[Castillo(2019)]{castillo2019fairness}
C.~Castillo.
\newblock Fairness and transparency in ranking.
\newblock In \emph{International Conference on Research and Development in
  Information Retrieval (SIGIR)}, 2019.

\bibitem[Celis et~al.(2018)Celis, Straszak, and Vishnoi]{celis2018ranking}
L.~E. Celis, D.~Straszak, and N.~K. Vishnoi.
\newblock Ranking with fairness constraints.
\newblock In \emph{International Colloquium on Automata, Languages, and
  Programming (ICALP)}. Schloss Dagstuhl -- Leibniz Center for Informatics,
  2018.

\bibitem[Celis et~al.(2020)Celis, Mehrotra, and
  Vishnoi]{celis2020interventions}
L.~E. Celis, A.~Mehrotra, and N.~K. Vishnoi.
\newblock Interventions for ranking in the presence of implicit bias.
\newblock In \emph{Conference on Fairness, Accountability and Transparency
  (FAccT)}, 2020.

\bibitem[Chakraborty et~al.(2019)Chakraborty, Patro, Ganguly, Gummadi, and
  Loiseau]{chakraborty2019equality}
A.~Chakraborty, G.~K. Patro, N.~Ganguly, K.~P. Gummadi, and P.~Loiseau.
\newblock Equality of voice: Towards fair representation in crowdsourced top-k
  recommendations.
\newblock In \emph{Conference on Fairness, Accountability and Transparency
  (FAccT)}, 2019.

\bibitem[Cho et~al.(2020)Cho, Hwang, and Suh]{cho2020fair}
J.~Cho, G.~Hwang, and C.~Suh.
\newblock A fair classifier using kernel density estimation.
\newblock \emph{Conference on Neural Information Processing Systems (NeurIPS)},
  2020.

\bibitem[Cotter et~al.(2019)Cotter, Gupta, Jiang, Srebro, Sridharan, Wang,
  Woodworth, and You]{cotter2019training}
A.~Cotter, M.~Gupta, H.~Jiang, N.~Srebro, K.~Sridharan, S.~Wang, B.~Woodworth,
  and S.~You.
\newblock Training well-generalizing classifiers for fairness metrics and other
  data-dependent constraints.
\newblock In \emph{International Conference on Machine Learing (ICML)}, 2019.

\bibitem[Devlin et~al.(2019)Devlin, Chang, Lee, and Toutanova]{devlin2019bert}
J.~Devlin, M.-W. Chang, K.~Lee, and K.~Toutanova.
\newblock {BERT}: Pre-training of deep bidirectional transformers for language
  understanding.
\newblock In \emph{Annual Meeting of the Association for Computational
  Linguistics (ACL)}, 2019.

\bibitem[Dwork et~al.(2012)Dwork, Hardt, Pitassi, Reingold, and
  Zemel]{dwork2012fairness}
C.~Dwork, M.~Hardt, T.~Pitassi, O.~Reingold, and R.~Zemel.
\newblock Fairness through awareness.
\newblock In \emph{Innovations in Theoretical Computer Science Conference
  (ITCS)}, 2012.

\bibitem[Farnadi et~al.(2018)Farnadi, Kouki, Thompson, Srinivasan, and
  Getoor]{farnadi2018fairness}
G.~Farnadi, P.~Kouki, S.~K. Thompson, S.~Srinivasan, and L.~Getoor.
\newblock A fairness-aware hybrid recommender system.
\newblock \emph{arXiv preprint arXiv:1809.09030}, 2018.

\bibitem[Geyik et~al.(2019)Geyik, Ambler, and Kenthapadi]{geyik2019fairness}
S.~C. Geyik, S.~Ambler, and K.~Kenthapadi.
\newblock Fairness-aware ranking in search \& recommendation systems with
  application to linkedin talent search.
\newblock In \emph{Conference on Knowledge Discovery and Data Mining (KDD)},
  2019.

\bibitem[Gorantla et~al.(2020)Gorantla, Deshpande, and
  Louis]{gorantla2020ranking}
S.~Gorantla, A.~Deshpande, and A.~Louis.
\newblock Ranking for individual and group fairness simultaneously.
\newblock \emph{arXiv preprint arXiv:2010.06986}, 2020.

\bibitem[Han et~al.(2020)Han, Wang, Bendersky, and Najork]{han2020learning}
S.~Han, X.~Wang, M.~Bendersky, and M.~Najork.
\newblock Learning-to-rank with {BERT} in {TF}-ranking.
\newblock \emph{arXiv preprint arXiv:2004.08476}, 2020.

\bibitem[Hardt et~al.(2016)Hardt, Price, and Srebro]{hardt2016equality}
M.~Hardt, E.~Price, and N.~Srebro.
\newblock Equality of opportunity in supervised learning.
\newblock In \emph{Conference on Neural Information Processing Systems
  (NeurIPS)}, 2016.

\bibitem[Janson(2004)]{janson2004large}
S.~Janson.
\newblock Large deviations for sums of partly dependent random variables.
\newblock \emph{Random Structures \& Algorithms}, 24\penalty0 (3):\penalty0
  234--248, 2004.

\bibitem[Joachims(2002)]{joachims2002optimizing}
T.~Joachims.
\newblock Optimizing search engines using clickthrough data.
\newblock In \emph{Conference on Knowledge Discovery and Data Mining (KDD)},
  2002.

\bibitem[Kallus and Zhou(2019)]{kallus2019fairness}
N.~Kallus and A.~Zhou.
\newblock The fairness of risk scores beyond classification: Bipartite ranking
  and the {xAUC} metric.
\newblock In \emph{Conference on Neural Information Processing Systems
  (NeurIPS)}, 2019.

\bibitem[Kamishima et~al.(2011)Kamishima, Akaho, and
  Sakuma]{kamishima2011fairness}
T.~Kamishima, S.~Akaho, and J.~Sakuma.
\newblock Fairness-aware learning through regularization approach.
\newblock In \emph{11th International Conference on Data Mining Workshops},
  2011.

\bibitem[Kamishima et~al.(2012{\natexlab{a}})Kamishima, Akaho, Asoh, and
  Sakuma]{kamishima2012enhancement}
T.~Kamishima, S.~Akaho, H.~Asoh, and J.~Sakuma.
\newblock Enhancement of the neutrality in recommendation.
\newblock In \emph{Decisions@ RecSys}, 2012{\natexlab{a}}.

\bibitem[Kamishima et~al.(2012{\natexlab{b}})Kamishima, Akaho, Asoh, and
  Sakuma]{kamishima2012fairness}
T.~Kamishima, S.~Akaho, H.~Asoh, and J.~Sakuma.
\newblock Fairness-aware classifier with prejudice remover regularizer.
\newblock In \emph{European Conference on Machine Learning and Data Mining
  (ECML PKDD)}, 2012{\natexlab{b}}.

\bibitem[Kamishima et~al.(2014)Kamishima, Akaho, Asoh, and
  Sakuma]{kamishima2014correcting}
T.~Kamishima, S.~Akaho, H.~Asoh, and J.~Sakuma.
\newblock Correcting popularity bias by enhancing recommendation neutrality.
\newblock In \emph{RecSys Posters}, 2014.

\bibitem[Kim et~al.(2020)Kim, Chen, and Talwalkar]{kim2020fact}
J.~S. Kim, J.~Chen, and A.~Talwalkar.
\newblock Fact: A diagnostic for group fairness trade-offs.
\newblock In \emph{International Conference on Machine Learing (ICML)}, 2020.

\bibitem[Kuhlman et~al.(2019)Kuhlman, VanValkenburg, and
  Rundensteiner]{kuhlman2019fare}
C.~Kuhlman, M.~VanValkenburg, and E.~Rundensteiner.
\newblock {FARE}: Diagnostics for fair ranking using pairwise error metrics.
\newblock In \emph{International World Wide Web Conference (WWW)}, 2019.

\bibitem[Kusner et~al.(2017)Kusner, Loftus, Russell, and
  Silva]{kusner2017counterfactual}
M.~J. Kusner, J.~Loftus, C.~Russell, and R.~Silva.
\newblock Counterfactual fairness.
\newblock In \emph{Conference on Neural Information Processing Systems
  (NeurIPS)}, 2017.

\bibitem[Liu(2011)]{liu2011learning}
T.-Y. Liu.
\newblock \emph{Learning to rank for information retrieval}.
\newblock Springer Science \& Business Media, 2011.

\bibitem[Lo(2015)]{ceo2}
D.~Lo.
\newblock When you {Google} image {CEO}, the first female photo on the results
  page is {Barbie}.
\newblock \url{https://www.glamour.com/story/google-search-ceo}, 2015.
\newblock Accessed: 2021-05-26.

\bibitem[Manning et~al.(2008)Manning, Sch{\"u}tze, and
  Raghavan]{manning2008introduction}
C.~D. Manning, H.~Sch{\"u}tze, and P.~Raghavan.
\newblock \emph{Introduction to information retrieval}.
\newblock Cambridge University Press, 2008.

\bibitem[Mehrabi et~al.(2019)Mehrabi, Morstatter, Saxena, Lerman, and
  Galstyan]{mehrabi2019survey}
N.~Mehrabi, F.~Morstatter, N.~Saxena, K.~Lerman, and A.~Galstyan.
\newblock A survey on bias and fairness in machine learning.
\newblock \emph{arXiv preprint arXiv:1908.09635}, 2019.

\bibitem[Mitra and Craswell(2018)]{mitra2018introduction}
B.~Mitra and N.~Craswell.
\newblock An introduction to neural information retrieval.
\newblock \emph{Foundations and Trends in Information Retrieval}, 2018.

\bibitem[Morik et~al.(2020)Morik, Singh, Hong, and
  Joachims]{morik2020controlling}
M.~Morik, A.~Singh, J.~Hong, and T.~Joachims.
\newblock Controlling fairness and bias in dynamic learning-to-rank.
\newblock In \emph{International Conference on Research and Development in
  Information Retrieval (SIGIR)}, 2020.

\bibitem[Narasimhan et~al.(2020)Narasimhan, Cotter, Gupta, and
  Wang]{narasimhan2020pairwise}
H.~Narasimhan, A.~Cotter, M.~R. Gupta, and S.~Wang.
\newblock Pairwise fairness for ranking and regression.
\newblock In \emph{Conference on Artificial Intelligence (AAAI)}, 2020.

\bibitem[Nguyen et~al.(2016)Nguyen, Rosenberg, Song, Gao, Tiwary, Majumder, and
  Deng]{nguyen2016ms}
T.~Nguyen, M.~Rosenberg, X.~Song, J.~Gao, S.~Tiwary, R.~Majumder, and L.~Deng.
\newblock {MS MARCO}: A human-generated machine reading comprehension dataset.
\newblock In \emph{Conference on Neural Information Processing Systems
  (NeurIPS)}, 2016.

\bibitem[Nogueira and Cho(2019)]{nogueira2019passage}
R.~Nogueira and K.~Cho.
\newblock Passage re-ranking with {BERT}.
\newblock \emph{arXiv preprint arXiv:1901.04085}, 2019.

\bibitem[Peysakhovich and Kroer(2019)]{peysakhovich2019fair}
A.~Peysakhovich and C.~Kroer.
\newblock Fair division without disparate impact.
\newblock \emph{arXiv preprint arXiv:1906.02775}, 2019.

\bibitem[Radlinski et~al.(2009)Radlinski, Bennett, Carterette, and
  Joachims]{radlinski2009redundancy}
F.~Radlinski, P.~N. Bennett, B.~Carterette, and T.~Joachims.
\newblock Redundancy, diversity and interdependent document relevance.
\newblock In \emph{International Conference on Research and Development in
  Information Retrieval (SIGIR)}, 2009.

\bibitem[Rezaei et~al.(2020)Rezaei, Fathony, Memarrast, and
  Ziebart]{rezaei2020fairness}
A.~Rezaei, R.~Fathony, O.~Memarrast, and B.~Ziebart.
\newblock Fairness for robust log loss classification.
\newblock In \emph{Conference on Artificial Intelligence (AAAI)}, 2020.

\bibitem[Robertson(1977)]{robertson1977probability}
S.~E. Robertson.
\newblock The probability ranking principle in ir.
\newblock \emph{Journal of Documentation}, 33\penalty0 (4):\penalty0 294--304,
  1977.

\bibitem[Sapiezynski et~al.(2019)Sapiezynski, Zeng, E~Robertson, Mislove, and
  Wilson]{sapiezynski2019quantifying}
P.~Sapiezynski, W.~Zeng, R.~E~Robertson, A.~Mislove, and C.~Wilson.
\newblock Quantifying the impact of user attentionon fair group representation
  in ranked lists.
\newblock In \emph{International World Wide Web Conference (WWW)}, 2019.

\bibitem[Singh and Joachims(2017)]{singh2017equality}
A.~Singh and T.~Joachims.
\newblock Equality of opportunity in rankings.
\newblock In \emph{Workshop on Prioritizing Online Content (WPOC) at NeurIPS},
  2017.

\bibitem[Singh and Joachims(2018)]{singh2018fairness}
A.~Singh and T.~Joachims.
\newblock Fairness of exposure in rankings.
\newblock In \emph{Conference on Knowledge Discovery and Data Mining (KDD)},
  2018.

\bibitem[Singh and Joachims(2019)]{singh2019policy}
A.~Singh and T.~Joachims.
\newblock Policy learning for fairness in ranking.
\newblock In \emph{Conference on Neural Information Processing Systems
  (NeurIPS)}, 2019.

\bibitem[Steck(2018)]{steck2018calibrated}
H.~Steck.
\newblock Calibrated recommendations.
\newblock In \emph{Conference on Recommender Systems (RecSys)}, 2018.

\bibitem[Tabibian et~al.(2020)Tabibian, G{\'o}mez, De, Sch{\"o}lkopf, and
  Rodriguez]{tabibian2020design}
B.~Tabibian, V.~G{\'o}mez, A.~De, B.~Sch{\"o}lkopf, and M.~G. Rodriguez.
\newblock On the design of consequential ranking algorithms.
\newblock In \emph{Uncertainty in Artificial Intelligence (UAI)}, 2020.

\bibitem[Tan et~al.(2020)Tan, Yeom, Fredrikson, and Talwalkar]{tan2020learning}
Z.~Tan, S.~Yeom, M.~Fredrikson, and A.~Talwalkar.
\newblock Learning fair representations for kernel models.
\newblock In \emph{Conference on Uncertainty in Artificial Intelligence
  (AISTATS)}, 2020.

\bibitem[Tsintzou et~al.(2019)Tsintzou, Pitoura, and
  Tsaparas]{tsintzou2018bias}
V.~Tsintzou, E.~Pitoura, and P.~Tsaparas.
\newblock Bias disparity in recommendation systems.
\newblock In \emph{Workshop on Recommendation in Multi-stakeholder Environments
  at {RecSys}}, 2019.

\bibitem[Usunier et~al.(2005)Usunier, Amini, and
  Gallinari]{usunier2005generalization}
N.~Usunier, M.~R. Amini, and P.~Gallinari.
\newblock Generalization error bounds for classifiers trained with
  interdependent data.
\newblock \emph{Conference on Neural Information Processing Systems (NeurIPS)},
  2005.

\bibitem[Vapnik(2013)]{vapnik2013nature}
V.~Vapnik.
\newblock \emph{The nature of statistical learning theory}.
\newblock Springer, 2013.

\bibitem[Vogel et~al.(2020)Vogel, Bellet, and
  Cl{\'e}men{\c{c}}on]{vogel2020learning}
R.~Vogel, A.~Bellet, and S.~Cl{\'e}men{\c{c}}on.
\newblock Learning fair scoring functions: Fairness definitions, algorithms and
  generalization bounds for bipartite ranking.
\newblock \emph{arXiv preprint arXiv:2002.08159}, 2020.

\bibitem[Weston et~al.(2011)Weston, Bengio, and Usunier]{weston2011wsabie}
J.~Weston, S.~Bengio, and N.~Usunier.
\newblock {WSABIE}: Scaling up to large vocabulary image annotation.
\newblock In \emph{International Joint Conference on Artificial Intelligence
  (IJCAI)}, 2011.

\bibitem[Wick et~al.(2019)Wick, Panda, and Tristan]{NEURIPS2019_373e4c5d}
M.~Wick, S.~Panda, and J.-B. Tristan.
\newblock Unlocking fairness: a trade-off revisited.
\newblock In \emph{Conference on Neural Information Processing Systems
  (NeurIPS)}, 2019.

\bibitem[Williamson and Menon(2019)]{williamson2019fairness}
R.~Williamson and A.~Menon.
\newblock Fairness risk measures.
\newblock In \emph{International Conference on Machine Learing (ICML)}, 2019.

\bibitem[Woodworth et~al.(2017)Woodworth, Gunasekar, Ohannessian, and
  Srebro]{woodworth2017learning}
B.~Woodworth, S.~Gunasekar, M.~I. Ohannessian, and N.~Srebro.
\newblock Learning non-discriminatory predictors.
\newblock In \emph{Workshop on Computational Learning Theory (COLT)}, 2017.

\bibitem[Wu et~al.(2018)Wu, Zhang, and Wu]{wu2018discrimination}
Y.~Wu, L.~Zhang, and X.~Wu.
\newblock On discrimination discovery and removal in ranked data using causal
  graph.
\newblock In \emph{Conference on Knowledge Discovery and Data Mining (KDD)},
  2018.

\bibitem[Yadav et~al.(2019)Yadav, Du, and Joachims]{yadav2019fair}
H.~Yadav, Z.~Du, and T.~Joachims.
\newblock Fair learning-to-rank from implicit feedback.
\newblock \emph{arXiv preprint arXiv:1911.08054}, 2019.

\bibitem[Yang and Stoyanovich(2017)]{yang2017measuring}
K.~Yang and J.~Stoyanovich.
\newblock Measuring fairness in ranked outputs.
\newblock In \emph{Scientific and Statistical Database Management Conference
  (SSDBM)}, 2017.

\bibitem[Yang et~al.(2019)Yang, Gkatzelis, and Stoyanovich]{yang2019balanced}
K.~Yang, V.~Gkatzelis, and J.~Stoyanovich.
\newblock Balanced ranking with diversity constraints.
\newblock In \emph{International Joint Conference on Artificial Intelligence
  (IJCAI)}, 2019.

\bibitem[Zafar et~al.(2017)Zafar, Valera, Rogriguez, and
  Gummadi]{zafar2017fairness}
M.~B. Zafar, I.~Valera, M.~G. Rogriguez, and K.~P. Gummadi.
\newblock Fairness constraints: Mechanisms for fair classification.
\newblock In \emph{Conference on Uncertainty in Artificial Intelligence
  (AISTATS)}, 2017.

\bibitem[Zehlike and Castillo(2020)]{zehlike2020reducing}
M.~Zehlike and C.~Castillo.
\newblock Reducing disparate exposure in ranking: A learning to rank approach.
\newblock In \emph{International World Wide Web Conference (WWW)}, 2020.

\bibitem[Zehlike et~al.(2017)Zehlike, Bonchi, Castillo, Hajian, Megahed, and
  Baeza-Yates]{zehlike2017fa}
M.~Zehlike, F.~Bonchi, C.~Castillo, S.~Hajian, M.~Megahed, and R.~Baeza-Yates.
\newblock {FA*IR}: A fair top-k ranking algorithm.
\newblock In \emph{Conference on Information and Knowledge Management (CIKM)},
  2017.

\bibitem[Zehlike et~al.(2020)Zehlike, S{\"u}hr, Castillo, and
  Kitanovski]{zehlike2020fairsearch}
M.~Zehlike, T.~S{\"u}hr, C.~Castillo, and I.~Kitanovski.
\newblock Fairsearch: A tool for fairness in ranked search results.
\newblock In \emph{International World Wide Web Conference (WWW)}, 2020.

\bibitem[Zhang et~al.(2018)Zhang, Lemoine, and Mitchell]{zhang2018mitigating}
B.~H. Zhang, B.~Lemoine, and M.~Mitchell.
\newblock Mitigating unwanted biases with adversarial learning.
\newblock In \emph{Proceedings of the 2018 AAAI/ACM Conference on AI, Ethics,
  and Society}, 2018.

\bibitem[Zhu et~al.(2018)Zhu, Hu, and Caverlee]{zhu2018fairness}
Z.~Zhu, X.~Hu, and J.~Caverlee.
\newblock Fairness-aware tensor-based recommendation.
\newblock In \emph{Conference on Information and Knowledge Management (CIKM)},
  2018.

\end{thebibliography}
